%% file: mainv5.tex
\documentclass[onefignum,onetabnum]{siamart171218}
\usepackage{graphicx,amsfonts,latexsym,amscd}
\usepackage{verbatim}
\usepackage{epstopdf}
\usepackage{bbm,color}
\usepackage{enumitem}
\usepackage{mathtools}
\usepackage{subcaption}
\usepackage{mathrsfs}
\usepackage{makecell}
\usepackage{capt-of}
\usepackage{float}
\usepackage{hyperref}
\hypersetup{
	colorlinks=true,
	linkcolor=blue,
	filecolor=magenta,      
	urlcolor=cyan,
}

\usepackage{comment}
\usepackage[font=small,skip=0pt]{caption}
\usepackage{multirow}
\usepackage{array}

\makeatletter
\newcommand{\subalign}[1]{%
	\vcenter{%
		\Let@ \restore@math@cr \default@tag
		\baselineskip\fontdimen10 \scriptfont\tw@
		\advance\baselineskip\fontdimen12 \scriptfont\tw@
		\lineskip\thr@@\fontdimen8 \scriptfont\thr@@
		\lineskiplimit\lineskip
		\ialign{\hfil$\m@th\scriptstyle##$&$\m@th\scriptstyle{}##$\hfil\crcr
			#1\crcr
		}%
	}%
}
\makeatother

\DeclareMathOperator*{\argmin}{arg\,min}

\DeclareMathAlphabet{\mathpzc}{OT1}{pzc}{m}{it}

\input{ex_shared}

\ifpdf
\hypersetup{
	pdftitle={A Weighted Difference of Anisotropic and Isotropic Total Variation for Relaxed Mumford-Shah Color and Multiphase Image Segmentation},
	pdfauthor={K. Bui, F. Park, Y. Lou, and J. Xin}
}
\fi

\begin{document}
	
	\maketitle
	\begin{abstract}
		In a class of piecewise-constant image segmentation models, we propose to incorporate a weighted difference of anisotropic and isotropic total variation (AITV) to regularize the partition boundaries in an image. In particular, we replace the total variation regularization in the Chan-Vese segmentation model and a fuzzy region competition model by the proposed AITV. To deal with the nonconvex nature of AITV, we apply the  difference-of-convex algorithm (DCA), in which the subproblems can be minimized by the primal-dual hybrid gradient method with linesearch. The convergence  of the DCA scheme is analyzed. In addition, a generalization to color image segmentation is  discussed. In the numerical experiments, we compare the proposed models with the classic convex approaches and the two-stage segmentation methods (smoothing and then thresholding) on various images, showing that our  models are  effective in image segmentation and robust with respect to impulsive noises.
	\end{abstract}

	
	
	\section{Introduction}
	\label{sec:intro}
	Image segmentation is an important problem in computer vision, where the goal is to partition a given image into salient regions that usually represent specific objects of interest.  Each partitioned region has uniform characteristics such as edges, intensities, colors, and textures. Mathematically, given an image $f: \Omega \rightarrow \mathbb{R}$, where the image domain $\Omega$ is a bounded and open subset of $\mathbb{R}^2$, the aim is to partition $\Omega$ into $N$ predetermined number of regions $\{\Omega_i\}_{i=1}^N$ such that $\Omega_i \cap \Omega_j =\emptyset$ for each $i \neq j$ and $\Omega = \bigcup_{i=1}^N \Omega_i$.  
	
	In the past two decades, image segmentation has been studied extensively using variational methods and partial differential equations as common and popular methodologies. One class of models, such as the snake model and geodesic contour model, uses edge-detection functions and evolves the curves toward sharp gradients \cite{caselles1997geodesic, cohen1991active,kass1988snakes, kichenassamy1995gradient}. However, these models are sensitive to noise. 
	As an alternative, region-based models that incorporate region and boundary information are robust to noise. 	
	One of the most fundamental region-based models is the Mumford--Shah (MS) model \cite{mumford1989optimal}, which approximates an image using piecewise-smooth functions. The MS model is formulated as
	\begin{align}
	\label{eq:ms_full}
	\min_{g, \Gamma}  \, \lambda \int_{\Omega}(f(x,y)-g(x,y))^2 \, dx\;dy +
	\mu \int_{\Omega\setminus \Gamma}|\nabla g(x,y)|^2 \, dx \;dy + \left|\Gamma\right|,
	\end{align}
	where $g: \Omega \rightarrow \mathbb{R}$ is a smooth approximation of the given image $f$, $\Gamma = \bigcup_{i=1}^N \partial \Omega_i$ is the union of the boundaries of the regions $\Omega_i$, $|\Gamma|$ denotes the arc length of $\Gamma$, and $\lambda, \mu$ are positive parameters. Unfortunately, solving \eqref{eq:ms_full} is extremely complex and difficult because it requires discretizing the unknown set of edges. 
	
	Instead of piecewise-smooth functions, the Chan--Vese (CV) model  \cite{chan-vese-2001} approximates $f$ by piecewise-constant functions with two constant values $c_1,c_2$ for the regions inside  and outside $\Gamma$. The CV model is expressed as
	\begin{align}
	\label{eq:original_CV}
	\min_{c_1, c_2, \Gamma} \lambda \int_{\text{inside}(\Gamma)}(f(x,y)-c_1)^2 \;dx \;dy +\lambda \int_{\text{outside}(\Gamma)}(f(x,y)-c_2)^2 \;dx \;dy \,+\, |\Gamma|.
	\end{align}
	Note that CV simultaneously optimizes $c_1$ and $c_2$  together with $\Gamma.$
	It is popular   to minimize \eqref{eq:original_CV}  via the level-set method \cite{osher1988fronts}. Let $\phi: \Omega \rightarrow \mathbb{R}$ be a Lipschitz function such that $	\Gamma = \{(x,y) \in \Omega: \phi(x,y) = 0\}$ and
	\begin{align*}
 \text{inside}(\Gamma) = \{(x,y) \in \Omega: \phi(x,y) > 0\}, \; \text{outside}(\Gamma) = \{ (x,y)\in \Omega: \phi(x,y) < 0\}.
	\end{align*}
	We denote the Heaviside function
	\begin{align*}
	H(\phi(x,y)) = \begin{cases}
	1 & \text{ if } \phi(x,y) \geq 0, \\
	0 & \text{ if } \phi(x,y) < 0.
	\end{cases}
	\end{align*}
	The level-set reformulation of \eqref{eq:original_CV} is
	\begin{gather}
	\begin{aligned} \label{eq:level_set_CV}
	\min_{\phi, c_1, c_2} &\lambda \int_{\Omega}(f(x,y)-c_1)^2 H(\phi(x,y)) + (f(x,y)-c_2)^2  (1- H(\phi(x,y))) \;dx \;dy\\  &+ \int_{\Omega} |\nabla H(\phi(x,y))| \;dx \;dy.
	\end{aligned}
	\end{gather}
	
	A numerical scheme for \eqref{eq:level_set_CV} requires solving the Euler--Lagrange equation for $\phi$, followed by updating $c_1,c_2$ as average intensities inside and outside of $\Gamma$, respectively; please see \cite{chan-vese-2001,getreuer2012chan} for details.
	Lie et al.~\cite{lie2006binary} introduced a binary level-set formulations of the MS model. Esedoglu and Tsai \cite{esedog2006threshold} later developed a more efficient algorithm using the Merriman--Bence--Osher scheme \cite{merriman1994motion}. Chan et al.~\cite{chan-esedoglu-nikolova-2004} proposed a convex relaxation of the CV model, formulated as
	\begin{gather}
	\begin{aligned} \label{eq:convex_CV}
	\min_{u(x,y) \in [0,1], c_1, c_2} &\lambda \int_{\Omega} (f(x,y) - c_1)^2 u(x,y) + (f(x,y)-c_2)^2(1-u(x,y)) \;dx \;dy\\ &+ \int_{\Omega} |\nabla u(x,y)| \; dx \;dy.
	\end{aligned}
	\end{gather}
	The segmented regions can be defined by thresholding $u$ as follows:
	\begin{align*}
	\text{inside}(\Gamma) = \{(x,y) \in \Omega: u(x,y) > \tau\}, \quad\text{outside}(\Gamma) = \{(x,y) \in \Omega: u(x,y) \leq \tau \},
	\end{align*}
	with a chosen constant $\tau \in [0,1]$. Since the objective function in \eqref{eq:convex_CV} is convex with respect to $u$, it can be minimized using popular convex optimization algorithms, such as split Bregman \cite{goldstein2009split}, alternating direction method of multipliers (ADMM) \cite{boyd2011distributed, gabay1983chapter}, and primal-dual hybrid gradient (PDHG) \cite{chambolle-pock-2011,esser2010general}. As a result, \eqref{eq:convex_CV} inspired various segmentation models  \cite{bae2011global,chambolle-cremers-2012,jung2014variational,lellmann2009convex,pock2008convex,yuan2010continuous, yuan2017bregman,zach2008fast} that can be solved by convex optimization.
	
		In \eqref{eq:convex_CV}, the total variation (TV) term $\|\nabla u\|_1 = \int_{\Omega} |\nabla u(x,y)| \;dx \;dy$ approximates the length of the curves that partition the segmented regions. Furthermore, it is the tightest convex relaxation of the jump term $\|\nabla u\|_0$, which counts the number of jump discontinuities. When $u$ is piecewise constant, $\|\nabla u\|_0$ is exactly the total arc length of the curves \cite{storath2015joint}. Unfortunately,  minimizing $\|\nabla u\|_0$ is an NP-hard combinatorial problem, and it is often replaced by $\|\nabla u\|_1$ that is algorithmically and theoretically easier to work with. Numerically, $\|\nabla u\|_1$ can be approximated isotropically \cite{rof-1992} or  anisotropically  \cite{choksi-2011,esedoglu2004decomposition}:
	\begin{align}
	J_{\text{iso}}(u) &= \int_{\Omega} \sqrt{|D_x u(x,y)|^2 + |D_y u(x,y)|^2} \;dx \;dy, \\
	J_{\text{ani}}(u) &=  \int_{\Omega} |D_x u(x,y)| + |D_y u(x,y)| \;dx \;dy,
	\end{align}
	where $D_x$ and $D_y$ denote the horizontal and vertical partial derivative operators, respectively. 
	
		\begin{figure}
		\centering
			\includegraphics[width=0.75\textwidth]{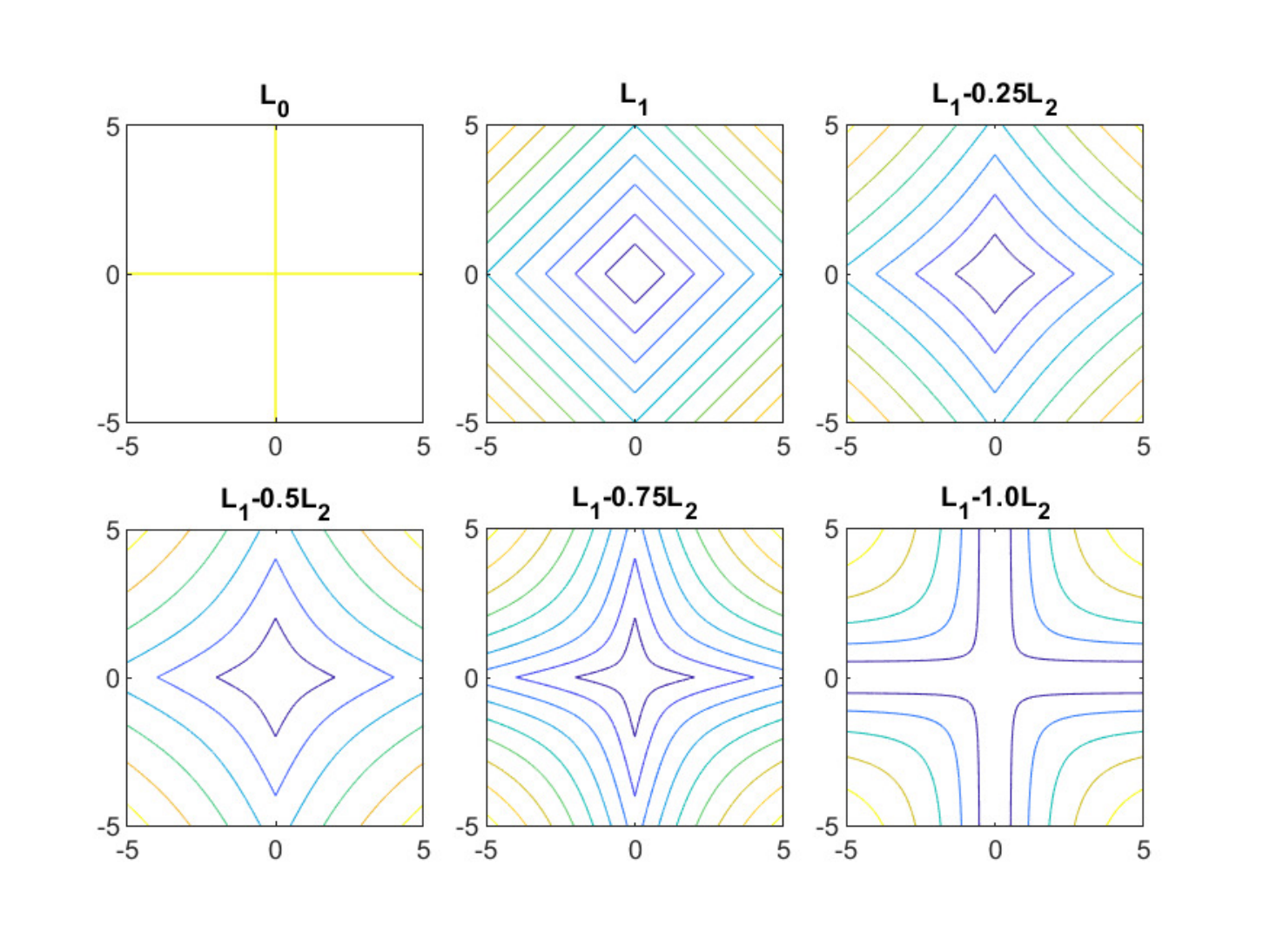}
		\caption{Contour lines of $\|x\|_0$ ($L_0$)  and $\|x\|_1 -\alpha \|x\|_2$ ($L_1-\alpha L_2$), where $x \in \mathbb{R}^2$ and $\alpha \in \{0, 0.25, 0.5, 0.75, 1.0\}$. As $\alpha$ increases, the contour lines of $L_1-\alpha L_2$ are closer to the ones of $L_0$.}
		\label{fig:contour}
		\vspace{-8mm}
	\end{figure}
	
In order to better approximate $\|\nabla u\|_0$, we consider the weighted anisotropic-isotropic TV (AITV), 
	\begin{gather}
	\begin{aligned} \label{eq:WDTV}
&J_{\text{ani}}(u) - \alpha J_{\text{iso}}(u)\\  &= \int_{\Omega} |D_x u(x,y)| + |D_y u(x,y)| - \alpha \sqrt{|D_x u(x,y)|^2 + |D_y u(x,y)|^2} \;dx \;dy
	\end{aligned} 
	\end{gather}
	with $\alpha \in [0,1]$. The AITV term was inspired by recent successes of $L_1-L_2$ minimization \cite{ding2019regularization,louY18,lou-2015-cs,louYX16,yin2014ratio,yin2015minimization} in compressed sensing. Compared with $L_1$, $L_p$ for $p \in (0,1)$ \cite{chartrand2007exact, lai2013improved, xu2012}, and $L_0$ \cite{tropp2004greed}, the $L_1-L_2$ penalty was shown to have the best performance in recovering sparse solutions when the sensing matrix is highly coherent or violates the restricted isometry property \cite{candes-2006}. 
	Figure \ref{fig:contour}  compares $L_0$, $L_1$, and $L_1-\alpha L_2$  by their contour lines in 2D. We observe that as $\alpha$ increases, the contour lines of $L_1-\alpha L_2$ are bending more inward and closer to the ones of $L_0$. This phenomenon illustrates that $L_1-\alpha L_2$ can encourage sparsity, and the constant $\alpha$ acts like a parameter controlling to what extent.
	By applying $L_1 - \alpha L_2$ on the gradient, Lou et al.~\cite{lou-2015}  proposed AITV  with a difference-of-convex algorithm (DCA) \cite{le2018dc,tao-1997,tao-1998}  for image denoising, deconvolution, and MRI reconstruction. Later, Li et al.~\cite{li2020} demonstrated the robustness of AITV with respect to impulsive noise corruption of the data. Both works \cite{li2020,lou-2015} showed that AITV preserves sharper image edges than the anisotropic TV. Moreover, AITV is preferred over the isotropic TV that tends to blur oblique edges \cite{birkholz2011unifying,condat2017discrete}. 
	
As edges are defined by gradient vectors, it is expected that AITV $(L_1-\alpha L_2)$ should produce sparser gradients and maintain sharper edges compared to TV $(L_1).$ A preliminary work that replaced $\|\nabla u\|_1$ by AITV in \eqref{eq:convex_CV} was conducted by Park et al.~\cite{park2016weighted}, showing better segmentation results than TV. However, this approach was limited to pre-determined values of $c_1$/$c_2$,  grayscale images, and two-phase segmentation (rather than multiphase).

	The CV model can be extended to vector-valued images \cite{chan2000active} and to multiphase segmentation \cite{brown2012completely,vese2002multiphase}. The vector-valued extension is  straightforward, i.e., replacing $f$ with a vector-valued input $\mathbf{f}: \Omega \rightarrow \mathbb{R}^C$ and replacing $c_1, c_2$ with vector-valued constants $\mathbf{c}_1, \mathbf{c}_2 \in \mathbb{R}^C,$ where $C$ is the number of channels in an image. The multiphase CV model relies on $\log_2(N)$ level-set functions to partition $\Omega$ into $N$ regions $\{\Omega_i\}_{i=1}^N$,
	and, 
	hence most CV-based multiphase segmentation methods are limited to power-of-two number of regions so that $\log_2(N)$ is an integer. There are two approaches that can deal with an arbitrary number of regions. One approach represents each region by a single level-set function \cite{samson2000level}, which unfortunately causes vacuums and  overlapping regions to appear. The other approach defines regions by membership functions,  referred to as  fuzzy region (FR) competition  \cite{li2010multiphase}.

	In this paper, we propose to incorporate the AITV term into both  CV and FR models together with an extension to color image segmentation. 
 To solve these models, we develop an alternating minimization framework that involves DCA  and PDHG with linesearch (PDHGLS) \cite{malitsky2018first}. We  provide convergence analysis of the proposed algorithms. Experimentally, we compare the proposed  models with the classic convex approaches and other segmentation methods to showcase the effectiveness and robustness of the AITV penalty. The major contributions of this work are threefold:
 \begin{itemize}
     \item We study the AITV regularization comprehensively  in image segmentation, including grayscale/color image and multiphase segmentation. 
     \item We propose an efficient algorithm that combines DCA and PDHGLS with guaranteed convergence.  To the best of our knowledge, this paper pioneers the implementation of PDHGLS in image segmentation.
     \item We conduct extensive experiments to demonstrate the effect of the constant $\alpha$ in AITV  on the segmentation performance and the robustness to impulsive noise. We compare the results with the two-stage segmentation methods. 
 \end{itemize}
	
	The paper is organized as follows. Section \ref{sec:notation} describes notations that will be used throughout the paper. In Section \ref{sec:AICV}, we introduce the AITV extension of the CV model, which can be solved by DCA with convergence analysis. In Section \ref{sec:AIFR}, we incorporate AITV into the FR model \cite{li2010multiphase} for multiphase segmentation with an algorithm similar to the CV model. In Section \ref{sec:color}, we  extend both CV and FR models to color image segmentation. Numerical results   are shown in Section \ref{sec:result}. Lastly, conclusions and future works are given in Section \ref{sec:conclusions}.

	\section{Notations} \label{sec:notation}
	 For  simplicity, we adopt the discrete notations for images and related models. 	The space $\mathbb{R}^n$ is equipped with the standard inner product $\langle x, y \rangle = \sum_{i=1}^n x_iy_i$ and standard Euclidean norm $\|x\|_2 = \sqrt{\langle x, x \rangle}$ for $x,y \in \mathbb{R}^n$.
	
	Without loss of generality, an image is represented as an $m \times n$ matrix, i.e. the image domain is $\Omega = \{ 1, 2, \ldots, m\} \times \{1,2, \ldots, n\}$. We denote $X \coloneqq \mathbb{R}^{m \times n}$ and the all-ones matrix in $X$ as $\mathbbm{1}$. The vector space $X$ is equipped with following inner product and norm:
	\begin{align*}
	\langle u, v \rangle_X &= \sum_{i=1}^m \sum_{j=1}^n u_{i,j} v_{i,j}, \quad\|u \|_X = \sqrt{\sum_{i=1}^m \sum_{j=1}^n u_{i,j}^2} \qquad \forall u,v \in X.
	\end{align*}
We denote	$D_x, D_y$ by the horizontal and vertical partial derivative operators, respectively, i.e.,
	\begin{align*}
	(D_x u)_{i,j} = \begin{cases}
	u_{i,j+1} - u_{i,j} &\text{ if } 1 \leq j \leq n-1, \\
	u_{i,1} - u_{i,n} &\text{ if } j = n,
	\end{cases}\\
	(D_y u)_{i,j} = \begin{cases}
	u_{i+1,j} - u_{i,j} &\text{ if } 1 \leq i \leq m-1, \\
	u_{1,j} - u_{m,j} &\text{ if } i = m.
	\end{cases}
	\end{align*}
	Let $Y \coloneqq X \times X$. Then the discrete gradient operator $D: X \rightarrow Y$ is defined as
	\begin{align*}
	(D u)_{i,j} = \left((D_x u)_{i,j}, (D_y u)_{i,j}\right) \in Y.
	\end{align*}
	For any $p=(p_x,p_y), q = (q_x, q_y) \in Y$, the inner product on $Y$ is defined by
	\begin{align*}
	\langle p,q \rangle_Y &= \langle p_x, q_x \rangle_X + \langle p_y , q_y \rangle_X,
	\end{align*}
	and the norms on $Y$ are
	\begin{align*}
	\|p\|_Y & = \sqrt{\sum_{i=1}^m \sum_{j=1}^n |(p_x)_{i,j}|^2 + |(p_y)_{i,j}|^2}, \qquad \|p\|_1 =  \sum_{i=1}^m \sum_{j=1}^n \left(|(p_x)_{i,j}| + |(p_y)_{i,j}| \right),\\
	\|p\|_{2,1} &= \sum_{i=1}^m \sum_{j=1}^n\sqrt{|(p_x)_{i,j}|^2 + |(p_y)_{i,j}|^2} = \sum_{i=1}^m \sum_{j=1}^n \|((p_x)_{i,j}, (p_y)_{i,j})\|_2.
	\end{align*}
	
We use a bold letter to denote a 3D tensor, e.g., $\mathbf{u} = (u_1, u_2, \ldots, u_N) \in X^N$. We further denote $\mathbf{u}_{< k} \coloneqq (u_1, \ldots, u_{k-1})$ and $\mathbf{u}_{>k} \coloneqq (u_{k+1}, \ldots, u_N)$ for $1 \leq k \leq N$. 
The notations $\mathbf{u}_{\leq k}$ and $\mathbf{u}_{\geq k}$ are defined similarly by including $u_k$. 
Note that  $\mathbf{u}_{<1}$ and $\mathbf{u}_{>N}$ are null or empty variables. 
	
	\section{Anisotropic-Isotropic Chan-Vese Model} \label{sec:AICV}
	
	Let $f \in X$ be an observed image. Suppose the image domain $\Omega$ has  $N=2^M$ non-overlapping regions, i.e. $\Omega = \bigcup_{i=1}^N \Omega_i$ and $\Omega_i \cap \Omega_j = \emptyset$ for each $i\neq j$. Let $\mathbf{u} = (u_1, \ldots, u_M) \in X^M$ and $\mathbf{c} = (c_1, \ldots, c_N) \in \mathbb{R}^N$. We propose an AITV-regularized  Chan-Vese (AICV) model for multiphase  segmentation as follows:
	\begin{align}\label{eq:AITV_MCV}
	\min_{\subalign{ \mathbf{u} &\in \mathcal{B}\\ \mathbf{c} &\in \mathbb{R}^N}}\sum_{k=1}^M \left(\|Du_k\|_1 - \alpha \|Du_k \|_{2,1} \right) + \lambda \sum_{\ell=1}^N \langle  f_{\ell}(\mathbf{c}), R_{\ell} (\mathbf{u}) \rangle_X, 
	\end{align}
	where 
$\mathcal{B} = \left \{\mathbf{u} \in X^M: (u_k)_{i,j} \in \{0,1\} \forall i,j,k  \right\},
$
	$f_{\ell}(\mathbf{c}) = (f-c_{\ell}\mathbbm{1})^2$ with square defined elementwise, and $R_{\ell}(\mathbf{u})$ is a function of $\mathbf u$ related to the region $\Omega_{\ell}$ such that
	\begin{align*}
    R_{\ell}(\mathbf{u})_{i,j} = \begin{cases}
    1 &\text{ if } (i,j) \in \Omega_{\ell}, \\
    0 &\text{ if } (i,j) \not \in \Omega_{\ell},
    \end{cases}
	\end{align*}
	with $\sum_{\ell=1}^N R_{\ell}(\mathbf{u})  = \mathbbm{1}$. Specifically when $N=2$ $(M=1)$, we have $R_1(\mathbf{u}) = u_1$ and $R_2(\mathbf{u}) = \mathbbm{1}-u_1$. When $N =4$ $(M=2)$, we have 
	\begin{alignat*}{3}
	&R_1(\mathbf{u})_{i,j} = (u_1)_{i,j}(u_2)_{i,j}, \quad &&R_2(\mathbf{u})_{i,j}  = (u_1)_{i,j}[1-(u_2)_{i,j}], \\ &R_3(\mathbf{u})_{i,j} = [1-(u_1)_{i,j}](u_2)_{i,j}, \quad &&R_4(\mathbf{u})_{i,j} =[1-(u_1)_{i,j}][1-(u_2)_{i,j}].
	\end{alignat*}
	When $N=8$ $(M=3)$, we have
	{\small
		\begin{alignat*}{3}
	&R_1(\mathbf{u})_{i,j}  = (u_1)_{i,j}(u_2)_{i,j}(u_3)_{i,j},  &&R_2(\mathbf{u})_{i,j}  = (u_1)_{i,j}(u_2)_{i,j}[1-(u_3)_{i,j}],\\ &R_3(\mathbf{u})_{i,j} = (u_1)_{i,j}[1-(u_2)_{i,j}] (u_3)_{i,j} , &&R_4(\mathbf{u})_{i,j} =(u_1)_{i,j}[1-(u_2)_{i,j}][1-(u_3)_{i,j}],\\
	&R_5(\mathbf{u})_{i,j}  = [1-(u_1)_{i,j}](u_2)_{i,j}(u_3)_{i,j},  && R_6(\mathbf{u})_{i,j}  =[1-(u_1)_{i,j}](u_2)_{i,j}[1-(u_3)_{i,j}],\\
	&R_7(\mathbf{u})_{i,j}  = [1-(u_1)_{i,j}][1-(u_2)_{i,j}](u_3)_{i,j},&& R_8(\mathbf{u})_{i,j} = [1-(u_1)_{i,j}][1-(u_2)_{i,j}][1-(u_3)_{i,j}].
	\end{alignat*}}%
	For $N =2^M$ with $M \geq 4$,  $R_{\ell}$ depends on $\ell$'s binary representation to decide whether to include $u_k$ or $\mathbbm{1}-u_k$ as a factor in $R_{\ell}$. 
	
	Due to the binary constraint set $\mathcal B$,  \eqref{eq:AITV_MCV} is a nonconvex optimization problem, thus numerically difficult to solve. We relax the binary constraint $\{0,1\}$ by a $[0,1]$ box constraint, which in turn has $R_{\ell} (\mathbf{u})_{i,j} \in [0,1]$. In particular, we rewrite \eqref{eq:AITV_MCV}   as an unconstrained formulation by introducing the indicator function
	\begin{align*}
\chi_{U}(u) = \begin{cases}
0 &\text{ if } u_{i,j} \in [0,1] \text{ for all } i,j, \\
+\infty &\text{ otherwise}.
\end{cases}
\end{align*} 
Hence, a relaxed model of \eqref{eq:AITV_MCV} can be expressed as
\begin{gather}
	\begin{aligned}
	\label{eq:relax_AITV_MCV}
	\min_{\subalign{\mathbf{u} &\in X^M\\ \mathbf{c} &\in \mathbb{R}^N}} \tilde{F}(\mathbf{u}, \mathbf{c}) \coloneqq \sum_{k=1}^M \Big(\|Du_k\|_1 - \alpha \|Du_k \|_{2,1} + \chi_{U}(u_k) \Big) + \lambda \sum_{\ell=1}^N \langle  f_{\ell}(\mathbf{c}), R_{\ell}(\mathbf{u}) \rangle_X.
	\end{aligned} 
	\end{gather}
	\subsection{Numerical Algorithm}
We propose an alternating minimization algorithm to find a solution of \eqref{eq:relax_AITV_MCV} with the following framework:
	\begin{align} \label{eq:minimize_u}
	\mathbf{u}^{t+1} & \in \argmin_{\mathbf{u}} \tilde{F}(\mathbf{u}, \mathbf{c}^t),  \\ \label{eq:minimize_c} 
	\mathbf{c}^{t+1} &\in \argmin_{\mathbf{c}} \tilde{F}(\mathbf{u}^{t+1}, \mathbf{c}),
	\end{align}
	where $t$ counts the (outer) iterations.
	Below, we discuss how to solve each subproblem. 
	
	
	We start with the $\mathbf c$-subproblem \eqref{eq:minimize_c}, as it is simpler than the other. Notice that
	we can solve  $c_{\ell}$ separately for each $\ell = 1, \ldots, N$, i.e.,
	\begin{align} \label{eq:reduced_c_prob}
	c^{t+1}_{\ell} &\in \argmin_{c_{\ell}} \lambda \langle f_{\ell}(\mathbf{c}), R_{\ell}(\mathbf{u}^{t+1}) \rangle_X =  \argmin_{c_{\ell}} \lambda \sum_{i=1}^m \sum_{j=1}^n (f_{i,j}-c_{\ell})^2 R_{\ell}(\mathbf{u}^{t+1})_{i,j}.
	\end{align}
If $\sum_{i=1}^m \sum_{j=1}^n R_{\ell}(\mathbf{u}^{t+1})_{i,j} \neq 0$, we differentiate the objective function in \eqref{eq:reduced_c_prob} with respect to $c_{\ell},$ set the derivative equal to zero, and solve for $c_{\ell};$ otherwise, since the objective function does not depend on $c_{\ell},$ the solution can take on any value, so we set the solution to 0 as a default. In summary, there is a closed-form solution to  \eqref{eq:reduced_c_prob} for updating $c_{\ell}^{t+1}$, i.e.,
	\begin{align}\label{eq:c_update}
	    	c^{t+1}_{\ell} = \begin{cases} \frac{\displaystyle\sum_{i=1}^m \sum_{j=1}^n f_{i,j} R_{\ell}(\mathbf{u}^{t+1})_{i,j}}{\displaystyle\sum_{i=1}^m \sum_{j=1}^n R_{\ell}(\mathbf{u}^{t+1})_{i,j}} &\text{ if }\displaystyle \sum_{i=1}^m \sum_{j=1}^n R_{\ell}(\mathbf{u}^{t+1})_{i,j} \neq 0, \\
	    	0 &\text{ if } \displaystyle \sum_{i=1}^m \sum_{j=1}^n R_{\ell}(\mathbf{u}^{t+1})_{i,j}  = 0.
	    	\end{cases}
	\end{align}
	The formula \eqref{eq:c_update} implies that $c_{\ell}^{t+1}$ is the mean intensity value of the region $\Omega_{\ell} \subset \Omega$ at the $(t+1)$-th iteration.
	
	The $\mathbf u$-subproblem \eqref{eq:minimize_u}  is   separable with respect to each $k$, i.e.,
	\begin{align}\label{eq:u_subproblem}
	u_{k}^{t+1} \in \argmin_{u_{k}} \|Du_k\|_1 - \alpha \|Du_k\|_{2,1} + \chi_{U}(u_k) + \lambda \langle r_{k}(\mathbf{c}^{t}, \mathbf{u}_{<k}^{t+1}, \mathbf{u}_{>k}^t), u_k \rangle_X,
	\end{align}
	where $r_{k}(\mathbf{c}^{t}, \mathbf{u}_{<k}^{t+1}, \mathbf{u}_{>k}^t)$ is
	a multivariate polynomial of $(\mathbf{u}_{<k}^{t+1}, \mathbf{u}_{>k}^t)$
	obtained by rewriting $\sum_{\ell=1}^{N} \langle f_{\ell}(\mathbf{c}), R_{\ell}(\mathbf{u})\rangle_X$ in \eqref{eq:relax_AITV_MCV} and getting the coefficients in front of $u_k$.  Because a general form of $r_k$ is complicated, we provide some specific examples in smaller dimensions.
When $N=2$ $(M=1)$, we have $r_1(\mathbf{c}, \mathbf{u}_{<1}, \mathbf{u}_{>1})_{i,j}  = (f_{i,j}-c_1)^2 -(f_{i,j}-c_2)^2$;
	when $N=4$ $(M=2)$, we have
	\begin{alignat*}{2}
	r_1(\mathbf{c}, \mathbf{u}_{<1}, \mathbf{u}_{>1})_{i,j} &= \left[(f_{i,j}-c_1)^2 - (f_{i,j}-c_2)^2 -(f_{i,j}-c_3)^2 + (f_{i,j}-c_4)^2\right](u_2)_{i,j},\\
	 &+(f_{i,j}-c_2)^2-(f_{i,j}-c_4)^2\\
	 r_2(\mathbf{c}, \mathbf{u}_{<2}, \mathbf{u}_{>2})_{i,j} &= \left[(f_{i,j}-c_1)^2 - (f_{i,j}-c_2)^2 - (f_{i,j}-c_3)^2 + (f_{i,j}-c_4)^2\right](u_1)_{i,j} \\
	 &+(f_{i,j}-c_3)^2-(f_{i,j}-c_4)^2.
	\end{alignat*}
	 
In order to minimize \eqref{eq:u_subproblem}, we apply a descent algorithm called DCA \cite{le2018dc,tao-1997,tao-1998} for solving a difference-of-convex (DC) optimization problem of the form $\displaystyle \min_{u \in X} g(u)-h(u),$
where $g$ and $h$ are proper, lower semicontinuous, and strongly convex functions. The algorithm consists of two steps per iteration with $u^0$ as initialization:
\begin{align} \label{eq:DCA}
\begin{cases}
v^t &\in \partial{h}(u^t), \\
u^{t+1} &\in \displaystyle \argmin_{u \in X} g(u) - \langle v^t, u \rangle_X.
\end{cases}
\end{align}
	 
	For each $k=1,\ldots, M$, we can express \eqref{eq:u_subproblem} as a DC function $g(u_k) - h(u_k)$
	with
	\begin{align}
	\begin{cases}
	g(u_k) &= \|Du_k\|_1 + \chi_{U}(u_k) + \lambda \langle r_{k}(\mathbf{c}^{t}, \mathbf{u}_{<k}^{t+1}, \mathbf{u}_{>k}^t), u_k \rangle_X + c\|u_k\|_X^2,\\
	h(u_k) &= \alpha \|Du_k\|_{2,1} + c\|u_k\|_X^2,
	\end{cases}
	\end{align}
	where $c > 0$ enforces strong convexity on the functions $g$ and $h$. Experimentally, $c$ can be chosen arbitrarily small for better performance. We then compute the subgradient of $h(u),$ i.e.,
	\begin{align*}
	\alpha\frac{D_x^{\top} D_x u + D_y^{\top}D_yu}{\sqrt{|D_xu|^2+|D_yu|^2 }} +2cu \in \partial{h(u)}. 
	\end{align*}
	Therefore, the $u$-subproblem in \eqref{eq:DCA} can be expressed as
	\begin{gather}
	\begin{aligned}\label{eq:DC_linear}
	u_k^{t+1} = \argmin_{u_k}  &\|Du_k\|_1 + \chi_{U}(u_k) + \lambda \langle r_{k}(\mathbf{c}^{t}, \mathbf{u}_{<k}^{t+1}, \mathbf{u}_{>k}^t), u_k \rangle_X + c\|u_k\|_X^2\\
	&-\alpha \langle Du_k, q_k^t \rangle_Y - 2c \langle u_k, u_k^t \rangle_X,
	\end{aligned}
	\end{gather}
where $q_k^t \coloneqq ((q_x)_k^t, (q_y)_k^t) = (D_x u_k^t, D_y u_k^t)/\sqrt{|D_xu_k^t|^2 + |D_yu_k^t|^2}$. Note that we compute $q_k^t$ elementwise and  adopt the convention that if the denominator is zero at some point, the corresponding $q_k^t$ value is set to zero, which aligns with the subgradient definition. 
To solve the convex problem \eqref{eq:DC_linear}, we apply the PDHG algorithm \cite{chambolle-pock-2011, esser2010general, zhu2008efficient} since it was demonstrated in \cite{chambolle-pock-2011} that PDHG  solves imaging models with the TV term \cite{rof-1992} efficiently.

In general,	the PDHG algorithm \cite{chambolle-pock-2011, esser2010general, zhu2008efficient} targets at a saddle-point problem
	\begin{align*}
	\min_u \max_v \Psi(u) + \langle Au, v \rangle_Y - \Phi(v),
	\end{align*}
	where $\Psi,\Phi$ are convex functions and $A$ is a linear operator. The PDHG algorithm is outlined as
	\begin{align*}
	u^{\eta+1} &= (I+ \tau \partial{\Psi})^{-1}(u^{\eta} - \tau A^{\top} v^{\eta}), \\
	\bar{u}^{\eta+1} &= u^{\eta+1} + \theta(u^{\eta+1}- u^{\eta}) ,\\
	v^{\eta+1} &=(I+ \sigma \partial \Phi)^{-1}(v^{\eta}+ \sigma A \bar{u}^{\eta+1}),
	\end{align*}
	with $\tau, \sigma >0, \theta \in [0,1]$. The inverse  is defined by the proximal operator, i.e.,
	\begin{align*}
	(I+ \tau \partial{\Psi})^{-1}(z) = \min_u \left(\Psi(u)+ \frac{\|u-z\|_X^2}{2 \tau} \right),
	\end{align*}
	and similarly for $(I+ \sigma \partial \Phi)^{-1}$.

	In order to apply PDHG for
	the $u_k$-problem in  \eqref{eq:DC_linear}, we define its saddle-point formulation:
	\begin{gather}
	\begin{aligned}\label{eq:saddle_point_reformulation}
	\min_{u_k} \max_{(p_x)_k, (p_y)_k}   &\langle D_x u_k, (p_x)_k \rangle_X + \langle D_y u_k, (p_y)_k \rangle_X + \chi_{U}(u_k)\\  &+ \lambda \langle r_{k}(\mathbf{c}^{t}, \mathbf{u}_{<k}^{t+1}, \mathbf{u}_{>k}^t), u_k \rangle_X + c\|u_k\|_X^2
	-\alpha\langle Du_k, q^t_k \rangle_Y -2c\langle u_k,\, u^t_k\rangle_X\\ &- \chi_{P}((p_x)_k) - \chi_{P}((p_y)_k),
	\end{aligned}
\end{gather}
	where $(p_x)_k,(p_y)_k$ are dual variables of $D_xu_k, D_y u_k,$ and $P = \{ p: |p_{i,j}| \leq 1  \ \forall i,j\}$ is a convex set.
	Please refer to \cite{chambolle2010introduction, chambolle-pock-2011} for the derivation of the saddle-point formulation in more details.
	 Then we have
	\begin{align*}
	\Psi_{k,t}(u_k) &= \chi_{U}(u_k) + \lambda \langle r_{k}(\mathbf{c}^{t}, \mathbf{u}_{<k}^{t+1}, \mathbf{u}_{>k}^t), u_k \rangle_X + c\|u_k\|_X^2,\\
	&\qquad-\alpha\langle Du_k, q^t_k \rangle_Y -2c\langle u_k,\, u^t_k\rangle_X ,\\
	Au_k &= (D_x u_k, D_y u_k) \\
	\Phi((p_x)_k, (p_y)_k) &= \chi_P((p_x)_k) + \chi_P((p_y)_k). 
	\end{align*}
	
	With the initial condition $u^{t,0}_k = u^t_k$, the $u$-subproblem can be computed as
	\begin{gather}
	\begin{aligned}\label{eq:u_opt}
	u_k^{t,\eta+1} = &\left(I+ \tau \partial{\Psi}_{k,t}\right)^{-1}\left(u_k^{t,\eta} - \tau \left(D_x^{\top}(p_x)_k^{\eta} +D_y^{\top}(p_y)_k^{\eta} \right)\right)\\ 
	 =&\min_{0 \leq (u_k)_{i,j} \leq 1} \Bigg\{ \lambda \langle r_{k}(\mathbf{c}^{t}, \mathbf{u}_{<k}^{t+1}, \mathbf{u}_{>k}^t), u_k \rangle_X + c\|u_k\|_X^2 \\ 
	& \qquad \qquad -\alpha\langle Du_k, q^t_k \rangle_Y -2c\langle u_k,\, u^t_k\rangle_X \\
&\qquad \qquad + \frac{\|u_k- \left(u_k^{t,\eta} - \tau \left(D_x^{\top}(p_x)_k^{\eta} +D_y^{\top}(p_y)_k^{\eta} \right)\right)\|_X^2}{2 \tau}\Bigg\},
	\end{aligned}
	\end{gather}
	where $\eta$ indexes the inner iteration, as opposed to $t$ for the outer iteration.
To solve \eqref{eq:u_opt}, we  derive a closed-form solution  that is similar to the one for the $u$-subproblem of \eqref{eq:convex_CV} determined in \cite{goldstein2010geometric}. In particular, we observe that the objective function in \eqref{eq:u_opt} is proper, continuous, and strongly convex with respect to $u_k$, so it has a unique minimizer. By ignoring the constraint and differentiating the  objective function in \eqref{eq:u_opt} with respect to $u_k$, we obtain
	\begin{align*}
\tilde u_k^{t,\eta+1}=	\frac{2c u_k^t + \frac{1}{\tau} u_k^{t,\eta}}{2c + \frac{1}{\tau}} - \frac{\lambda r_k(\mathbf{c}^t, \mathbf{u}_{<k}^{t+1}, \mathbf{u}_{>k}^t)- \alpha D^{\top} q_k^t + (D_x^{\top} (p_x)_k^{\eta} + D_y^{\top} (p_y)_k^{\eta})}{2c+ \frac{1}{\tau}}.
	\end{align*}
If $(\tilde u_k^{t,\eta+1})_{i,j}$ lies in the interval $[0,1]$, then the $(i,j)$-entry of the unique minimizer also coincides with the minimizer of the constrained problem \eqref{eq:u_opt}. If $(\tilde u_k^{t, \eta+1})_{i,j}$ is outside of the interval, then the $(i,j)$-entry of the unique minimizer lies at the interval endpoint closest to the unconstrained minimizer due to the quadratic  objective function. As a result, we project $\tilde u_k^{t,\eta+1}$ onto  $[0,1]$, leading to a closed-form solution for $u_k^{t,\eta+1}$:
	\begin{equation}\label{eq:u_descent}
	u^{t,\eta+1}_k
	= \min\{\max\{\tilde u_k^{t,\eta+1}, 0 \}, 1\},
	\end{equation}
	where $\min$ and $\max$ are executed elementwise. 
	
It is straightforward to derive 	closed-form solutions for $(p_x)_k,(p_y)_k$ in \eqref{eq:saddle_point_reformulation} given by
	\begin{gather}
	\begin{aligned}\label{eq:p_ascent}
	(p_x)_k^{\eta+1} &=\text{Proj}_P ((p_x)_k^{\eta} + \sigma D_x \bar{u}_k^{t,\eta+1}), \\
	(p_y)_k^{\eta+1} &=\text{Proj}_P ((p_y)_k^{\eta} + \sigma D_y \bar{u}^{t,\eta+1}_k)
	\end{aligned}
	\end{gather}
	with $\bar{u}_k^{t,\eta+1} = u_k^{t,\eta+1}+\theta(u_k^{t,\eta+1}-u_k^{t,\eta})$
and
$
	\text{Proj}_P(p) = \frac{p}{\max\{|p|, 1\}}.
$	We see that \eqref{eq:u_descent} is projected gradient descent of the primal variable $u$  with entrywise box constraint $[0,1]$, while \eqref{eq:p_ascent} is projected gradient ascent of the dual variable $(p_x, p_y)$ that is constrained to the set $P$.
The update order between the primal variable $u_k^{t, \eta}$ and the dual variables $(p_x)_k^{\eta}, (p_y)_k^{\eta}$ does not matter for PDHG \cite{chambolle-pock-2011, malitsky2018first}.
To further improve the speed and solution quality of PDHG, we incorporate a linesearch technique \cite{malitsky2018first} that starts with the primal variable, followed by the dual update. The PDHG algorithm with linesearch is referred to as PDHGLS. Both PDHG and PDHGLS provide a saddle-point solution $(u_k^*, (p_x)_k^*, (p_y)_k^*)$ for \eqref{eq:saddle_point_reformulation}  upon convergence \cite{chambolle-pock-2011, malitsky2018first}. Since \eqref{eq:DC_linear} is convex,
$u_k^*$ is indeed its solution, independent of the choice between using PDHG or PDHGLS. We summarize the proposed DCA-PDHGLS algorithm to solve \eqref{eq:relax_AITV_MCV}  in Algorithm \ref{alg:DCA_PDHG}. 
	\begin{algorithm}[t]
	\caption{DCA-PDHGLS algorithm to solve \eqref{eq:relax_AITV_MCV}}
	\label{alg:DCA_PDHG}\smallskip
	\textbf{Input: }\begin{itemize}
		\item Image $f$
		\item model parameters $\alpha, \lambda >0$ 
		\item strong convexity parameter $c > 0$
		\item PDHGLS initial step size $\tau_0 > 0$
		\item PDHGLS primal-dual step size ratio $\beta > 0$
		\item PDHGLS parameter $\delta \in (0,1)$
		\item PDHGLS step size multiplier $\mu \in (0,1)$
	\end{itemize}
	\begin{algorithmic}[1]
			\STATE  Set $u_k^0 = 1$ $(k=1, \ldots, M)$ for some region $\Sigma \subset \Omega$ and $0$ elsewhere.
			\STATE Compute $\mathbf{c}^0 = (c_1^0, \ldots, c_N^0)$ by \eqref{eq:c_update}.
			\STATE Set $t \coloneqq 0$.
			\WHILE{stopping criterion for DCA is not satisfied}
			\FOR{$k=1$ to $M$}
			\STATE Set $u_k^{t,0} \coloneqq u_k^t$ and $(p_x)_k^0 = (p_y)_k^0 = 0$.
			\STATE Compute $((q_x)_k^{t}, (q_y)_k^{t}) = (D_x u_k^{t}, D_y u_k^{t})/\sqrt{|D_xu_k^{t}|^2 + |D_yu_k^{t}|^2}$.
			\STATE Set $\theta_0 =1$.
			\STATE Set $\eta \coloneqq  0$.
			\WHILE{stopping criterion for PDHGLS is not satisfied}
			\STATE Compute $u_k^{t,\eta+1}$ by \eqref{eq:u_descent} with $\tau \coloneqq \tau_{\eta}$.
			\STATE Set $\tau_{\eta+1} = \tau_{\eta} \sqrt{1+\theta_{\eta}}$.\\
			\textbf{Linesearch:}
			\STATE Compute $\theta_{\eta+1} = \frac{\tau_{\eta+1}}{\tau_{\eta}}$ and $\sigma_{\eta+1} = \beta \tau_{\eta+1}$.
			\STATE Compute $\bar{u}_k^{t,\eta+1} = u_k^{t,\eta+1}+\theta_{\eta+1}(u_k^{t,\eta+1}-u_k^{t,\eta})$.
			\STATE Compute $p_k^{\eta+1} \coloneqq ((p_x)_k^{\eta+1}, (p_y)_k^{\eta+1})$ by \eqref{eq:p_ascent} with $\sigma \coloneqq \sigma_{\eta+1}$.
			\IF{$\sqrt{\beta} \tau_{\eta+1} \|(D_x^{\top} (p_x)_k^{\eta+1},  D_y^{\top} (p_y)_k^{\eta+1}) -(D_x^{\top}(p_x)_k^{\eta}, D_y^{\top}(p_y)_k^{\eta}) \|_Y \leq  \delta \|p_k^{\eta+1}-p_k^{\eta}\|_Y$}
			\STATE Set $\eta \coloneqq \eta+1$, and break linesearch
			\ELSE 
			\STATE Set $\tau_{\eta+1} \coloneqq \mu \tau_{\eta+1}$ and go back to line 13.
			\ENDIF\\
			\textbf{End of linesearch}
			\ENDWHILE
			\STATE Set $u_k^{t+1} \coloneqq u_k^{t,\eta}$.
			\ENDFOR
			\STATE Compute $\mathbf{c}^{t+1}$ by \eqref{eq:c_update}. 
			\STATE Set $t \coloneqq t+1.$
			\ENDWHILE
	\end{algorithmic}
	\textbf{Output:} $(\mathbf{u}, \mathbf{c}) \coloneqq (\mathbf{u}^t, \mathbf{c}^t)$.
	\end{algorithm}
	\subsection{Convergence Analysis} \label{subsec:converge} We analyze the convergence of the sequence\\ $\{(\mathbf{u}^t, \mathbf{c}^t)\}_{t=1}^{\infty}$ generated by \eqref{eq:minimize_u} and \eqref{eq:minimize_c}, which are solved by \eqref{eq:DC_linear} and \eqref{eq:c_update}, respectively.
We establish in Lemma~\ref{lemma:decrease1} that the sequence $\{\tilde{F}(\mathbf{u}^t, \mathbf{c}^t)\}_{t=1}^{\infty}$ decreases sufficiently, followed by the convergence result in Theorem~\ref{thm:opt}.

	\begin{lemma}\label{lemma:decrease1} Suppose $\alpha \in [0,1]$ and $\lambda > 0$. Let $\{(\mathbf{u}^t, \mathbf{c}^t)\}_{t=1}^{\infty}$ be a sequence such that $\mathbf{u}^t$ is generated by \eqref{eq:DC_linear} and $\mathbf{c}^t$ is generated by  \eqref{eq:c_update}. Then we have
		\begin{align*}
		\tilde{F}(\mathbf{u}^t, \mathbf{c}^t) - \tilde{F}(\mathbf{u}^{t+1}, \mathbf{c}^{t+1}) \geq 2c \sum_{k=1}^M \|u_k^t - u_k^{t+1}\|_X^2.
		\end{align*}
	\end{lemma}
\begin{proof}
	Since $\mathbf{c}^{t+1}$  satisfies $\eqref{eq:c_update}$, we have
	\begin{align} \label{eq:c_decrease}
	\tilde{F}(\mathbf{u}^{t+1}, \mathbf{c}^{t+1}) \leq \tilde{F}(\mathbf{u}^{t+1}, \mathbf{c}^{t}).
	\end{align} 
Then we estimate 
\begin{gather}
			\begin{aligned}\label{eq:diff}
			&\tilde{F}((\mathbf{u}_{\leq k-1}^{t+1}, \mathbf{u}_{\geq k}^t), \mathbf{c}^t) - 		\tilde{F}((\mathbf{u}_{\leq k}^{t+1}, \mathbf{u}_{\geq k+1}^t), \mathbf{c}^t)\\ =& \|Du_k^t\|_1 - \|Du_k^{t+1}\|_1
			  - \alpha(\|Du_k^t\|_{2,1} - \|Du_k^{t+1}\|_{2,1}) + \chi_{U}(u_k^t)-\chi_{U}(u_k^{t+1})\\   &+ \lambda \sum_{\ell=1}^N \langle f_{\ell}(\mathbf{c}), R_{\ell} (\mathbf{u}_{\leq k-1}^{t+1}, \mathbf{u}_{\geq k}^t) -  R_{\ell} (\mathbf{u}_{\leq k}^{t+1}, \mathbf{u}_{\geq k+1}^t)\rangle_X.
			\end{aligned}
			\end{gather}
	It follows from the first-order optimality condition of \eqref{eq:DC_linear} at $u^{t+1}_k$ that there exists $p_k^{t+1} \in \partial \left(\|Du_k^{t+1}\|_1 + \chi_{U}(u_k^{t+1}) \right)$ such that
			\begin{align*}
			0 = p_k^{t+1} -\alpha D^{\top} q^{t}_k + 2c(u_k^{t+1} - u_k^{t}) + \lambda r_{k}(\mathbf{c}^{t}, \mathbf{u}_{<k}^{t+1}, \mathbf{u}_{>k}^t).
			\end{align*}
			Taking the inner product with $u_k^{t} - u_k^{t+1}$ and rearranging it, we obtain
\begin{gather}
\begin{aligned} \label{eq:equal}
&\lambda \langle r_{k}(\mathbf{c}^{t}, \mathbf{u}_{<k}^{t+1}, \mathbf{u}_{>k}^t), u^{t}_k - u_k^{t+1} \rangle_X \\
	=& - \langle  p_k^{t+1} - \alpha D^{\top} q_k^t, u_k^t -u_k^{t+1} \rangle_X +2c \|u_k^{t+1}-u_k^t\|_X^2.
\end{aligned}
\end{gather}
The last term in \eqref{eq:diff} can be simplified to \begin{align*}
\displaystyle \sum_{\ell=1}^N \langle f_{\ell}, R_{\ell} (\mathbf{u}_{\leq k-1}^{t+1}, \mathbf{u}_{\geq k}^t)  -  R_{\ell} (\mathbf{u}_{\leq k}^{t+1}, \mathbf{u}_{\geq k+1}^t)\rangle_X = \langle r_k(\mathbf{c}^t, \mathbf{u}_{<k}^{t+1}, \mathbf{u}_{>k}^t), u_k^t -u_{k}^{t+1} \rangle_X,
\end{align*}
as $R_{\ell}(\mathbf{u})$ consists of terms with at most one $u_k$, and the terms without $u_k^t$ and $u_k^{t+1}$ are cancelled out.
Together with  \eqref{eq:diff} and \eqref{eq:equal}, we get
\begin{gather}
\begin{aligned}\label{eq:long_eq}
&\tilde{F}((\mathbf{u}_{\leq k-1}^{t+1}, \mathbf{u}_{\geq k}^t), \mathbf{c}^t) - 		\tilde{F}((\mathbf{u}_{\leq k}^{t+1}, \mathbf{u}_{\geq k+1}^t), \mathbf{c}^t)\\
  =& \|Du_k^t\|_1 - \|Du_k^{t+1}\|_1- \alpha(\|Du_k^t\|_{2,1} - \|Du_k^{t+1}\|_{2,1})\\&+ \chi_{U}(u_k^t)-\chi_{U}(u_k^{t+1}) + \lambda \langle r_k(\mathbf{c}^t, \mathbf{u}_{<k}^{t+1}, \mathbf{u}_{>k}^t), u_k^t -u_{k}^{t+1} \rangle_X \\
=&\|Du_k^t\|_1 - \|Du_k^{t+1}\|_1- \alpha(\|Du_k^t\|_{2,1} - \|Du_k^{t+1}\|_{2,1})\\
  &+ \chi_{U}(u_k^t)-\chi_{U}(u_k^{t+1})-  \langle p_k^{t+1} - \alpha D^{\top} q_k^t, u_k^t -u_k^{t+1} \rangle_X +2c \|u_k^{t+1}-u_k^t\|_X^2 \\
= &\left[ \left(\|Du_k^t\|_1 - \langle p_k^{t+1}, u_k^t -u_k^{t+1} \rangle_X +\chi_{U}(u_k^t) \right)- \|Du_k^{t+1}\|_1 -\chi_{U}(u_k^{t+1}) \right] \\
 &+\alpha (\|Du_k^{t+1}\|_{2,1} - \langle D^{\top}q_k^t, u_k^t - u_k^{t+1}\rangle_X - \|Du_k^t\|_{2,1}) + 2c  \|u_k^{t+1}-u_k^t\|_X^2.
\end{aligned}
\end{gather}
The definitions of convexity and subgradient yield that
\begin{align}
\|Du^{t}_k\|_1 +\chi_{U}(u^t_k) - \langle p_k^{t+1}, u_k^t - u_k^{t+1} \rangle_X  &\geq \|Du^{t+1}_k\|_1  +\chi_{U}(u^{t+1}_k), \\ \label{eq:convex_ineq}
\|Du_k^{t+1}\|_{2,1} - \langle D^{\top}q_k^t, u_k^{t+1}-u_k^t \rangle_X &\geq \|Du_k^t\|_{2,1}.
\end{align}
Combining \eqref{eq:long_eq}-\eqref{eq:convex_ineq}, we have
\begin{align*}
\tilde{F}((\mathbf{u}_{\leq k-1}^{t+1}, \mathbf{u}_{\geq k}^t), \mathbf{c}^t) - 		\tilde{F}((\mathbf{u}_{\leq k}^{t+1}, \mathbf{u}_{\geq k+1}^t), \mathbf{c}^t) \geq 2c \|u_k^{t+1}-u_k^t\|_X^2.
\end{align*}
Summing over $k = 1, \ldots, M$ leads to
\begin{gather}
\begin{aligned} \label{eq:sum_k}
\tilde{F}(\mathbf{u}^t, \mathbf{c}^{t}) - \tilde{F}(\mathbf{u}^{t+1}, \mathbf{c}^{t}) &= \sum_{k=1}^M \tilde{F}((\mathbf{u}_{\leq k-1}^{t+1}, \mathbf{u}_{\geq k}^t), \mathbf{c}^t) - 		\tilde{F}((\mathbf{u}_{\leq k}^{t+1}, \mathbf{u}_{\geq k+1}^t), \mathbf{c}^t) \\
&\geq 2c \sum_{k=1}^M \|u_k^{t+1} - u_k^t \|_X^2.
\end{aligned}
\end{gather}
Therefore, \eqref{eq:c_decrease} and \eqref{eq:sum_k} establish the desired result.
\end{proof}
\begin{theorem}\label{thm:opt}
	Suppose $\alpha \in [0,1]$ and $\lambda >0$.  Let $\{(\mathbf{u}^t, \mathbf{c}^t)\}_{t=1}^{\infty}$ be a sequence such that $\mathbf{u}^t$ is generated by \eqref{eq:DC_linear} and $\mathbf{c}^t$ is generated by  \eqref{eq:c_update}. We have the following:
	\begin{enumerate}[label=(\alph*)]
		\item $\{(\mathbf{u}^t, \mathbf{c}^t)\}_{t=1}^{\infty}$ is bounded.
		\item For $k=1, \ldots, M$, we have $\|u_k^{t+1}-u_k^t\|_X \rightarrow 0$ as $t \rightarrow \infty$.
		\item The sequence $\{(\mathbf{u}^t, \mathbf{c}^t)\}_{t=1}^{\infty}$ has a limit point $(\mathbf{u}^*, \mathbf{c}^*)$ satisfying
	\begin{align}\label{eq:crit_u}
	\mathbf{0} \in \partial\|Du_k^*\|_1 - \alpha \partial \|Du_k^*\|_{2,1} + \partial\chi_{U}(u_k^*)+ \lambda r_k(\mathbf{c}^*, \mathbf{u}_{<k}^{*}, \mathbf{u}_{>k}^*) 
	\end{align}
	for $k = 1, \ldots, M,$ and
	 \begin{align} \label{eq:c_opt}
	0 &\in \frac{\partial\tilde{F}(\mathbf{u}^*, \mathbf{c}^*)}{\partial c_{\ell}}, \quad \ell= 1, \ldots, N.	\end{align}
		\end{enumerate} 
\end{theorem}
\begin{proof}
	(a) As each entry of $u_k^t$ is bounded by $[0,1]$ for $k=1, \ldots, M$, $\{\mathbf{u}^t\}_{t=1}^{\infty}$ is a bounded sequence. It further follows from  \eqref{eq:c_update} that 
$
0 \leq |c_{\ell}^{t+1}| \leq \max_{i,j} |f_{i,j}|.
$
 Therefore, $\{\mathbf{c}^t\}_{t=1}^{\infty}$ is also bounded, and altogether so is the sequence $\{(\mathbf{u}^t, \mathbf{c}^t)\}_{t=1}^{\infty}$. 
	
	\medskip
	(b) Since $\alpha \|Du_k\|_{2,1} \leq \|Du_k\|_1$ for $\alpha \in [0,1]$, we have
	\begin{align}\label{eq:nonnegativity}
	\tilde{F}(\mathbf{u}, \mathbf{c}) \geq \sum_{k=1}^M \chi_{U}(u_k) + \lambda \sum_{\ell=1}^N \langle f_{\ell}, R_{\ell}(\mathbf{u}) \rangle_X \geq 0,
 	\end{align}
 which implies that $\tilde{F}(\mathbf{u}, \mathbf{c})$ is lower bounded. As it is also decreasing by Lemma \ref{lemma:decrease1},  the sequence $\{\tilde{F}(\mathbf{u}^t, \mathbf{c}^t)\}_{t=1}^{\infty}$  converges. By a telescope summation of \eqref{eq:sum_k}, we obtain
	\begin{align*}
	\tilde{F}(\mathbf{u}^1, \mathbf{c}^1) - \lim_{t \rightarrow \infty} 	\tilde{F}(\mathbf{u}^t, \mathbf{c}^t) \geq 2c \sum_{t=1}^{\infty} \sum_{k=1}^M \|u_k^t -u_k^{t+1}\|_X^2 = 2c \sum_{k=1}^M \sum_{t=1}^{\infty} \|u_k^t -u_k^{t+1}\|_X^2.
	\end{align*}
	Therefore, $\sum_{t=1}^{\infty} \|u_k^t -u_k^{t+1}\|_X^2<\infty$,  leading to $\displaystyle \lim_{t \rightarrow \infty} \|u_k^t - u_k^{t+1}\|_X^2=0$ for $k=1, \ldots, M$.
	
	\medskip
	(c) By Bolzano-Weierstrass Theorem,  the bounded sequence $\{(\mathbf{u}^t, \mathbf{c}^t)\}_{t=1}^{\infty}$  has a convergent subsequence $\{(\mathbf{u}^{t_L}, \mathbf{c}^{t_L})\}_{L=1}^{\infty}$ such that $\displaystyle \lim_{L \rightarrow \infty} (\mathbf{u}^{t_L}, \mathbf{c}^{t_L}) = (\mathbf{u}^*, \mathbf{c}^*)$ . 
By (b), $\displaystyle \lim_{L\rightarrow \infty} u_k^{t_L+1} - u_k^{t_L} = 0$. As $\displaystyle \lim_{L \rightarrow \infty} u_k^{t_L+1} = \lim_{L \rightarrow \infty} u_k^{t_L} =u_k^{*},$ we have $\displaystyle \lim_{L \rightarrow \infty} \mathbf{u}^{t_L+1} = \mathbf{u}^*$.
Since $\mathbf{u}^{t_L}$ is generated by \eqref{eq:DC_linear}, all of its entries are bounded by $[0,1]$; otherwise, the objective function would be at $+\infty$. Hence, $\chi_{U}(u_k^{t_L}) = 0$ and similarly $\chi_{U}(u_k^{t_L+1}) = 0$ for all $L$, from which follows that $\chi_{U}(u_k^*) = 0$. In short, we have
	\begin{align} \label{eq:lim_chi}
	\lim_{L \rightarrow \infty} \chi_{U}(u_k^{t_L}) = \chi_{U}(u_k^*) \quad \text{ for } k = 1, \ldots, M.
	\end{align}

Now we establish \eqref{eq:c_opt} by showing that $ \tilde{F}(\mathbf{u}^*, \mathbf{c}^*) \leq  \tilde{F}(\mathbf{u}^*, \mathbf{c})$ for all   $\mathbf{c} \in \mathbb{R}^n.$ On one hand, we have
\begin{gather}\label{eq:thm1}
	\begin{aligned}
&\lim_{L \rightarrow \infty} \tilde{F}(\mathbf{u}^{t_L}, \mathbf{c}^{t_L})\\
=&\lim_{L \rightarrow \infty} \left[\sum_{k=1}^M \left( \|Du_k^{t_L}\|_1 - \alpha \|Du_k^{t_L}\|_{2,1} + \chi_U (u_k^{t_L})\right) +\lambda \sum_{\ell=1}^N \langle f_{\ell}(\mathbf{c}^{t_L}), R_{\ell}(\mathbf{u}^{t_L}) \rangle_X \right] \\
=&\sum_{k=1}^M \lim_{L \rightarrow \infty}\left( \|Du_k^{t_L}\|_1 - \alpha \|Du_k^{t_L}\|_{2,1} + \chi_U (u_k^{t_L})\right) +\lambda \sum_{\ell=1}^N \lim_{L \rightarrow \infty}\langle f_{\ell}(\mathbf{c}^{t_L}), R_{\ell}(\mathbf{u}^{t_L}) \rangle_X\\
=& \sum_{k=1}^M \left( \|Du_k^{*}\|_1 - \alpha \|Du_k^*\|_{2,1} + \chi_U (u_k^{*})\right) +\lambda \sum_{\ell=1}^N \langle f_{\ell}(\mathbf{c}^*), R_{\ell}(\mathbf{u}^*) \rangle_X =\tilde{F}(\mathbf{u}^*, \mathbf{c}^*).\\
		\end{aligned}
	\end{gather}
We can take the limit as all the terms of $\tilde{F}$ except for $\chi_U$ are continuous with respect to $(\mathbf{u}, \mathbf{c}).$ On the other hand, we have
\begin{gather}
	\begin{aligned}
	&\lim_{L \rightarrow \infty} \tilde{F}( \mathbf{u}^{t_L}, \mathbf{c}) \\
	=& \lim_{L \rightarrow \infty} \left[\sum_{k=1}^M \left( \|Du_k^{t_L}\|_1 - \alpha \|Du_k^{t_L}\|_{2,1} + \chi_U (u_k^{t_L})\right) +\lambda \sum_{\ell=1}^N \langle f_{\ell}(\mathbf{c}), R_{\ell}(\mathbf{u}^{t_L}) \rangle_X \right] \\
	=& \sum_{k=1}^M \lim_{L \rightarrow \infty}\left( \|Du_k^{t_L}\|_1 - \alpha \|Du_k\|_{2,1} + \chi_U (u_k^{t_L})\right) +\lambda \sum_{\ell=1}^N \lim_{L \rightarrow \infty}\langle f_{\ell}(\mathbf{c}), R_{\ell}(\mathbf{u}^{t_L}) \rangle_X\\
	=& \sum_{k=1}^M \left( \|Du_k^{*}\|_1 - \alpha \|Du_k^*\|_{2,1} + \chi_U (u_k^{*})\right) +\lambda \sum_{\ell=1}^N \langle f_{\ell}(\mathbf{c}), R_{\ell}(\mathbf{u}^*) \rangle_X = \tilde{F}(\mathbf{u}^*, \mathbf{c}).
	\end{aligned}\label{eq:thm2}
	\end{gather}
It follows from \eqref{eq:minimize_c} that for all $L \in \mathbb{N}$, we have
	\begin{align}\label{eq:F_c_ineq}
	\tilde{F}(\mathbf{u}^{t_L}, \mathbf{c}^{t_L}) \leq \tilde{F}(\mathbf{u}^{t_L}, \mathbf{c}) \quad \forall \; \mathbf{c} \in \mathbb{R}^N.
	\end{align}	
Combined with \eqref{eq:thm1}-\eqref{eq:thm2}, 
\begin{align*}
    \tilde{F}(\mathbf{u}^*, \mathbf{c}^*) = \lim_{L \rightarrow \infty} \tilde{F}(\mathbf{u}^{t_L}, \mathbf{c}^{t_L}) \leq \lim_{L \rightarrow \infty} \tilde{F}(\mathbf{u}^{t_L}, \mathbf{c}) = \tilde{F}(\mathbf{u}^*, \mathbf{c}) \quad \forall \; \mathbf{c} \in \mathbb{R}^N
\end{align*}
or, equivalently $\tilde{F}(\mathbf{u}^*, \mathbf{c}^*) = \displaystyle \inf_{\mathbf{c} \in \mathbb{R}^N} \tilde{F}(\mathbf{u}^*, \mathbf{c})$. 
	The minimization with respect to $\mathbf c$ can be expressed elementwise for each $c_{\ell}$, leading to the optimality condition of \eqref{eq:c_opt}.
	
	For the rest of the proof, we establish \eqref{eq:crit_u}. For each $k = 1, \ldots, M$, the optimality condition at the $(t_L +1)$th step of \eqref{eq:DC_linear} is 
\begin{gather}
	\begin{aligned}\label{eq:opt329}
	\mathbf{0} \in \partial & (\|Du_k^{t_L+1}\|_1 + \chi_{U} (u_k^{t_L+1})) + \lambda r_k(\mathbf{c}^{t_L}, \mathbf{u}_{<k}^{t_L+1}, \mathbf{u}_{>k}^{t_L}) +2c(u_k^{t_L+1} -u_k^{t_L})\\&- \alpha D^{\top} q_k^{t_L}.
	\end{aligned}
\end{gather}
Denote $s_k^L \coloneqq - \lambda r_k(\mathbf{c}^{t_L}, \mathbf{u}_{<k}^{t_L+1}, \mathbf{u}_k^{t_L}) -2c(u_k^{t_L+1} -u_k^{t_L})+ \alpha D^{\top} q_k^{t_L}.$ Then \eqref{eq:opt329} implies that
	\begin{align} \label{eq:opt_eq3}
s_k^L\in \partial (\|Du_k^{t_L+1}\|_1 + \chi_{U} (u_k^{t_L+1})).
	\end{align}
Since $r_k(\mathbf{c}, \mathbf{u}_{<k}, \mathbf{u} _{>k})$ is continuous in $(\mathbf{c},\mathbf{u}_{<k}, \mathbf{u}_{>k})$, we have \[ \lim_{L \rightarrow \infty} r_k(\mathbf{c}^{t_L}, \mathbf{u}_{<k}^{t_L+1}, \mathbf{u}_{>k}^{t_L}) = r_k(\mathbf{c}^*, \mathbf{u}^*_{<k}, \mathbf{u}^*_{>k}).\]

To compute the limit of $D^{\top} q_k^{t_L},$ we recall the multivariate subgradient of\\
$
\partial \|Du_k\|_{2,1} = \prod_{(i,j)} \partial \|(Du_k)_{i,j}\|_2,
$
where 
\begin{align*}
\partial \|(x_1, x_2)\|_2 = \begin{cases}
\left\{\frac{(x_1,x_2)}{\sqrt{x_1^2+x_2^2}}\right\} &\text{ if } (x_1,x_2) \neq (0,0) \in \mathbb{R}^2, \\
\{(y_1,y_2) \in \mathbb{R}^2: y_1^2+y_2^2 \leq 1\} & \text{ if } (x_1, x_2) = (0,0).
\end{cases}
\end{align*}
Let  $((v^*_{x,k})_{i,j}, (v^*_{y,k})_{i,j}) \coloneqq ((D_x u^*_k)_{i,j},(D_y u^*_k)_{i,j})$ be the discrete gradient of $u^*_k$ at entry $(i,j)$ for $k=1, \ldots, M,$ which satisfies
\begin{align*}
&\partial \|(v_{x,k}^*)_{i,j}, (v_{y,k}^*)_{i,j}\|_2=\\ &\qquad \qquad \begin{cases} \left\{\frac{((v_{x,k}^*)_{i,j},(v_{y,k}^*)_{i,j} )}{\sqrt{|(v_{x,k}^*)_{i,j}|^2 + |(v_{y,k}^*)_{i,j}|^2}}\right\} &\text{ if }  ((v_{x,k}^*)_{i,j},(v_{y,k}^*)_{i,j} ) \neq (0,0), \\
\{(y_1,y_2) \in \mathbb{R}^2: y_1^2+y_2^2 \leq 1\} & \text{ if } ((v_{x,k}^*)_{i,j},(v_{y,k}^*)_{i,j} ) = (0,0).
\end{cases}
\end{align*}
Note that we define $q_k^{t_L}$ in the following way
\begin{align}\label{eq:q_formula}
(q_k^{t_L})_{i,j} = \begin{cases} \frac{((D_x u_k^{t_L})_{i,j},(D_y u_k^{t_L})_{i,j} )}{\sqrt{|(D_x u_k^{t_L})_{i,j}|^2 + |(D_y u_k^{t_L})_{i,j}|^2}} &\text{ if }  ((D_x u_k^{t_L})_{i,j},(D_y u_k^{t_L})_{i,j} ) \neq (0,0), \\
(0,0) &\text{ if }  ((D_x u_k^{t_L})_{i,j},(D_y u_k^{t_L})_{i,j} ) = (0,0).
\end{cases}
\end{align}
Denote $q_k^* \coloneqq \displaystyle \lim_{L \rightarrow \infty} q_k^{t_L}$. 
Therefore, by \eqref{eq:q_formula},  when $((v_x^*)_{i,j}, (v_y^*)_{i,j}) \neq (0,0)$, we have
\begin{align*}
(q_k^*)_{i,j} = \lim_{L \rightarrow \infty} (q_k^{t_L})_{i,j} =  \frac{((v_{x,k}^*)_{i,j},(v_{y,k}^*)_{i,j} )}{\sqrt{|(v_{x,k}^*)_{i,j}|^2 + |(v_{y,k}^*)_{i,j}|^2}} \in \partial \|((v_{x,k}^*)_{i,j}, (v_{y,k}^*)_{i,j})\|_2,
\end{align*}
and when $((v_x^*)_{i,j}, (v_y^*)_{i,j}) = (0,0)$, we have
\begin{align*}
(q_k^{t_L})_{i,j}  \in \{(y_1,y_2) \in \mathbb{R}^2: y_1^2+y_2^2 \leq 1\} \subseteq \partial \|((v_{x,k}^*)_{i,j}, (v_{y,k}^*)_{i,j})\|_2
\end{align*}
for all $L \in \mathbb{N}$ so that taking the limit $L \rightarrow \infty$ yields  $(q_k^*)_{i,j} \in \partial \|((v_{x,k}^*)_{i,j}, (v_{y,k}^*)_{i,j})\|_2$. By the chain rule of the subgradient (Corollary 16 in \cite{hantoute2008subdifferential}), we have \[\partial\|(Du_k^*)_{i,j}\|_2 = D^{\top} \partial \|((v_{x,k}^*)_{i,j}, (v_{y,k}^*)_{i,j})\|_2.\]
Since $D^\top$ is a linear operator (thus continuous), we get 
\begin{align}
\lim_{L\rightarrow \infty} D^{\top} q_k^{t_L} = D^{\top} q_k^* \in  \partial \|Du_k^*\|_{2,1}.
\end{align}
In short, we obtain that 
$s_k^* \coloneqq \displaystyle
\lim_{L \rightarrow \infty} s_k^L =  -\lambda r_k(\mathbf{c}^*, \mathbf{u}_{<k}^*, \mathbf{u}_{>k}^*)+ \alpha D^{\top}q_k^*.$

It further follows from \eqref{eq:opt_eq3} and the subgradient definition that \begin{gather}
\begin{aligned}
\|Du_k\|_1 + \chi_{U}(u_k) &\geq \|Du_k^{t_L+1}\|_1 + \chi_{U}(u_k^{t_L+1}) + \langle  s_k^L, u_k - u_k^{t_L+1} \rangle\\ &= \|Du_k^{t_L+1}\|_1 + \langle  s_k^L, u_k - u_k^{t_L+1} \rangle
\end{aligned}
\end{gather}
for all $u_k \in X$ and $L \in \mathbb{N}$. By continuity, we obtain
\begin{align*}
\|Du_k\|_1 + \chi_{U}(u_k) &\geq \lim_{L \rightarrow \infty} \left(\|Du_k^{t_L+1}\|_1 + \langle  s_k^L, u_k - u_k^{t_L+1} \rangle \right)\\ &= \|Du_k^*\|_1 + \langle s_k^*, u_k - u_k^* \rangle = \|Du_k^*\|_1 + \chi_{U}(u_k^*) + \langle s_k^*, u_k - u_k^* \rangle,
\end{align*}
 where the last equality is due to $\chi_{U}(u_k^*) = 0$. Since both $\|Du\|_1$ and $\chi_{U}(u)$ are convex,  $s_k^* \in \partial(\|Du_k^*\|_1 + \chi_{U}(u_k^*)) =  \partial\|Du_k^*\|_1 + \partial \chi_{U}(u_k^*)$.
 Therefore, we have
\begin{align*}
\mathbf{0} &\in \partial\|Du_k^*\|_1  + \partial\chi_{U}(u_k^*)+ \lambda r_k(\mathbf{c}^*, \mathbf{u}_{<k}^{*}, \mathbf{u}_{>k}^*)- \alpha D^{\top} q_k^* \\
&\subseteq \partial\|Du_k^*\|_1 - \alpha \partial \|Du_k^*\|_{2,1} + \partial\chi_{U}(u_k^*)+ \lambda r_k(\mathbf{c}^*, \mathbf{u}_{<k}^{*}, \mathbf{u}_{>k}^*).
\end{align*}
This concludes the proof.
\end{proof}
\begin{remark}
	The limit point $(\mathbf{u}^*, \mathbf{c}^*)$ is not guaranteed to be a global optimal solution for \eqref{eq:relax_AITV_MCV} because the objective function is nonconvex, and $(\mathbf{u}^*, \mathbf{c}^*)$ may not even satisfy a first-order optimality condition $\mathbf{0} \in \partial_{(\mathbf{u}, \mathbf{c})} \tilde{F}(\mathbf{u}^*, \mathbf{c}^*)$. However, according to Theorem \ref{thm:opt} (c),  each coordinate $u_k^*$ or $c_{\ell}^*$ satisfies its respective first-order optimality condition, since $(\mathbf{u}^*, \mathbf{c}^*) = (u_1^*, \ldots, u_M^*, c_1^*, \ldots, c_N^*)$. In convex optimization, if $g$ is convex, a point $x^*$ is a critical point if $0 \in \partial{g}(x^*)$. \eqref{eq:c_opt} establishes $c_{\ell}^*$ to be a critical point of the function convex in $c_{\ell}$,
	\begin{align*}
	  \sum_{i=1}^m \sum_{j=1}^n (f_{i,j} - c_{\ell})^2 R_{\ell}(\mathbf{u})_{i,j},
	\end{align*}
	which is derived from \eqref{eq:relax_AITV_MCV} when minimizing for $c_{\ell}$.
	In DC optimization, a point $x^*$  is a critical point of DC function $g-h$ if $0 \in \partial g(u^*)-\partial h(u^*)$ \cite{le2018dc}. However, this optimality condition is not as strong as the optimality condition $0 \in \partial(g-h)(u^*)$ because  $\partial(g-h)(u^*) \subset \partial g(u^*)-\partial h(u^*)$ in terms of either the Clarke subdifferential or the Fréchet subdifferential \cite{le2018dc}. \eqref{eq:crit_u} establishes $u_k^*$ to be a DC critical point of the DC function
	\begin{align*}
	   \underbrace{\|Du_k\|_1  + \chi_{U}(u_k) + \lambda \langle r_k (\mathbf{c}, \mathbf{u}_{<k}, \mathbf{u}_{>k}), u_k \rangle_X}_{g(u_k)} -\underbrace{\alpha \|Du_k\|_{2,1}}_{h(u_k)},
	\end{align*}
	which is derived from \eqref{eq:relax_AITV_MCV} when minimizing for $u_k$.
\end{remark}
\section{Fuzzy Extension of the AICV Model}\label{sec:AIFR}
One limitation of the CV models is that they are only applicable for image segmentation that has specifically power-of-two number (i.e., $2^M$) of  regions.  To generalize to an arbitrary number of regions $N$, we associate each region $\Omega_{\ell}$  with a membership function $u_{\ell}$ for $\ell=1, \ldots, N$.  A membership function $u_{\ell}$ represents a region $\Omega_{\ell}$ in the following way:
\begin{align*}
    (u_{\ell})_{i,j} = \begin{cases}
    1 &\text{ if } (i,j) \in \Omega_{\ell}, \\
    0 & \text{ if } (i,j) \not \in \Omega_{\ell}.
    \end{cases}
\end{align*}
To avoid overlap between $u_{\ell}$'s, we enforce the constraint $\sum_{\ell=1}^N u_{\ell} = \mathbbm{1}$, but we relax it with a quadratic penalty to make the model numerically tractable. As such, we propose an AITV extension to the FR model, referred to as AIFR,
\begin{gather}
\begin{aligned}\label{eq:fuzzy_AITV}
	\min_{\subalign{\mathbf{u} &\in X^N\\ \mathbf{c} &\in \mathbb{R}^N}} \hat{F}(\mathbf{u}, \mathbf{c}) \coloneqq  &\sum_{\ell=1}^N \left(\|Du_{\ell}\|_1 - \alpha \|Du_{\ell}\|_{2,1} + \chi_{U}(u_{\ell})\right) + \lambda \sum_{\ell=1}^N \langle  f_{\ell}(\mathbf{c}),u_{\ell}\rangle_X\\  &+ \frac{\nu}{2} \left\|\sum_{\ell=1}^N u_{\ell}  - \mathbbm{1}\right\|_X^2
\end{aligned}
\end{gather}
with $\nu > 0$. 
Similarly to \eqref{eq:minimize_u}-\eqref{eq:minimize_c}, we adopt the alternating minimization framework to solve \eqref{eq:fuzzy_AITV}, i.e.,
	\begin{align} \label{eq:minimize_u2}
\mathbf{u}^{t+1} & \in \argmin_{\mathbf{u}} \hat{F}(\mathbf{u}, \mathbf{c}^t),  \\ \label{eq:minimize_c2} 
\mathbf{c}^{t+1} &\in \argmin_{\mathbf{c}} \hat{F}(\mathbf{u}^{t+1}, \mathbf{c}).
\end{align}
The $\mathbf c$-subproblem \eqref{eq:minimize_c2} has a closed-form solution for $\ell=1, \ldots, N$,
\begin{align}
\label{eq:c_update2}
c^{t+1}_{\ell} =\begin{cases} \frac{\displaystyle\sum_{i=1}^m \sum_{j=1}^n f_{i,j}(u_{\ell}^{t+1})_{i,j}}{\displaystyle\sum_{i=1}^m \sum_{j=1}^n (u_{\ell}^{t+1})_{i,j}} &\text{ if } \displaystyle\sum_{i=1}^m \sum_{j=1}^n (u_{\ell}^{t+1})_{i,j} \neq 0, \\
0 &\text{ if } \displaystyle\sum_{i=1}^m \sum_{j=1}^n (u_{\ell}^{t+1})_{i,j} = 0.
\end{cases}
\end{align}
 For \eqref{eq:minimize_u2}, we can find $u_{\ell}^{t+1}$ coordinatewise with respect to $\ell$ by solving
\begin{gather}
\begin{aligned} \label{eq:fuzzy_u}
u_{\ell}^{t+1} \in \argmin_{u_{\ell}} &\|Du_{\ell}\|_1 - \alpha \|Du_{\ell}\|_{2,1}  + \chi_{U}(u_{\ell}) + \lambda  \langle f_{\ell}(\mathbf{c}), u_{\ell} \rangle_X\\ &+ \frac{\nu}{2} \left \|\sum_{j < \ell } u_j^{t+1} + u_{\ell} + \sum_{j > \ell} u_{\ell}^t  - \mathbbm{1} \right \|_X^2.
\end{aligned}
\end{gather}
Applying  DCA \eqref{eq:DCA} to solve for \eqref{eq:fuzzy_u} gives
\begin{gather}
\begin{aligned}\label{eq:DC_linear2}
u_{\ell}^{t+1} = \argmin_{u_{\ell}}  &\|Du_{\ell}\|_1 + \chi_{U}(u_{\ell}) + \lambda \langle f_{\ell}(\mathbf{c}), u_{\ell} \rangle_X\\ & + \frac{\nu}{2} \left \|\sum_{j < \ell } u_j^{t+1} + u_{\ell} + \sum_{j > \ell} u_{\ell}^t - \mathbbm{1} \right \|_X^2+ c\|u_{\ell}\|_X^2\\ &-\alpha \langle Du_{\ell}, q_{\ell}^t \rangle_Y - 2c \langle u_{\ell}, u_{\ell}^t \rangle_X,
\end{aligned}
\end{gather}
 where $q_{\ell}^t \coloneqq ((q_x)_{\ell}^t, (q_y)_{\ell}^t) = (D_x u_{\ell}^t, D_y u_{\ell}^t)/\sqrt{|D_xu_{\ell}^t|^2 + |D_yu_{\ell}^t|^2}$ if the denominator is not zero. Similarly to \eqref{eq:DC_linear},  we  apply PDHGLS to find $u_{\ell}^{t+1}$ in \eqref{eq:DC_linear2} with the following iteration:
 \begin{align} \label{eq:fuzzy_pdhg_u}
 u^{t,\eta+1}_{\ell} &= \min\left\{\max\left\{\frac{2c u_{\ell}^t + \frac{1}{\tau} u_{\ell}^{t,\eta} + \nu \left(\mathbbm{1} -\sum_{j < \ell } u_j^{t+1} - \sum_{j > \ell} u_{\ell}^t\right)}{2c + \frac{1}{\tau} + \nu } \right. \right. \\
 &\nonumber \qquad\qquad \qquad \left. \left.- \frac{\lambda f_{\ell}(\mathbf{c})- \alpha D^{\top} q_{\ell}^t + (D_x^{\top} (p_x)_{\ell}^{\eta} + D_y^{\top} (p_y)_{\ell}^{\eta})}{2c+ \frac{1}{\tau} + \nu}, 0 \right\}, 1\right\},\\
 \bar{u}_{\ell}^{t,\eta+1} &= u_{\ell}^{t,\eta+1}+\theta(u_{\ell}^{t,\eta+1}-u_{\ell}^{t,\eta}) \label{eq:fuzzy_px_pdhg},\\
	(p_x)_{\ell}^{\eta+1} &=\text{Proj}_P ((p_x)_{\ell}^{\eta} + \sigma D_x \bar{u}_{\ell}^{t,\eta+1}), \\ \label{eq:fuzzy_py_pdhg}
(p_y)_{\ell}^{\eta+1} &=\text{Proj}_P ((p_y)_{\ell}^{\eta} + \sigma D_y \bar{u}^{t,\eta+1}_{\ell})
\end{align}
for $u^{t,0}_{\ell} = u^t_{\ell}$ and $\tau, \sigma >0$, $\theta \in [0,1]$. 
The proposed algorithm is referred to as DCA-PDHGLS, summarized in Algorithm \ref{alg:DCA_PDHG2}. 
Convergence analysis of the sequence $\{(\mathbf{u}^t, \mathbf{c}^t)\}_{t=1}^{\infty}$ generated by \eqref{eq:DC_linear2} and \eqref{eq:c_update2} can be established similarly to the one in Section \ref{subsec:converge}. Hence, we have the following theorem, but for the sake of brevity, the proof is omitted.

	\begin{algorithm}[t]
	\caption{DCA-PDHGLS algorithm to solve \eqref{eq:fuzzy_AITV}}
	\label{alg:DCA_PDHG2}\smallskip
	\textbf{Input: }\begin{itemize}
		\item Image $f$
		\item model parameters $\alpha, \lambda >0$ 
		\item strong convexity parameter $c > 0$
		\item quadratic penalty parameter $\nu > 0$
		\item PDHGLS initial step size $\tau_0 > 0$
		\item PDHGLS primal-dual step size ratio $\beta > 0$
		\item PDHGLS parameter $\delta \in (0,1)$
		\item PDHGLS step size multiplier $\mu \in (0,1)$
	\end{itemize}
	\begin{algorithmic}[1]
		\STATE  Set $u_{\ell}^0 = 1$ $(\ell=1, \ldots, N)$ for some region $\Sigma \subset \Omega$ and $0$ elsewhere.
		\STATE Compute $\mathbf{c}^0 = (c_1^0, \ldots, c_N^0)$ by \eqref{eq:c_update2}.
		\STATE Set $t \coloneqq 0$.
		\WHILE{stopping criterion for DCA is not satisfied}
		\FOR{$\ell=1$ to $M$}
		\STATE Set $u_{\ell}^{t,0} \coloneqq u_{\ell}^t$ and $(p_x)_{\ell}^0 = (p_y)_{\ell}^0 = 0$.
		\STATE Compute $((q_x)_{\ell}^{t}, (q_y)_{\ell}^{t}) = (D_x u_{\ell}^{t}, D_y u_{\ell}^{t})/\sqrt{|D_xu_{\ell}^{t}|^2 + |D_yu_{\ell}^{t}|^2}$.
		\STATE Set $\theta_0 =1$.
		\STATE Set $\eta \coloneqq  0$.
		\WHILE{stopping criterion for PDHGLS is not satisfied}
		\STATE Compute $u_{\ell}^{t,\eta+1}$ by \eqref{eq:fuzzy_pdhg_u} with $\tau \coloneqq \tau_{\eta}$.
		\STATE Set $\tau_{\eta+1} = \tau_{\eta} \sqrt{1+\theta_{\eta}}$.\\
		\textbf{Linesearch:}
		\STATE Compute $\theta_{\eta+1} = \frac{\tau_{\eta+1}}{\tau_{\eta}}$ and $\sigma_{\eta+1} = \beta \tau_{\eta+1}$.
		\STATE Compute $\bar{u}_{\ell}^{t,\eta+1} = u_{\ell}^{t,\eta+1}+\theta_{\eta+1}(u_{\ell}^{t,\eta+1}-u_{\ell}^{t,\eta})$.
		\STATE Compute $p_{\ell}^{\eta+1} \coloneqq((p_x)_{\ell}^{\eta+1}, (p_y)_{\ell}^{\eta+1})$ by \eqref{eq:fuzzy_px_pdhg}-\eqref{eq:fuzzy_py_pdhg} with $\sigma \coloneqq \sigma_{\eta+1}$.
		\IF{$\sqrt{\beta} \tau_{\eta+1} \|(D_x^{\top} (p_x)_{\ell}^{\eta+1},  D_y^{\top} (p_y)_{\ell}^{\eta+1}) -(D_x^{\top}(p_x)_{\ell}^{\eta}, D_y^{\top}(p_y)_{\ell}^{\eta}) \|_Y \leq  \delta \|p_{\ell}^{\eta+1}-p_{\ell}^{\eta}\|_Y$}
		\STATE Set $\eta \coloneqq \eta+1$, and break linesearch
		\ELSE 
		\STATE Set $\tau_{\eta+1} \coloneqq \mu \tau_{\eta+1}$ and go back to line 13.
		\ENDIF\\
		\textbf{End of linesearch}
		\ENDWHILE
		\STATE Set $u_{\ell}^{t+1} \coloneqq u_{\ell}^{t,\eta}$.
		\ENDFOR
		\STATE Compute $\mathbf{c}^{t+1}$ by \eqref{eq:c_update2}.
		\STATE Set $t \coloneqq t+1.$
		\ENDWHILE
	\end{algorithmic}
	\textbf{Output:} $(\mathbf{u}, \mathbf{c}) \coloneqq (\mathbf{u}^t, \mathbf{c}^t)$.
\end{algorithm}

\begin{theorem}\label{thm:convFR}
	Suppose $\alpha \in [0,1]$ and $\lambda >0$.  Let $\{(\mathbf{u}^t, \mathbf{c}^t)\}_{t=1}^{\infty}$ be a sequence such that $\mathbf{u}^t$ is generated by \eqref{eq:DC_linear2} and $\mathbf{c}^t$ is generated by  \eqref{eq:c_update2}. We have the following:
	\begin{enumerate}[label=(\alph*)]
		\item $\{(\mathbf{u}^t, \mathbf{c}^t)\}_{t=1}^{\infty}$ is bounded.
		\item For $\ell=1, \ldots, N$, we have $\|u_{\ell}^{t+1}-u_{\ell}^t\|_X \rightarrow 0$ as $t \rightarrow \infty$.
		\item The sequence $\{(\mathbf{u}^t, \mathbf{c}^t)\}_{t=1}^{\infty}$ has a limit point $(\mathbf{u}^*, \mathbf{c}^*)$ satisfying
\begin{align}
		&\mathbf{0} \in \partial\|Du_{\ell}^*\|_1 - \alpha \partial \|Du_{\ell}^*\|_{2,1} + \partial\chi_{U}(u_{\ell}^*)+ \lambda f_{\ell}(\mathbf{c}^*) + \nu \left(\sum_{j=1}^N u_j^* - \mathbbm{1} \right),\\
		&0 \in \frac{\partial\hat{F}(\mathbf{u}^*, \mathbf{c}^*)}{\partial c_{\ell}} \quad \forall\; \ell = 1, \ldots, N.
		\end{align}
	\end{enumerate} 
\end{theorem}

\section{Extension to Color Images}\label{sec:color}
Both AICV  \eqref{eq:relax_AITV_MCV} and AIFR \eqref{eq:fuzzy_AITV} models can be extended to color image segmentation. Let $\mathbf{f} = (f_r, f_g, f_b) \in X^3$ be a color image and $(c_{\ell,r}, c_{\ell,g}, c_{\ell,b}) \in \mathbb{R}^3$ for $\ell = 1, \ldots, N$. By replacing $f_{\ell}(\mathbf{c})$ with 
\begin{align*}
\mathbf{f}_{\ell} (\mathbf{c}_r, \mathbf{c}_g, \mathbf{c}_b) = \sum_{\iota \in \{r,g,b\}} (f_{\iota} - c_{\ell,\iota} \mathbbm{1})^2,
\end{align*}
where $\mathbf{c}_{\iota} = (c_{1, \iota}, \ldots, c_{N,\iota} )$ for $\iota \in \{r,g,b\}$, the AICV model for color segmentation is
\begin{align}
\label{eq:color_relax_AITV_MCV}
\min_{\subalign{\mathbf{u} &\in X^M\\ \mathbf{c}_r, \mathbf{c}_g, &\mathbf{c}_b \in \mathbb{R}^N}} \sum_{k=1}^M \left(\|Du_k\|_1 - \alpha \|Du_k \|_{2,1} + \chi_{U}(u_k) \right) + \lambda \sum_{\ell=1}^N \langle  \mathbf{f}_{\ell} (\mathbf{c}_r, \mathbf{c}_g, \mathbf{c}_b) , R_{\ell}(\mathbf{u}) \rangle_X.
\end{align}
Similarly, the AIFR model for color segmentation is
\begin{gather}
\begin{aligned}\label{eq:color fuzzy_AITV}
\min_{\subalign{\mathbf{u} &\in X^N\\ \mathbf{c}_r, \mathbf{c}_g, &\mathbf{c}_b \in \mathbb{R}^N}} &\sum_{\ell=1}^N \left(\|Du_{\ell}\|_1 - \alpha \|Du_{\ell}\|_{2,1} + \chi_{U}(u_{\ell})\right) + \lambda \sum_{\ell=1}^N \langle  \mathbf{f}_{\ell} (\mathbf{c}_r, \mathbf{c}_g, \mathbf{c}_b),u_{\ell}\rangle_X\\  &+ \frac{\nu}{2} \left\|\sum_{\ell=1}^N u_{\ell}  - \mathbbm{1}\right\|_X^2.
\end{aligned}
\end{gather}
For \eqref{eq:color_relax_AITV_MCV} and \eqref{eq:color fuzzy_AITV}, their respective update formulas for $\mathbf{c}_{\iota}$ with $\iota \in \{r,g,b\}$  are
\begin{align}\label{eq:color_c_update2} 
c_{\ell,\iota} &= \begin{cases}\frac{\displaystyle\sum_{i=1}^m \sum_{j=1}^n (f_{\iota})_{i,j} R_{\ell}(\mathbf{u})_{i,j}}{\displaystyle\sum_{i=1}^m \sum_{j=1}^n R_{\ell}(\mathbf{u})_{i,j}} &\text{ if } \displaystyle\sum_{i=1}^m \sum_{j=1}^n R_{\ell}(\mathbf{u})_{i,j} \neq 0, \\
0 & \text{ if  } \displaystyle\sum_{i=1}^m \sum_{j=1}^n R_{\ell}(\mathbf{u})_{i,j} = 0\end{cases}
\end{align}
and
\begin{align}
\label{eq:color_c_update3}
c_{\ell, \iota} &=
\begin{cases}\frac{\displaystyle\sum_{i=1}^m \sum_{j=1}^n (f_{\iota})_{i,j}(u_{\ell})_{i,j}}{\displaystyle\sum_{i=1}^m \sum_{j=1}^n (u_{\ell})_{i,j}} &\text{ if } \displaystyle \sum_{i=1}^m \sum_{j=1}^n (u_{\ell})_{i,j} \neq 0, \\
0 &\text{ if } \displaystyle \sum_{i=1}^m \sum_{j=1}^n (u_{\ell})_{i,j} =0.
\end{cases}
\end{align}
The update formulas for $\mathbf{u}$ are similar to their grayscale counterparts since only $f_{\ell}$ needs to be replaced with $\mathbf{f}_{\ell}$. Hence, their algorithms are straightforward to derive, thus omitted.

\section{Numerical Results}\label{sec:result}
In this section, we present extensive experiments on various synthetic and real images to demonstrate the effectiveness of  AITV in image segmentation. In particular, we compare the AICV and AIFR models for $\alpha \in \{ 0, 0.25, 0.5, 0.75, 1.0\}$ with the two-stage segmentation methods that use $L_1+L_2^2$~\cite{cai2017three,cai2013two}, $L_0$ \cite{storath2014fast, xu2011image}, and real-time Mumford-Shah ($R_{MS}$) \cite{strekalovskiy2014real} penalties. When $\alpha =0$, the AICV model reduces to the original CV ($L_1$ CV) model~\cite{chan2000active,chan-vese-2001}, while the AIFR model becomes the fuzzy region competition ($L_1$ FR) model \cite{li2010multiphase}. The two-stage segmentation methods find a smooth approximation $\bar{f}$ of the underlying image $f$ with certain regularization, followed by $k$-means clustering on $\bar{f}$ to obtain the segmentation result. Specifically, Cai et al.~\cite{cai2017three,cai2013two} proposed an $L_1+L_2^2$ regularization problem\footnote{Code is available at \url{https://xiaohaocai.netlify.app/download/}.}
\begin{align}\label{eq:L1+L2^2}
\min_{u} \lambda \|f - u \|_X^2 + \gamma \|Du\|_Y^2 +  \|Du\|_{2,1}.
\end{align}
Throughout our numerical experiments, we set $\gamma = 1$, which is suggested in \cite{cai2017three,cai2013two}. The $L_0$-regularized model \cite{storath2014fast, xu2011image} is given by
\begin{align}\label{eq:L0}
\min_{u} \lambda \|f-u\|_X^2 + \|D_xu\|_0 + \|D_yu\|_0,
\end{align}
where $\| \cdot\|_0$ counts the number of nonzero entries of the matrix. The model in \eqref{eq:L0} can be solved in two different ways. One is   by alternating minimization with half-quadratic splitting \cite{xu2011image}\footnote{Code is available at \url{http://www.cse.cuhk.edu.hk/~leojia/projects/L0smoothing/}.}. Another approach \cite{storath2014fast} incorporates weights for a better isotropic discretizatation than the original $L_0$ model,
followed by ADMM\footnote{Code is available at \url{https://github.com/mstorath/Pottslab}.}.  The  $R_{MS}$ model \cite{strekalovskiy2014real} replaces  the $L_0$ norm in \eqref{eq:L0} by $R_{MS}(u) = \sum_{i=1}^m \sum_{j=1}^n \min \{\gamma u_{i,j}, 1\}$,
thus leading to 
\begin{align}\label{eq:RMS}
  \min_{u} \lambda \|f-u\|_X^2 + R_{MS}(D_xu) + R_{MS}(D_yu).
\end{align}
In our numerical experiments, we consider the piecewise-constant limit case, where $\gamma \rightarrow \infty$. Its implementation is described in \cite[Algorithm 1]{strekalovskiy2014real}.
We refer to the models \eqref{eq:L1+L2^2}, \eqref{eq:L0}, and \eqref{eq:RMS} as $L_1+L_2^2$, $L_0$, and $R_{MS}$, respectively.

For the proposed Algorithms~\ref{alg:DCA_PDHG} and \ref{alg:DCA_PDHG2}, we set $c= 10^{-8}$, $\tau_0 = 1/8$, $\beta = 1.0$, $\delta = 0.9999$, and $\mu = 7.5 \times 10^{-5}$, as suggested in \cite{lou-2015,malitsky2018first}. The parameter $\lambda$  depends on the image, which will be specified for each testing case.
For the stopping criteria, we use the relative error 
\begin{align}
\text{relerr}(u,v) = 
\frac{\|u- v\|_X}{\max \{ \|u\|_X, \|v\|_X, \epsilon\}},
\end{align}
where $\epsilon$ is the machine's precision.
Following \cite{lou-2015}, we choose the  stopping criterion for the inner PDHGLS algorithm as $\text{relerr}(u^{t,\eta+1}, u^{t,\eta})<10^{-6}$. As for the outer iterations, DCA minimization terminates when $\text{relerr}(u^{t+1}, u^{t})<10^{-6}$ and $\text{relerr}(u^{t+1}, u^{t})<10^{-4}$ for 2-phase and 4-phase AICV models, respectively.
For the AIFR models, we use the same stopping criterion in \cite{li2016multiphase} for the outer iterations, i.e.,  
 when all the relative errors of the membership functions are less than $10^{-4}$. We further adjust the maximum number of outer/inner iterations for  multiple channels and multiphase segmentation, which are selected empirically for each image.

We shall apply postprocessing to define the segmented regions. In particular, we convert the results of Algorithm \ref{alg:DCA_PDHG} to a binary output by setting any pixel values greater than or equal to 0.5 to 1, and  0 otherwise.  For the results from Algorithm \ref{alg:DCA_PDHG2}, we set a pixel value $(u_{\ell})_{i,j}$   to 1 if it is the maximum among all the membership functions $\{u_k\}_{k=1}^N$ at pixel $(i,j)$, and 0 otherwise. For a grayscale image $f$, we define its reconstructed image
\begin{align} \label{eq:reconstruct_equation}
\tilde{f} = \sum_{k=1}^{N} c_k \mathbbm{1}_{\tilde{\Omega}_k},
\end{align}
where $\{c_k\}_{k=1}^N$ and $\{\tilde{\Omega}_k\}_{k=1}^N$ are sets of constants and regions obtained by a segmentation algorithm, respectively, and $\mathbbm{1}_{\tilde{\Omega}_k}$ is a binary image corresponding to the region $\tilde{\Omega}_k$. The matrix $\mathbbm{1}_{\tilde{\Omega}_k}$ is obtained by thresholding for Algorithms \eqref{alg:DCA_PDHG} and \eqref{alg:DCA_PDHG2} or by $k$-means clustering for the two-stage segmentation framework. 
Specifically for Algorithms \ref{alg:DCA_PDHG} and \ref{alg:DCA_PDHG2}, the constants $\{c_k\}_{k=1}^N$  are the final outputs of \eqref{eq:c_update} and \eqref{eq:c_update2}, respectively. For the two-stage segmentation framework, we compute a smoothed image of $f$ by one of the models \eqref{eq:L1+L2^2}-\eqref{eq:RMS}, thus getting $\bar{f}$, and define the constants in \eqref{eq:reconstruct_equation} by
\begin{align}\label{eq:two_stage_grayscale_c}
    c_k = \frac{\displaystyle \sum_{i=1}^m \sum_{j=1}^n \bar{f}_{i,j}(\mathbbm{1}_{\tilde{\Omega}_k})_{i,j} }{\displaystyle \sum_{i=1}^m \sum_{j=1}^n (\mathbbm{1}_{\tilde{\Omega}_k})_{i,j}}, \; k = 1, \ldots, N.
\end{align}
As $k$-means clustering applied to $\bar{f}$ does not produce an empty cluster, the denominator of \eqref{eq:two_stage_grayscale_c} is nonzero. Similarly, the color image $\mathbf{f}$ is approximated by $\tilde{\mathbf{f}} = (\tilde{f}_r, \tilde{f}_g, \tilde{f}_b)$ given by
\begin{align}\label{eq:color_reconstruction_eq}
   \tilde{f}_{\iota} = \sum_{k=1}^N c_{k, \iota}  \mathbbm{1}_{\tilde{\Omega}_k} \text{ for } \iota \in \{r, g, b\},
\end{align}
where $\{c_{k, \iota}\}_{k=1}^N$ is a set of constants for channel $\iota$. For the color versions of Algorithms \ref{alg:DCA_PDHG} and \ref{alg:DCA_PDHG2}, the constants are obtained by \eqref{eq:color_c_update2} and \eqref{eq:color_c_update3}, respectively. 
For the color version of the two-stage segmentation framework, the constants are computed by \eqref{eq:two_stage_grayscale_c} applied to each channel of the smoothed image $\bar{\mathbf{f}} =(\bar{f}_r, \bar{f}_g, \bar{f}_b)$.
\begin{figure}[t!]
	\centering
	\begin{tabular}{c@{}c@{}c@{}}
		\subcaptionbox{\label{fig:synthetic_grayscale}}{\includegraphics[scale = 0.25]{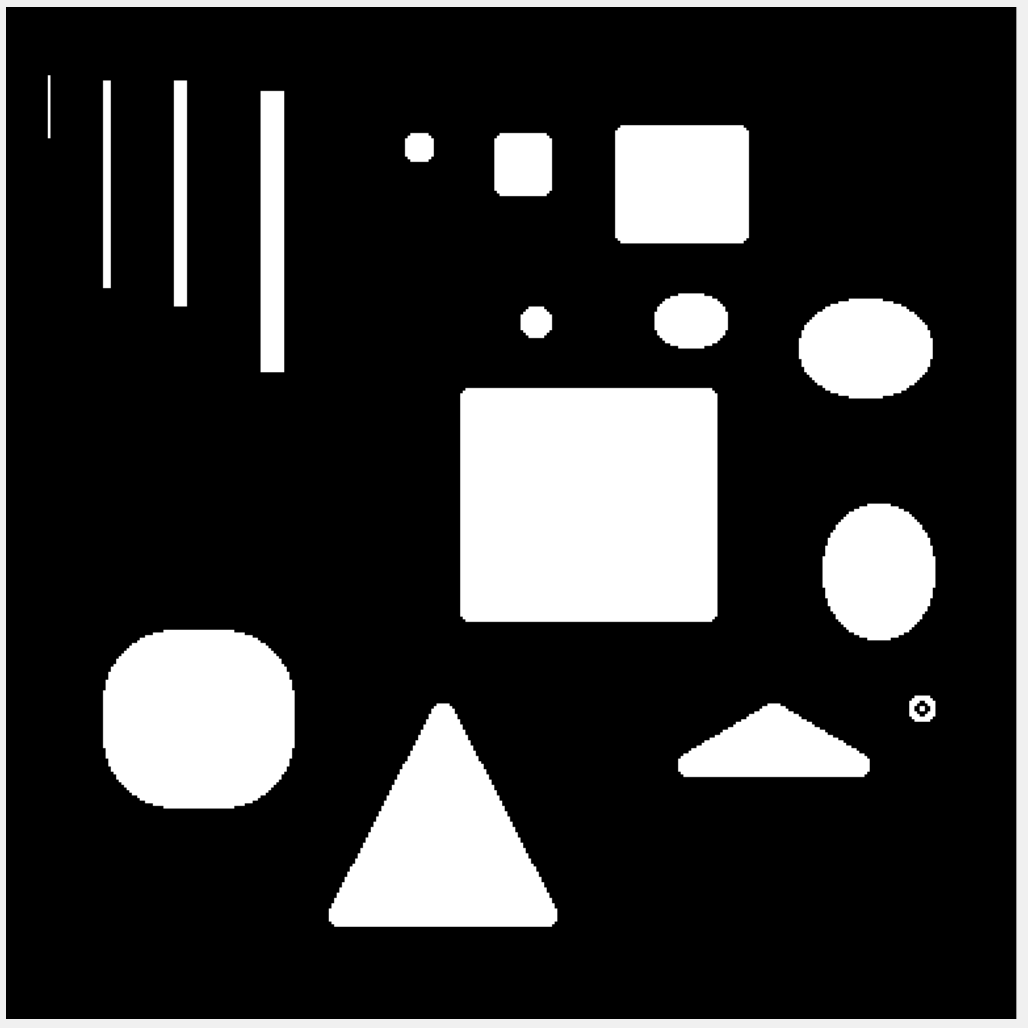}} &
		 \subcaptionbox{\label{fig:synthetic_2phase}}{\includegraphics[scale = 0.25]{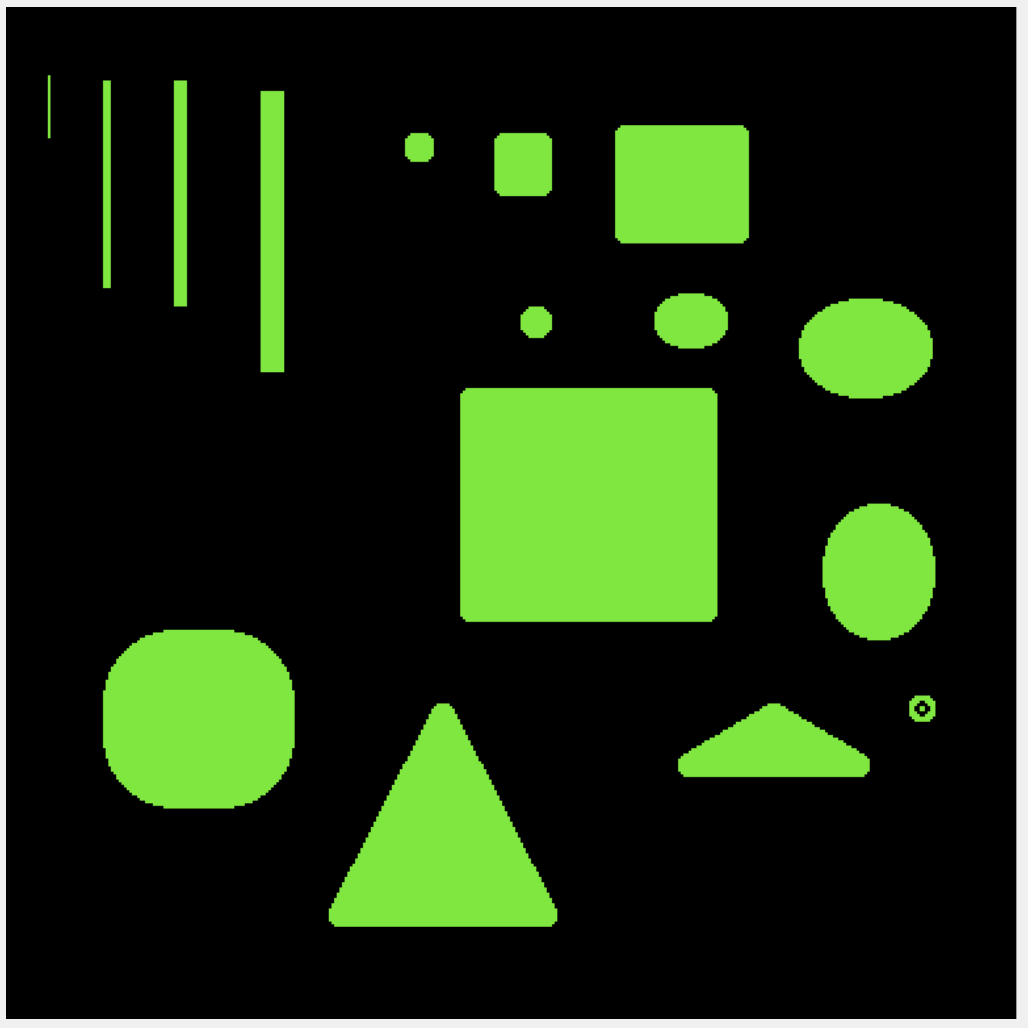}} &  \subcaptionbox{\label{fig:synthetic_4phase}}{\includegraphics[scale = 0.25]{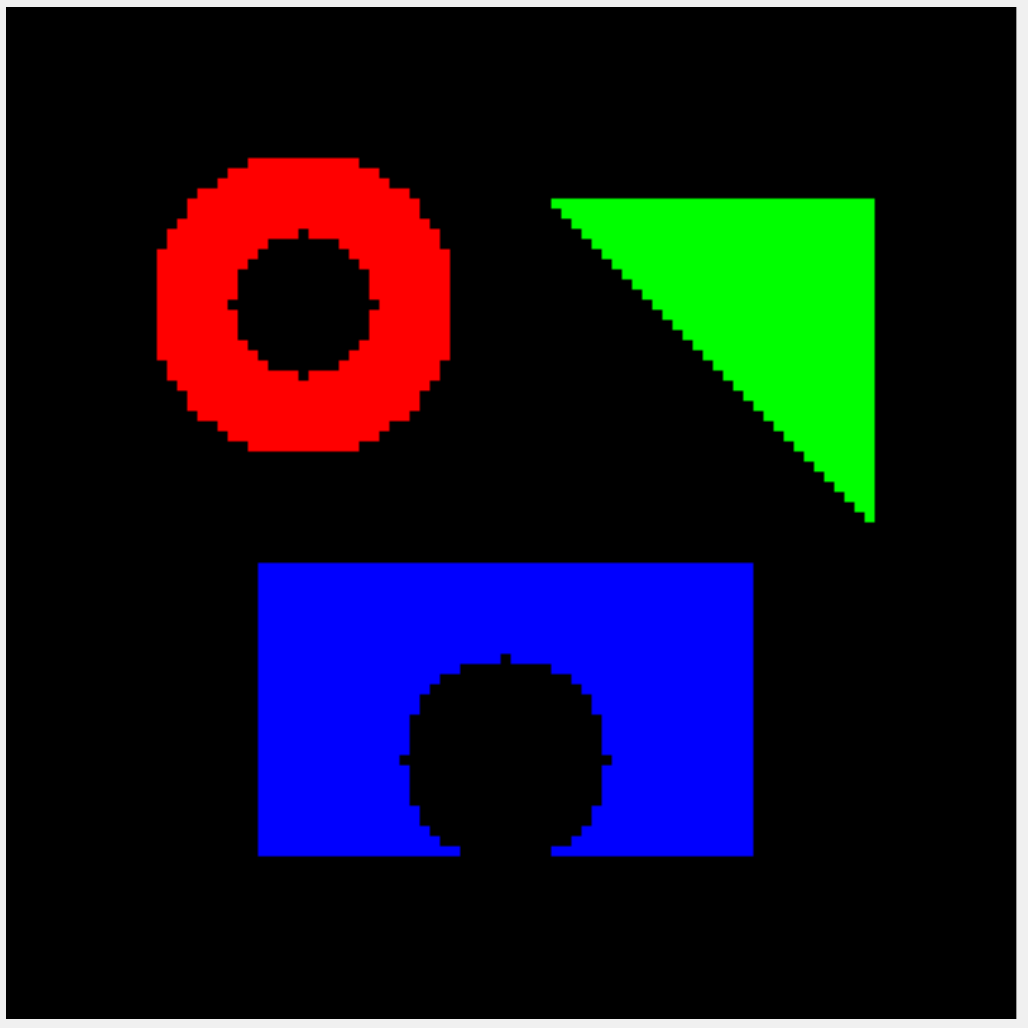}} 
	\end{tabular}
	\caption{Synthetic images for image segmentation. (a) Grayscale image for two-phase segmentation. Size: $385 \times 385$. (b) Color image for two-phase segmentation. Size: $385 \times 385.$ (c) Color image for four-phase segmentation. Size: $100 \times 100$. }
	\label{fig:synthetic}
\end{figure}

\begin{table}[t!]
	\caption{DICE indices of  various segmentation models applied to Figure \ref{fig:synthetic_grayscale} corrupted with different levels of impulsive noise.}
	\begin{center}
		\resizebox{\textwidth}{!}{%
			\begin{tabular}{l|c|c|c|c|c|c|c|c} \Xhline{5\arrayrulewidth}
				\makecell{Salt \&\\ Pepper ($\%$)} & 0& 10 & 20 & 30 & 40 & 50 & 60 & 70 \\ \Xhline{5\arrayrulewidth}
				$L_1-L_2$ CV                 & \textbf{1}      & 0.9977 & 0.9932 & 0.9854 & 0.9594 & 0.9062 & 0.8138 & 0.7643 \\ \hline
				$L_1-0.75L_2$CV             & \textbf{1}      & 0.9978 & 0.9929 & 0.9853 & 0.9795 & 0.9727 & 0.9678 & 0.9550 \\ \hline
				$L_1-0.5L_2$ CV              & \textbf{1}       & 0.9975 & 0.9941 & 0.9893 & 0.9850 & 0.9801 & \textbf{0.9726} & \textbf{0.9554} \\ \hline
				$L_1-0.25L_2$ CV             & \textbf{1}       & 0.9974 & 0.9954 & 0.9910 & 0.9870 & \textbf{0.9823} & 0.9711 & 0.9483 \\ \hline
				$L_1$ CV                    & \textbf{1}       & \textbf{0.9981} & 0.9960 & 0.9922 & 0.9877 & 0.9802 & 0.9681 & 0.9338 \\ \hline
				$L_1-L_2$ FR            &\textbf{1}       & 0.8753 & 0.7719 & 0.6833 & 0.6129 & 0.5425 & 0.4702 & 0.4138 \\ \hline
				$L_1-0.75L_2$ FR        & \textbf{1}      & 0.9896 & 0.9841 & 0.9693 & 0.9585 & 0.9437 & 0.9183 & 0.7775 \\ \hline
				$L_1-0.5L_2$ FR        & 0.9998 & 0.9978 & 0.9956 & 0.9923 & \textbf{0.9879} & 0.9788 & 0.9495 & 0.7760 \\ \hline
				$L_1-0.25L_2$ FR       & 0.9995 & 0.9979 & \textbf{0.9961} & \textbf{0.9925} & 0.9865 & 0.9737 & 0.9347 & 0.6883 \\ \hline
				$L_1$ FR & 0.9992 & 0.9978 & 0.9949 & 0.9877 & 0.9812 & 0.9663 & 0.8990 & 0.5053 \\ \hline
				$L_1+L_2^2$ & 0.9996 & 0.9961 & 0.9925 & 0.9857 & 0.9733 & 0.9328 & 0.8375 & 0.6840 \\ \hline
				$L_0$  \cite{xu2011image}                & \textbf{1}       & 0.8731 & 0.7666 & 0.6736 & 0.5943 & 0.5226 & 0.4601 & 0.4035\\ \hline
					$L_0$  \cite{storath2014fast}                & 0.9995       & 0.9944 & 0.9874 & 0.9792 & 0.9738 & 0.9690 & 0.9605 & 0.9474\\ \hline
					$R_{MS}$ & 0.9995 & 0.9969 & 0.9947 & 0.9887 & 0.9851 & 0.9784 & 0.9670 & 0.9312\\\hline
				\Xhline{5\arrayrulewidth}
				\makecell{Random-\\valued (\%)}  &0 & 10 & 20 & 30 & 40 & 50 & 60 & 70 \\ \Xhline{5\arrayrulewidth}
				$L_1-L_2$ CV                 & \textbf{1}  & 0.9986 & 0.9957 & 0.9909 & 0.9846 & 0.9739 & 0.9534 & 0.9542 \\ \hline
				$L_1-0.75L_2$ CV             &\textbf{1}   & 0.9988 & 0.9971 & 0.9948 & 0.9926 & 0.9894 & \textbf{0.9840} & \textbf{0.9712} \\ \hline
				$L_1-0.5L_2$ CV              & \textbf{1}  & 0.9989 & \textbf{0.9973} & 0.9958 & 0.9930 & \textbf{0.9899} & 0.9816 & 0.9614 \\ \hline
				$L_1-0.25L_2$ CV             & \textbf{1}  & \textbf{0.9990} & 0.9971 & 0.9957 & \textbf{0.9935} & 0.9898 & 0.9808 & 0.9560 \\ \hline
				$L_1$ CV                   & \textbf{1}  & 0.9984 & 0.9972 & \textbf{0.9959} & 0.9928 & 0.9863 & 0.9700 & 0.9332 \\ \hline
				$L_1-L_2$ FR& \textbf{1}  & 0.9505 & 0.9053 & 0.8578 & 0.8015 & 0.7369 & 0.6478 & 0.5662 \\ \hline
				$L_1-0.75L_2$ FR& \textbf{1}  & 0.9987 & 0.9971 & 0.9945 & 0.9913 & 0.9879 & 0.9715 & 0.5364 \\ \hline
				$L_1-0.5L_2$ FR & 0.9998 & 0.9984 & 0.9972 & 0.9955 & 0.9921 & 0.9833 & 0.9538 & 0.3540 \\ \hline
				$L_1-0.25L_2$ FR & 0.9995 & 0.9983 & 0.9972 & 0.9940 & 0.9880 & 0.9763 & 0.9299 & 0.5984 \\ \hline
				$L_1$ FR & 0.9992 & 0.9983 & 0.9970 & 0.9925 & 0.9833 & 0.9643 & 0.8800 & 0.4503 \\ \hline
				$L_1+L_2^2$ & 0.9996 & 0.9980 & 0.9960 & 0.9937 & 0.9903 & 0.9858 & 0.9776 & 0.9668 \\ \hline
				$L_0$ \cite{xu2011image}  & \textbf{1}  & 0.8753 & 0.7697 & 0.6768 & 0.5981 & 0.5247 & 0.4627 & 0.4054\\ \hline
					$L_0$  \cite{storath2014fast}                & 0.9995       & 0.9966 & 0.9933 & 0.9904 & 0.9874 & 0.9810 & 0.9688 & 0.9462\\ \hline
					$R_{MS}$ & 0.9995 & 0.9983 & 0.9971 & 0.9954 & 0.9932 & 0.9850 & 0.9731 & 0.9361\\\hline
		\end{tabular}}
	\end{center}
	\label{tab:twophase_grayscale}
\end{table}     

All the algorithms are coded in MATLAB R2019a and all the computations are performed on a Dell laptop with a 1.80 GHz Intel Core i7-8565U processor and 16.0 GB of RAM. The codes are available at  \url{https://github.com/kbui1993/L1mL2Segmentation}.

\begin{figure}[p]
	\begin{minipage}{\linewidth}
		\centering
		\resizebox{\textwidth}{!}{%
			\begin{tabular}{c@{}c@{}c@{}c@{}c@{}c}
				\subcaptionbox{60\% SPIN\label{fig:grayscale_sp_60}}{\includegraphics[width = 1.00in]{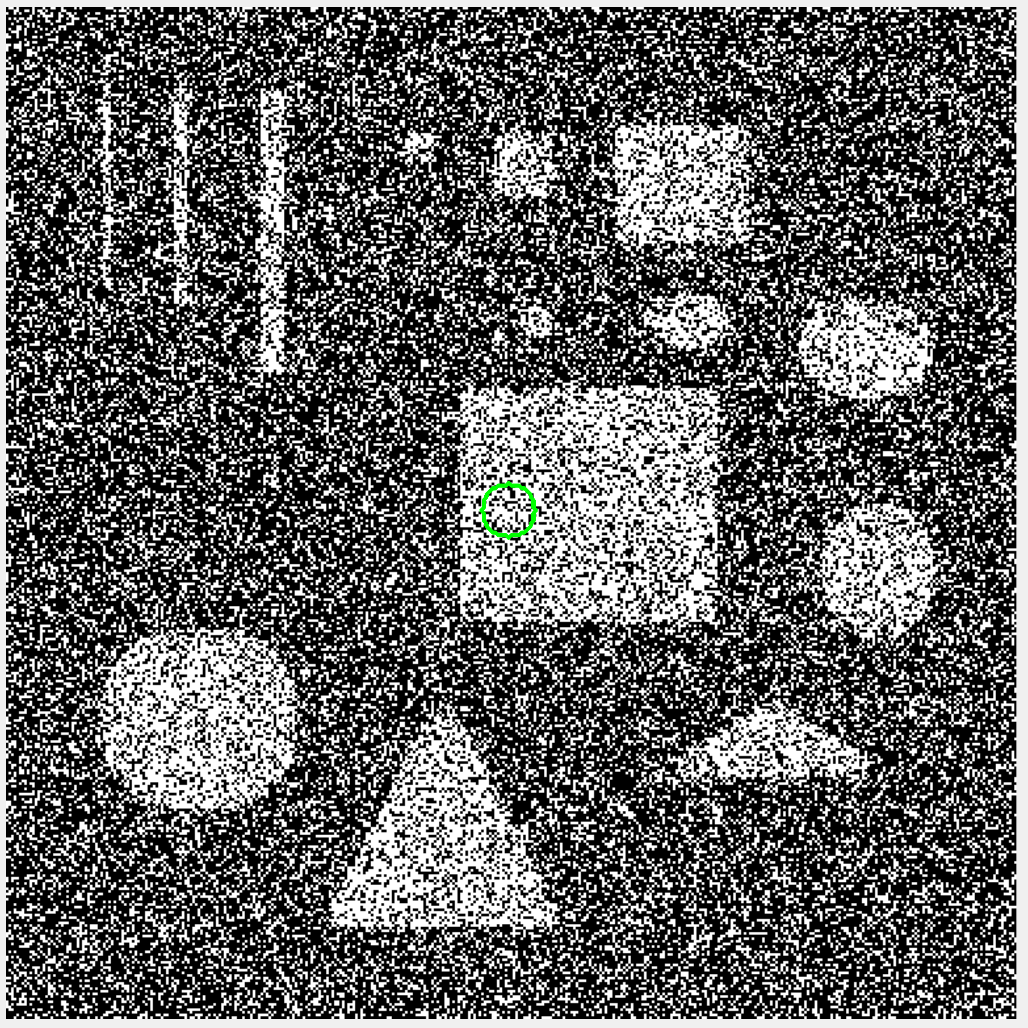}}& \subcaptionbox{$L_1+L_2^2$}{\includegraphics[width = 1.00in]{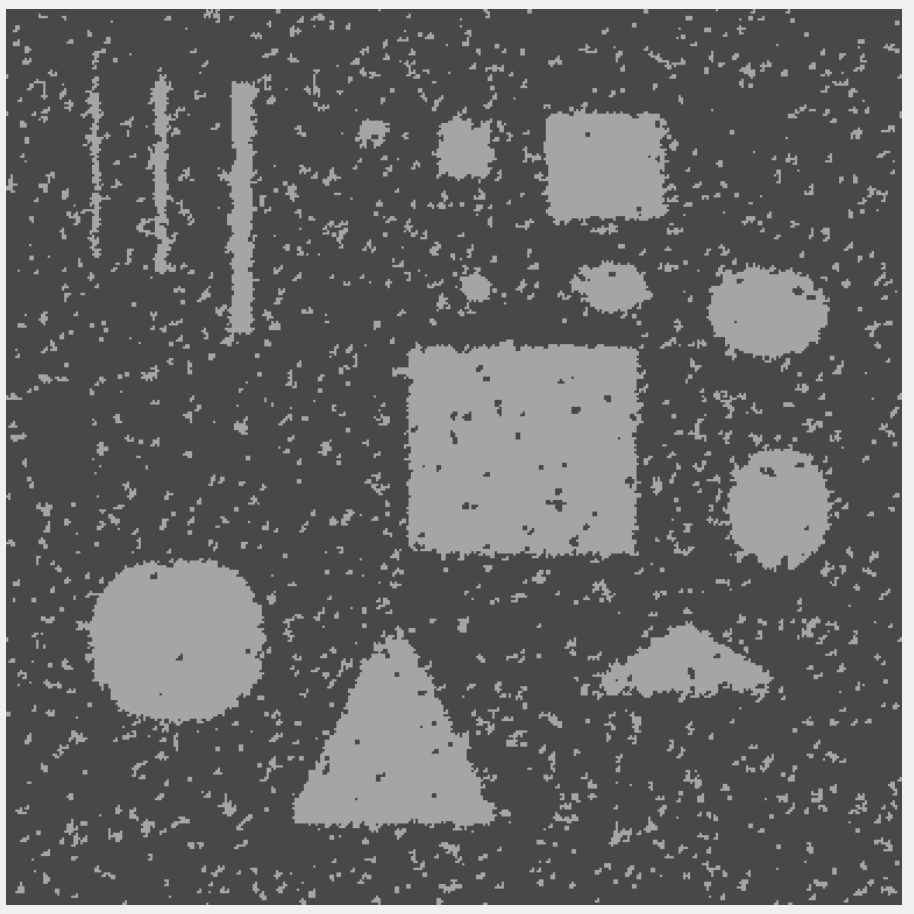}}  &	\subcaptionbox{$L_0$ \cite{xu2011image}}{\includegraphics[width = 1.00in]{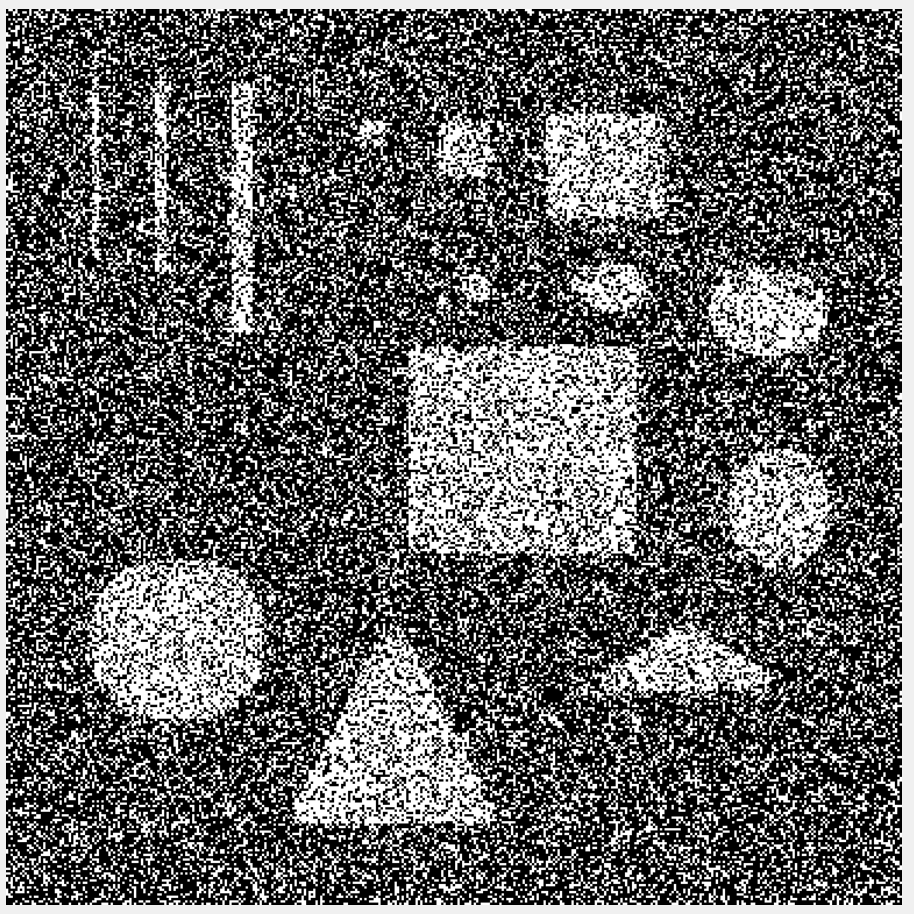}} &	\subcaptionbox{$L_0$ \cite{storath2014fast}}{\includegraphics[width = 1.00in]{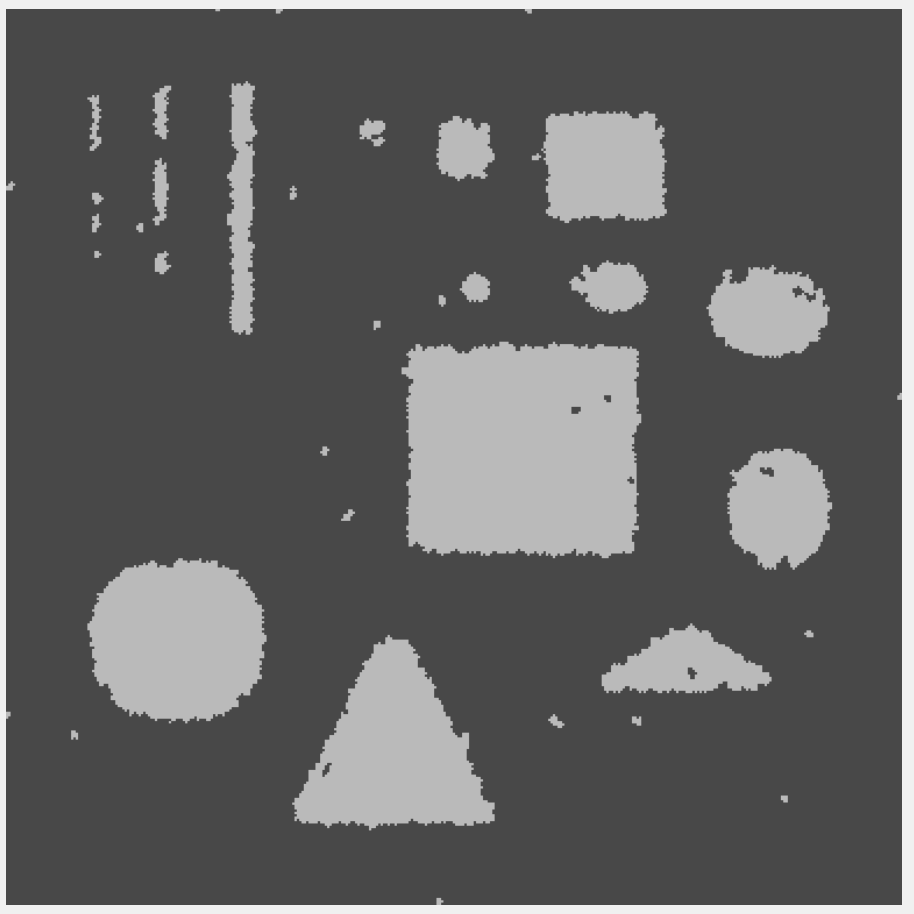}}&	\subcaptionbox{$R_{MS}$}{\includegraphics[width = 1.00in]{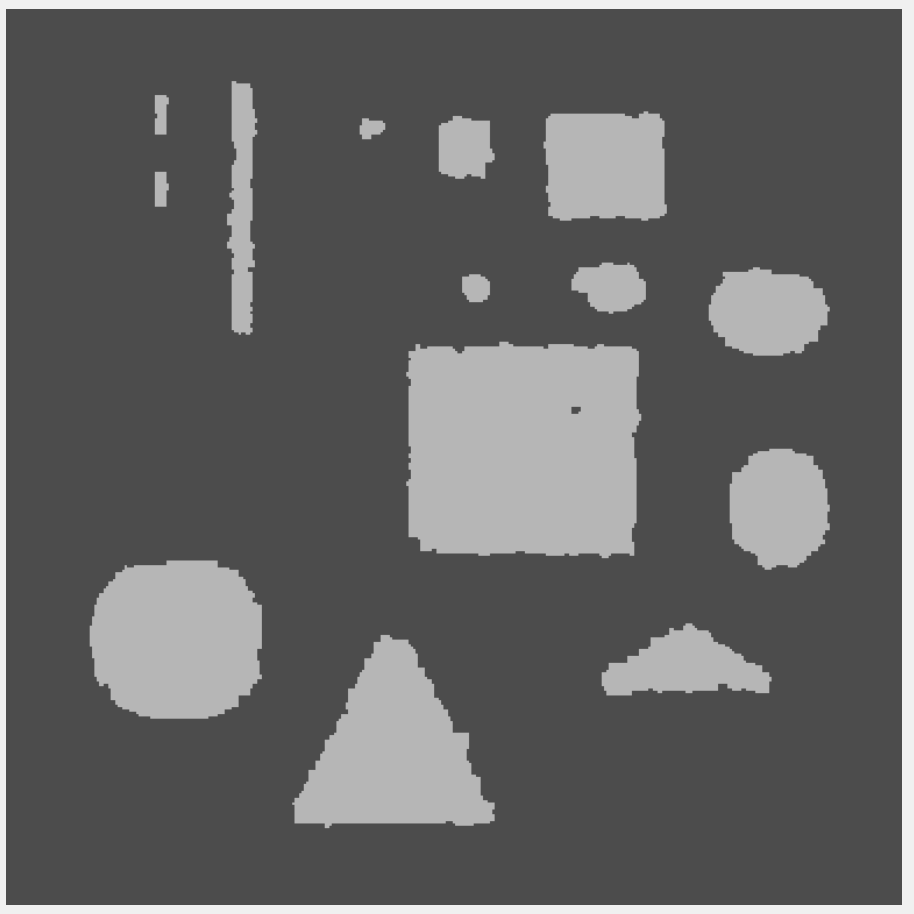}}\\
				\subcaptionbox{$L_1-L_2$ CV}{\includegraphics[width = 1.00in]{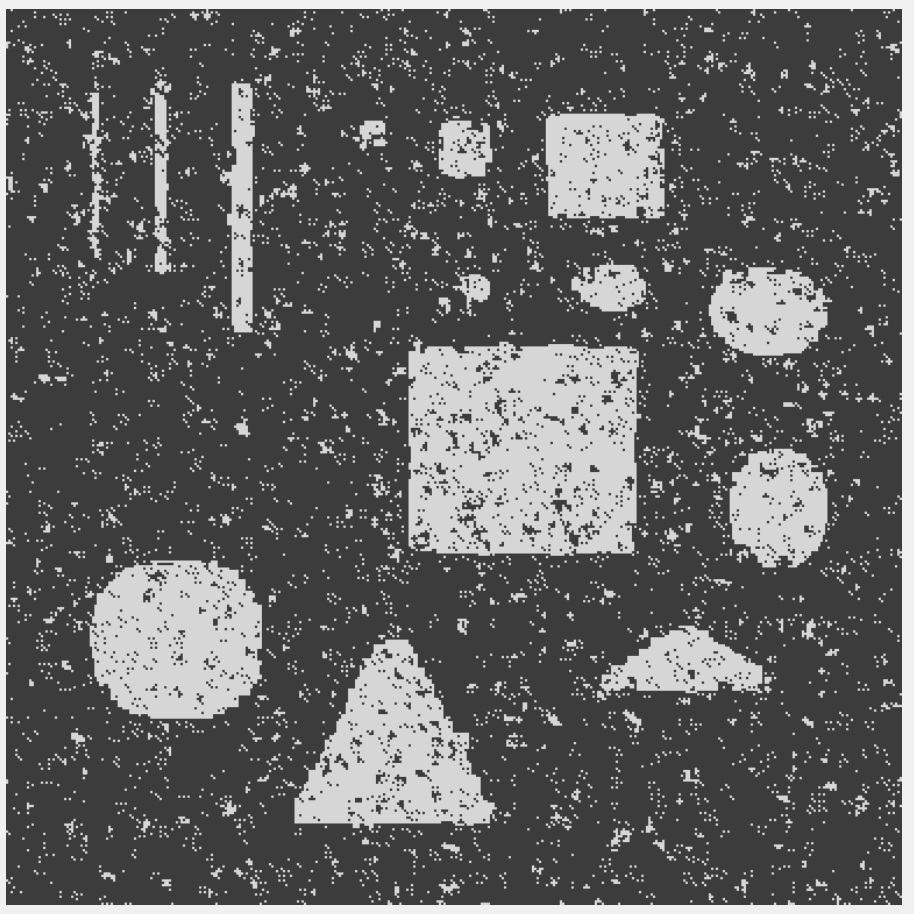}}  & \subcaptionbox{$L_1-0.75L_2$ CV}{\includegraphics[width = 1.00in]{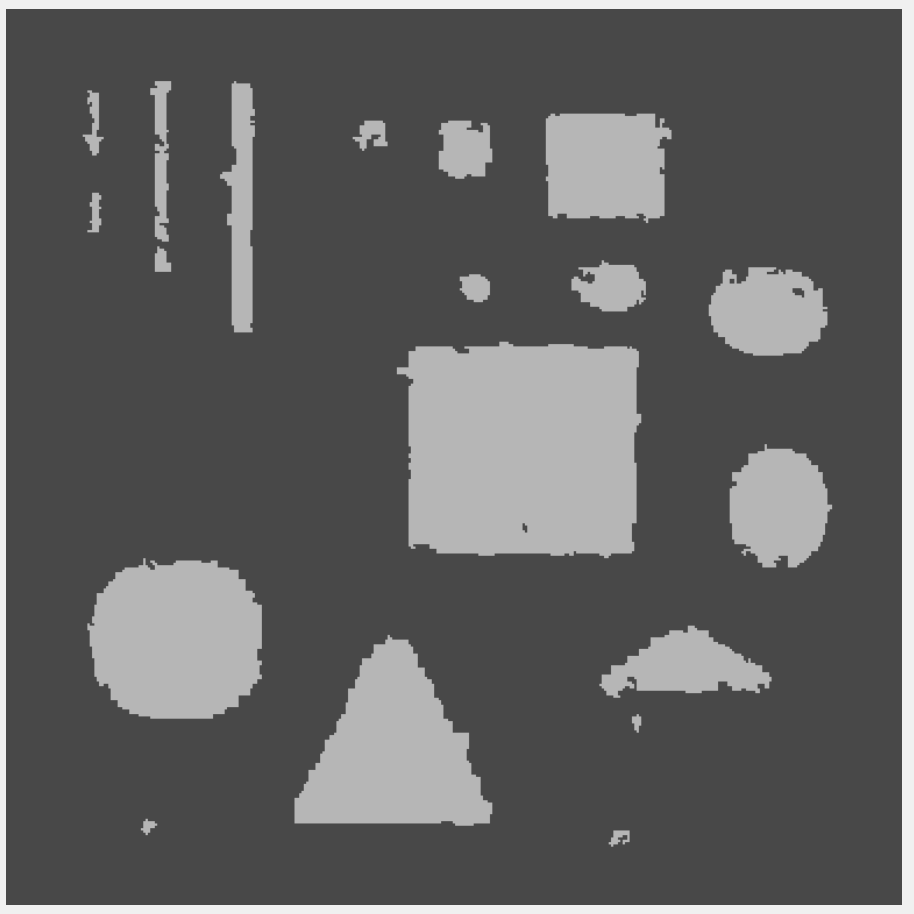}} & \subcaptionbox{$L_1-0.5L_2$ CV}{\includegraphics[width = 1.00in]{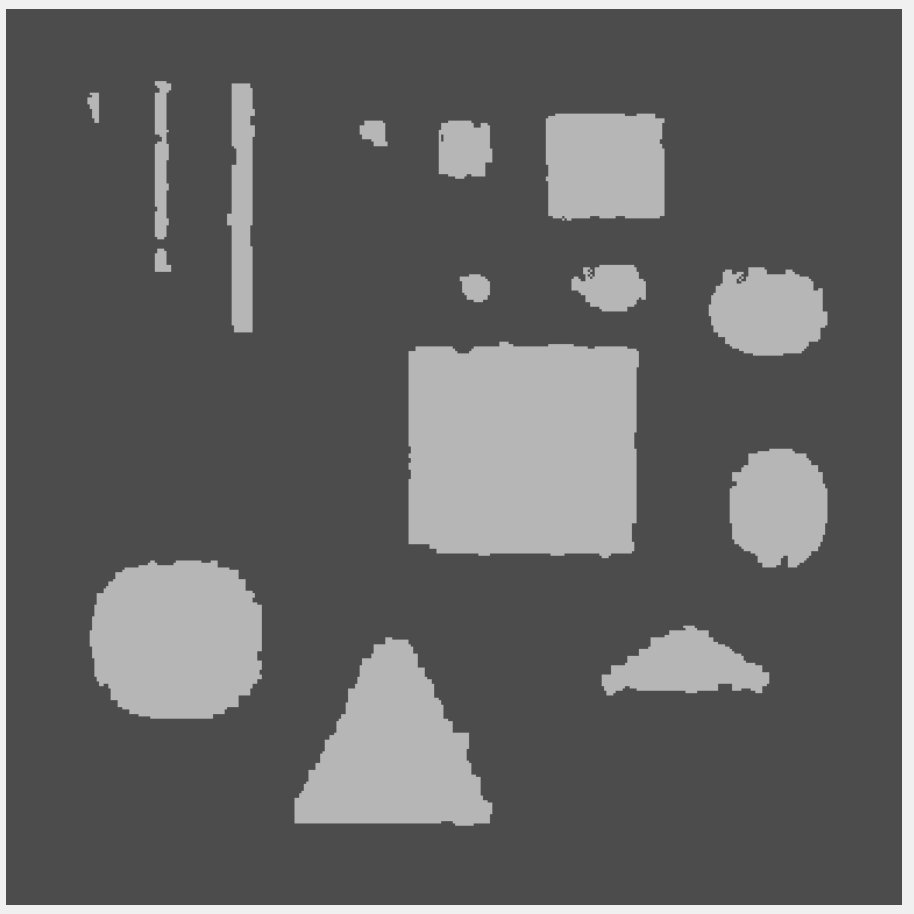}} &
				\subcaptionbox{$L_1-0.25L_2$ CV}{\includegraphics[width = 1.00in]{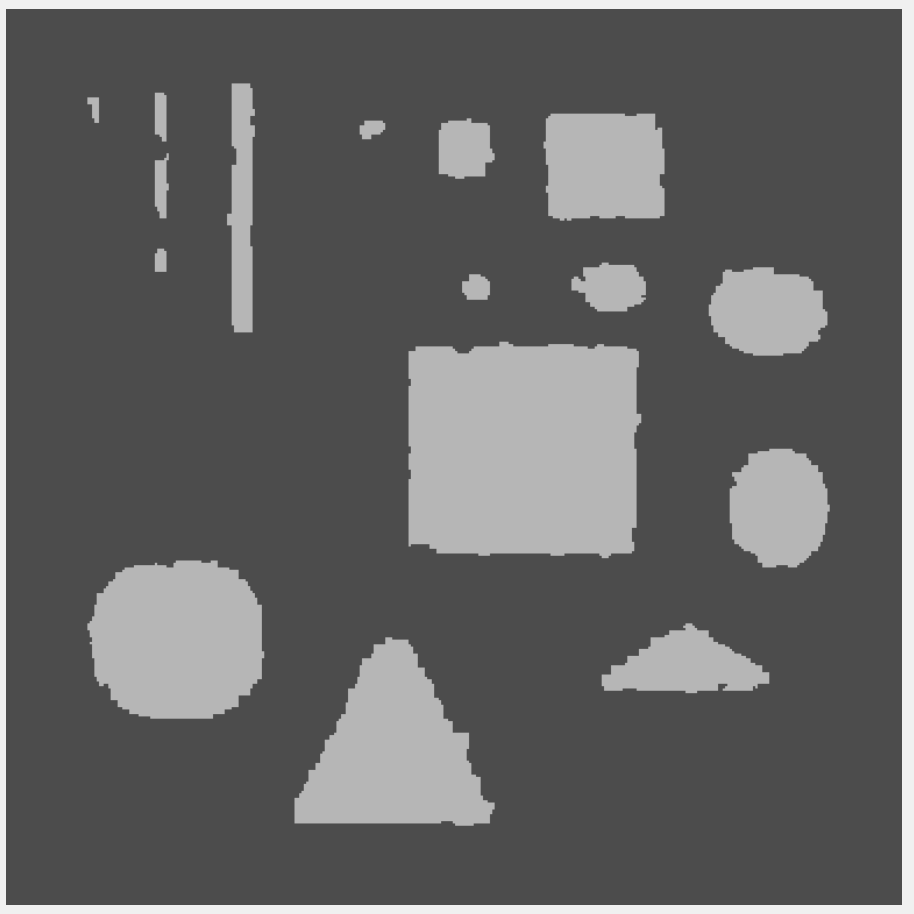}} & \subcaptionbox{$L_1$ CV}{\includegraphics[width = 1.00in]{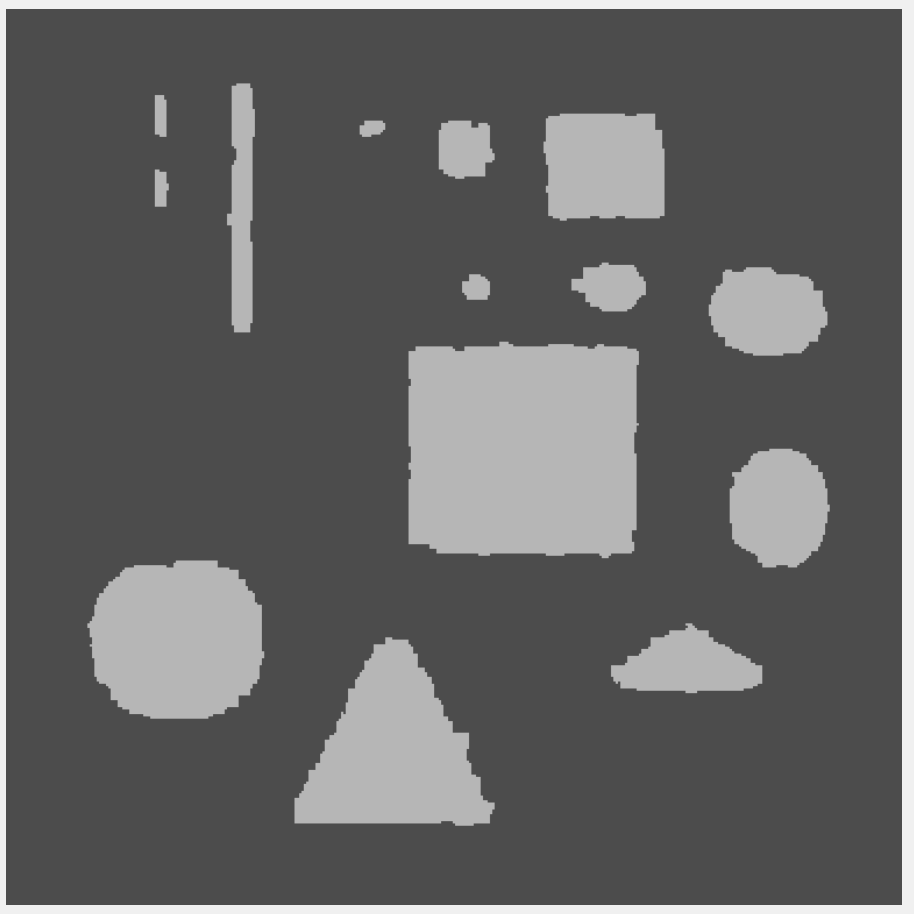}}\\ \subcaptionbox{$L_1-L_2$ FR}{\includegraphics[width = 1.00in]{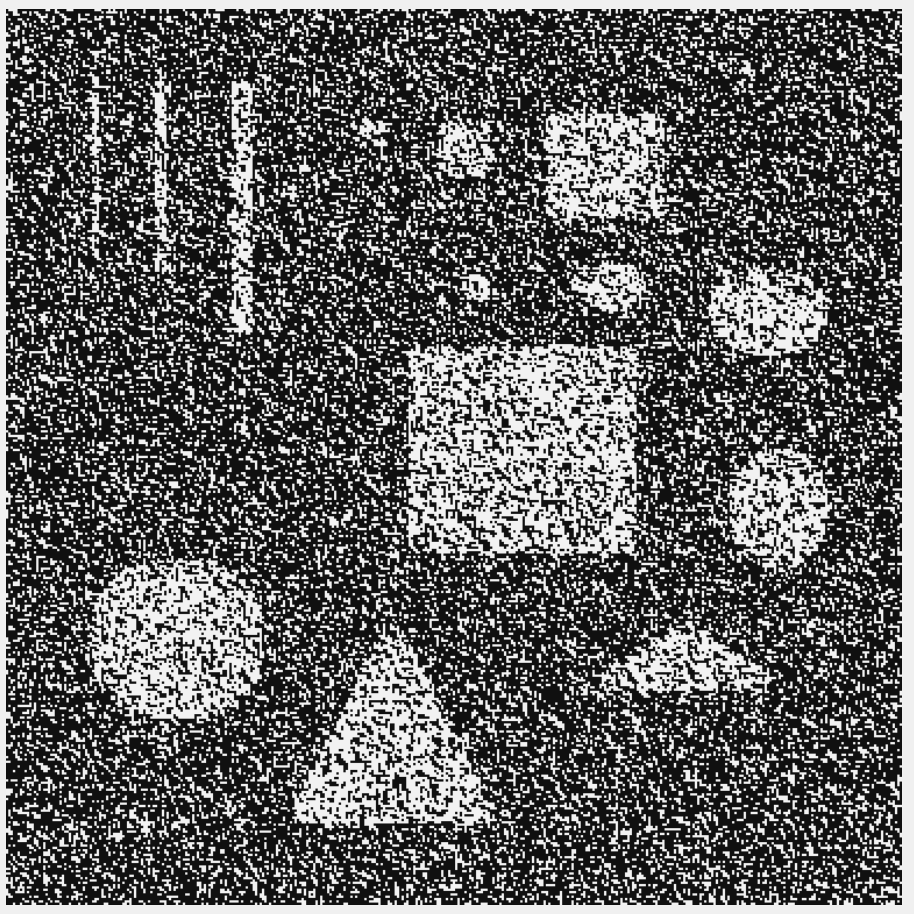}}
				 & \subcaptionbox{$L_1-0.75L_2$ FR}{\includegraphics[width = 1.00in]{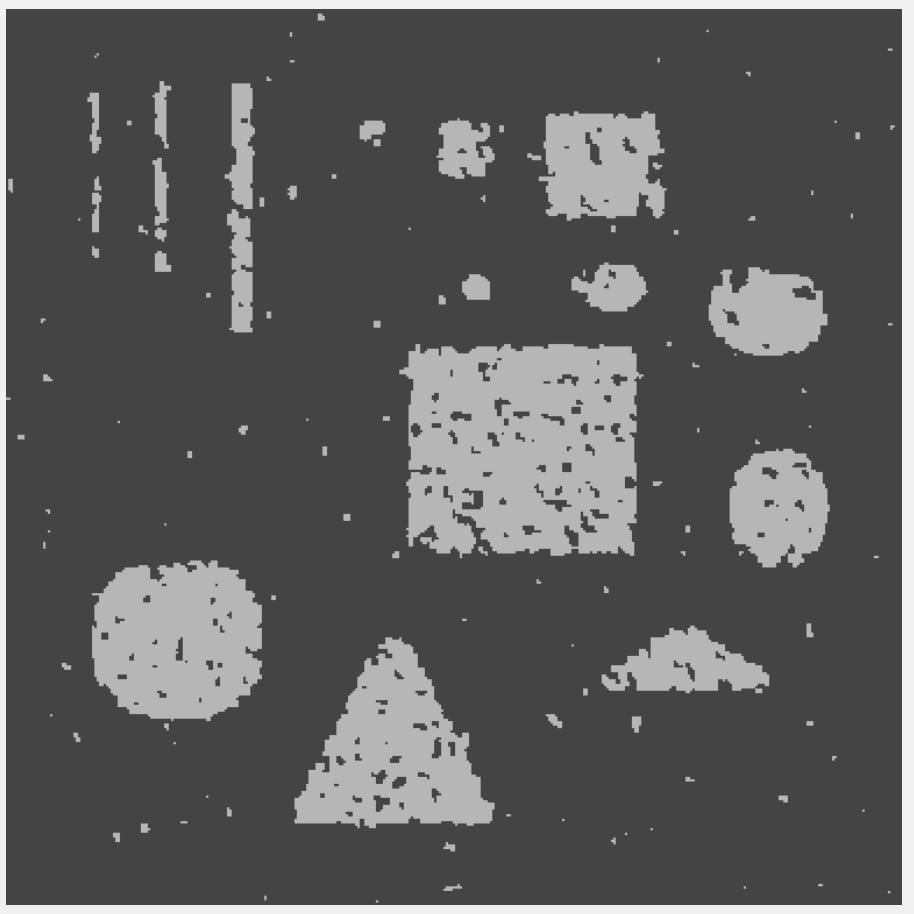}} & \subcaptionbox{$L_1-0.5L_2$ FR}{\includegraphics[width = 1.00in]{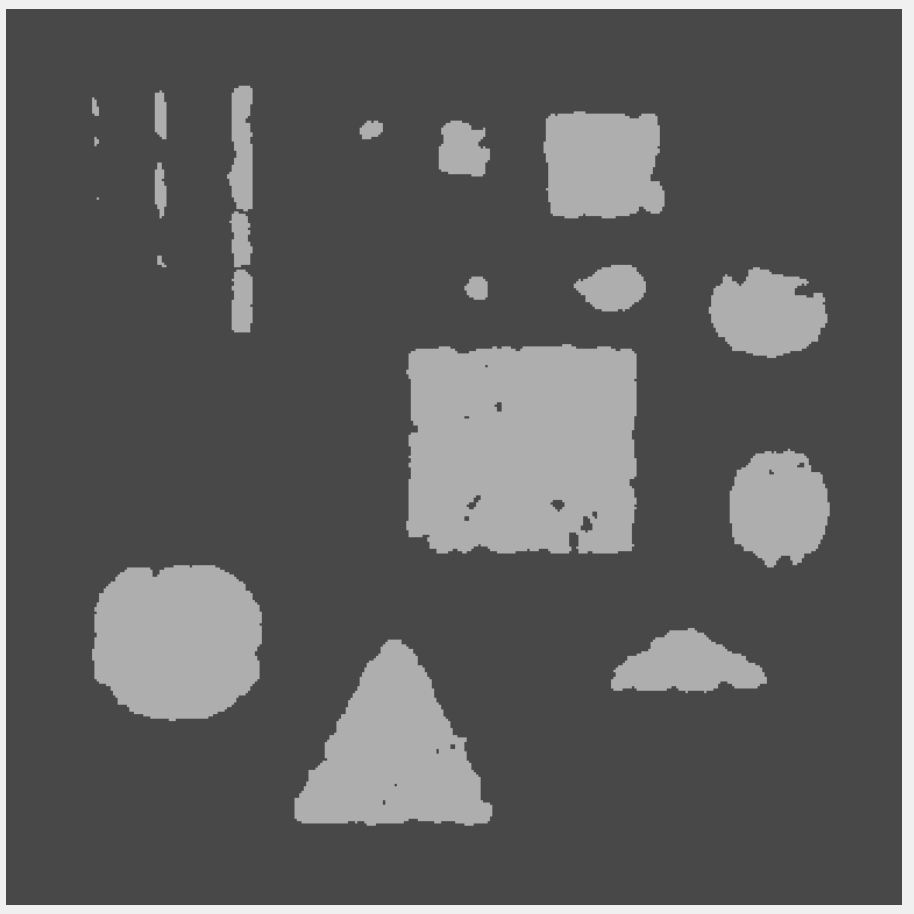}} &
				\subcaptionbox{$L_1-0.25L_2$ FR}{\includegraphics[width = 1.00in]{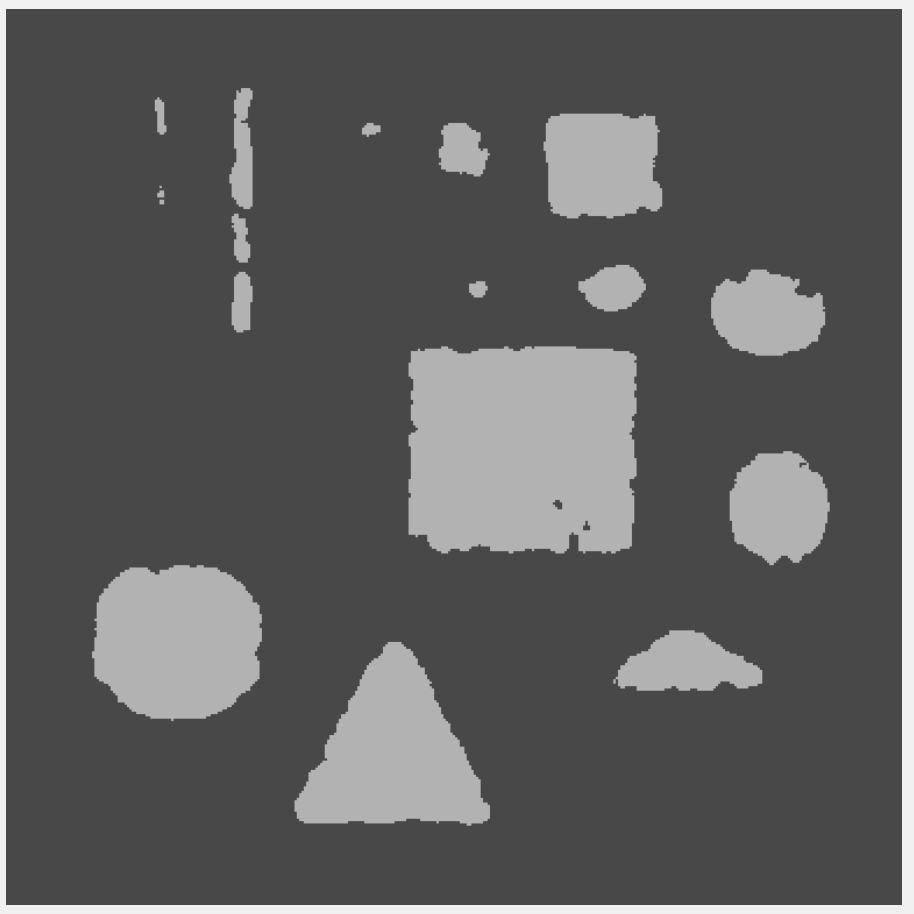}} & \subcaptionbox{$L_1$ FR}{\includegraphics[width = 1.00in]{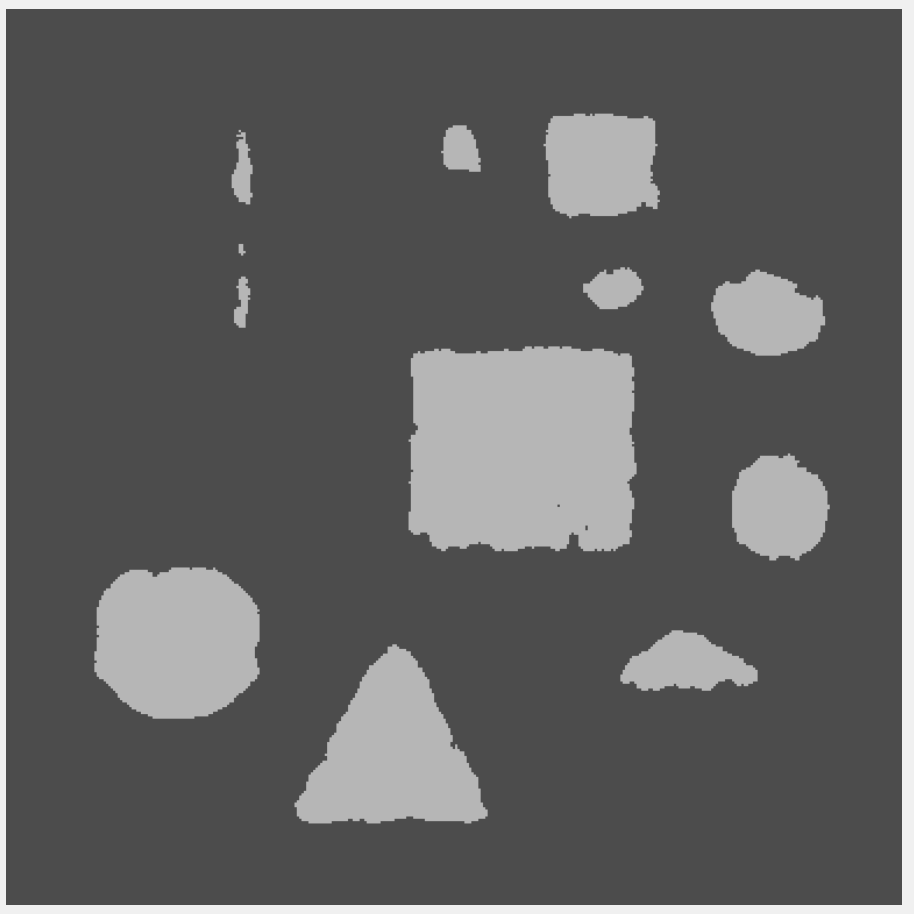}}
		\end{tabular}}
		\caption{Reconstruction results on Figure \ref{fig:synthetic_grayscale} corrupted with 60\% SPIN.  }
		\label{fig:grayscale_60_spin}
	\end{minipage}
	\begin{minipage}{\linewidth}
		\centering
		\resizebox{\textwidth}{!}{%
			\begin{tabular}{c@{}c@{}c@{}c@{}c@{}c}
				\subcaptionbox{60\% RVIN\label{fig:grayscale_rv_60}}{\includegraphics[width = 1.00in]{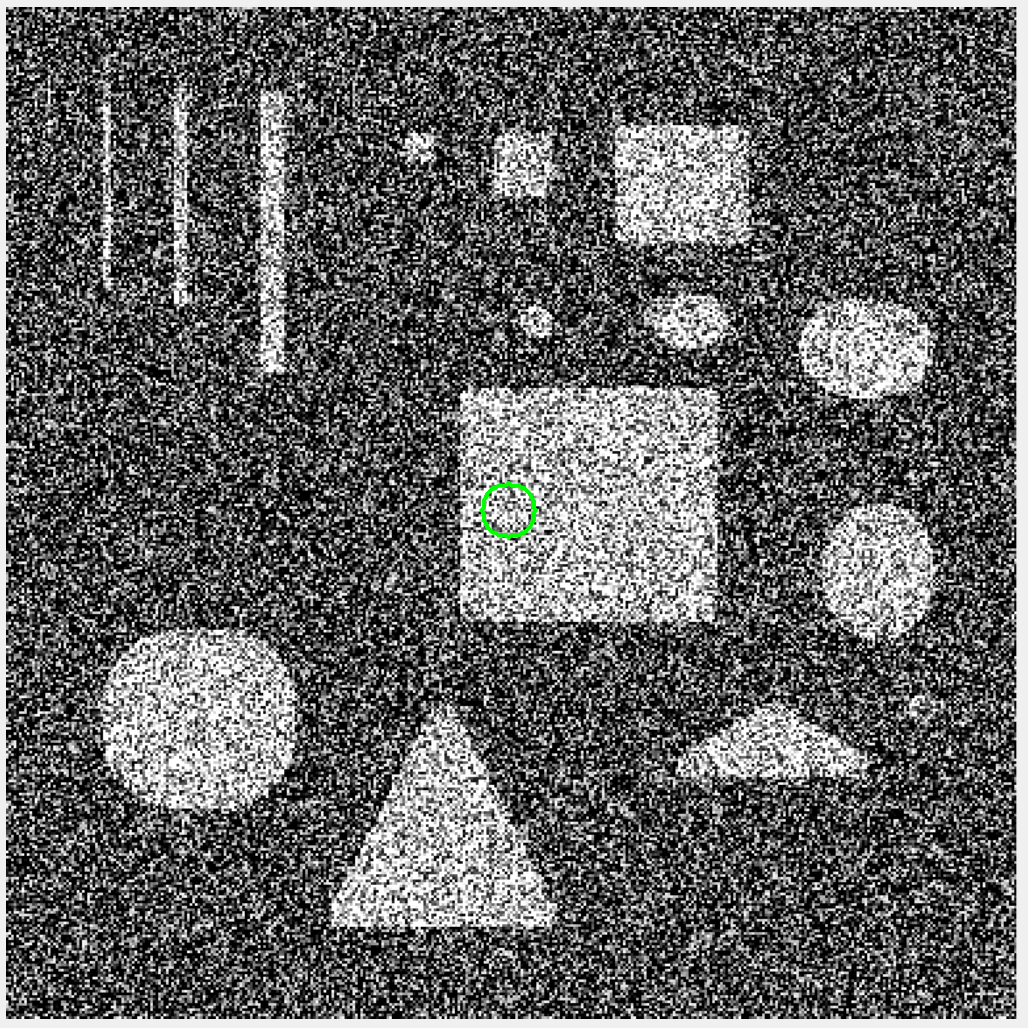}} &
				\subcaptionbox{$L_1+L_2^2$}{\includegraphics[width = 1.00in]{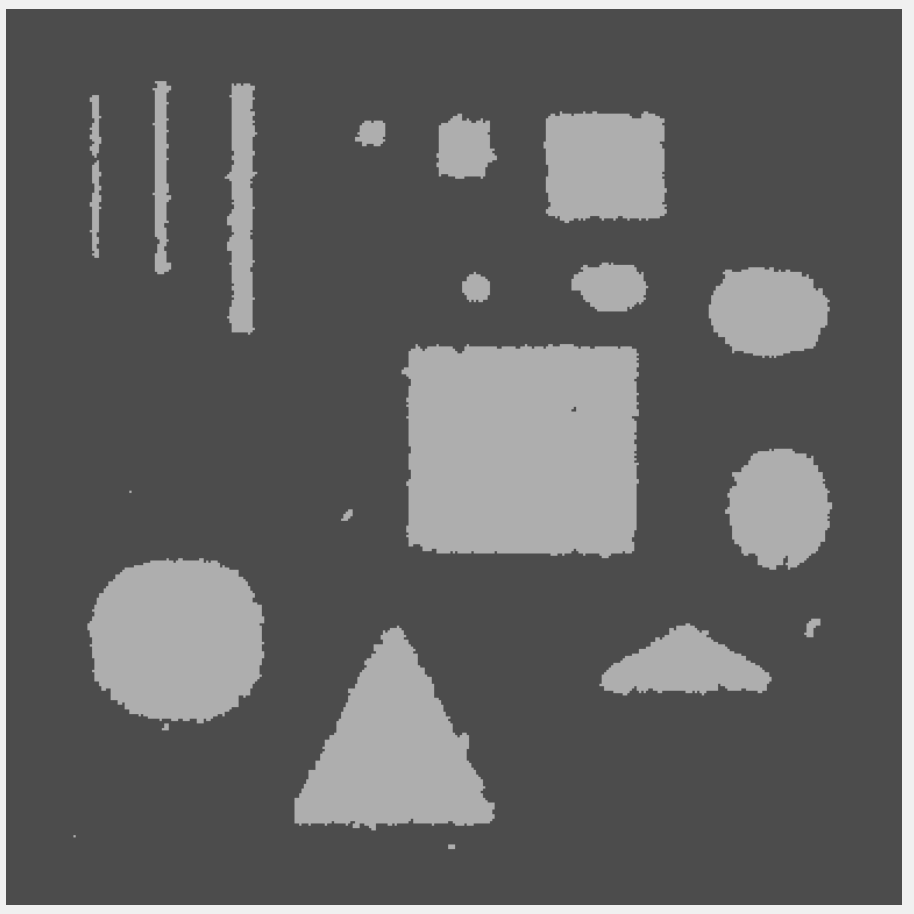}} &
				\subcaptionbox{$L_0$ \cite{xu2011image}}{\includegraphics[width = 1.00in]{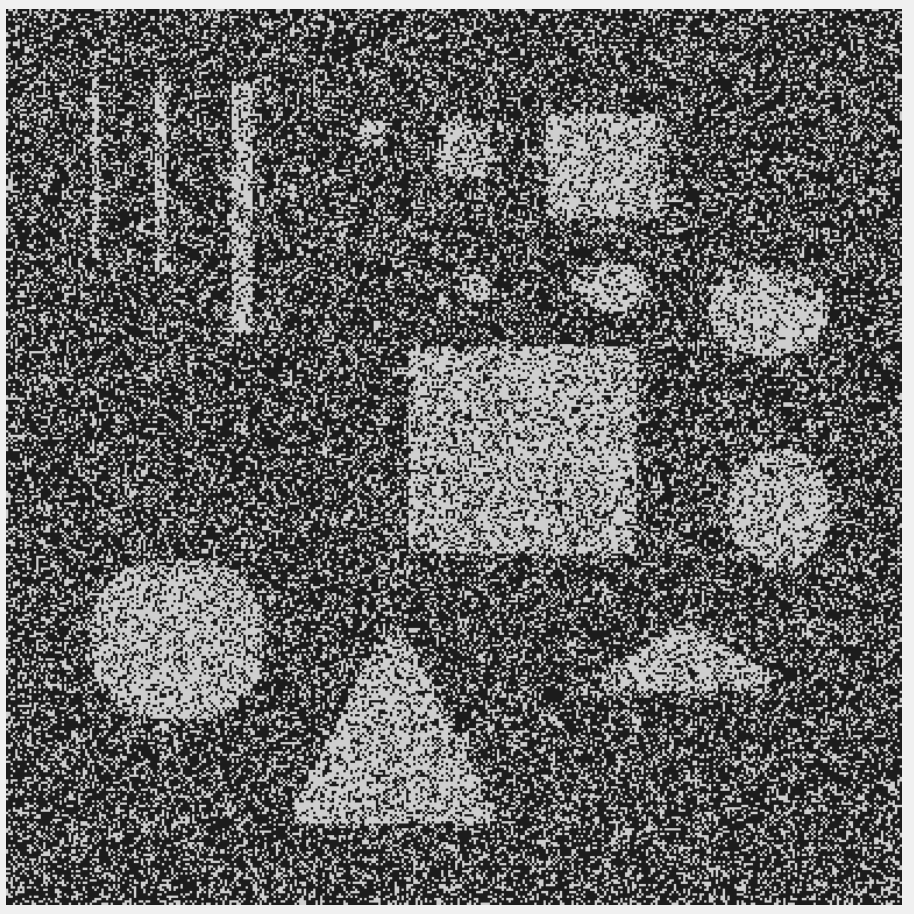}} &	\subcaptionbox{$L_0$ \cite{storath2014fast}}{\includegraphics[width = 1.00in]{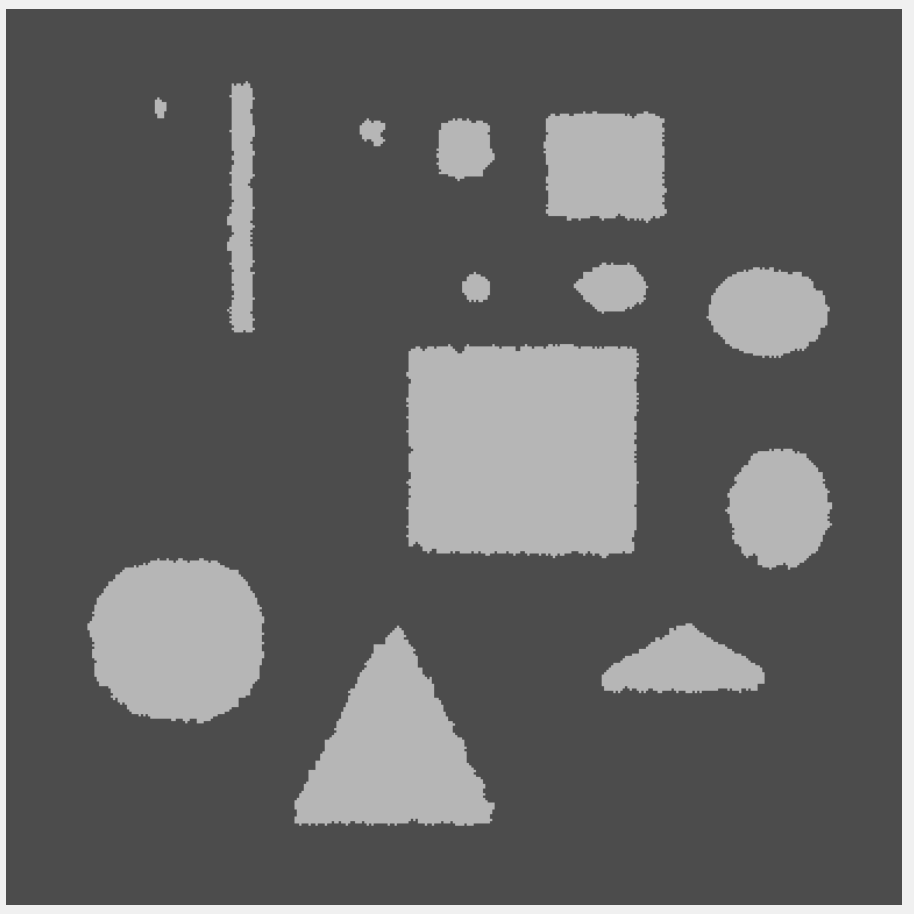}} &	\subcaptionbox{$R_{MS}$}{\includegraphics[width = 1.00in]{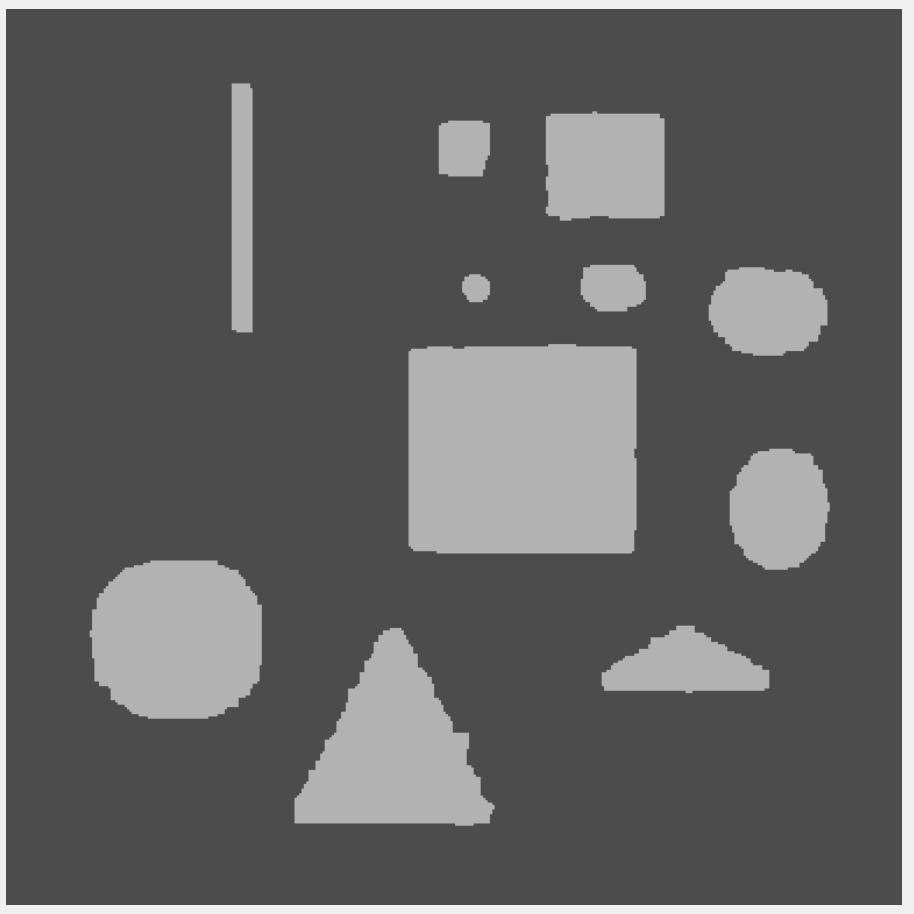}}\\ \subcaptionbox{$L_1-L_2$ CV }{\includegraphics[width = 1.00in]{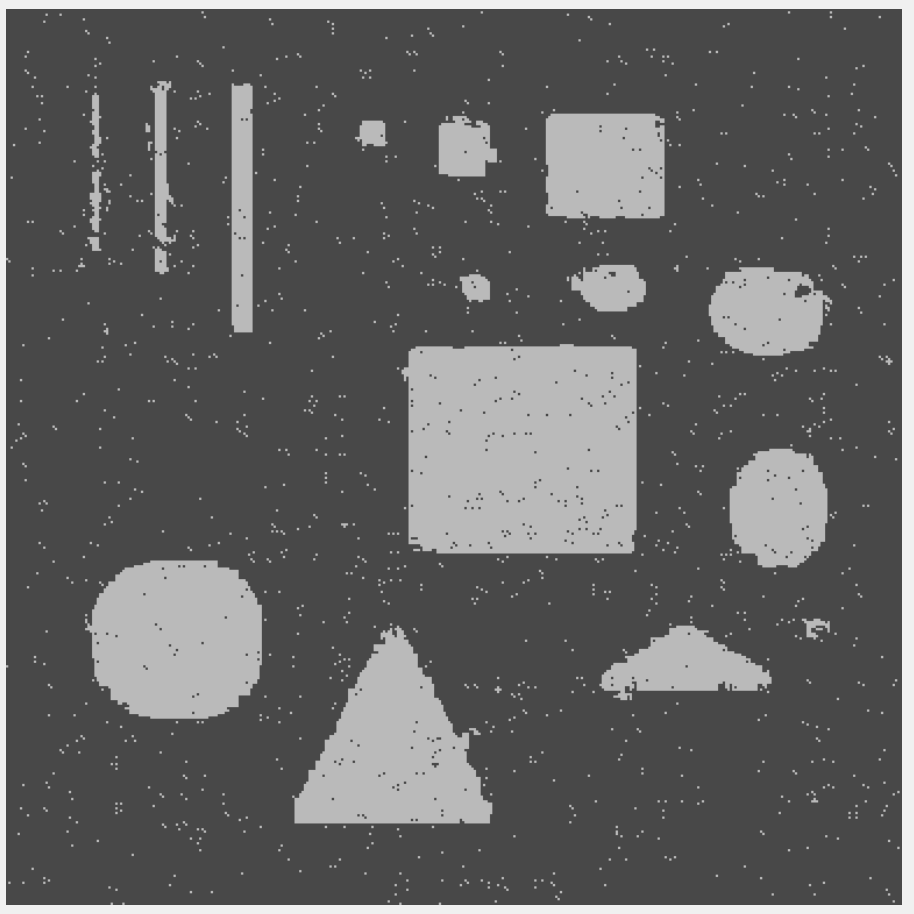}} & \subcaptionbox{$L_1-0.75L_2$ CV}{\includegraphics[width = 1.00in]{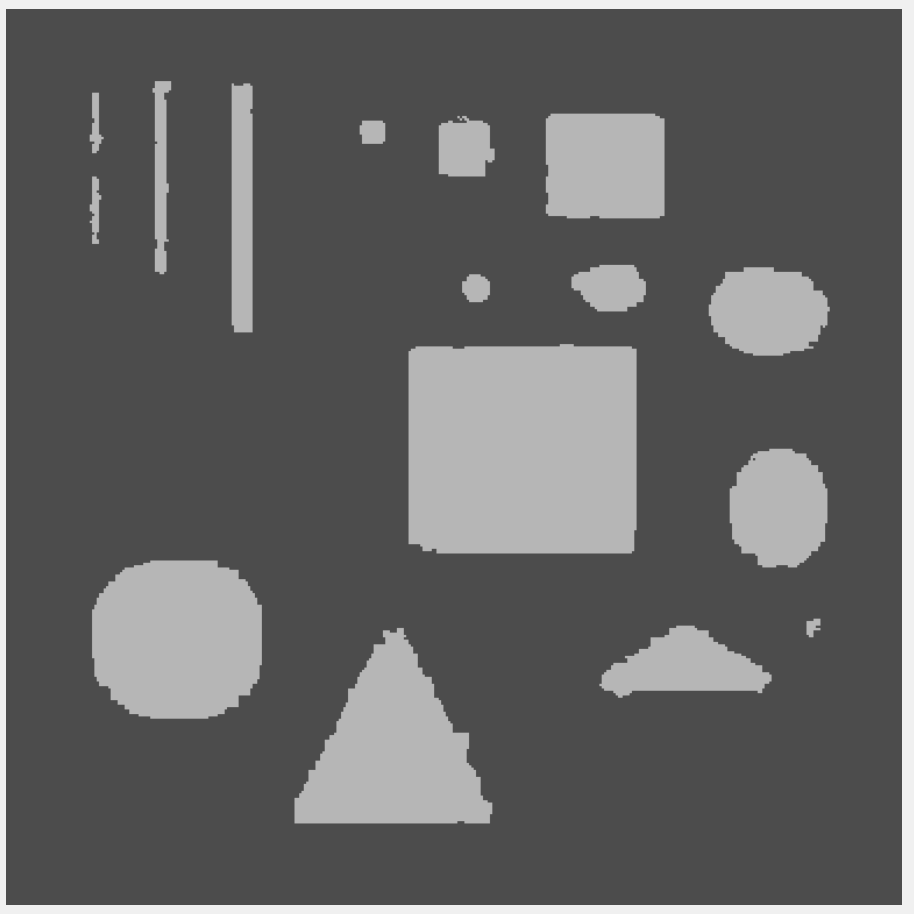}} & \subcaptionbox{$L_1-0.5L_2$ CV}{\includegraphics[width = 1.00in]{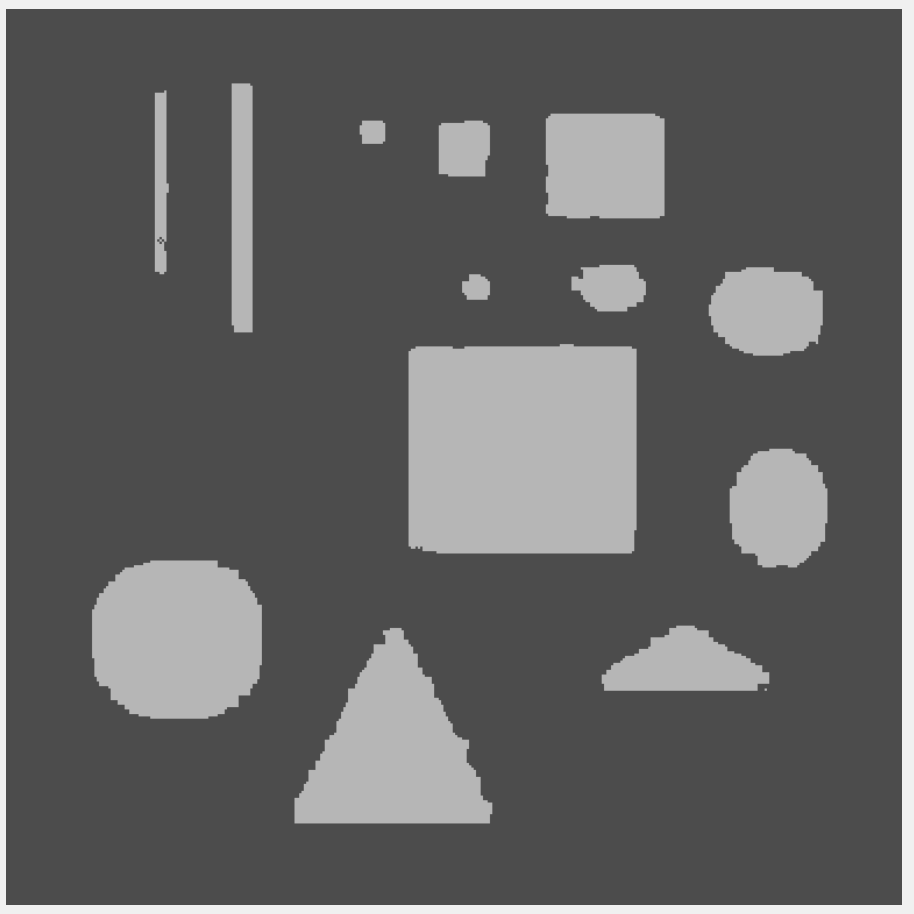}} &
				\subcaptionbox{$L_1-0.25L_2$ CV}{\includegraphics[width = 1.00in]{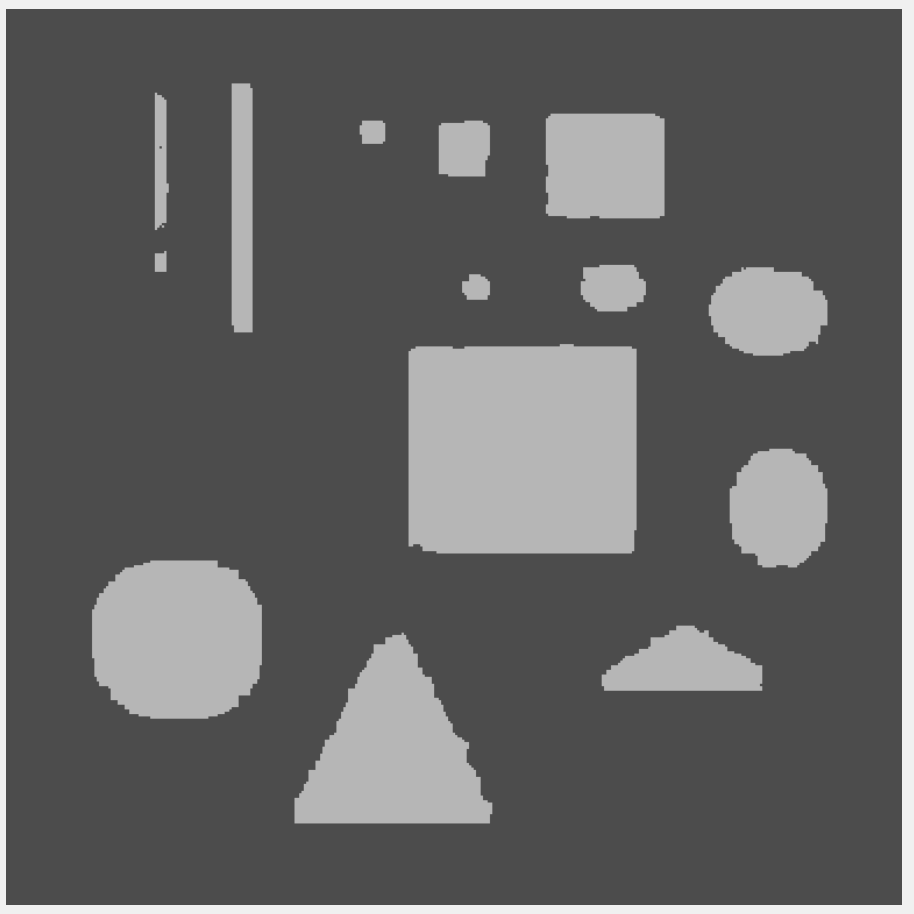}} & \subcaptionbox{$L_1$ CV}{\includegraphics[width = 1.00in]{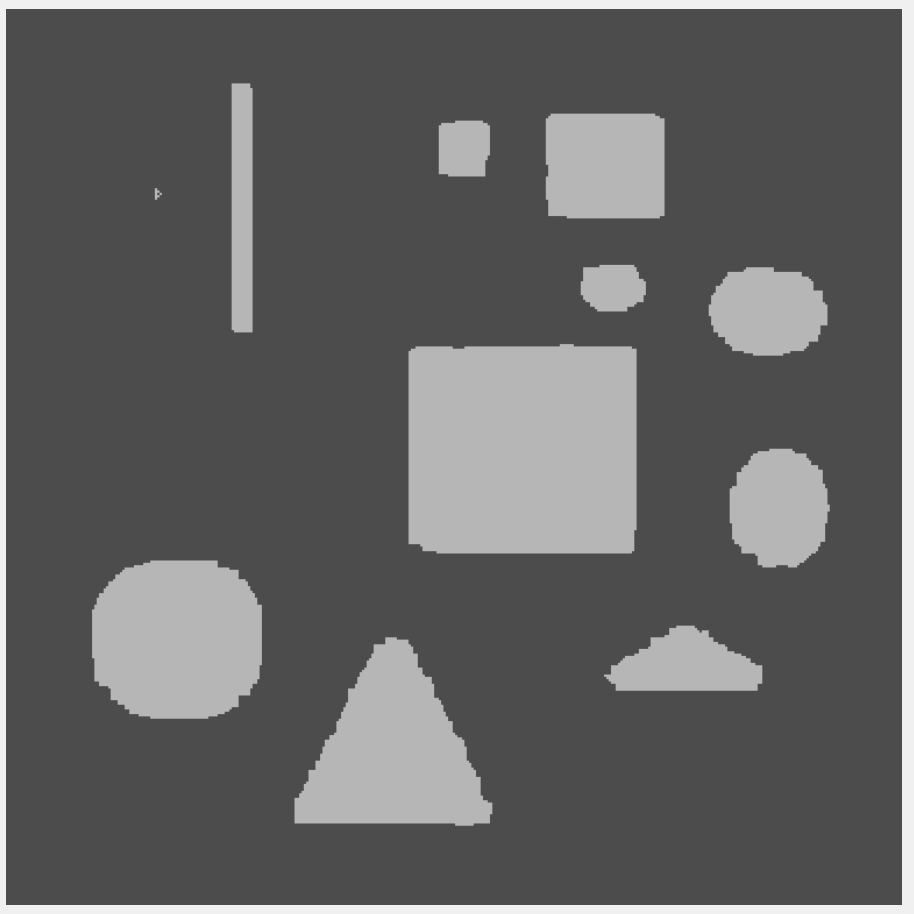}}\\
				\subcaptionbox{$L_1-L_2$ FR}{\includegraphics[width = 1.00in]{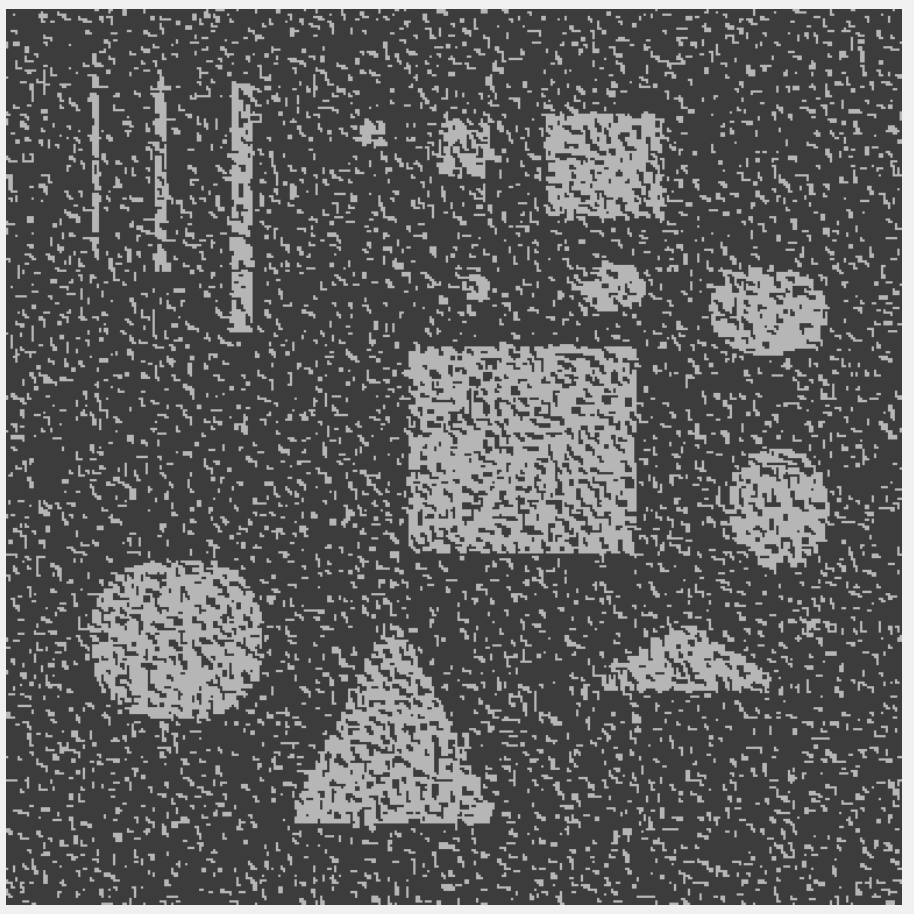}} & \subcaptionbox{$L_1-0.75L_2$ FR}{\includegraphics[width = 1.00in]{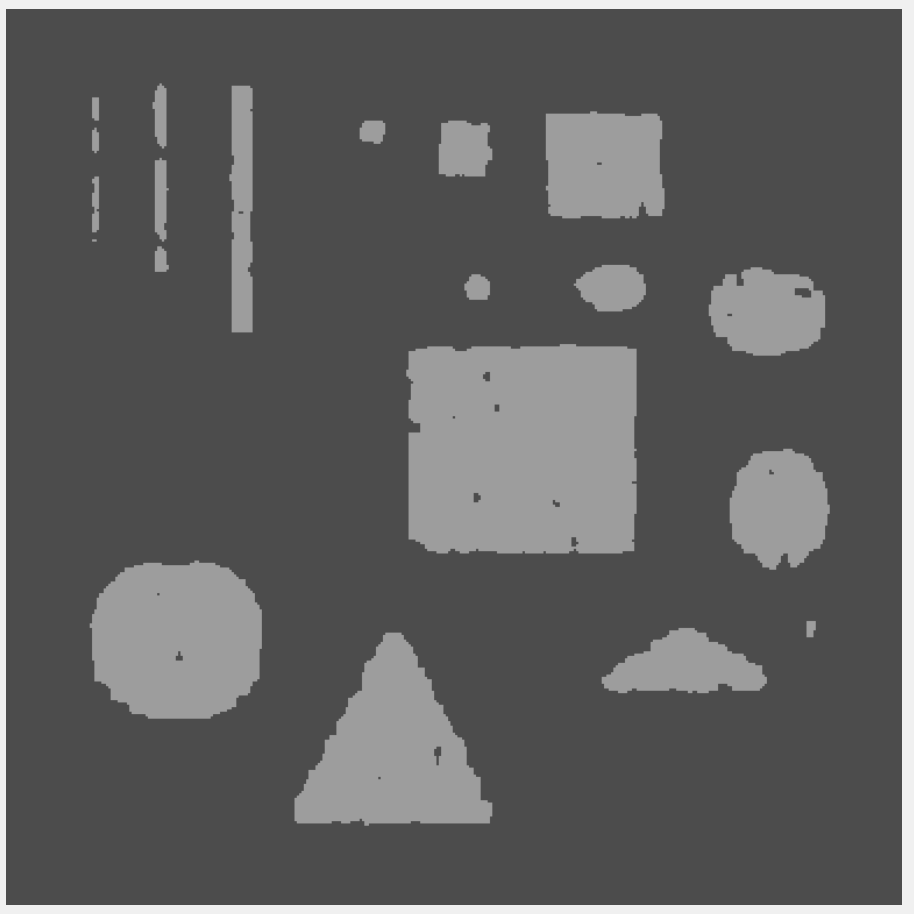}} & \subcaptionbox{$L_1-0.5L_2$ FR}{\includegraphics[width = 1.00in]{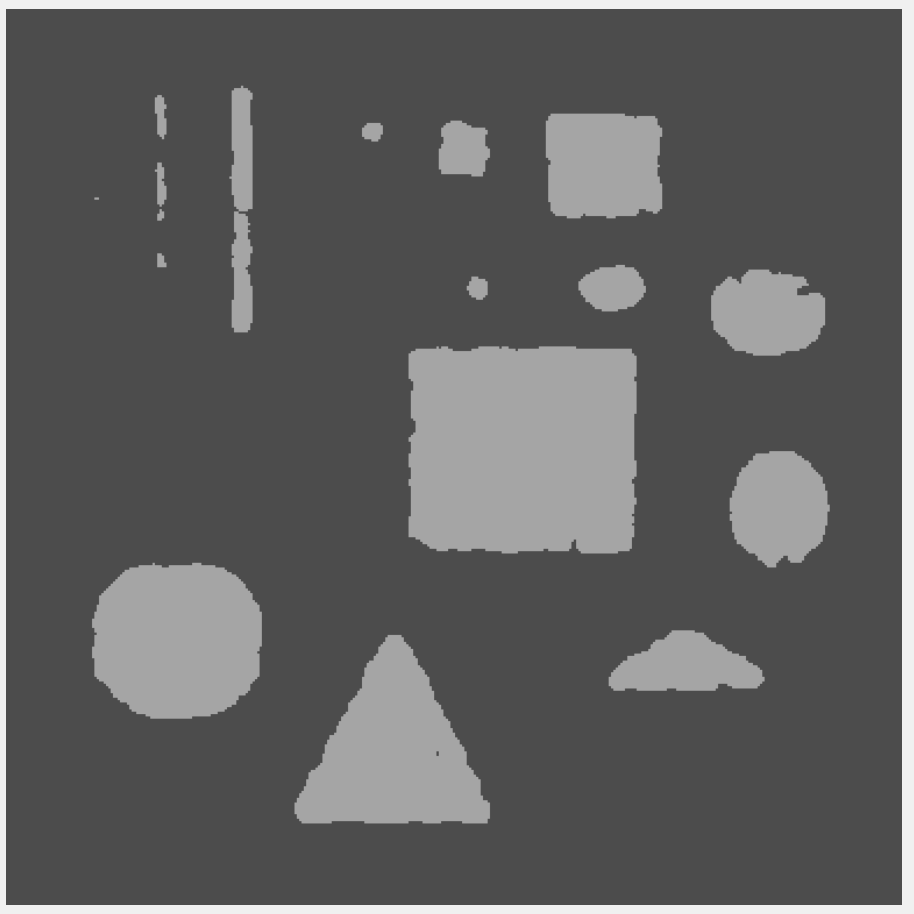}} &
				\subcaptionbox{$L_1-0.25L_2$ FR}{\includegraphics[width = 1.00in]{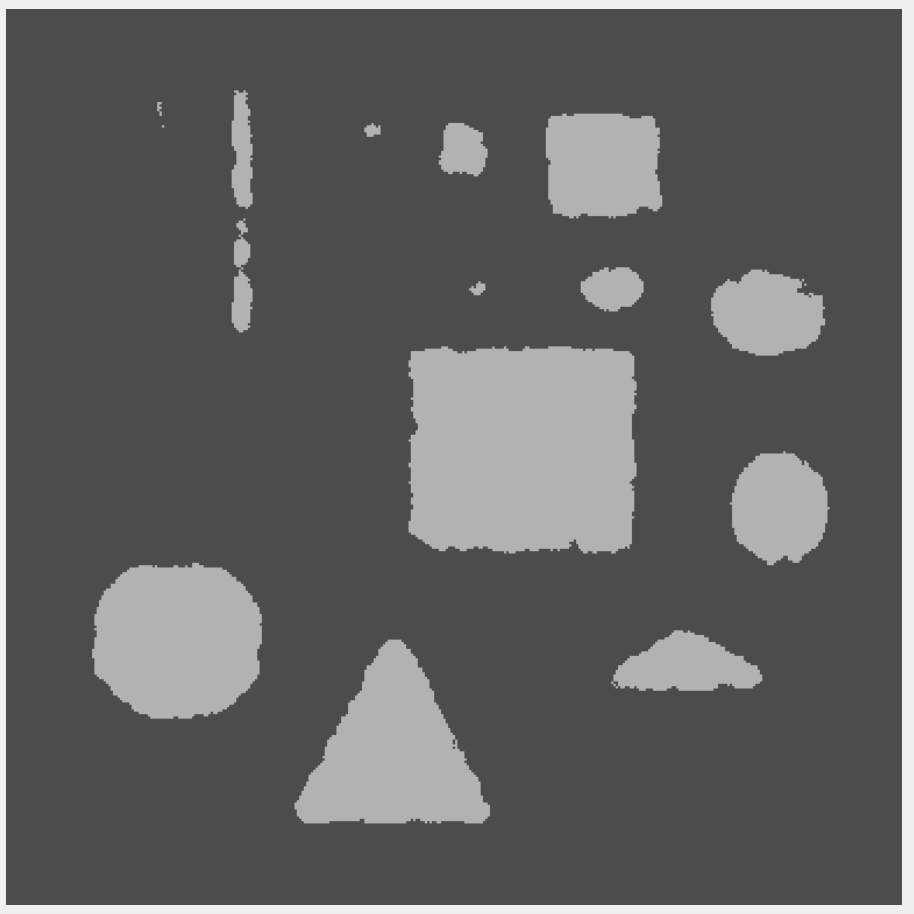}} & \subcaptionbox{$L_1$ FR}{\includegraphics[width = 1.00in]{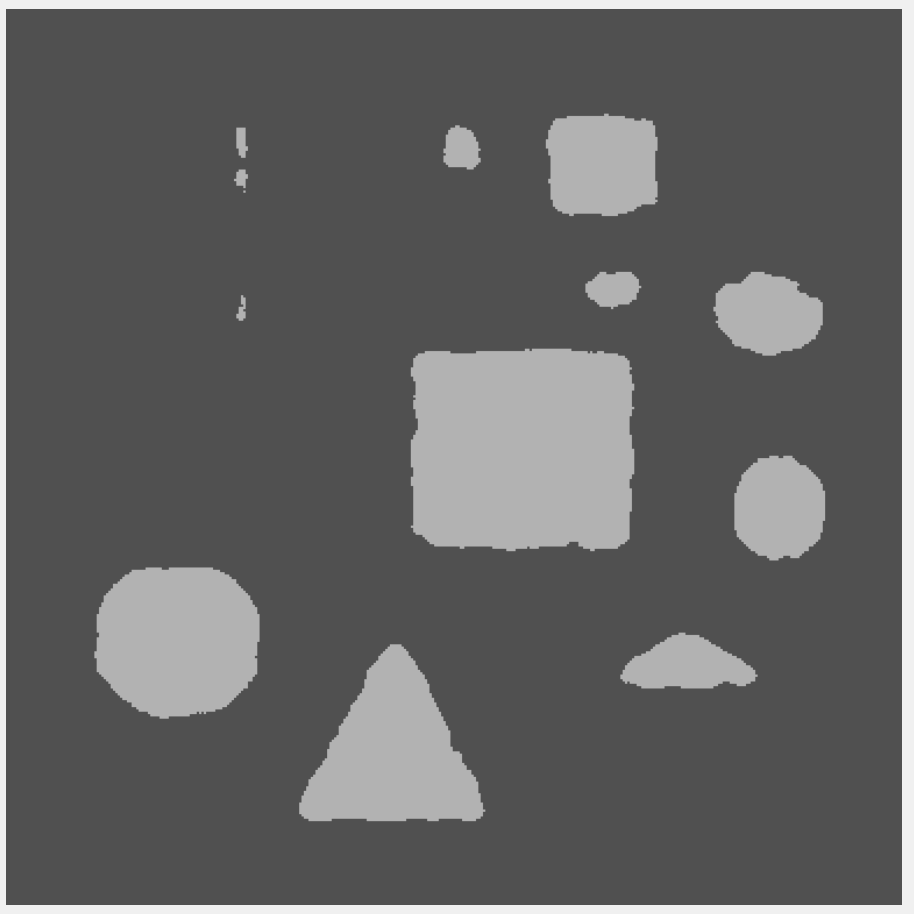}}\\ 
		\end{tabular}}
		\caption{Reconstruction results on Figure \ref{fig:synthetic_grayscale} corrupted with 60\% RVIN.  }
		\label{fig:grayscale_60_rvin}
	\end{minipage}
\end{figure}
\subsection{Synthetic Images}
We apply various segmentation algorithms on the synthetic images presented in Figure \ref{fig:synthetic}. We scale the intensity values of all the images to be $[0,1]$ to ease the parameter tuning. To demonstrate the robustness of the  algorithms with respect to noises, we  contaminate the original images with either salt-and-pepper impulsive noise (SPIN) or random-valued impulsive noise (RVIN). To  evaluate the model performance, we compute the DICE index \cite{dice1945measures} between the segmentation result and the ground truth. The metric is defined by
\begin{align*}
\text{DICE} = 2\frac{\#\{A(i) \cap A'(i)\}}{\#\{A(i)\} +\#\{A'(i)\}},
\end{align*}
where $A(i)$ is the set of pixels with label $i$ in the ground-truth image $f$ or $\mathbf{f}$, $A'(i)$ is the set of pixels with label $i$ in the segmented image $\tilde{f}$ or $\tilde{\mathbf{f}}$, and $\#\{A\}$ refers to the number of pixels in the set $A$.  If the DICE index equals 1, it means the perfect alignment of the segmentation result to the ground truth. For two-phase segmentation, we compute the DICE index only  for the object of interest, not the background. For multiphase segmentation, we compute the mean of the DICE indices across the regions, including the background.

For the two-phase AICV model, the initialization $u_1^0$ in Algorithm \ref{alg:DCA_PDHG} is a binary step function
that represents a circle of radius 10 in the center of the image (i.e., taking the value 1 if inside the circle and 0 elsewhere). Since the binary step function forms two regions in an image, it can be used as initialization for the two-phase AIFR model,  i.e., $u_1^0$ and $u_2^0 = \mathbbm{1}-u_1^0$ for Algorithm \ref{alg:DCA_PDHG2}. The initialization for the four-phase segmentation requires two step functions, which are set to be two circles of radius 30 shifted by 5 pixels to the right of the image center and another by 5 pixels to the left. The circle functions are used here for simplicity. Contours of the initialization are marked as colored circles in the noisy images.

For Figure \ref{fig:synthetic_grayscale}, we set $\lambda = 2$ for all methods, except for $L_0$ \cite{xu2011image} in which $\lambda = 50$. For the AIFR models, we set $\nu = 10$. The maximum number of inner iterations for the AITV models is 300, while the maximum number of outer iterations is 20 for AICV and 40 for AIFR. Table \ref{tab:twophase_grayscale} records the DICE indices of the segmentation results for varying levels of both SPIN and RVIN from 0\% to 70\%. 
When the noise level is at least 50\%, both $L_1-0.5L_2$ and $L_1-0.25L_2$ CV models outperform $L_1$ CV. For AIFR, $L_1-0.5L_2$ and $L_1-0.25L_2$ outperform $L_1$ across all levels of SPIN corruption. In addition,  $L_1-L_2$ FR is less robust than other values of $\alpha$ when the noise level increases. Most of the best results in the cases of intermediate to high RVIN noise levels are attained by the proposed models.
Figures \ref{fig:grayscale_60_spin}-\ref{fig:grayscale_60_rvin} display the segmentation results of Figure \ref{fig:synthetic_grayscale} corrupted with 60\% SPIN and 60\% RVIN, respectively. (We note that the contrast of the reconstructed images is different from Figure \ref{fig:synthetic_grayscale} because the impulsive noise in the corrupted image skews the values of the constants $\{c_k\}_{k=1}^N$ computed by the segmentation algorithms. This phenomenon repeats for Figures \ref{fig:synthetic_2phase}-\ref{fig:synthetic_4phase}.) As $\alpha$ decreases in both the AICV and AIFR models, the results become less noisy, but they have less segmented regions. Therefore, $\alpha =0.5$ yields the best compromise in the case of SPIN. For RVIN, the AICV and AIFR results are not as noisy as in the case of SPIN, and hence $\alpha =0.75$
is the best for RVIN. The two-stage methods generally produce  noisy results in the presence of SPIN and RVIN. 
\begin{figure}[t!]
	\begin{minipage}{\linewidth}
		\centering
		\resizebox{\textwidth}{!}{%
			\begin{tabular}{c@{}c@{}c@{}c@{}c@{}c}
				\subcaptionbox{40\% SPIN \label{fig:color_40_sp}}{\includegraphics[width = 1.00in]{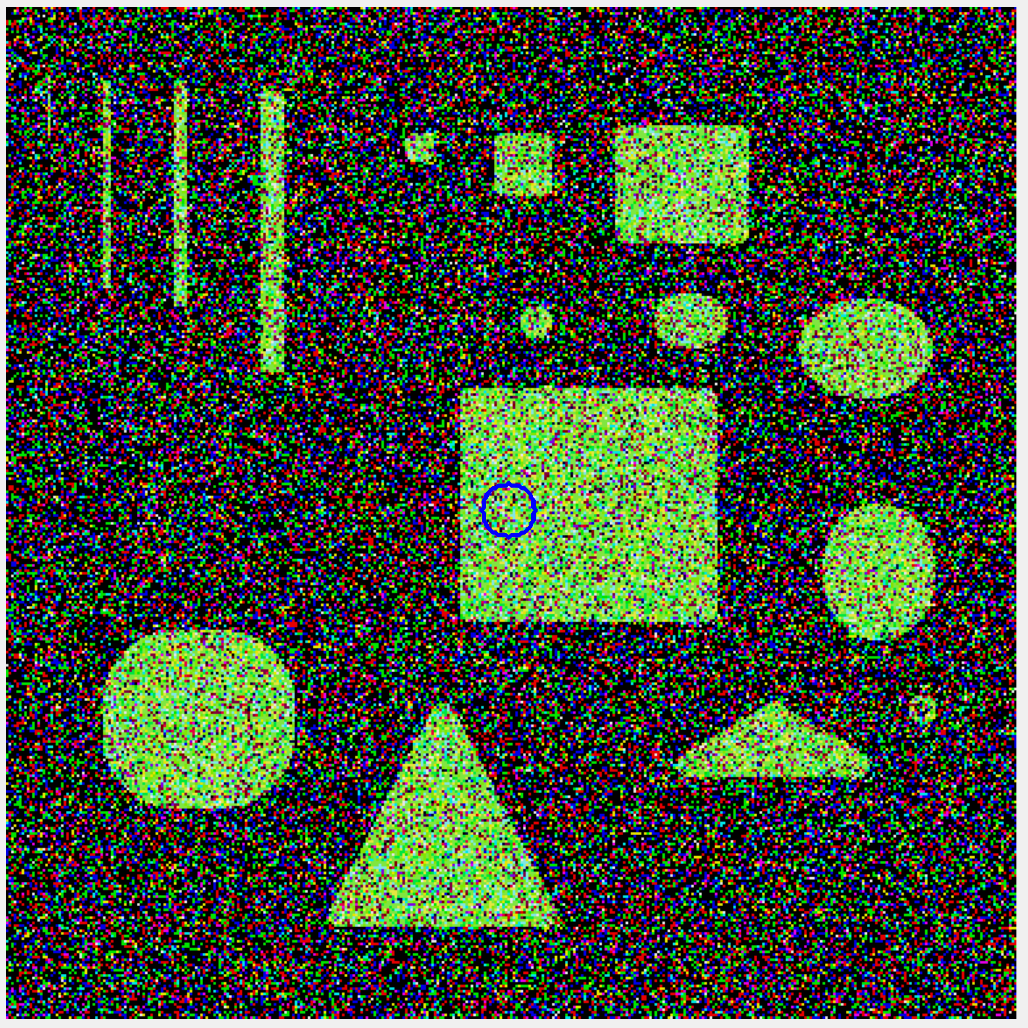}} & \subcaptionbox{$L_1-L_2$ CV}{\includegraphics[width = 1.00in]{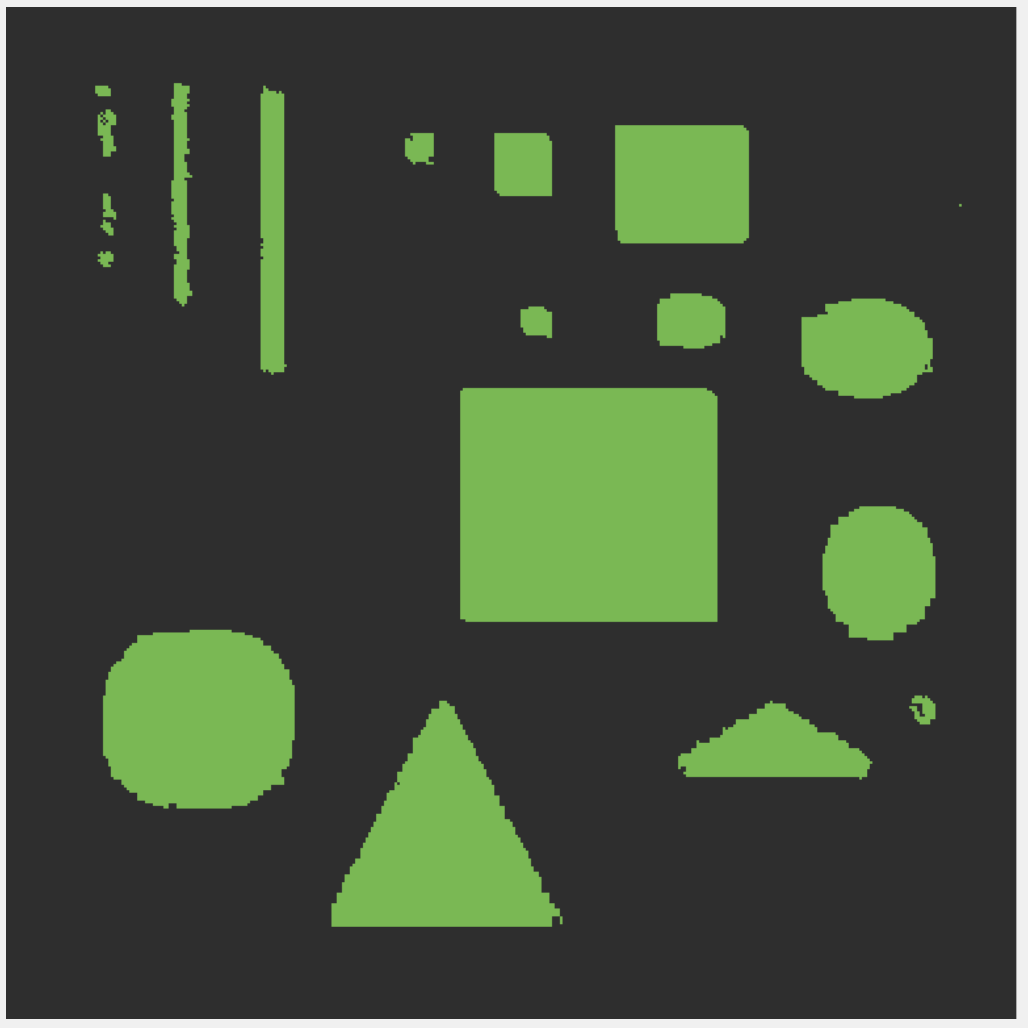}}  & \subcaptionbox{$L_1$ CV}{\includegraphics[width = 1.00in]{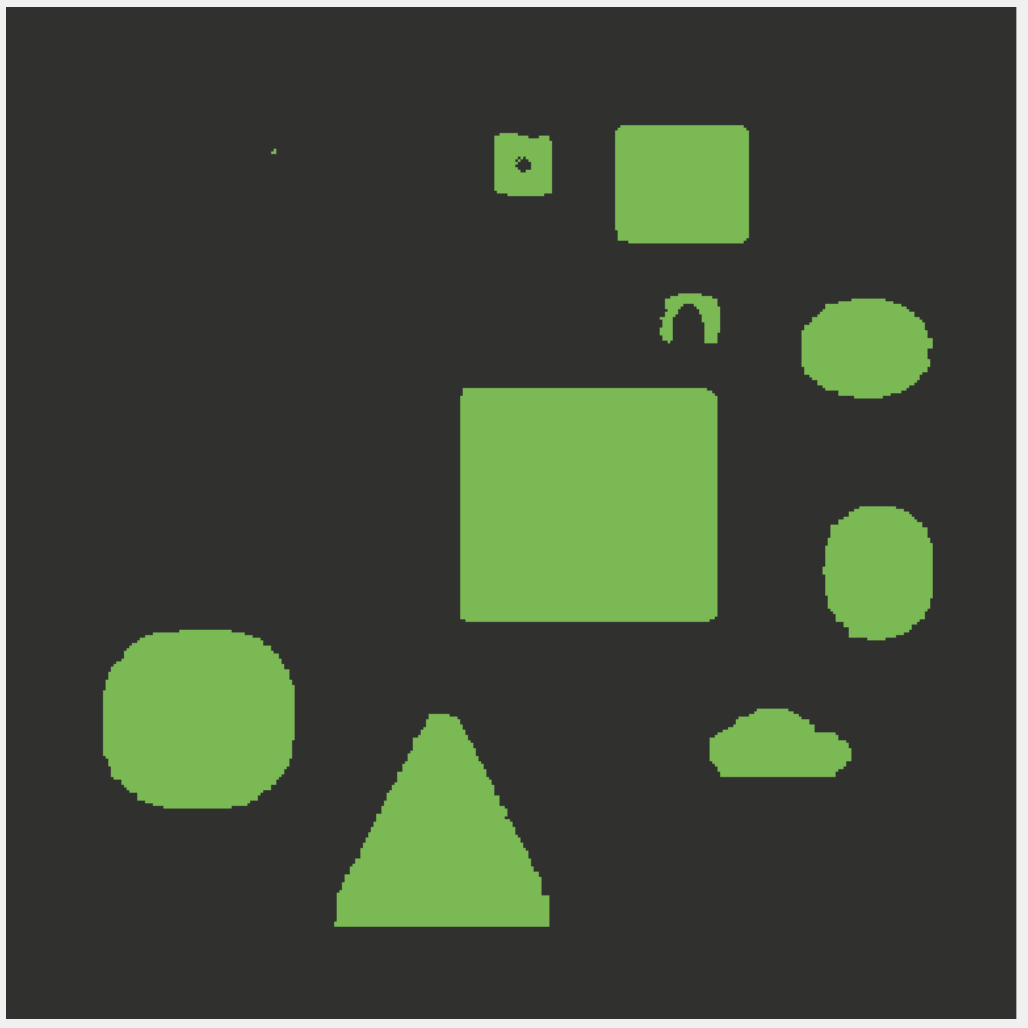}} & \subcaptionbox{$L_1 - 0.75L_2$ FR}{\includegraphics[width = 1.00in]{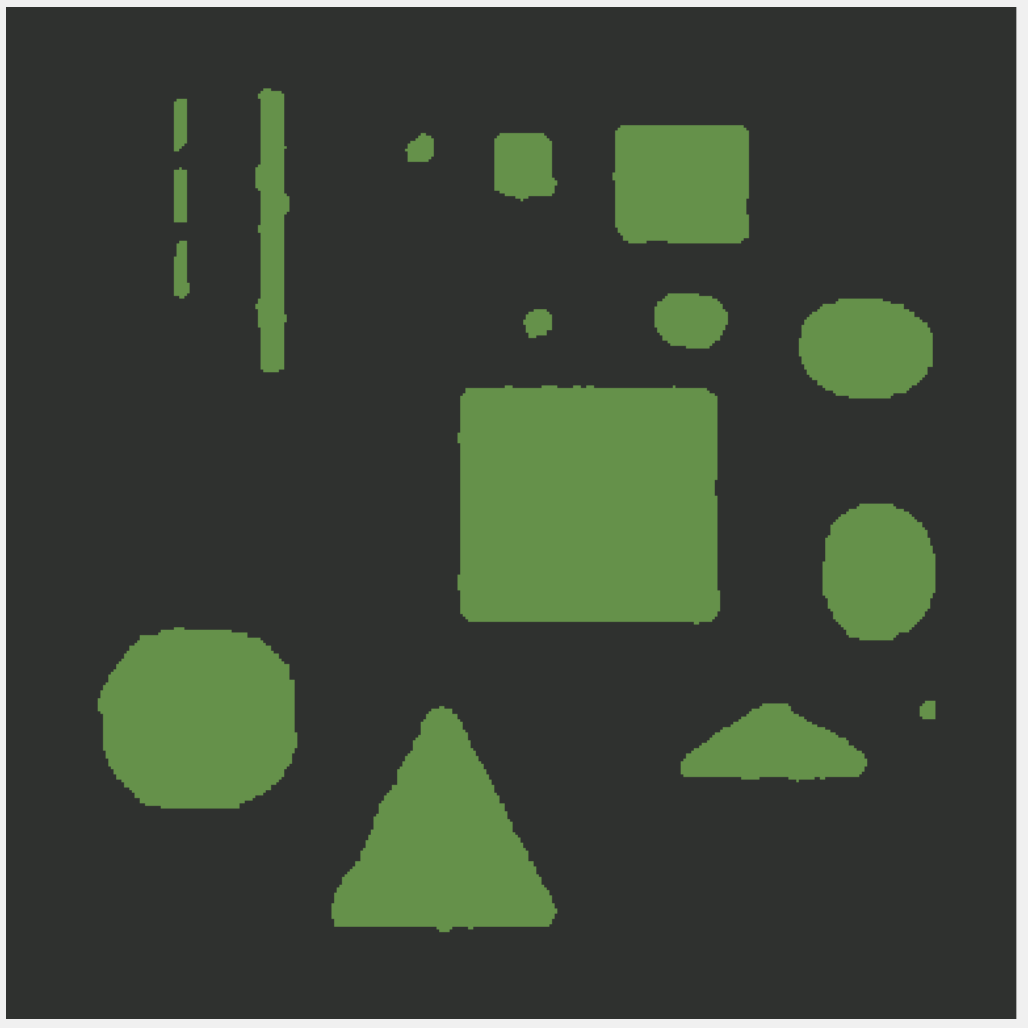}} & \subcaptionbox{$L_1$ FR}{\includegraphics[width = 1.00in]{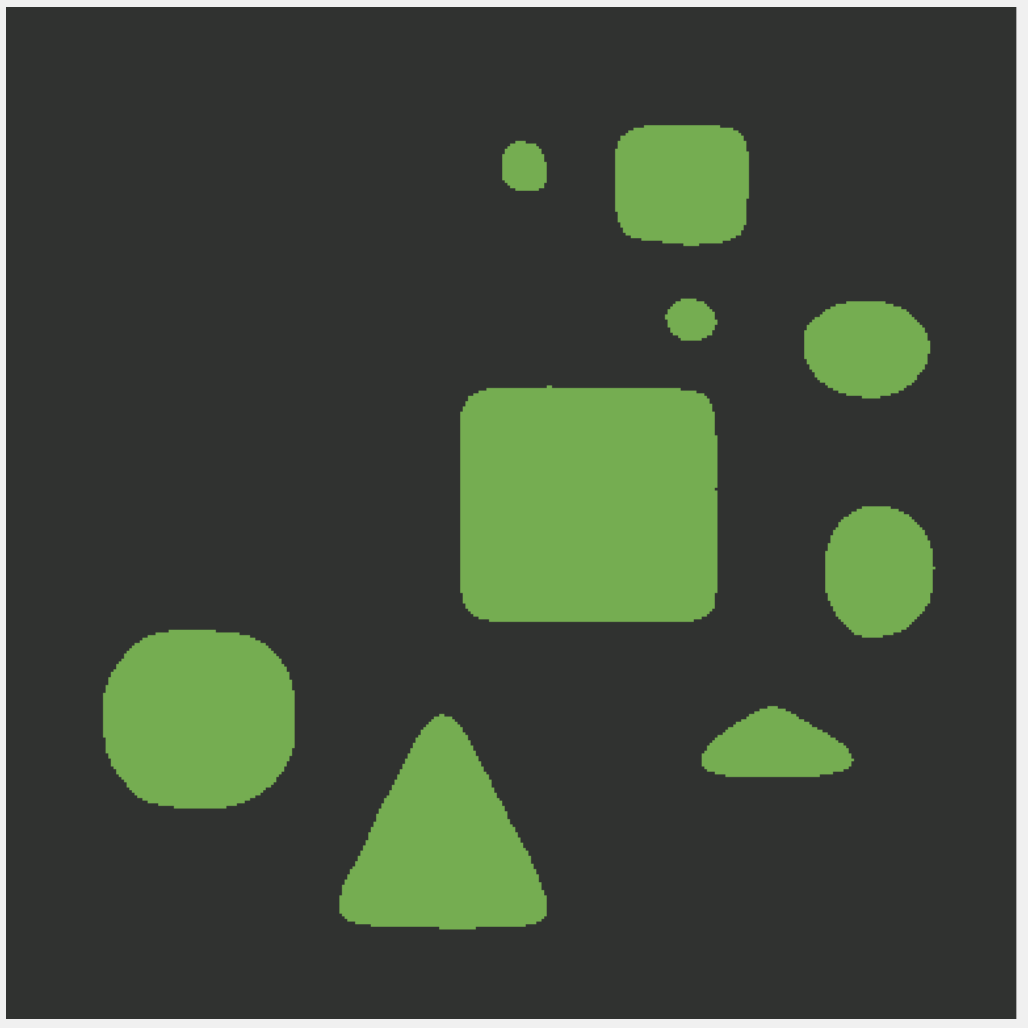}} &\subcaptionbox{$L_1+L_2^2$}{\includegraphics[width = 1.00in]{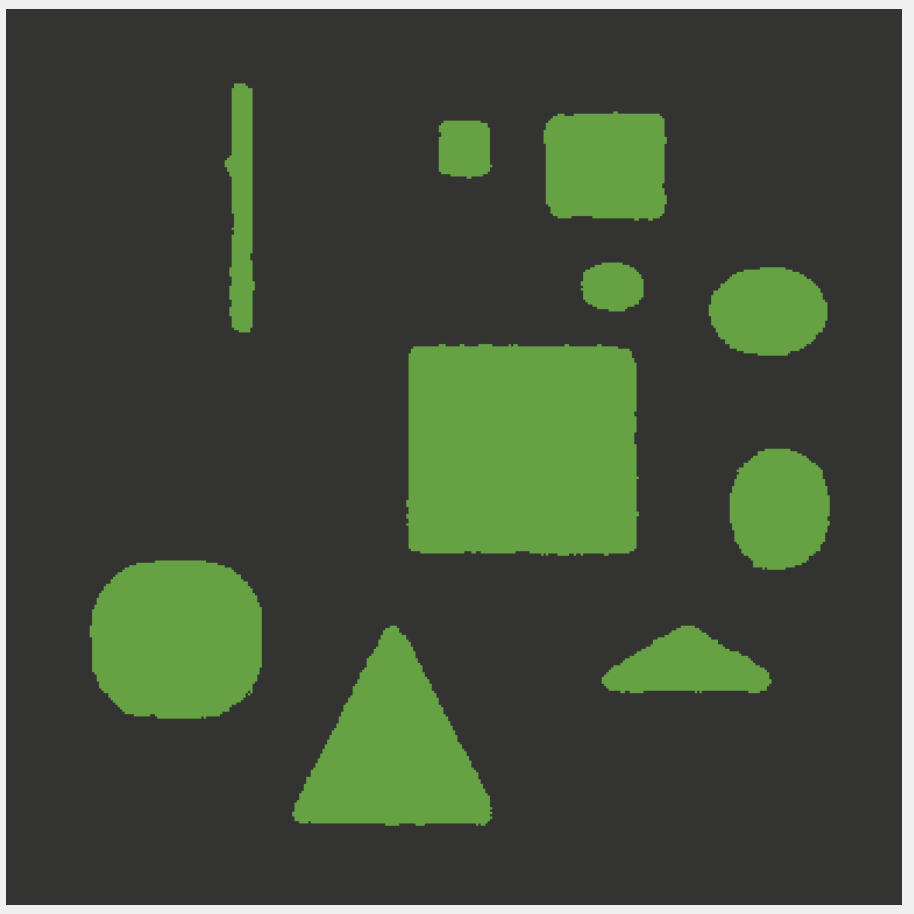}} \\
					\subcaptionbox{40\% RVIN \label{fig:color_rv_40}}{\includegraphics[width = 1.00in]{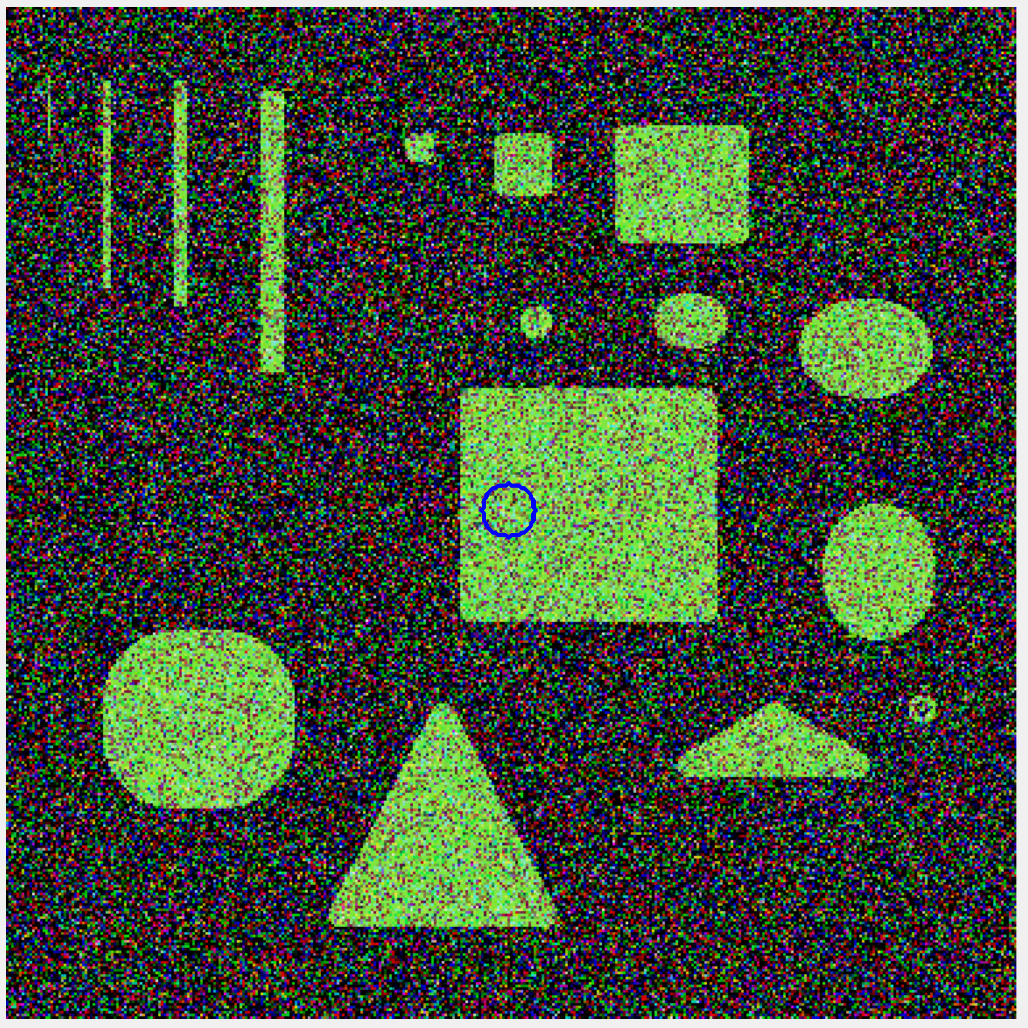}}&
					\subcaptionbox{$L_1-L_2$ CV}{\includegraphics[width = 1.00in]{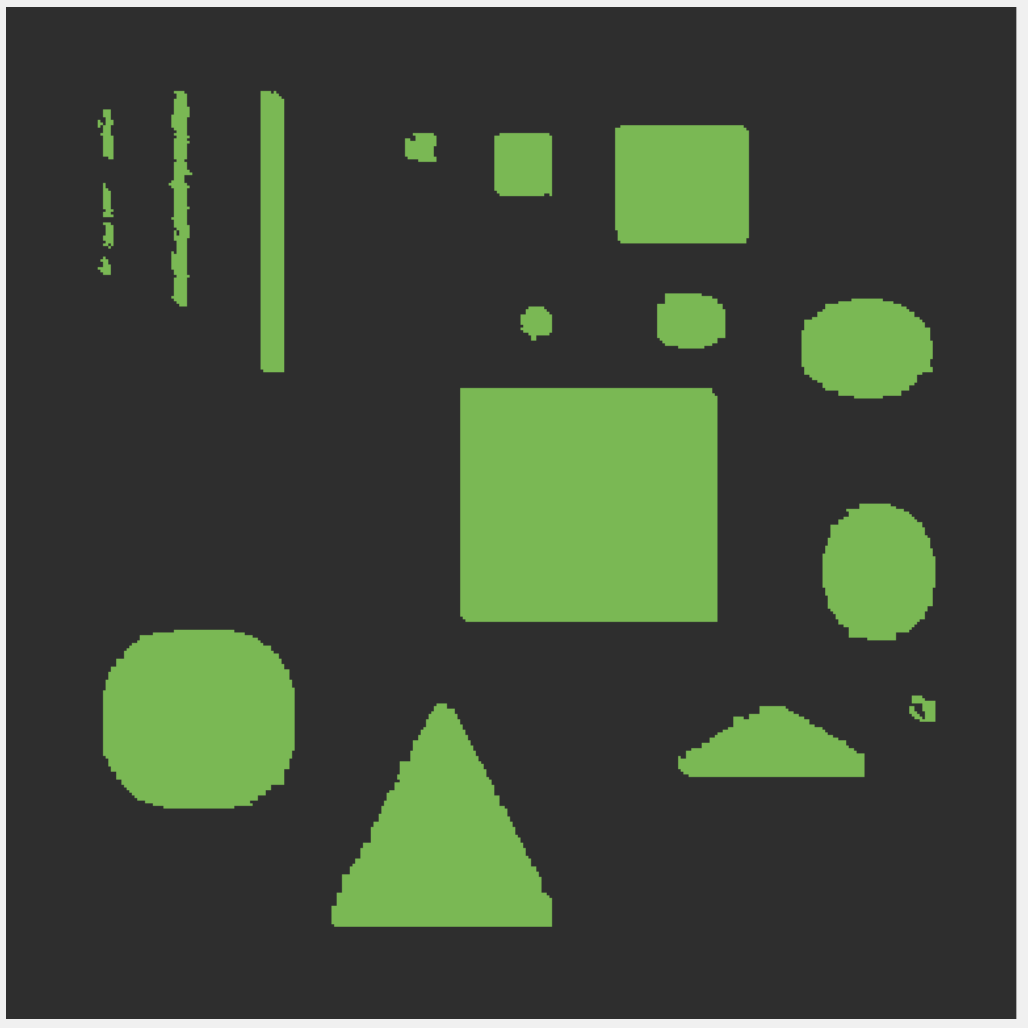}} & \subcaptionbox{$L_1$ CV}{\includegraphics[width = 1.00in]{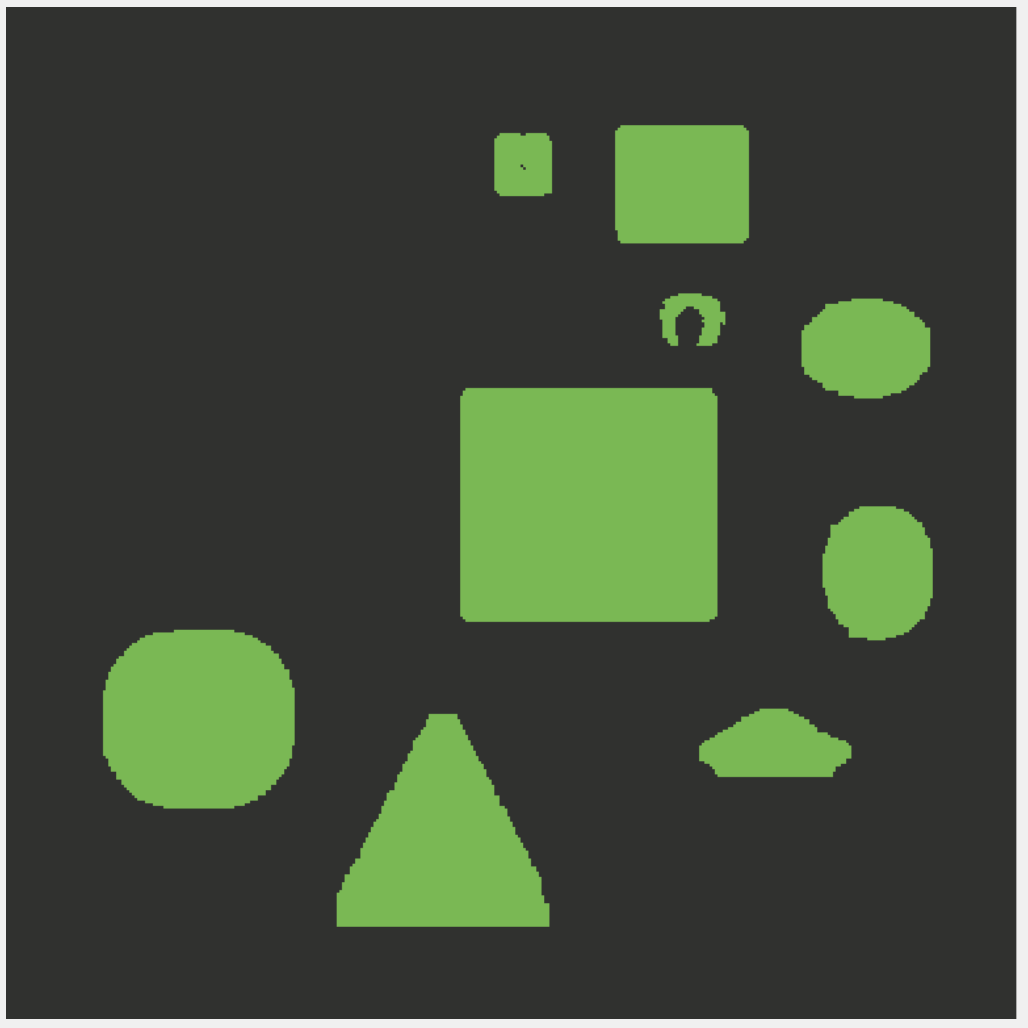}}  & \subcaptionbox{$L_1-0.75L_2$ FR}{\includegraphics[width = 1.00in]{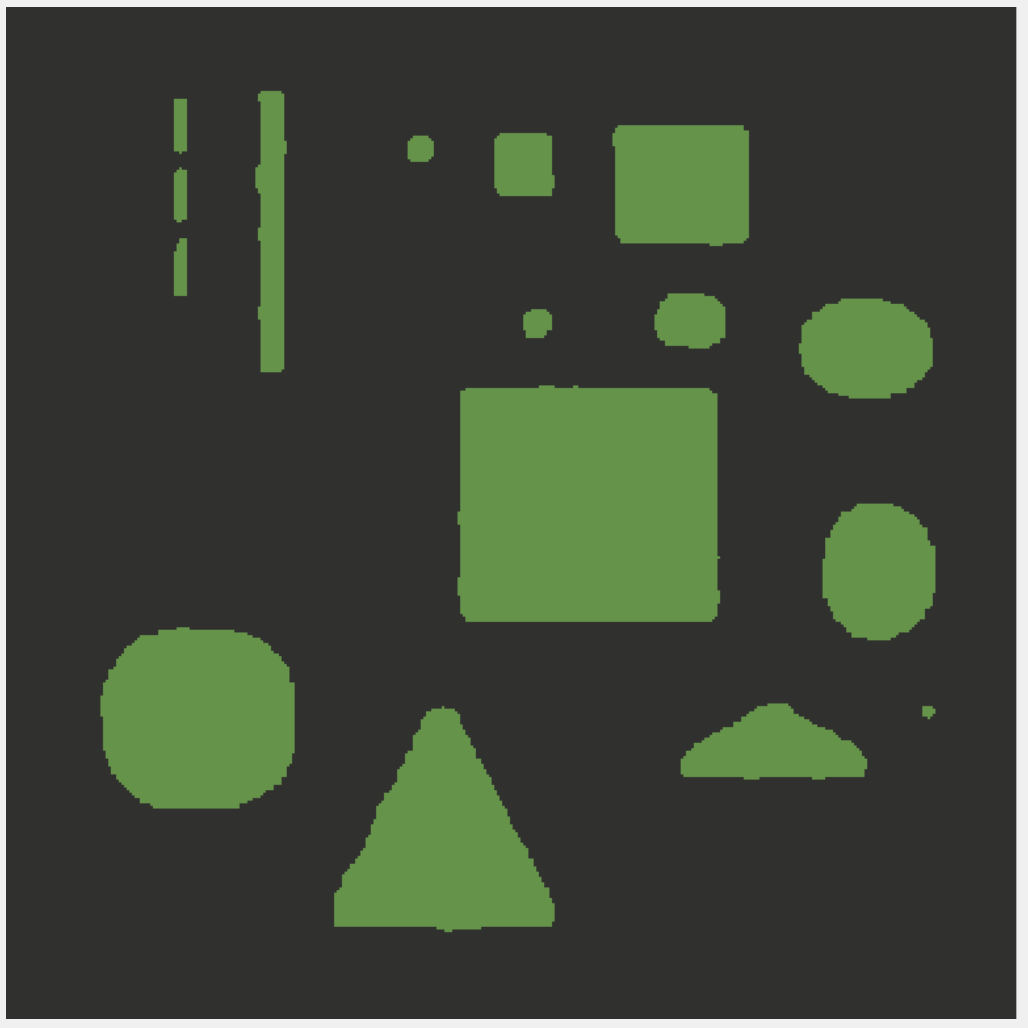}} & \subcaptionbox{$L_1$ FR}{\includegraphics[width = 1.00in]{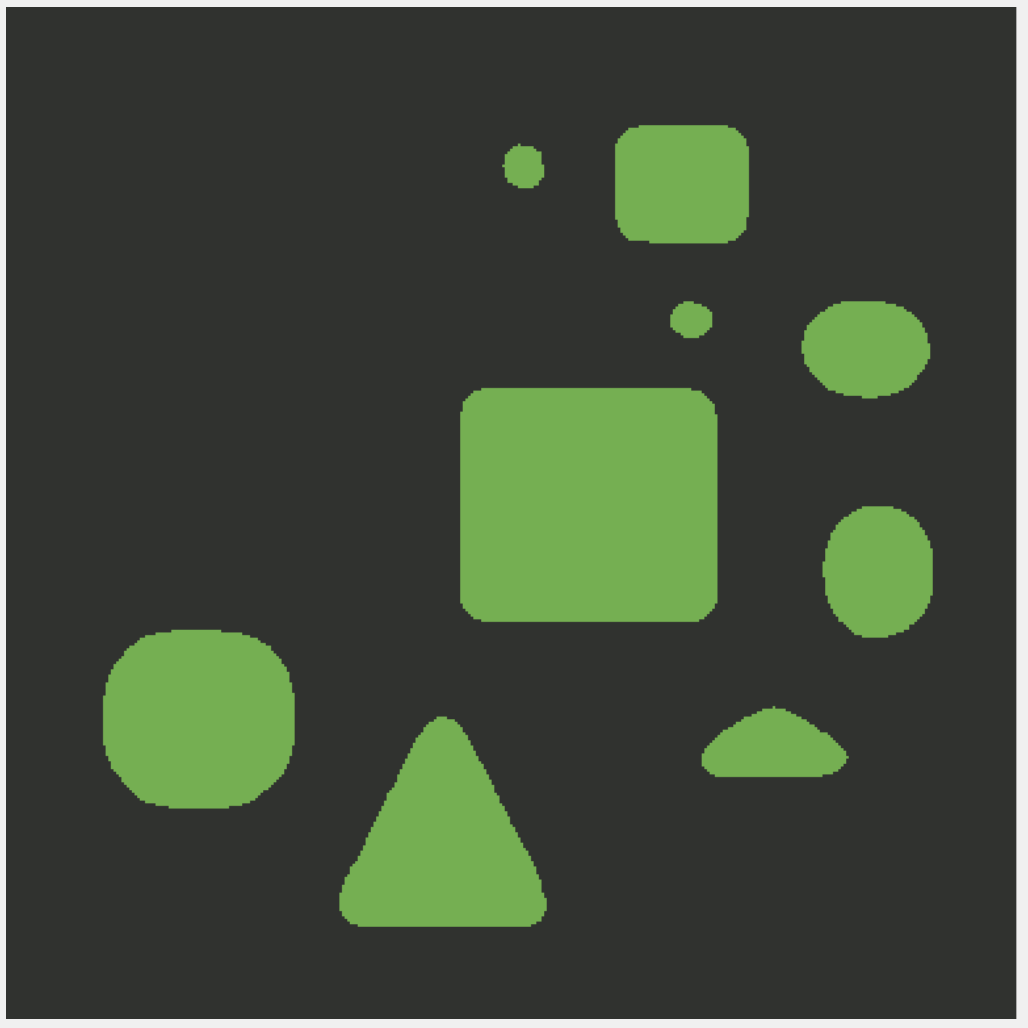}} &
				\subcaptionbox{$L_1+L_2^2$}{\includegraphics[width = 1.00in]{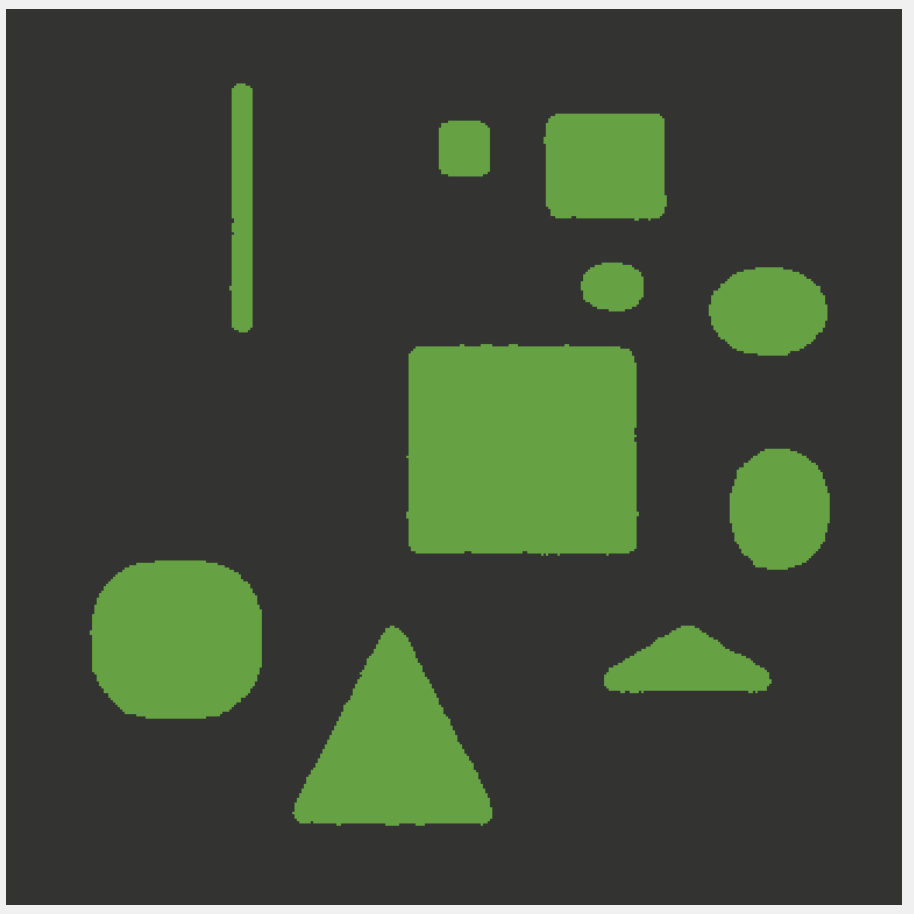}} 
		\end{tabular}}
			\caption{Reconstruction results on Figure \ref{fig:synthetic_2phase} corrupted with 40\% SPIN (top) and 40\% RVIN (bottom). }
					\label{fig:2phase_40_spin}
	\end{minipage}
\end{figure}
\begin{table}[th!]
		\caption{DICE indices of   various segmentation models applied to Figure \ref{fig:synthetic_2phase} corrupted with different levels of impulsive noise.}
	\begin{center}
		\begin{tabular}{l|c|c|c|c|c|c}\Xhline{5\arrayrulewidth}
			Salt \& Pepper (\%)  & 0 & 10 & 20 & 30 & 40 & 50 \\ \Xhline{5\arrayrulewidth}
			$L_1-L_2$ CV                 & \textbf{1}      & \textbf{0.9979} & 0.9952 & \textbf{0.9920} & \textbf{0.9867} & \textbf{0.9775} \\ \hline
			$L_1-0.75L_2$ CV             & 0.9994 & 0.9978 & \textbf{0.9957} & 0.9896 & 0.9856 & 0.9737 \\ \hline
			$L_1-0.5L_2$ CV              & 0.9992 & 0.9970 & 0.9910 & 0.9889 & 0.9826 & 0.9512 \\ \hline
			$L_1-0.25L_2$ CV             & 0.9982 & 0.9924 & 0.9904 & 0.9829 & 0.9726 & 0.9308 \\ \hline
			$L_1$ CV                    & 0.9938 & 0.9918 & 0.9808 & 0.9755 & 0.9457 & 0.9109 \\ \hline
			$L_1-L_2$ FR & 0.9977 & 0.9960 & 0.9931 & 0.9685 & 0.8187 & 0.7273 \\ \hline
			$L_1-0.75L_2$ FR & 0.9979 & 0.9955 & 0.9920 & 0.9873 & 0.9795 & 0.9626 \\ \hline
			$L_1-0.5L_2$ FR & 0.993  & 0.9908 & 0.9802 & 0.9720 & 0.9635 & 0.9409 \\ \hline
			$L_1-0.25L_2$ FR & 0.9818 & 0.9786 & 0.9690 & 0.9462 & 0.9441 & 0.9195 \\ \hline
			$L_1$ FR & 0.9774 & 0.9705 & 0.9524 & 0.9383 & 0.9301 & 0.8906 \\ \hline
			$L_1+L_2^2$ & 0.9931  & 0.9907 & 0.9874 & 0.9794 & 0.9726 & 0.9686 \\ \hline
			$L_0$ \cite{xu2011image} & \textbf{1}        & 0.8734 & 0.7687 & 0.6745 & 0.5945 & 0.4307 \\
			\hline
			$L_0$ \cite{storath2014fast} & 0.9939& 0.9904 & 0.9823 & 0.9762 & 0.9543 & 0.9266 \\
			\hline$R_{MS}$ & 0.9853& 0.9801 & 0.9676 & 0.9444 & 0.9116 & 0.8225 \\
			\hline
			\Xhline{5\arrayrulewidth}
			Random-valued (\%) & 0 & 10 & 20 & 30 & 40 & 50 \\ \Xhline{5\arrayrulewidth}
			$L_1-L_2$ CV                 & \textbf{1}        & \textbf{0.9987} & \textbf{0.9966} & \textbf{0.9932} & \textbf{0.9887} & \textbf{0.9826} \\ \hline
			$L_1-0.75L_2$ CV             & 0.9994 & 0.9983 & 0.9960 & 0.9915 & 0.9877 & 0.9759 \\ \hline
			$L_1-0.5L_2$ CV              & 0.9992 & 0.9975 & 0.9916 & 0.9899 & 0.9815 & 0.9535 \\ \hline
			$L_1-0.25L_2$ CV             & 0.9982 & 0.9928 & 0.9913 & 0.9784 & 0.9748 & 0.9344 \\ \hline
			$L_1$ CV                    & 0.9938 & 0.9920 & 0.9798 & 0.9773 & 0.9493 & 0.9145 \\ \hline
			$L_1-L_2$ FR & 0.9977 & 0.9965 & 0.9943 & 0.9902 & 0.9071 & 0.7154 \\ \hline
			$L_1-0.75L_2$ FR & 0.9979 & 0.9960 & 0.9921 & 0.9879 & 0.9815 & 0.9520 \\ \hline
			$L_1-0.5L_2$ FR & 0.993  & 0.9907 & 0.9797 & 0.9742 & 0.9644 & 0.9526 \\ \hline
			$L_1-0.25L_2$ FR & 0.9818 & 0.9781 & 0.9702 & 0.9620 & 0.9534 & 0.9161 \\ \hline
			$L_1$ FR & 0.9774 & 0.9656 & 0.9533 & 0.9519 & 0.9316 & 0.8770 \\ \hline
			$L_1+L_2^2$ & 0.9931  & 0.9912 & 0.9877 & 0.9812 & 0.9755 & 0.9726 \\ \hline
			$L_0$ \cite{xu2011image} &\textbf{1}        & 0.9032 & 0.7991 & 0.6972 & 0.6089 & 0.5312\\ \hline
			$L_0$ \cite{storath2014fast} & 0.9939& 0.9852 & 0.9846 & 0.9786 & 0.9573 & 0.9298 \\
			\hline
			$R_{MS}$ & 0.9853& 0.9797 & 0.9782 & 0.9465 & 0.9074 & 0.8260 \\ \hline
		\end{tabular}
	\end{center}
	\label{tab:twophase_color}
\end{table}

\begin{table}
		\captionof{table}{DICE indices of   various segmentation models applied to Figure \ref{fig:synthetic_4phase} corrupted with different levels of impulsive noise.}
		\begin{center}
			\begin{tabular}{l|c|c|c|c|c}\Xhline{5\arrayrulewidth}
				Salt \& Pepper (\%)   & 0 & 10 & 20 & 30 & 40 \\ \Xhline{5\arrayrulewidth}
				$L_1-L_2$ CV                 & 0.9990 & 0.9762 & 0.9524 & 0.9245 & 0.8548 \\ \hline
				$L_1-0.75 L_2$ CV             & 0.9992 & 0.9763 & 0.9649 & 0.9288 & 0.8978 \\ \hline
				$L_1-0.5 L_2$ CV              & 0.9992 & 0.9789 & 0.9704 & 0.9509 & 0.9292 \\ \hline
				$L_1-0.25 L_2$ CV             & 0.9994 & 0.9852 & 0.9686 & 0.9608 & 0.9448 \\ \hline
				$L_1$ CV                    & 0.9987 & 0.9832 & 0.9788 & 0.9597 & 0.9496 \\ \hline
				$L_1-L_2$ FR             & 0.9994 & 0.7869 & 0.6566 & 0.5424 & 0.4552 \\ \hline
				$L_1-0.75L_2$  FR         & 0.9994 & 0.9328 & 0.8736 & 0.8058 & 0.6541 \\ \hline
				$L_1-0.5L_2$ FR  & 0.9980 & 0.9905 & 0.9847 & 0.9720 & 0.8976 \\ \hline
				$L_1-0.25 L_2$ FR & 0.9976 & 0.9921 & 0.9863 & 0.9801 & \textbf{0.9753} \\ \hline
				$L_1$ FR                & 0.9976 & \textbf{0.9924} & \textbf{0.9869} & \textbf{0.9804} & 0.9474 \\ \hline
				$L_1+L_2^2$ & 0.9984 & 0.9904 & 0.9691 & 0.8984 & 0.7562 \\ \hline
				$L_0$          \cite{xu2011image}             & \textbf{1} & 0.7611 & 0.6284 & 0.5134 & 0.4225\\ \hline			$L_0$ \cite{storath2014fast} & 0.9997& 0.9245 & 0.7977 & 0.6536 & 0.4884 \\\hline			$R_{MS}$ & \textbf{1} & 0.9900 & 0.9771 & 0.9649 & 0.9575 \\
			\hline
				\Xhline{5\arrayrulewidth}
				Random-valued (\%)   & 0      & 10    & 20    & 30    & 40    \\ \Xhline{5\arrayrulewidth}
				$L_1-L_2$ CV                 & 0.9990  & 0.9895 & 0.9757 & 0.9594 & 0.9261 \\ \hline
				$L_1-0.75 L_2$ CV             & 0.9992 & 0.9910 & 0.9831 & 0.9755 & 0.9664 \\ \hline
				$L_1-0.5L_2$ CV              & 0.9992 & 0.9934 & 0.9875 & 0.9797 & 0.9737 \\ \hline
				$L_1-0.25L_2$ CV             & 0.9994 & 0.9934 & 0.9876 & 0.9798 & 0.9771 \\ \hline
				$L_1$ CV                    & 0.9987 & 0.9941 & 0.9884 & 0.9789 & 0.9761 \\ \hline
				$L_1-L_2$ FR & 0.9994 & 0.8841 & 0.7118 & 0.6604 & 0.5972 \\ \hline
				$L_1-0.75L_2$ FR & 0.9994 & 0.9916 & 0.9875 & 0.9353 & 0.8790 \\ \hline
				$L_1-0.5L_2$ FR & 0.998  & 0.9947 & \textbf{0.9912} & \textbf{0.9851} & \textbf{0.9833} \\ \hline
				$L_1-0.25L_2$ FR & 0.9976 & 0.9942 & \textbf{0.9912} & 0.9849 & 0.9821 \\ \hline
				$L_1$ FR & 0.9976 & 0.9921 & 0.9892 & \textbf{0.9851} & 0.9553 \\ \hline
				$L_1+L_2^2$ & 0.9984 & 0.9949 & 0.9857 & 0.9803 & 0.9705\\ \hline
				$L_0$         \cite{xu2011image}              & \textbf{1}      & 0.7744 & 0.6932 & 0.5302 & 0.4478 \\ \hline			$L_0$ \cite{storath2014fast} & 0.9997& 0.9828 & 0.9614 & 0.9482 & 0.9311 \\ \hline
				$R_{MS}$ & \textbf{1} & \textbf{0.9953} & 0.9900 & 0.9849 & 0.9831 \\ \hline
			\end{tabular}
		\end{center}
		\label{tab:fourphase_color}
		\end{table}
		
\begin{figure}[h!!!!]
	\begin{minipage}{\linewidth}
		\centering
		\resizebox{\textwidth}{!}{%
			\begin{tabular}{c@{}c@{}c@{}c@{}c@{}c}
				\subcaptionbox{40\% SPIN}{\includegraphics[width = 1.00in]{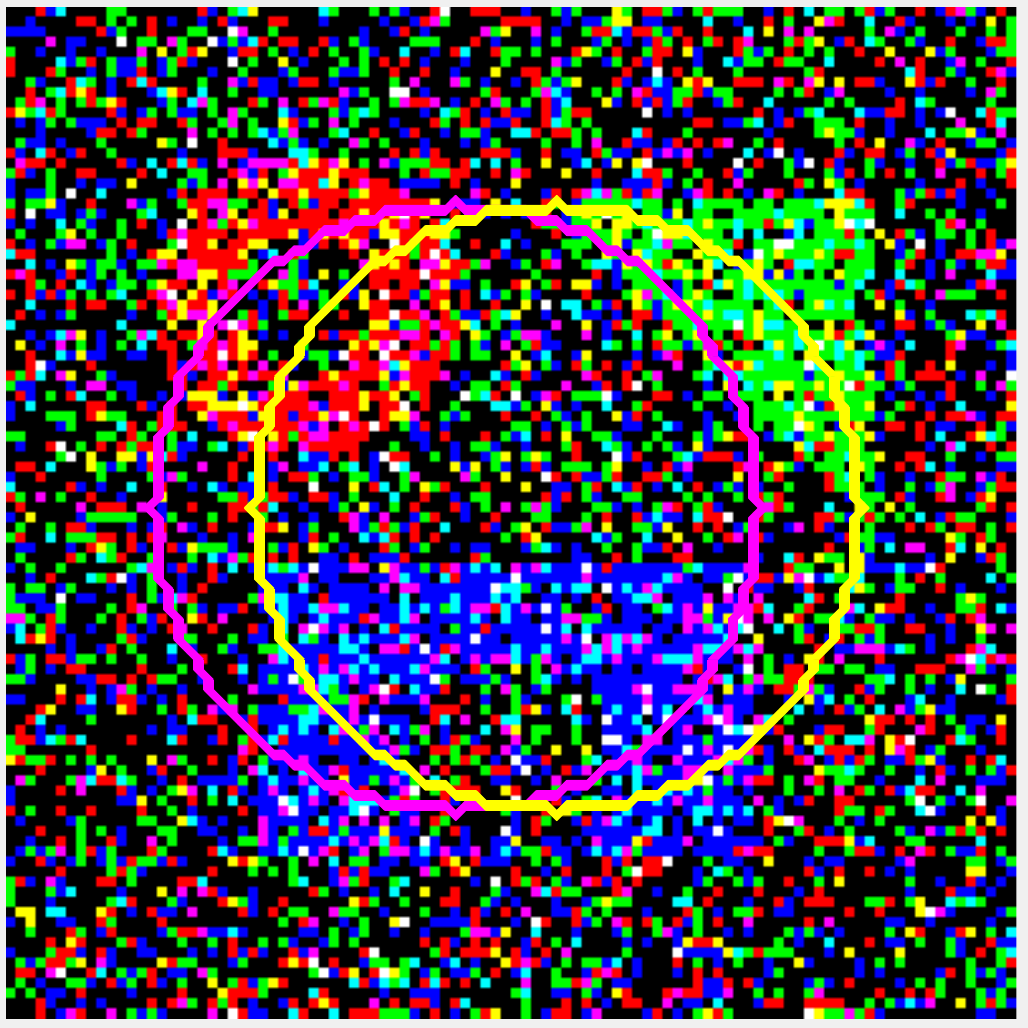}}& \subcaptionbox{$L_1-0.25L_2$ CV}{\includegraphics[width = 1.00in]{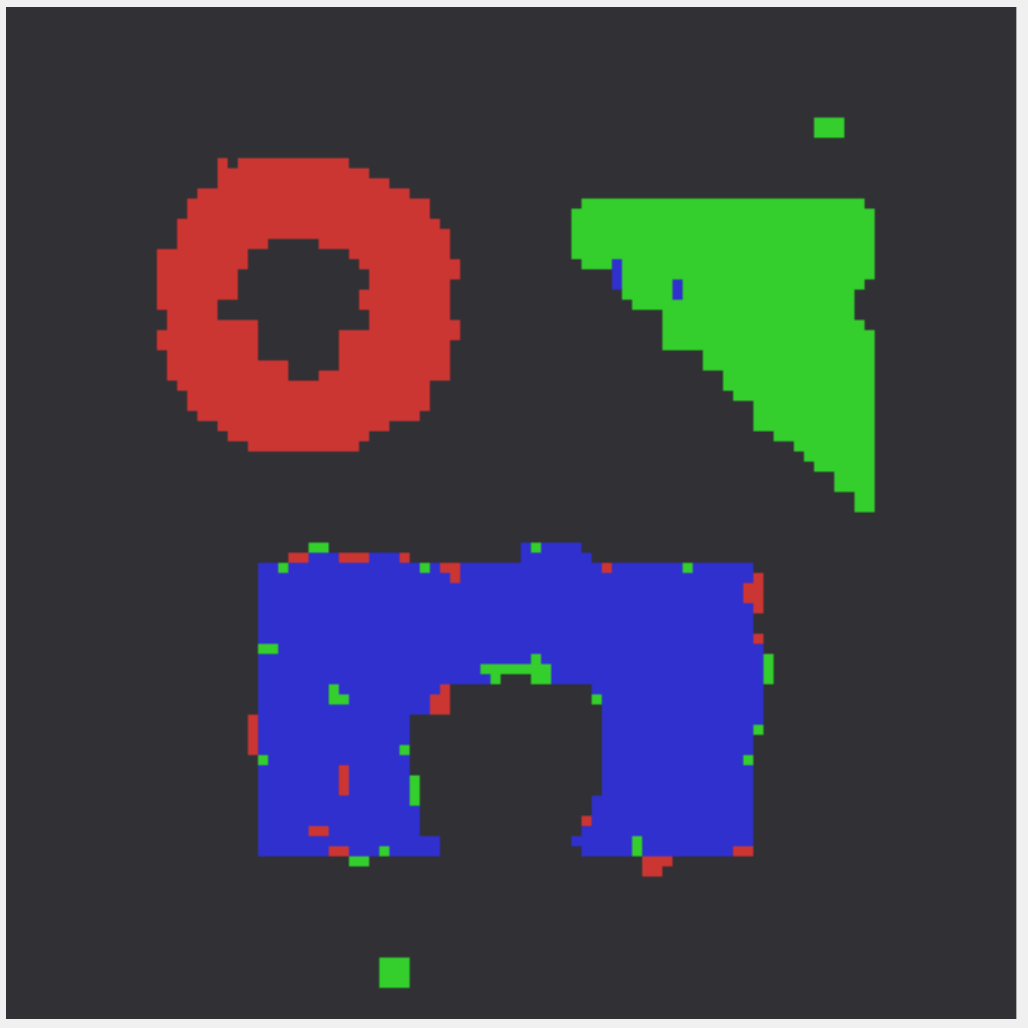}} & \subcaptionbox{$L_1$ CV}{\includegraphics[width = 1.00in]{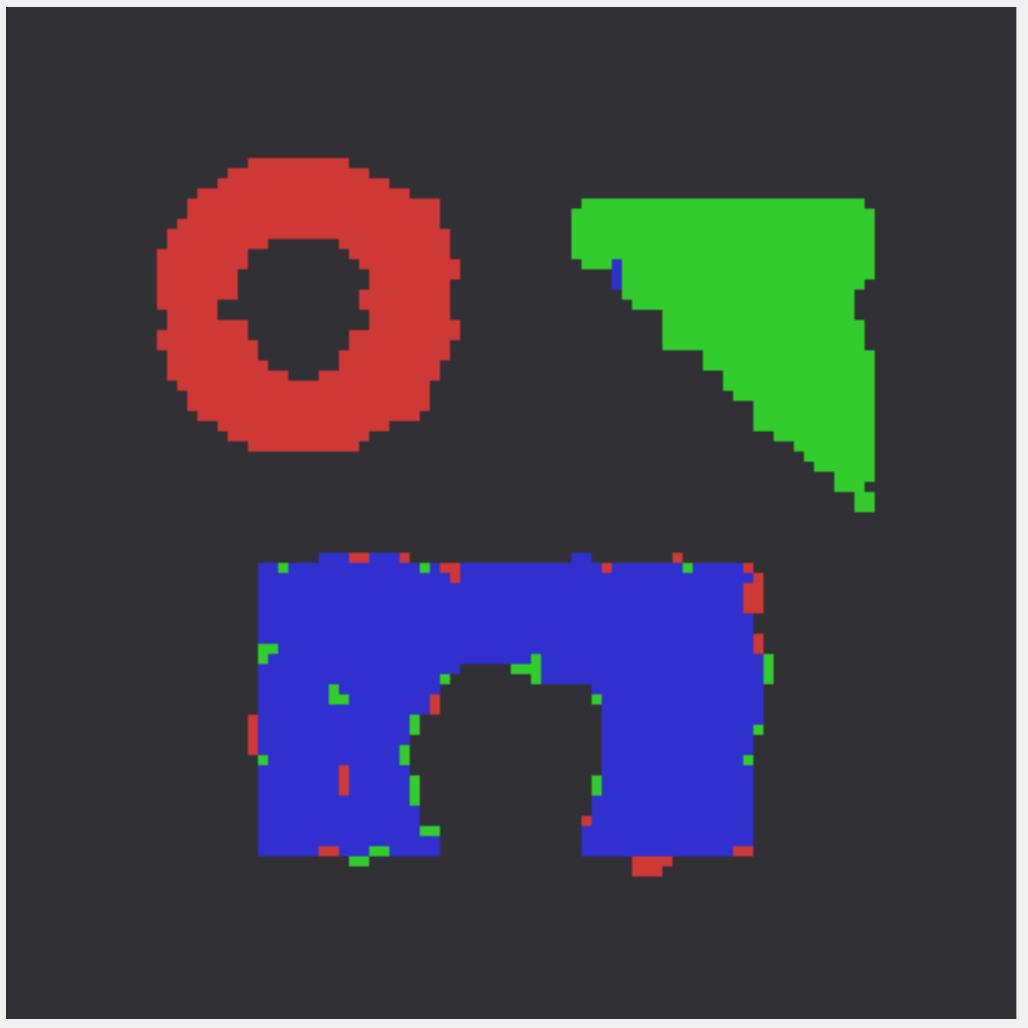}}  & \subcaptionbox{$L_1-0.25L_2$ FR}{\includegraphics[width = 1.00in]{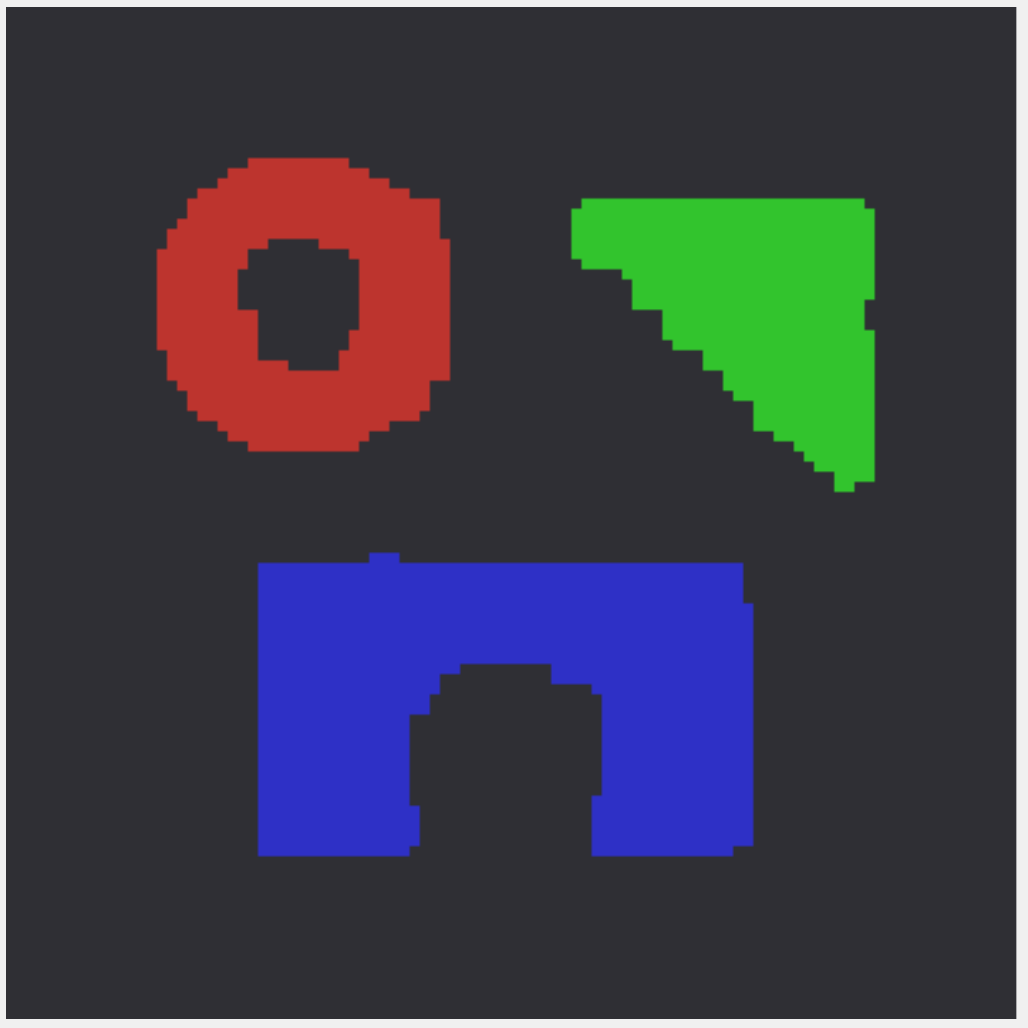}} & \subcaptionbox{$L_1$ FR}{\includegraphics[width = 1.00in]{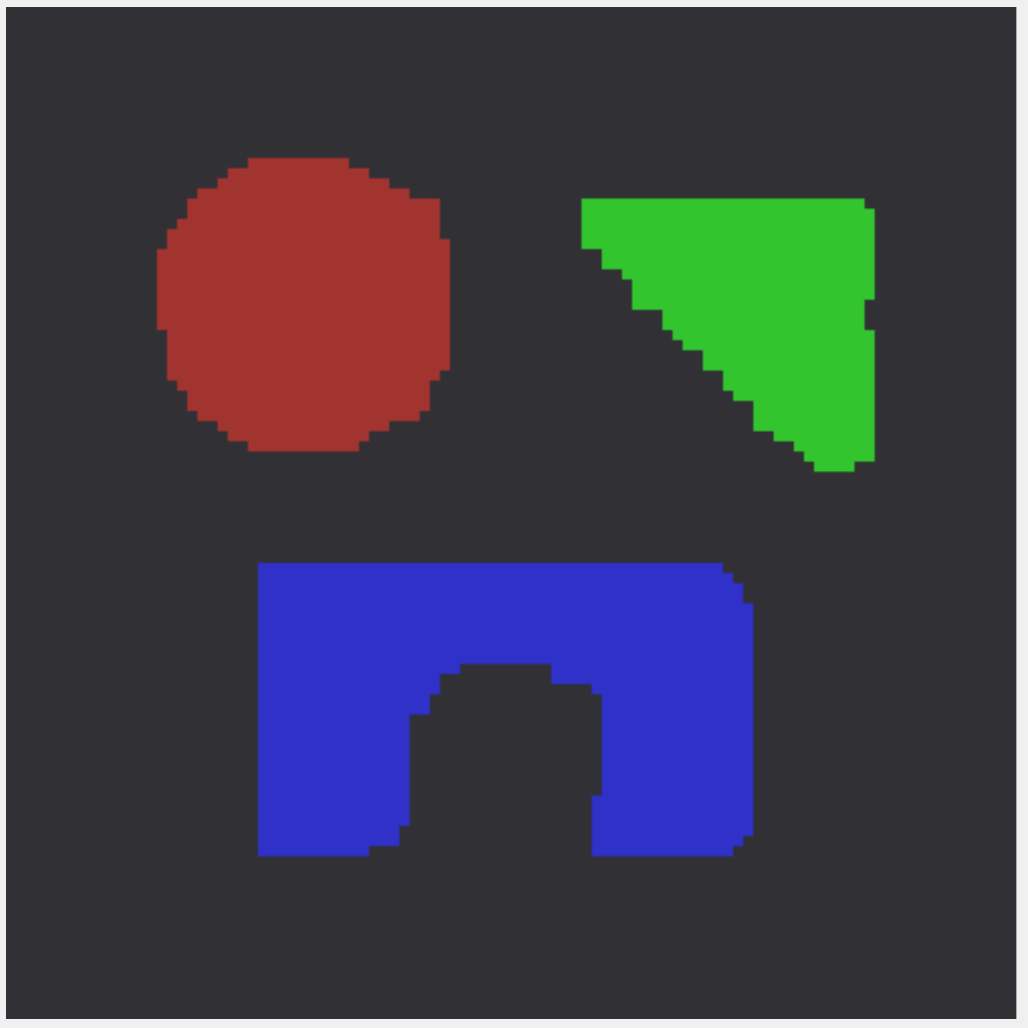}} &
				\subcaptionbox{$R_{MS}$}{\includegraphics[width = 1.00in]{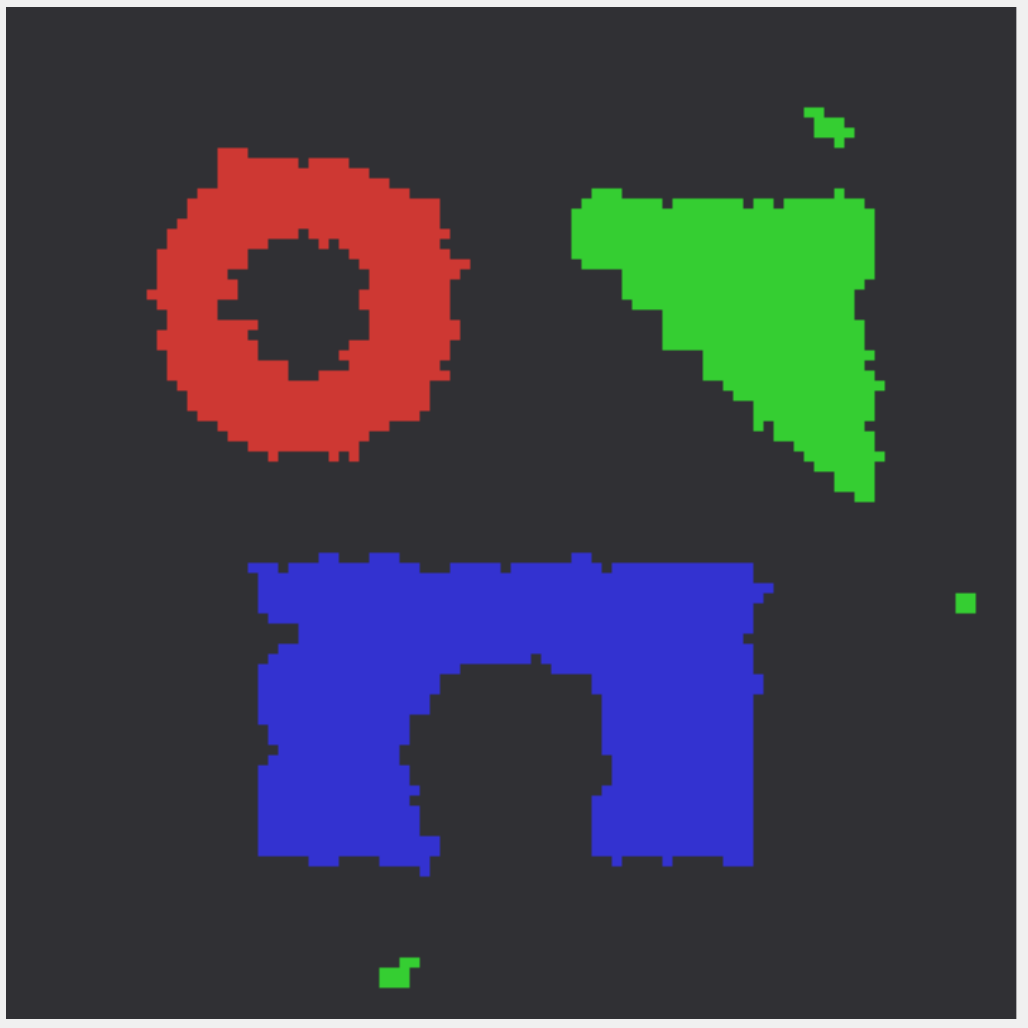}}\\
				\subcaptionbox{40\% RVIN}{\includegraphics[width = 1.00in]{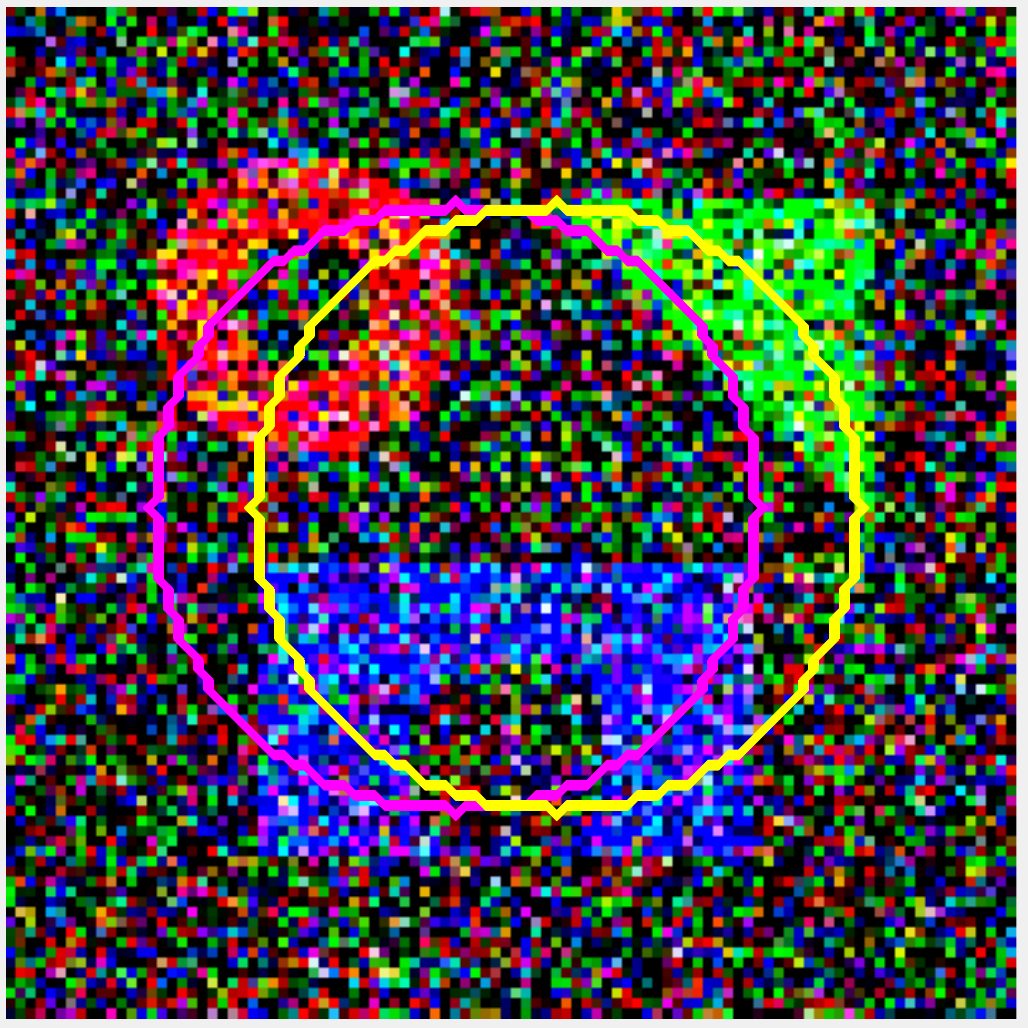}}& \subcaptionbox{$L_1-0.25L_2$ CV}{\includegraphics[width = 1.00in]{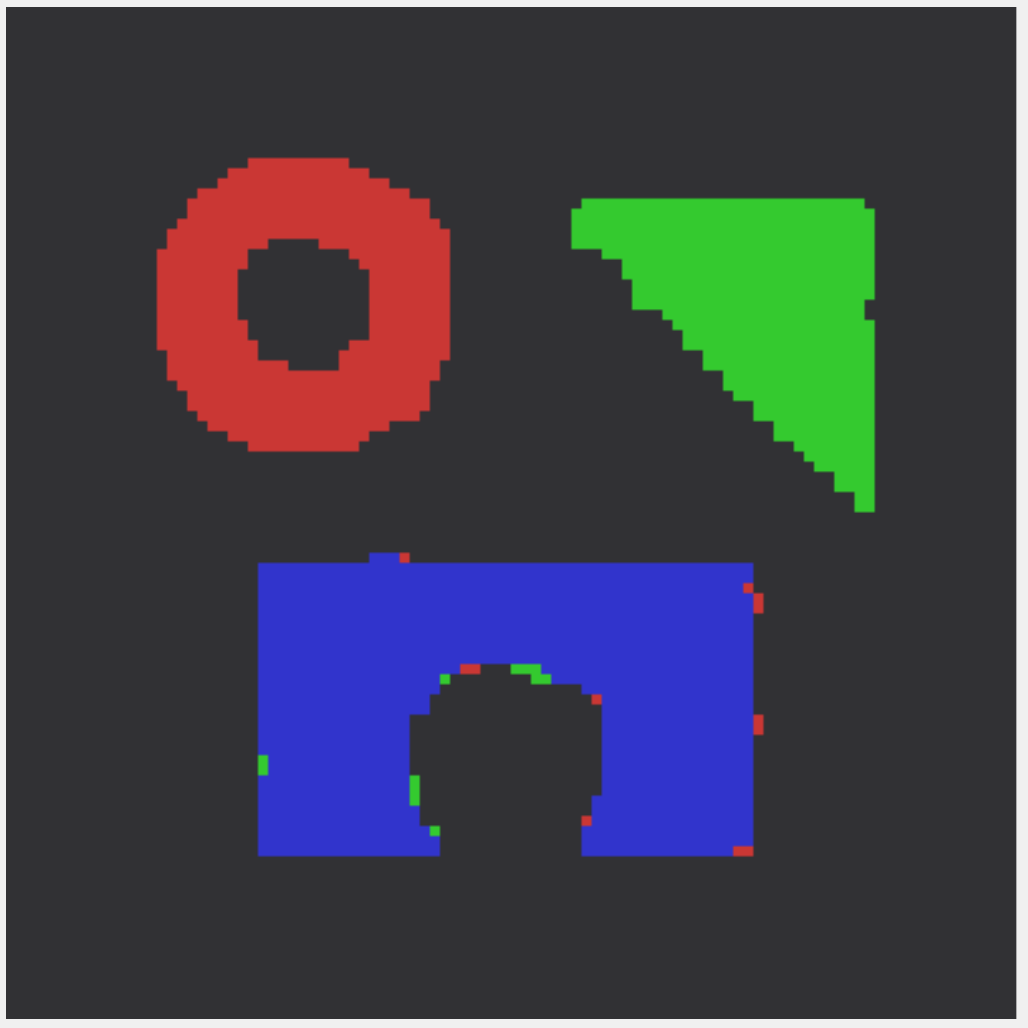}} & \subcaptionbox{$L_1$ CV}{\includegraphics[width = 1.00in]{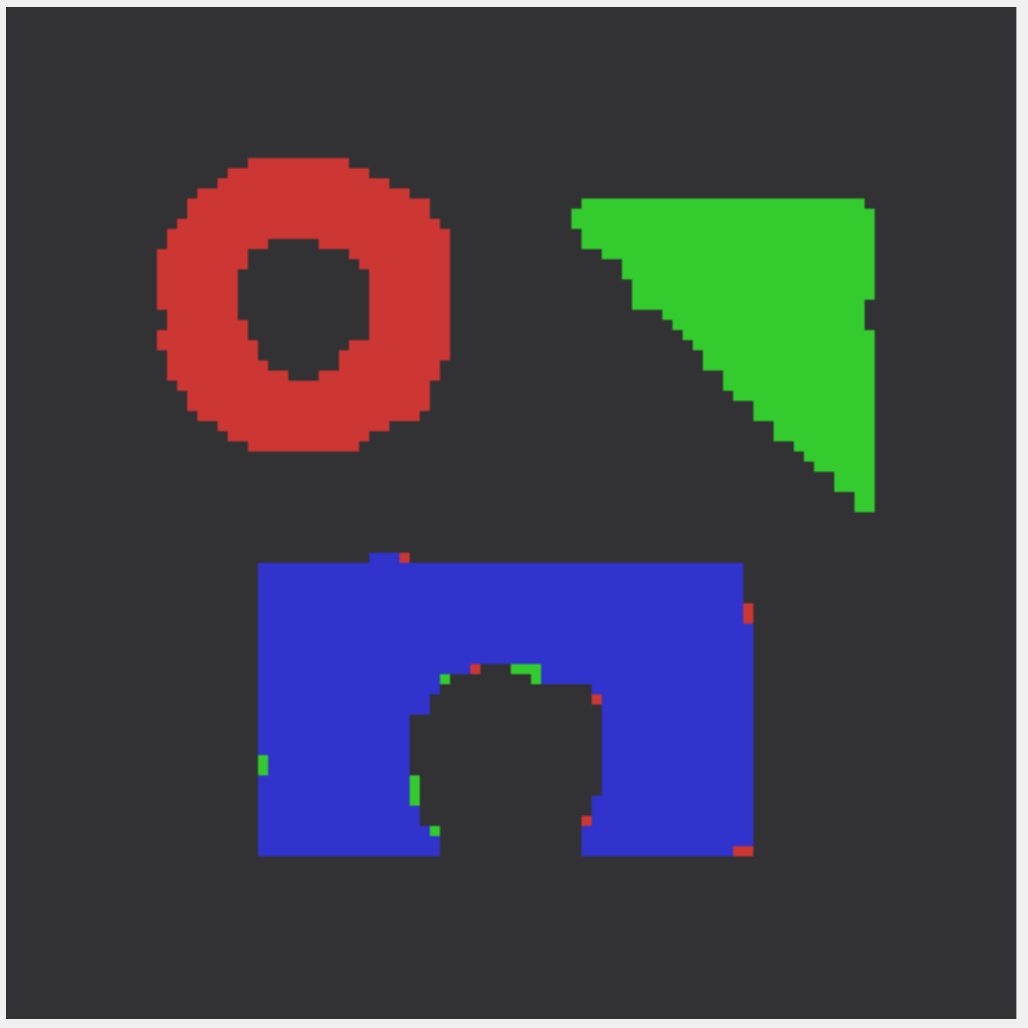}}  & \subcaptionbox{$L_1-0.5L_2$ FR}{\includegraphics[width = 1.00in]{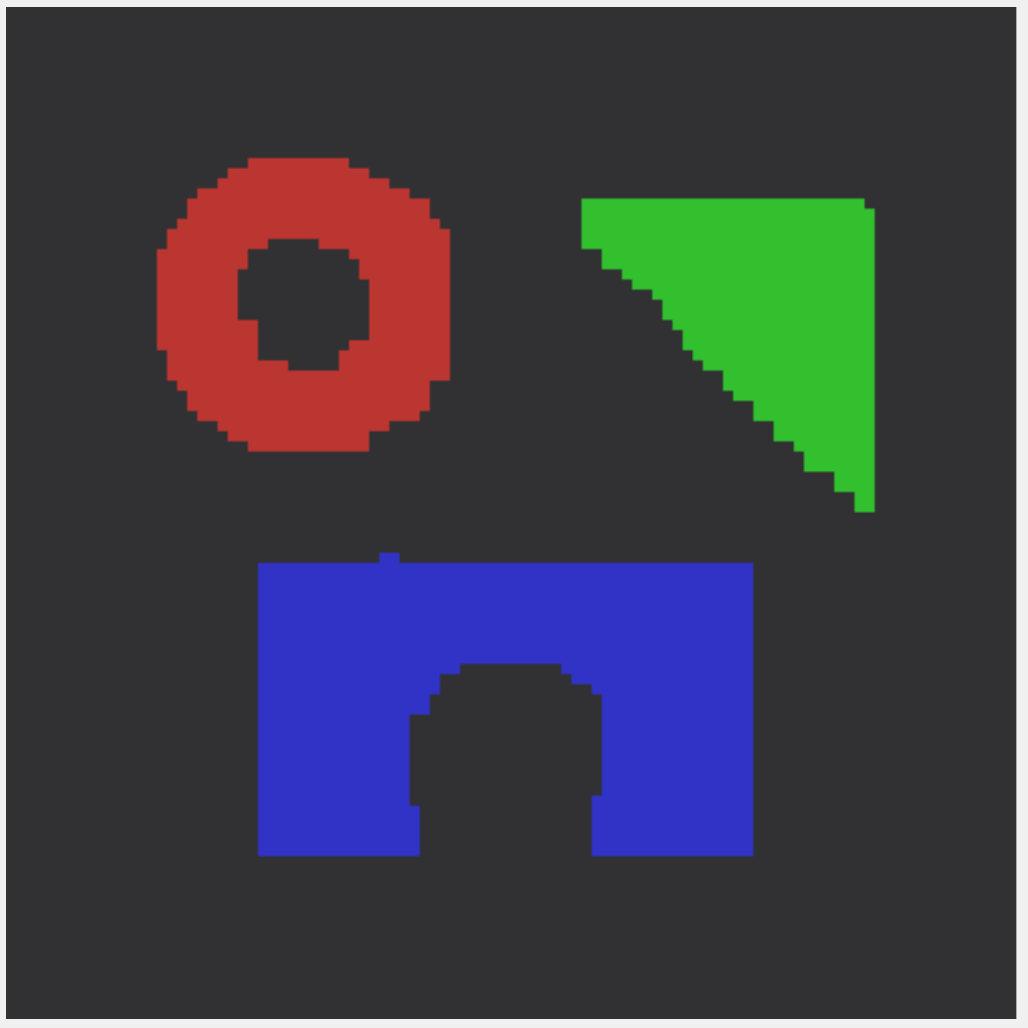}} & \subcaptionbox{$L_1$ FR}{\includegraphics[width = 1.00in]{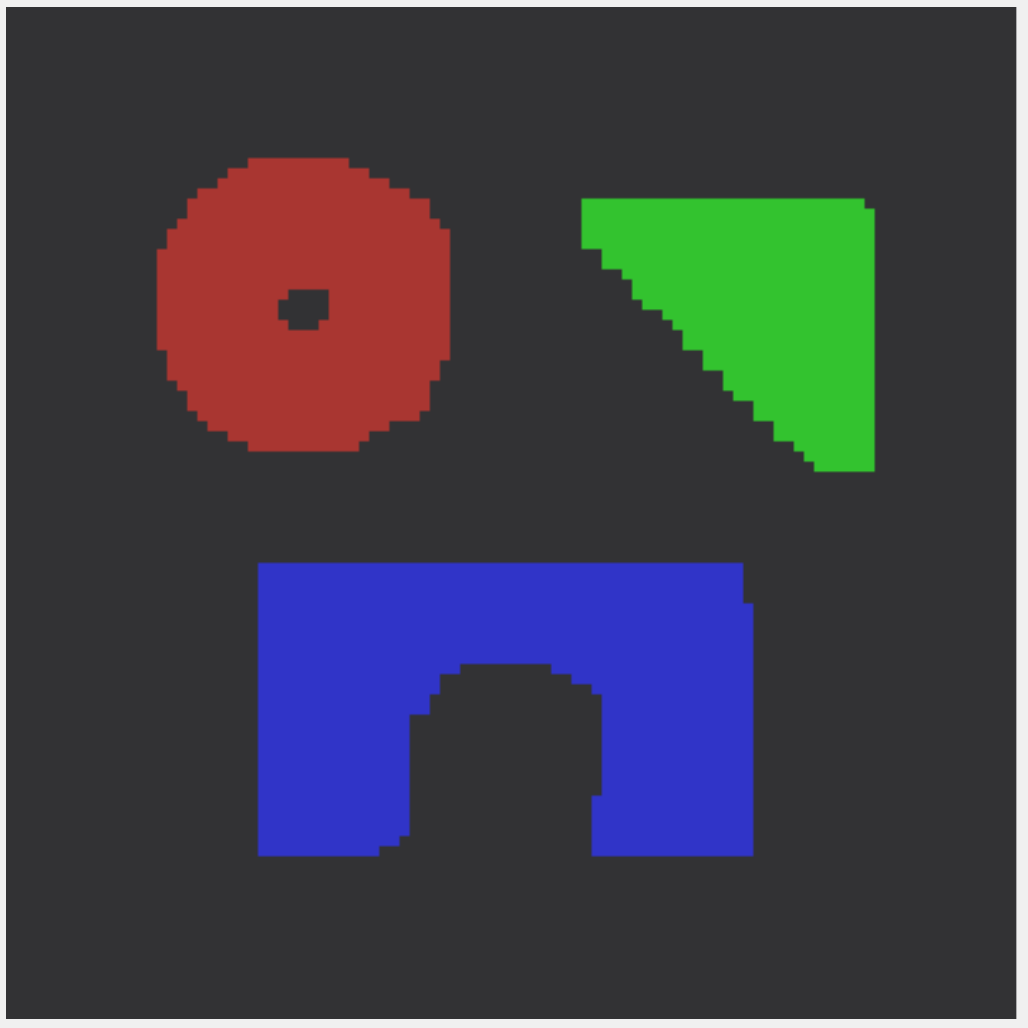}} &
				\subcaptionbox{$R_{MS}$}{\includegraphics[width = 1.00in]{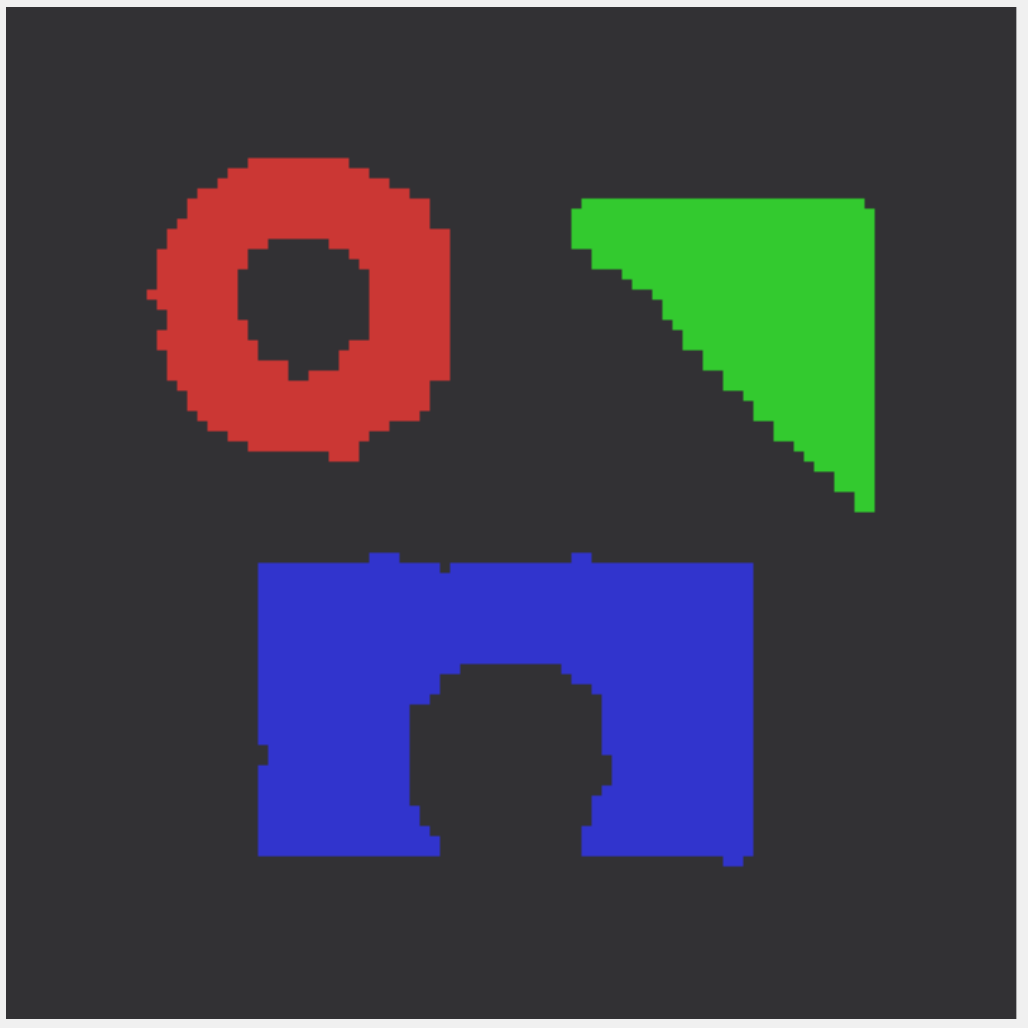}} 
		\end{tabular}}
		\caption{Reconstruction results on Figure \ref{fig:synthetic_4phase} corrupted with 40\% SPIN (top) and 40\% RVIN (bottom). }
		\label{fig:4phase_40_spin}
	\end{minipage}
\end{figure}

Figure \ref{fig:synthetic_2phase} is a color version of Figure \ref{fig:synthetic_grayscale}. We corrupt the image by 0\% to 50\% SPIN/RVIN  for each color channel. When a color image is corrupted with noise, one channel might be noisier than the others. In addition, image structures may vary with color channels, thus making the color extension of finding a balanced segmentation across all the color channels more challenging than for grayscale images. 
For Figure \ref{fig:synthetic_2phase}, we set $\lambda = 0.5$ for all methods, except for $L_0$ \cite{xu2011image} in which $\lambda = 50$. For the AIFR models, we set $\nu = 2.5$. The maximum number of inner/outer iterations are the same as the case for Figure \ref{fig:synthetic_grayscale}. The DICE indices of the segmentation results  are reported in Table \ref{tab:twophase_color}, which shows that $L_1-L_2$ CV generally yields the best results
and AIFR is slightly worse than its AICV counterpart but better than $L_1$ FR.  Figure \ref{fig:2phase_40_spin} presents the comparison results of  AICV (with optimal $\alpha$), $L_1$ CV, AIFR (with optimal $\alpha$), $L_1$ FR, and $L_1+L_2^2$ for 40\% SPIN and 40\% RVIN, showing that AICV and AIFR segment more salient regions than their $L_1$ counterparts and $L_1+L_1^2$. 

Figure \ref{fig:synthetic_4phase} is a color image for multiphase segmentation. 
We set $\lambda = 2.25$ for all methods, except for $L_0$ \cite{xu2011image} in which $\lambda = 50$. For the AIFR models, we set $\nu = 5$. The maximum number of inner iterations for the AITV models is 1000, while the maximum number of outer iterations is 40 for AICV and 160 for AIFR. Table \ref{tab:fourphase_color} presents the DICE indices of the segmentation results under 0\% to 40\% SPIN/RVIN contamination  for each color channel. 
For SPIN,  $L_1-0.25L_2$ FR is comparable to $L_1$ FR and outperforms it when the noise level is  40\% . For RVIN, $L_1-0.5L_2$ and $L_1-0.25L_2$ FR give the best results in general. We also observe that the smaller $\alpha$ is, the more robust AICV/AIFR are with respect to impulsive noise. The visual results are presented in Figure \ref{fig:4phase_40_spin} for 40\% SPIN/RVIN, clearly showing that AIFR provides the best segmentation. AICV and $L_1$ CV contain noise along the edges of the blue region, $L_1$ FR oversegments the red region, and $R_{MS}$ appears slightly worse than AIFR. 

Overall, the proposed AICV/AIFR methods are robust against impulsive noise, unlike the two-stage methods. For the three synthetic images, AICV and AIFR with  appropriately chosen $\alpha$  outperform their $L_1$ counterparts under a high level of impulsive noise. Unfortunately, there is no optimal choice of $\alpha$ that works for all images, as demonstrated by our experiments. For example, $\alpha = 1.0$ yields the highest DICE indices for Figure \ref{fig:synthetic_2phase} according to Table \ref{tab:twophase_color}, but it does not perform as well for Figure \ref{fig:synthetic_grayscale} according to Table \ref{tab:twophase_grayscale}.


\subsection{Real Images}

\begin{figure}[t]
	\begin{minipage}{\linewidth}
	\centering
	\begin{tabular}{c@{}c@{}c@{}}
		\subcaptionbox{\label{fig:sign}}{\includegraphics[scale = 0.25]{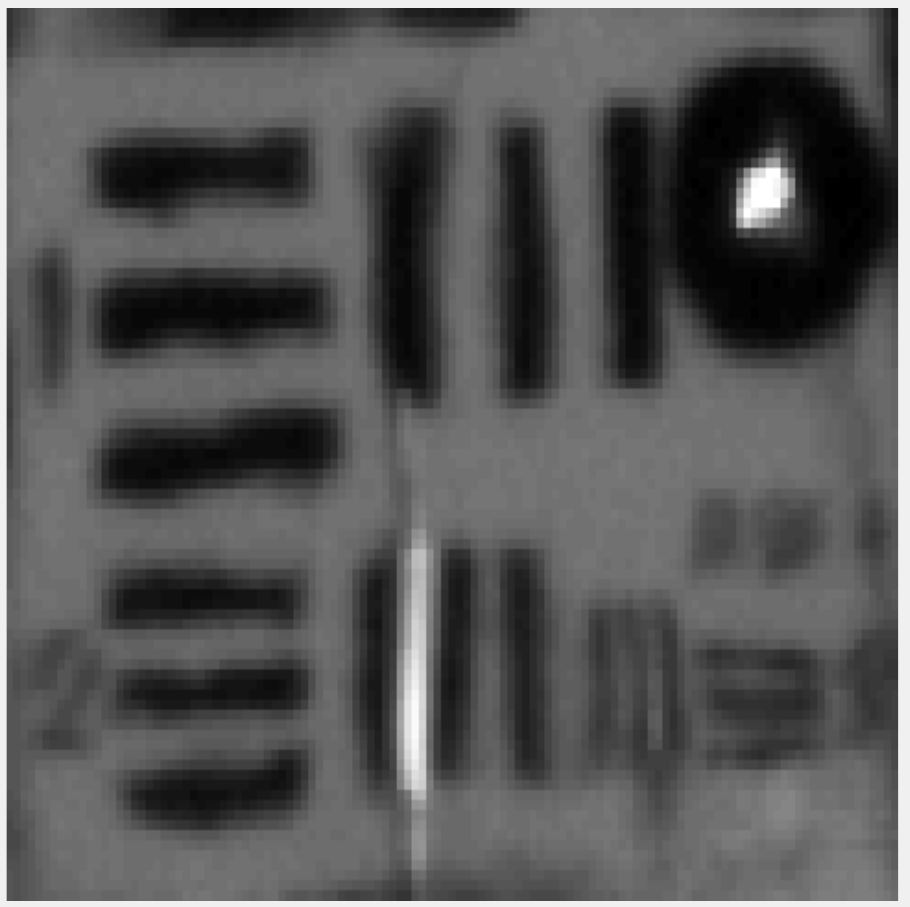}} &
		\subcaptionbox{\label{fig:hawk}}{\includegraphics[scale = 0.25]{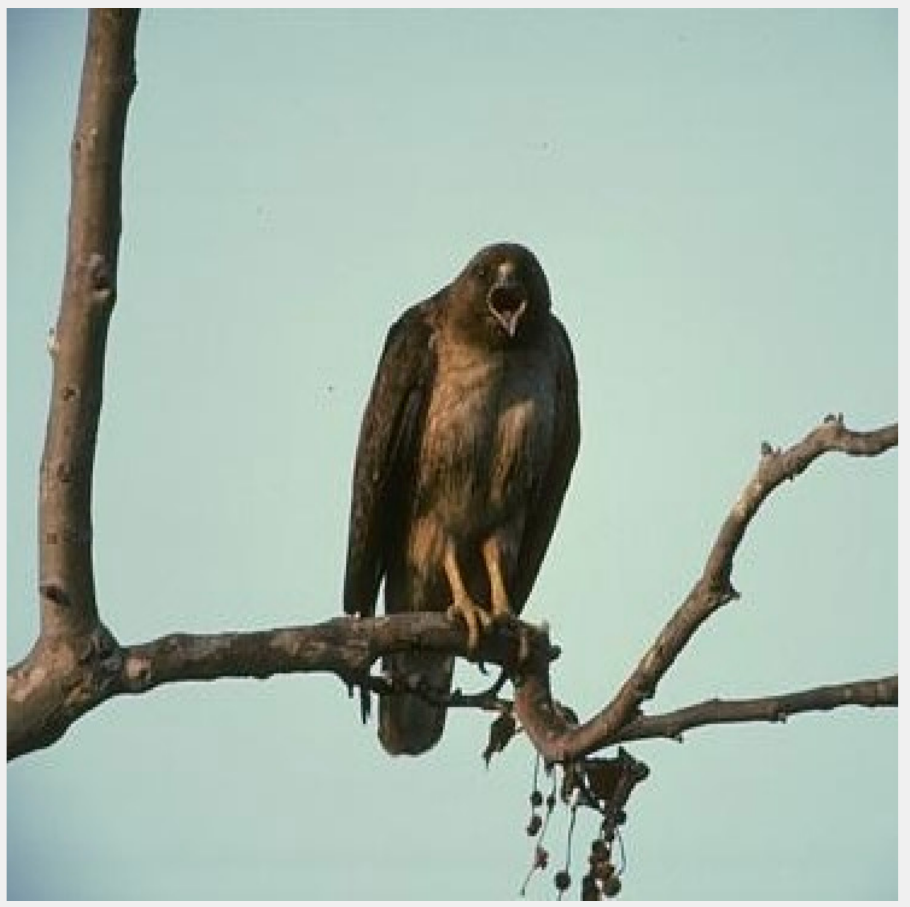}} &  \subcaptionbox{\label{fig:butterfly}}{\includegraphics[scale = 0.25]{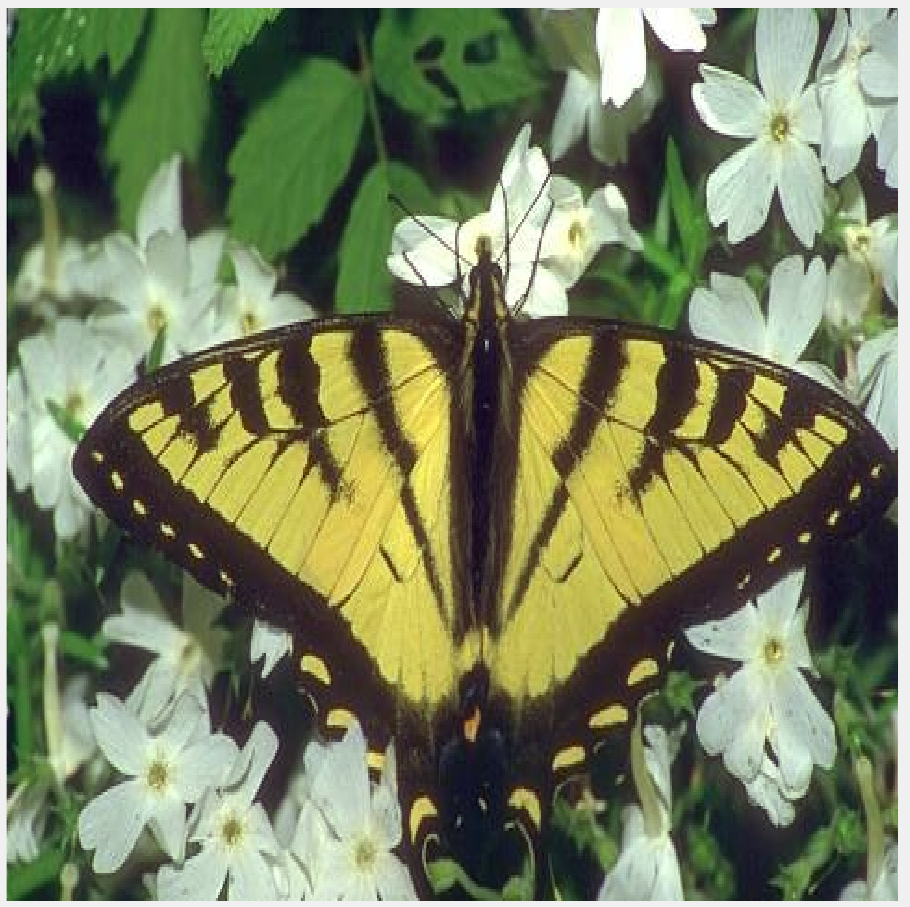}} 
		\subcaptionbox{\label{fig:flower}}{\includegraphics[scale = 0.25]{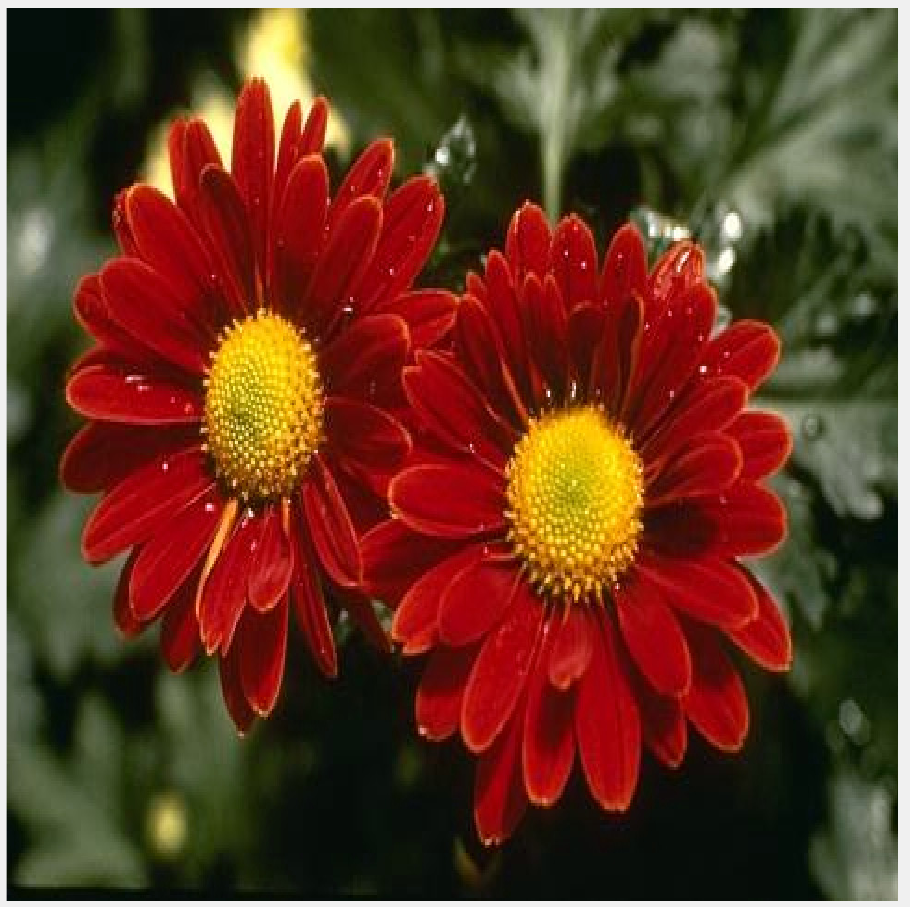}} 
				\subcaptionbox{\label{fig:pepper}}{\includegraphics[scale = 0.25]{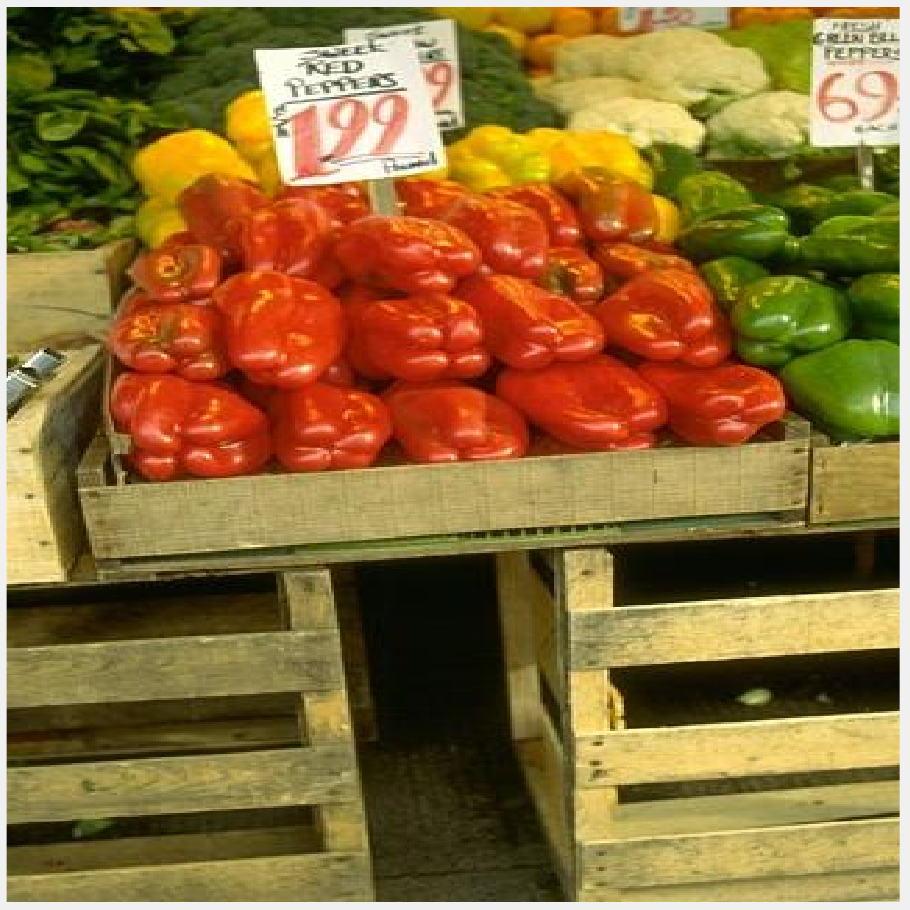}} 
	\end{tabular}
	\caption{Real images for image segmentation. (a) Close-up of a target board in a video. Size:  $89 \times 121$. (b) Image of a hawk.  Size: $318\times 370$. (c) Image of a butterfly.  Size: $321 \times 481$. (d) Image of a flower. Size: $321 \times 481$. (e) Image of peppers. Size: $481 \times 321$. }
	\label{fig:real_image}
\end{minipage}	
\end{figure}

\begin{table}[t]
	\captionof{table}{PSNR values of segmentation methods applied to real color images. NA stands for ``not applicable.''}
			\begin{center}
			\begin{tabular}{l|c|c|c|c}\hline
				& Figure \ref{fig:hawk}    & Figure \ref{fig:butterfly}   & Figure \ref{fig:flower} & Figure \ref{fig:pepper}  \\ \Xhline{5\arrayrulewidth}
				$L_1-L_2$ CV                 & 23.3949 &  21.9000   & NA & NA\\ \hline
				$L_1-0.75L_2$ CV             & 23.3933 &  21.9001   & NA & NA\\\hline
				$L_1-0.5L_2$ CV              & \textbf{23.4001} &  21.8976   & NA & NA\\\hline
				$L_1-0.25L_2$ CV             & 23.3913 &  21.8985   & NA & NA \\\hline
				$L_1$ CV                    & 23.3690 &  21.8977   & NA & NA\\\hline
				$L_1-L_2$ FR              & 23.4223 &  22.2574   & 21.8283 & 22.2597\\ \hline
				$L_1-0.75L_2$ FR          & 23.4014 &  \textbf{22.2578}   & 21.8383 & 22.4880\\ \hline
				$L_1-0.5L_2$ FR           & 23.3814 &  22.2576   & \textbf{21.8418} & \textbf{22.4901}\\ \hline
				$L_1-0.25L_2$ FR          & 23.3523 &  22.2575   &\textbf{21.8418} & 22.4672\\ \hline
				$L_1$ FR                 & 23.3173 &  22.2570   & 21.8409 & 21.9482\\ \hline
				$L_1+L_2^2$ & 23.2601 &  21.6077   & 21.1802 & 21.0277\\ \hline
				$L_0$        \cite{xu2011image}               & 23.2419 &  22.2570   & 21.7914 & 22.0361\\ \hline
				$L_0$        \cite{storath2014fast}               & 23.1985 &  17.7573   & 21.8129 & 21.9703\\ \hline$R_{MS}$               & 23.0865 &  17.7140   & 21.7832 & 22.0904\\ \hline
			\end{tabular}
		\end{center}
		\label{tab:psnr}
\end{table}

\begin{table}[t]
	\captionof{table}{Computational time (seconds) of segmentation methods applied to real color images. NA stands for ``not applicable.''}
			\begin{center}
			\begin{tabular}{l|c|c|c|c|c}\hline
				& Figure \ref{fig:sign} & Figure \ref{fig:hawk}    & Figure \ref{fig:butterfly}   & Figure \ref{fig:flower} & Figure \ref{fig:pepper}  \\ \Xhline{5\arrayrulewidth}
				$L_1-L_2$ CV                 & 2.06 & 16.09& 49.27 & NA & NA\\ \hline
				$L_1-0.75L_2$ CV             & 1.86 &  15.91   & 55.91 & NA & NA\\\hline
				$L_1-0.5L_2$ CV              & 2.08 &  15.89   & 70.68 & NA & NA\\\hline
				$L_1-0.25L_2$ CV             & 2.17 & 16.09 &  71.23   & NA & NA \\\hline
				$L_1$ CV                    & 1.78 &  16.23 & 54.94   & NA & NA\\\hline
				$L_1-L_2$ FR              & 2.51 &  43.65   & 66.27  & 191.30 & 212.28\\ \hline
				$L_1-0.75L_2$ FR          & 1.91 &  46.26   & 64.98 & 185.26 & 233.79 \\ \hline
				$L_1-0.5L_2$ FR           & 1.23 &  15.29   & 68.3 & 175.67 &  263.52\\ \hline
				$L_1-0.25L_2$ FR          & 0.92 &  13.18   & 69.49 & 182.08 & 227.62\\ \hline
				$L_1$ FR                 & 0.72 &  13.18   & 69.49 & 182.08 & 227.62\\ \hline
				$L_1+L_2^2$ & 0.24 &  1.8   & \textbf{1.2} & 1.75& 2.48\\ \hline
				$L_0$        \cite{xu2011image}               & \textbf{0.15} & \textbf{0.92}   & 1.71 & \textbf{1.6} & \textbf{1.97}\\ \hline
				$L_0$        \cite{storath2014fast}               & 0.17 &  2.96   & 3.06 & 3.05 & 4.26\\ \hline$R_{MS}$               & 0.61 &  6.60  & 17.71 & 17.24 & 20.10\\ \hline
			\end{tabular}
		\end{center}
		\label{tab:sec}
\end{table}

We apply the proposed methods and the two-stage methods on real images (all rescaled to $[0,1]$ for the pixel values) shown in Figure \ref{fig:real_image} without additive noise. Figure \ref{fig:sign} is provided in \cite{lou-2013} while Figures \ref{fig:hawk}-\ref{fig:pepper} are provided by the Berkeley Segmentation Dataset and Benchmark \cite{martin-2001}. Specifically, Figures \ref{fig:sign} and \ref{fig:hawk} are for two-phase segmentation,  Figure \ref{fig:butterfly} is for four-phase segmentation, and Figures \ref{fig:flower} and \ref{fig:pepper} are for five-phase and seven-phase segmentation, respectively. We set the maximum number of inner iterations for CV/FR methods as 300, and the maximum number of outer iterations for CV as 20. The maximum outer iteration number of  the FR methods depends on images, which is set to 40 for Figures \ref{fig:sign}-\ref{fig:hawk}, 80 for Figure \ref{fig:butterfly}, and 160 for Figures \ref{fig:flower}-\ref{fig:pepper}. 
Following the work of \cite{jung2017piecewise},
we compute the peak signal-to-noise ratio (PSNR) between the reconstructed image $\tilde{\mathbf{f}}$ derived by \eqref{eq:color_reconstruction_eq} and the original image $\mathbf{f}$. PSNR is defined by $10 \log_{10} \frac{3mn}{\sum_{\iota \in \{r,g,b\}} \|\tilde{f}_{\iota}-f_{\iota}\|_X^2}$, and it quantitatively measures the quality of the segmentation results for real color images without ground truth.  The PSNR values are recorded in Table \ref{tab:psnr}. As the CV methods are inapplicable to non-power-of-2 segmentation examples, we indicate by NA (``not applicable'') their results on Figures \ref{fig:flower}-\ref{fig:pepper} in Table \ref{tab:psnr}.

For Figure \ref{fig:sign}, we set $\lambda =100$ for all methods, except for $L_0$ \cite{xu2011image}  in which $\lambda = 10000$. For all FR methods, we set $\nu = 35$. The initialization for the CV and FR methods is a step function of a circle in the image center with radius 10. The segmentation results of these competing methods are displayed in Figure \ref{fig:sign_result}, each equipped with a zoomed-in region of the bottom right of the image.
We observe that as $\alpha$ decreases, the CV methods segment lesser regions, while the FR methods identify lesser gaps. The results of the two-stage methods are not as detailed as the results provided by $L_1-L_2$ CV and FR.   

For Figure \ref{fig:hawk}, we set $\lambda = 50$ for $L_0$ \cite{xu2011image}, $\lambda =10$ for the other methods, and $\nu = 10.0$ for the FR methods. The initialization for the CV and FR methods is the same as Figure \ref{fig:sign}. Quantitative comparison of these methods is listed in   Table \ref{tab:psnr}, showing that the AICV and AIFR methods outperform their $L_1$ counterparts. The visual results in Figure \ref{fig:hawk_result} demonstrate that  AICV and AIFR can segment finer details, especially on the branch on the left side of the image and on the hawk, than their $L_1$ counterparts, which thereby explains their higher PSNR values. 

For Figure \ref{fig:butterfly}, we set $\lambda = 1000$ for all methods and $\nu = 650$ for the FR methods. Initialization for the CV methods are two step functions of circles both with radius 10, one shifted 5 pixels to the left of the image center and the other shifted 5 pixels to the right. For the FR methods, the initialization of the membership functions are uniformly distributed in $[0,1]$ and then normalized. 
Figure \ref{fig:butterfly_result} compares the AIFR and AICV methods (using the optimal $\alpha$ value that corresponds to the highest PSNR in Table \ref{tab:psnr}) with their $L_1$ counterparts. 
As PSNR values are all similar, we do not observe much visual differences between the images in Figure \ref{fig:butterfly_result}.

\begin{figure}[t!]
\begin{minipage}{\linewidth}
\centering
\resizebox{\textwidth}{!}{%
	\begin{tabular}{c@{}c@{}c@{}c@{}c@{}c}
		\subcaptionbox{Original}{\includegraphics[width = 1.00in]{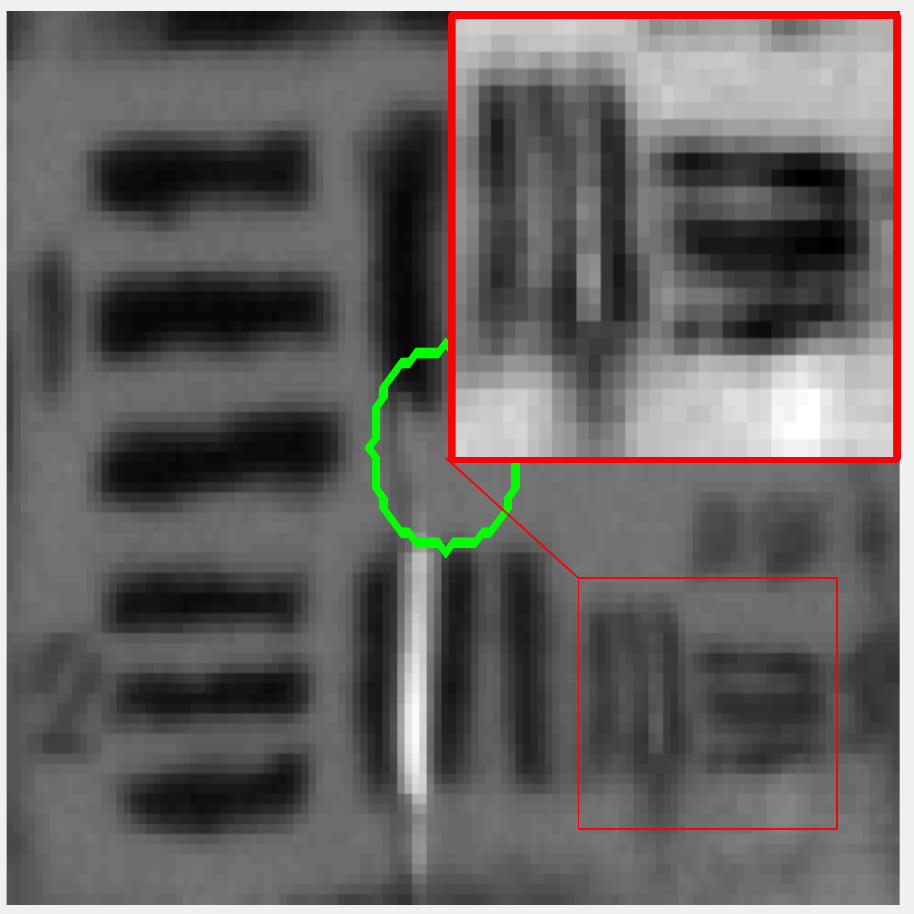}} & 		\subcaptionbox{$L_1+L_2^2$}{\includegraphics[width = 1.00in]{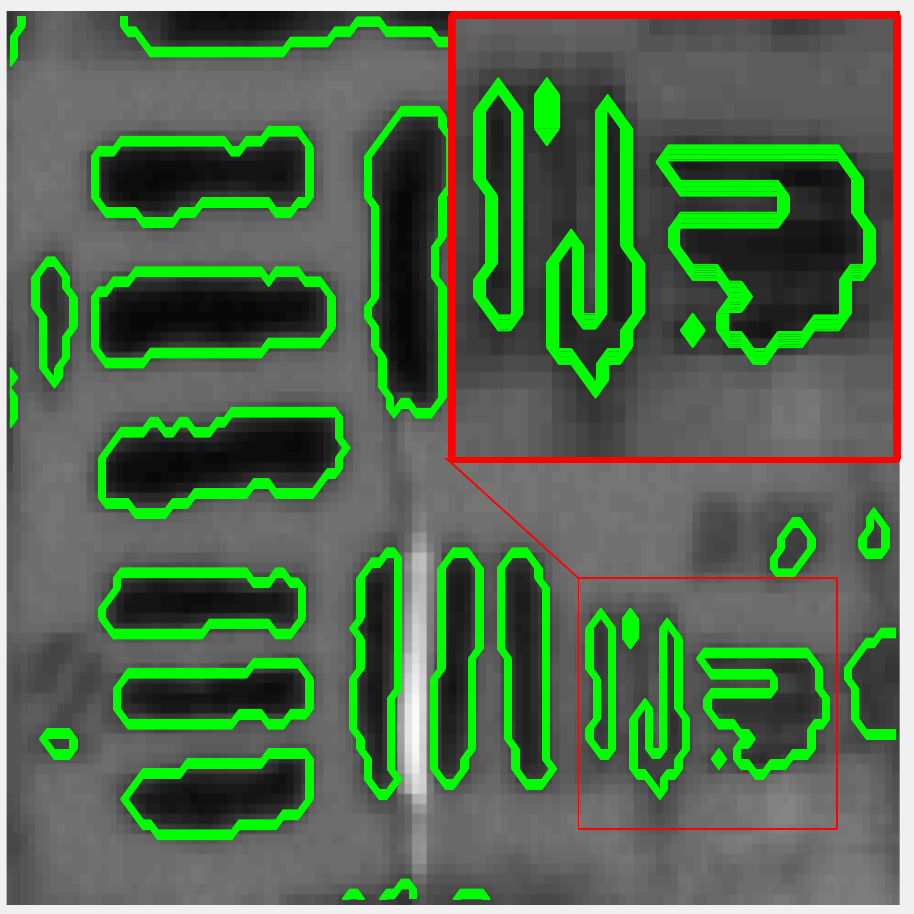}} &
		\subcaptionbox{$L_0$ \cite{xu2011image}}{\includegraphics[width = 1.00in]{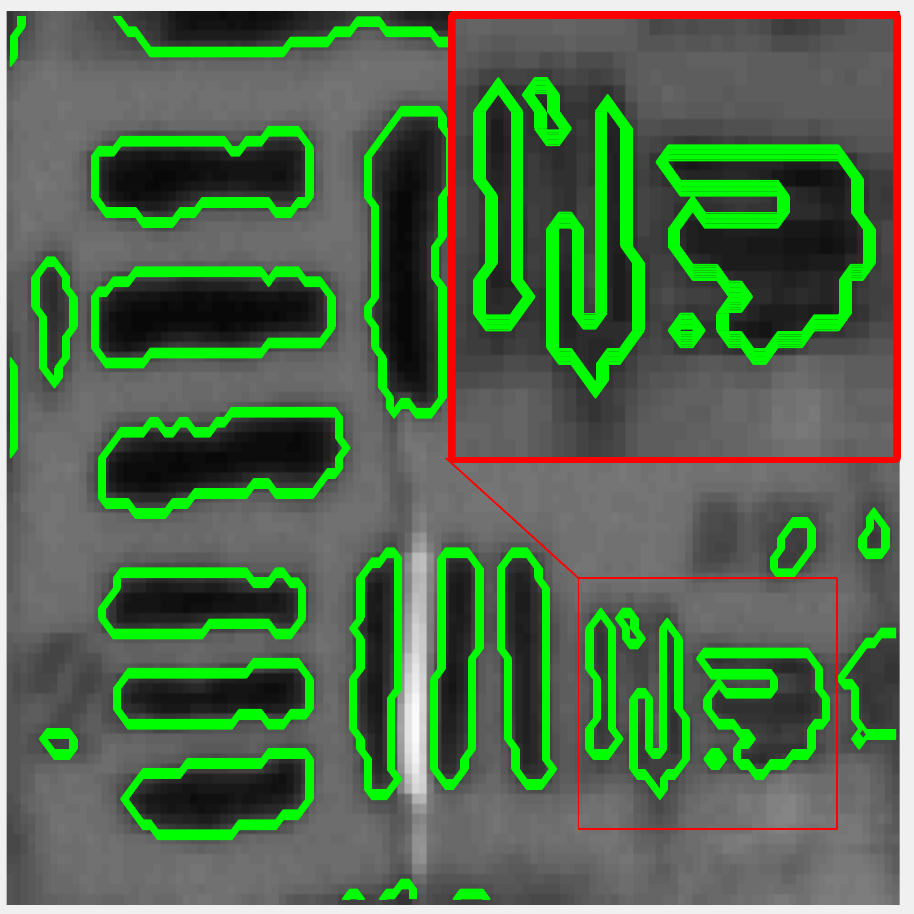}}& \subcaptionbox{$L_0$ \cite{storath2014fast}}{\includegraphics[width = 1.00in]{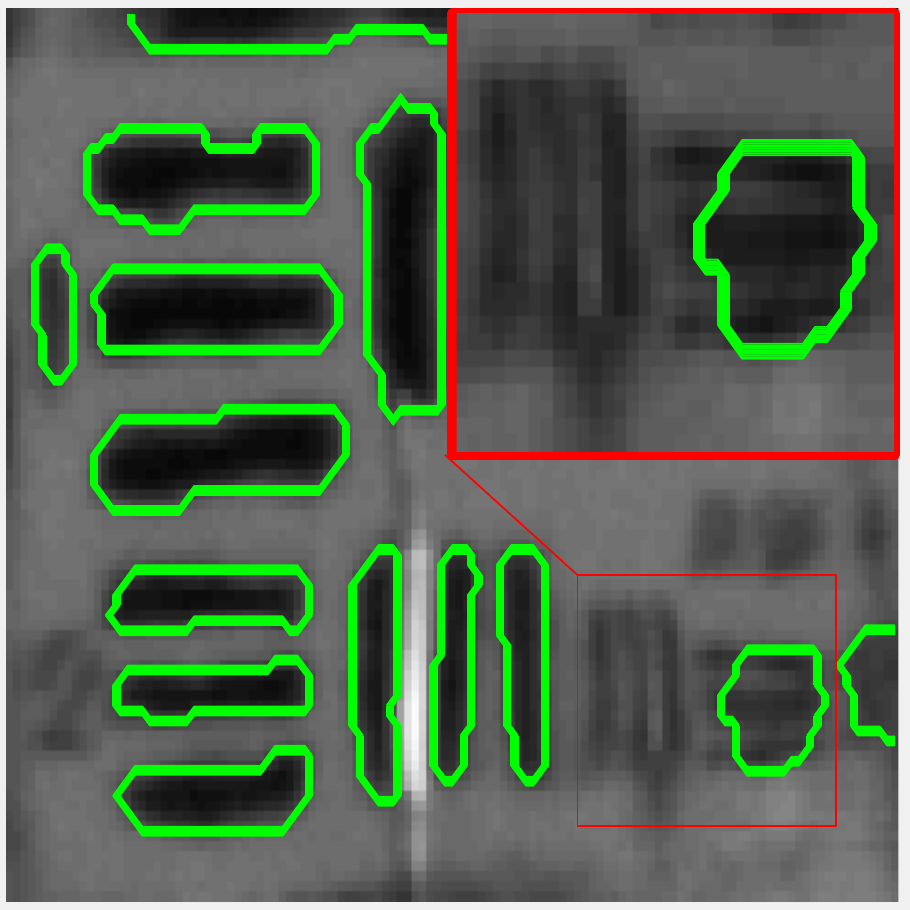}}& \subcaptionbox{$R_{MS}$}{\includegraphics[width = 1.00in]{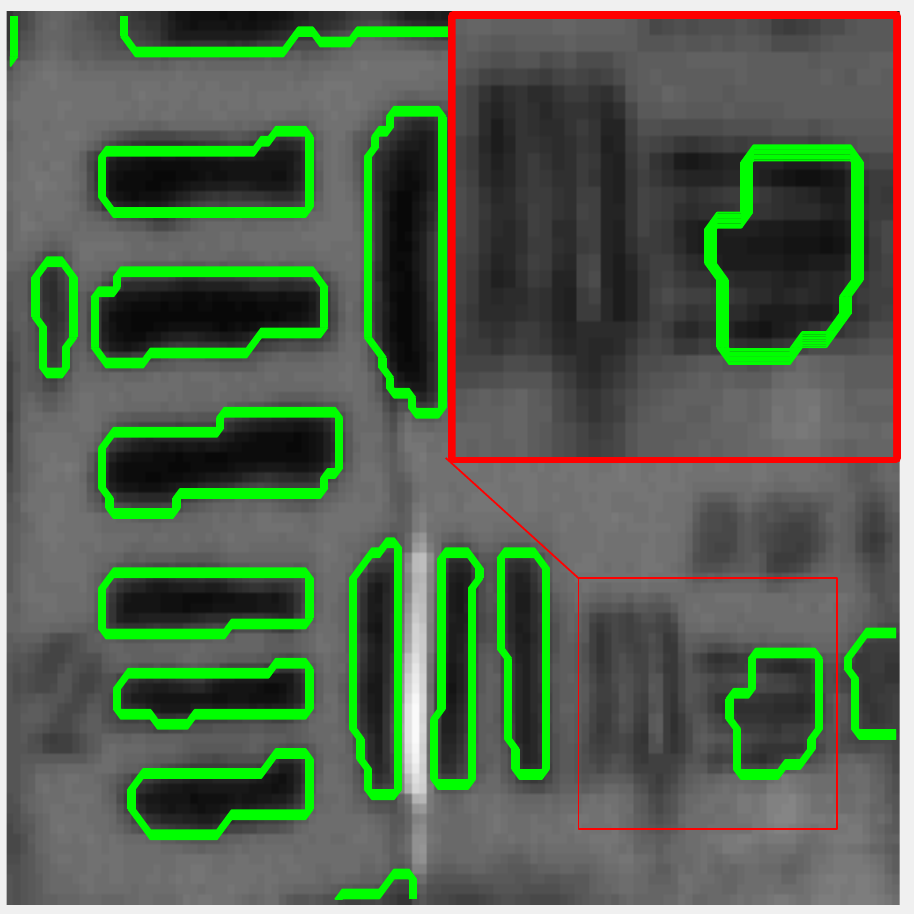}}  \\ \subcaptionbox{$L_1-L_2$ CV }{\includegraphics[width = 1.00in]{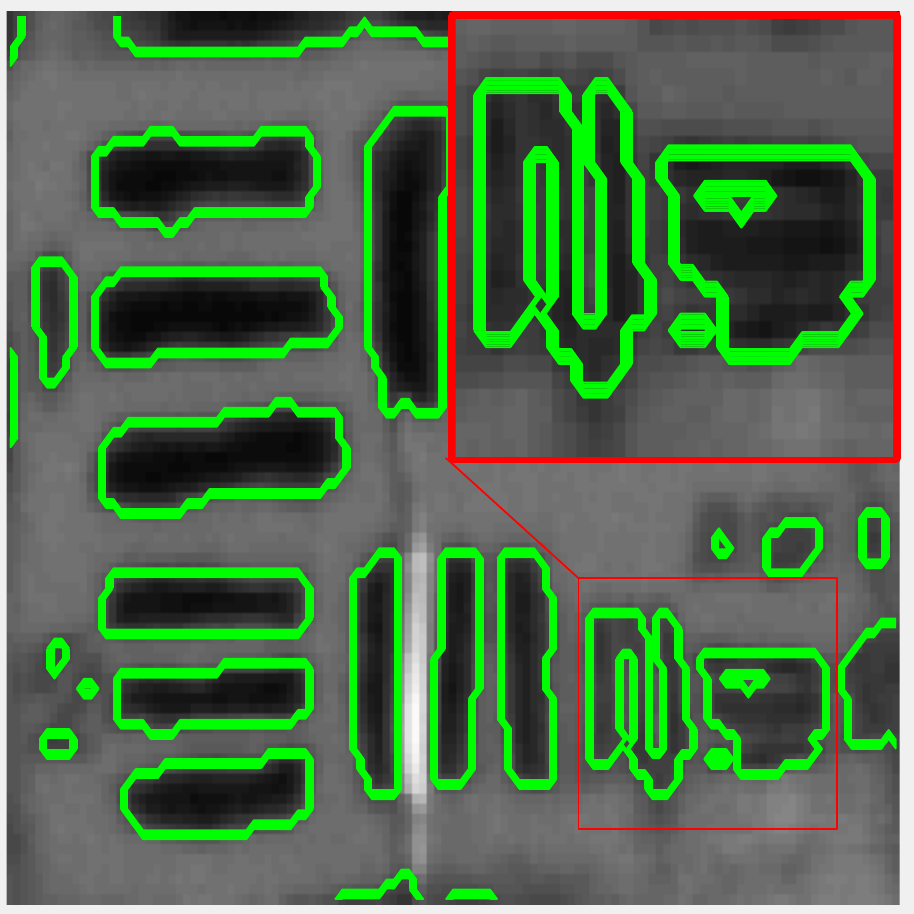}} & \subcaptionbox{$L_1-0.75L_2$ CV}{\includegraphics[width = 1.00in]{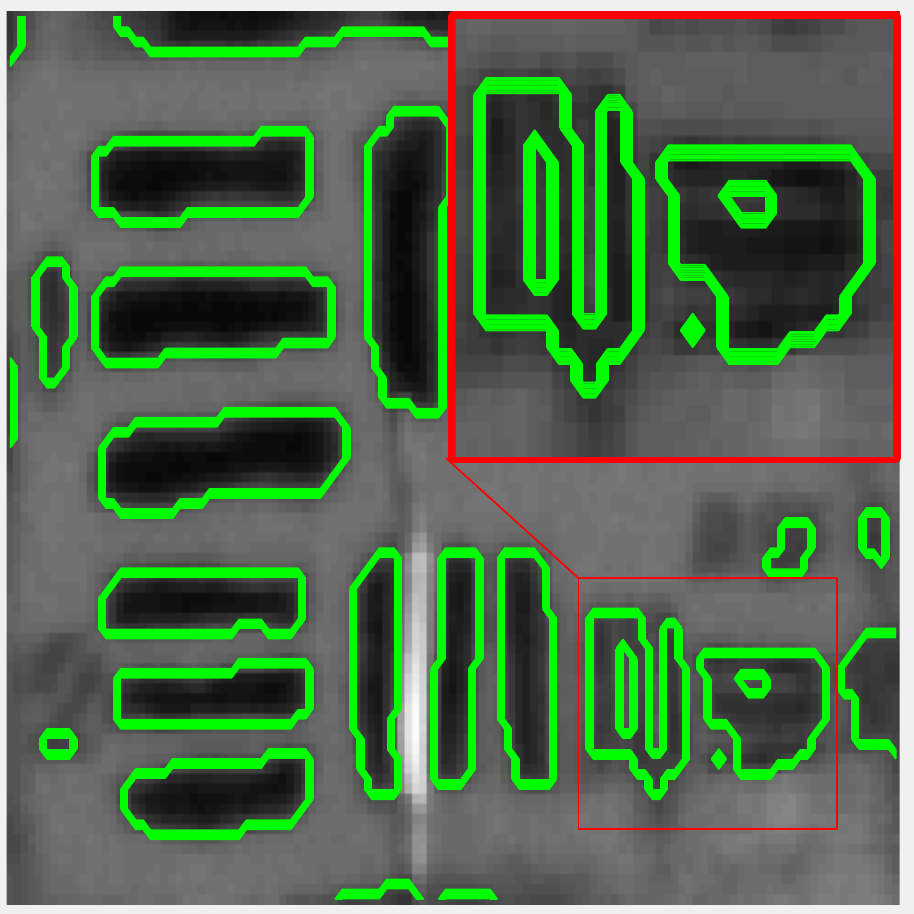}} & \subcaptionbox{$L_1-0.5L_2$ CV}{\includegraphics[width = 1.00in]{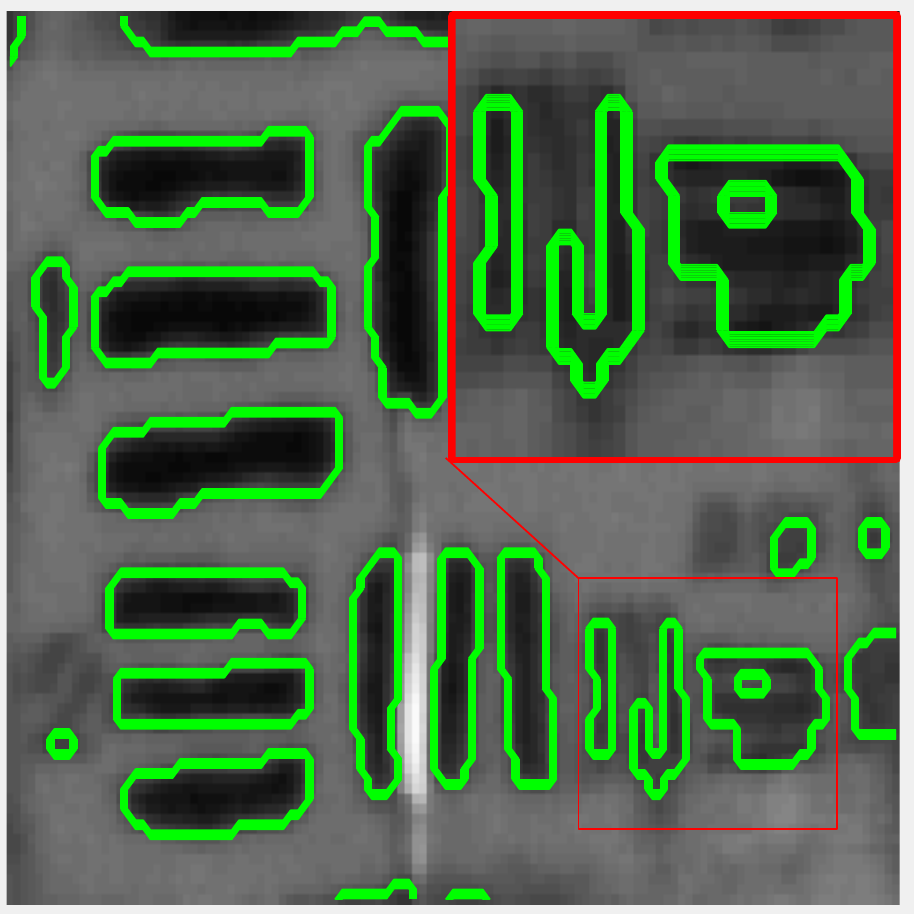}} &
		\subcaptionbox{$L_1-0.25L_2$ CV}{\includegraphics[width = 1.00in]{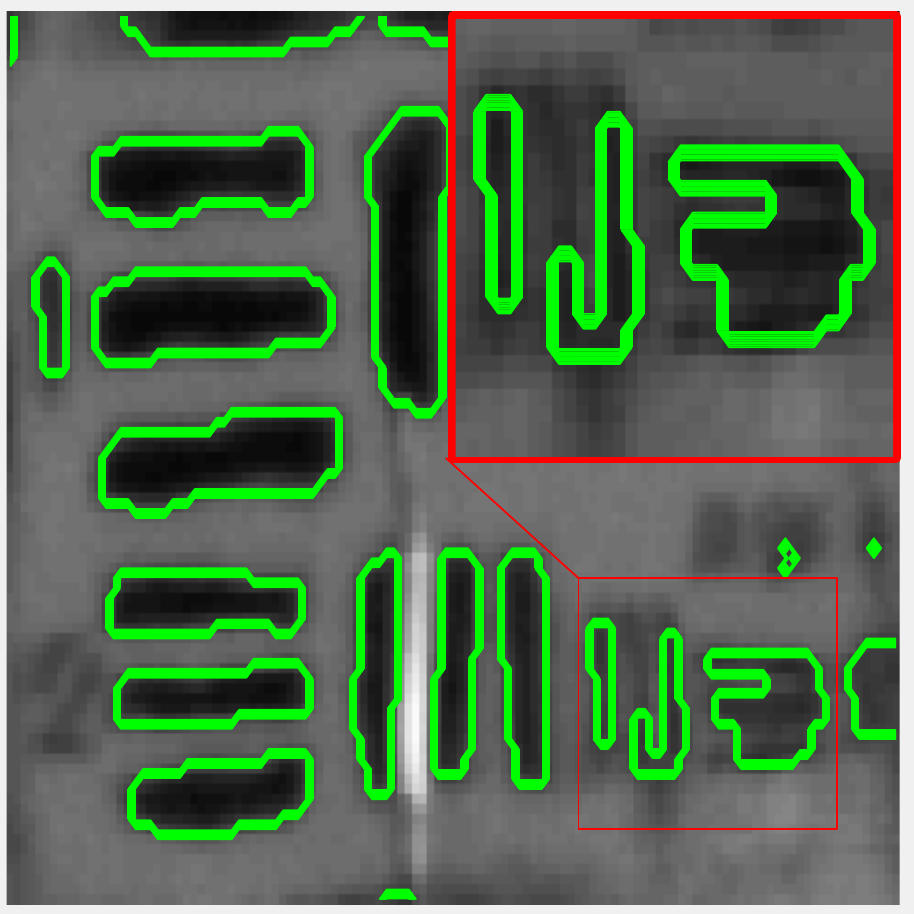}} & \subcaptionbox{$L_1$ CV}{\includegraphics[width = 1.00in]{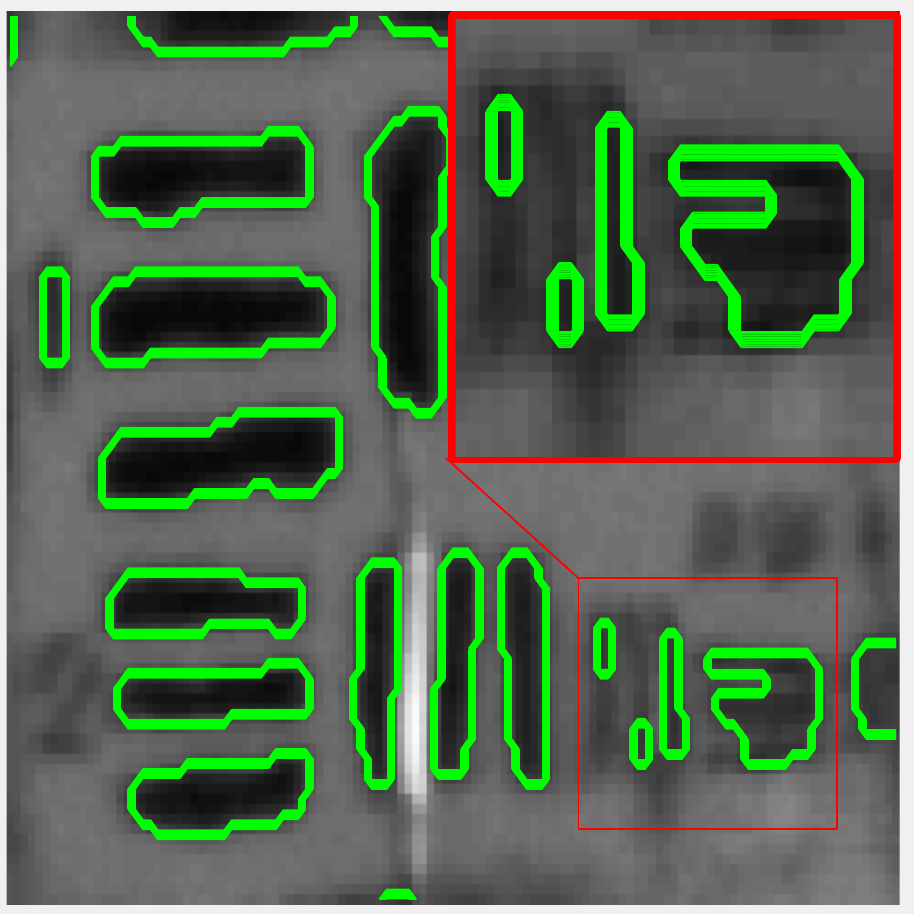}} \\ \subcaptionbox{$L_1-L_2$ FR}{\includegraphics[width = 1.00in]{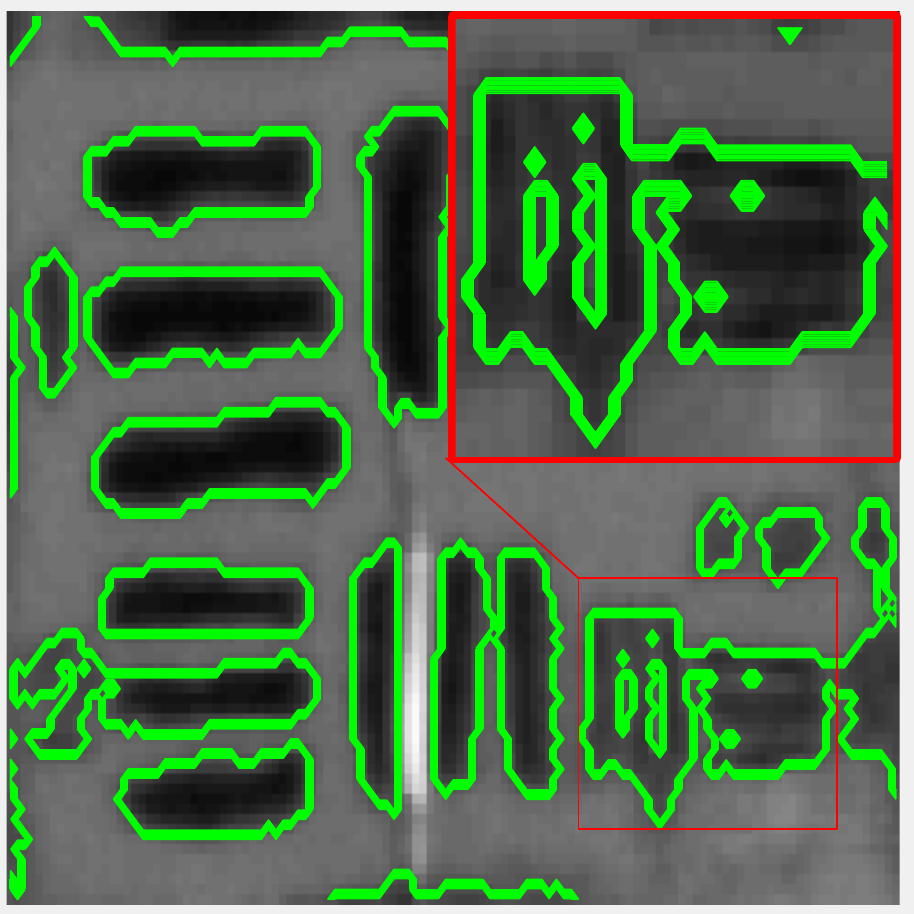}} & \subcaptionbox{$L_1-0.75L_2$ FR}{\includegraphics[width = 1.00in]{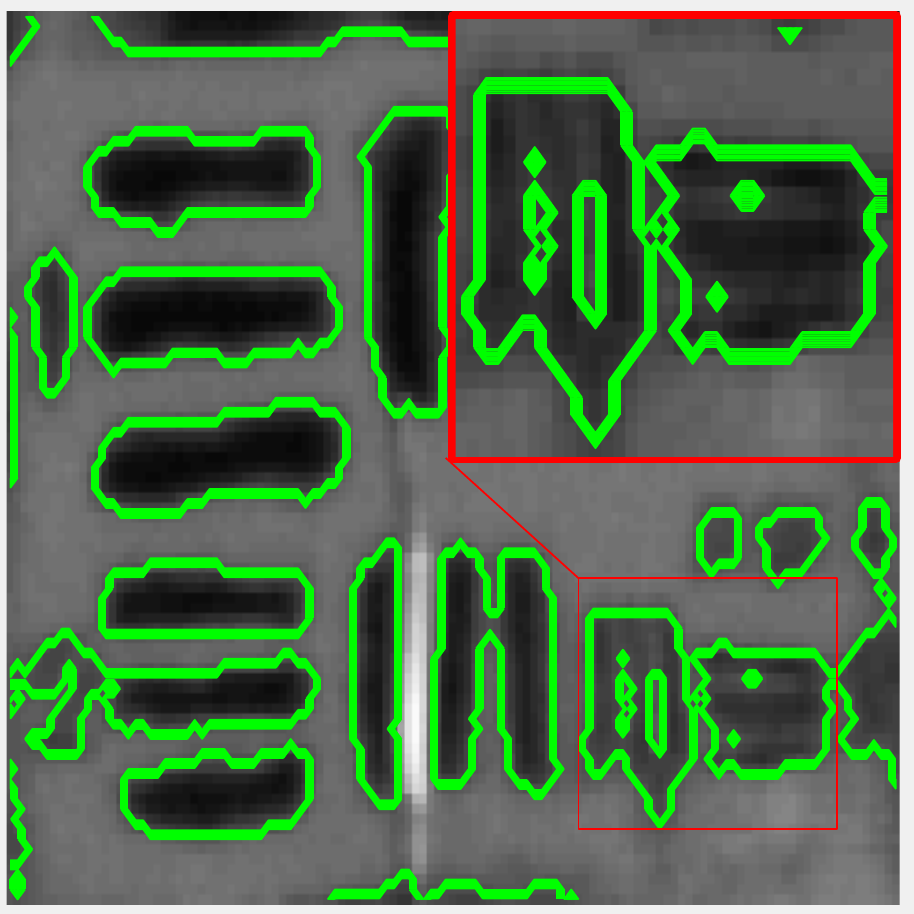}} & \subcaptionbox{$L_1-0.5L_2$ FR}{\includegraphics[width = 1.00in]{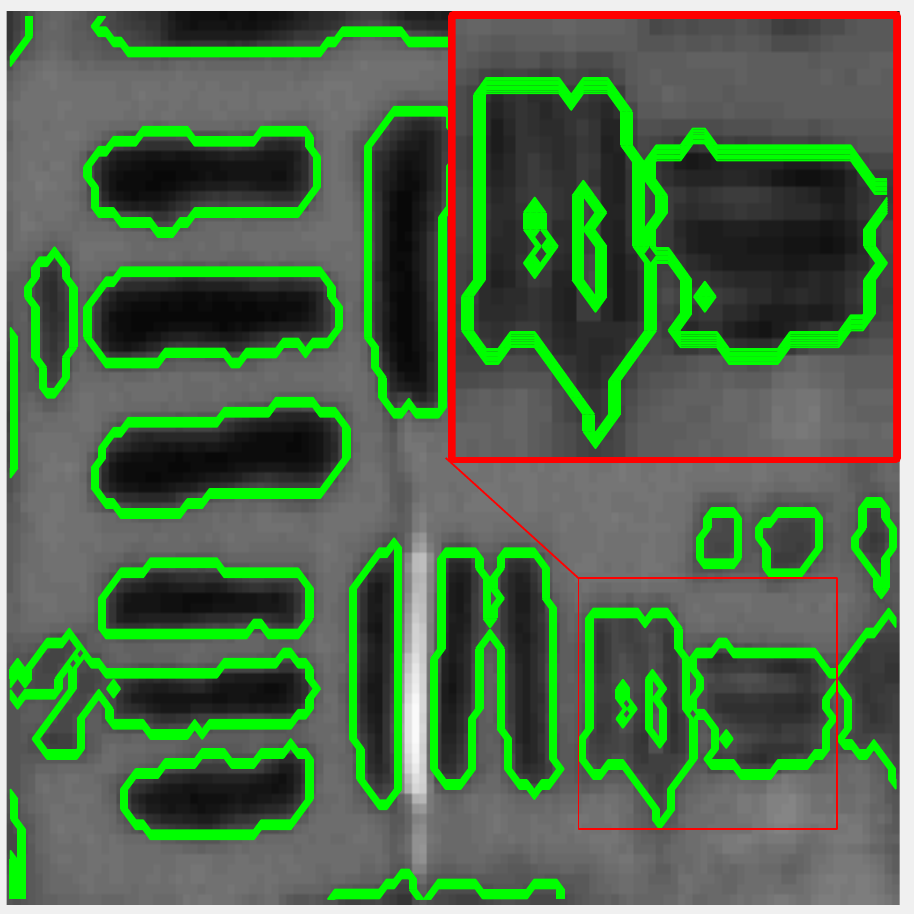}} &
		\subcaptionbox{$L_1-0.25L_2$ FR}{\includegraphics[width = 1.00in]{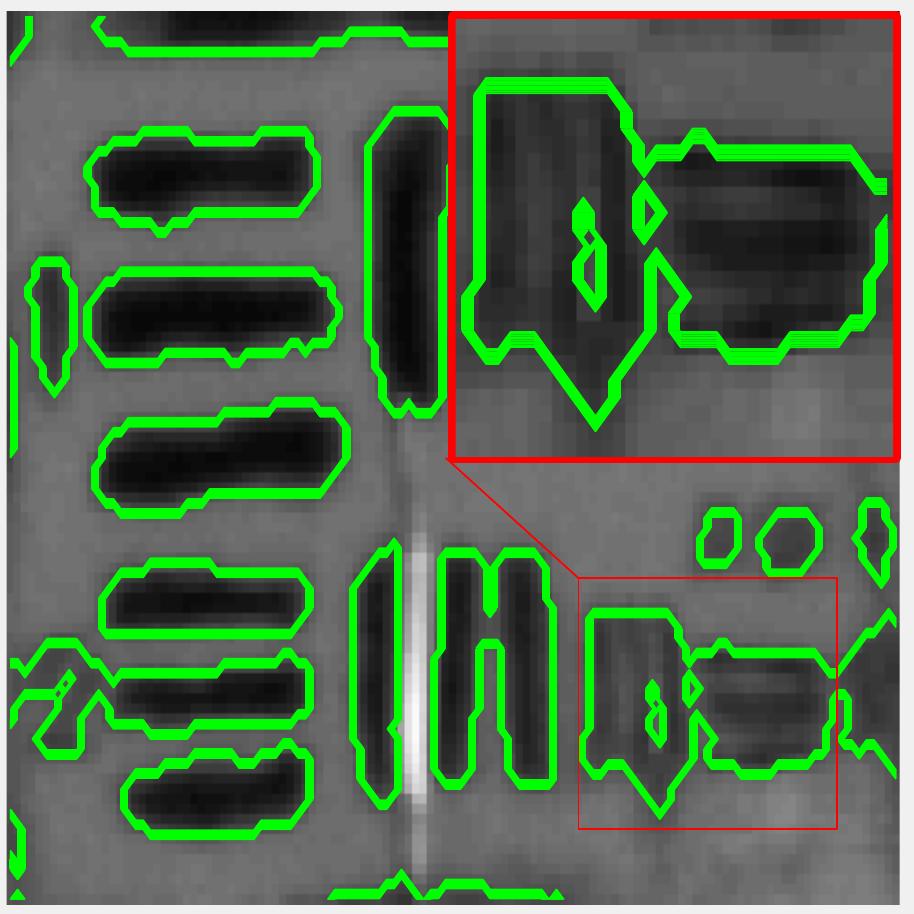}} & \subcaptionbox{$L_1$ FR}{\includegraphics[width = 1.00in]{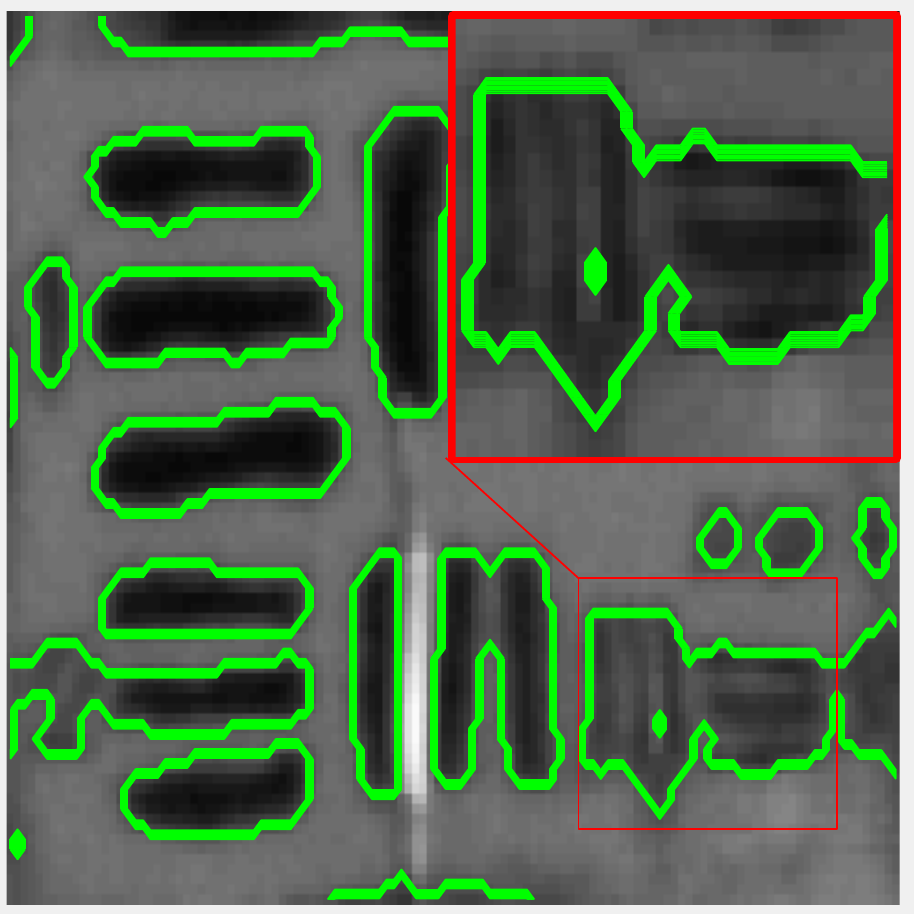}} 
\end{tabular}}
\caption{Segmentation results on Figure \ref{fig:sign}. (The images may need to be zoomed in on a pdf reader to see the differences.) }
\label{fig:sign_result}
\end{minipage}
\end{figure}

For Figure \ref{fig:flower}, we set $\lambda = 650$ for all methods, except $L_0$ \cite{xu2011image} in which $\lambda = 1000$. For the FR methods, we set $\nu = 1050$. For Figure \ref{fig:pepper}, we set $\lambda = 500$ for all methods and $\nu = 400$ for the FR methods.  Initialization of the membership functions for the FR methods is the same as for Figure \ref{fig:butterfly}. The segmentation results of the FR methods and the two-stage methods are shown in Figures \ref{fig:flower_result} and \ref{fig:pepper_result}. In Figure \ref{fig:flower_result}, the results of the FR methods have better contrast than the result of $L_1+L_2^2$ and thus they look more similar to the original image. In Figure \ref{fig:pepper_result}, $L_1-L_2$ FR, $L_1$ FR, and $L_0$ are unable to identify the yellow/orange peppers behind the red peppers, which explains their lower PSNR values. Although the results of the AIFR methods for $\alpha = 0.25,0.5, 0.75$ appear similar to $L_1+L_2^2$ and $R_{MS}$, $L_1-0.5L_2$ attains the best segmentation based on its PSNR value. 

Last, we report the computational times of the segmentation methods in Table \ref{tab:sec}. Admittedly, the proposed methods are slower compared to other segmentation methods. Besides, our computational times largely depend on the image size, the number of channels, and the number of $u_k$'s needed to segment. The acceleration of the proposed scheme will be left for future investigation.

In summary, given particular choices of $\alpha$, the AITV models outperform their $L_1$ counterparts and the two-stage methods. For Figure \ref{fig:sign}, larger values of $\alpha$ provide better segmentation results, but this may not be the case for other images. Thus, the optimal $\alpha$ value in an AITV model varies for an individual image. In addition, although the AITV methods tend to be slower than the two-stage methods, they are consistently more accurate based on their PSNR values. This observation is apparent in Figures \ref{fig:butterfly}-\ref{fig:pepper}, the most complex images tested in this section.

\begin{figure}[t!]

	\begin{minipage}{\linewidth}
		\centering
		\resizebox{\textwidth}{!}{%
			\begin{tabular}{c@{}c@{}c@{}c@{}c}
				\subcaptionbox{Original}{\includegraphics[width = 1.25in]{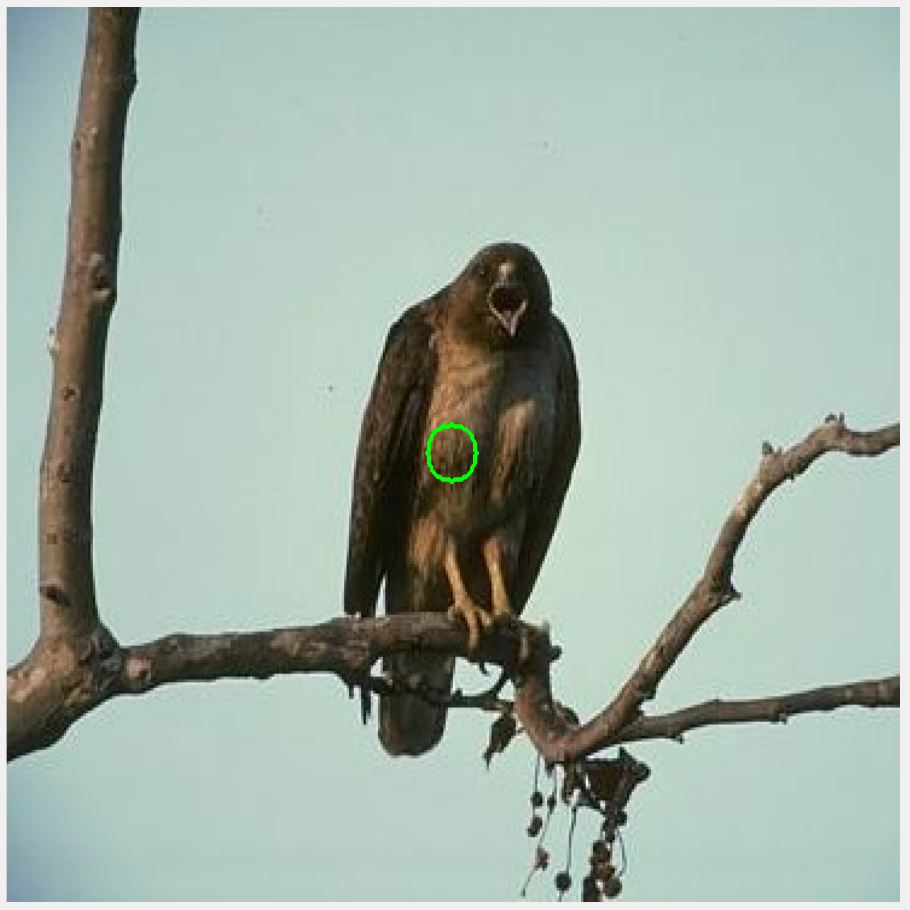}} & \subcaptionbox{$L_1-0.5L_2$ CV}{\includegraphics[width = 1.25in]{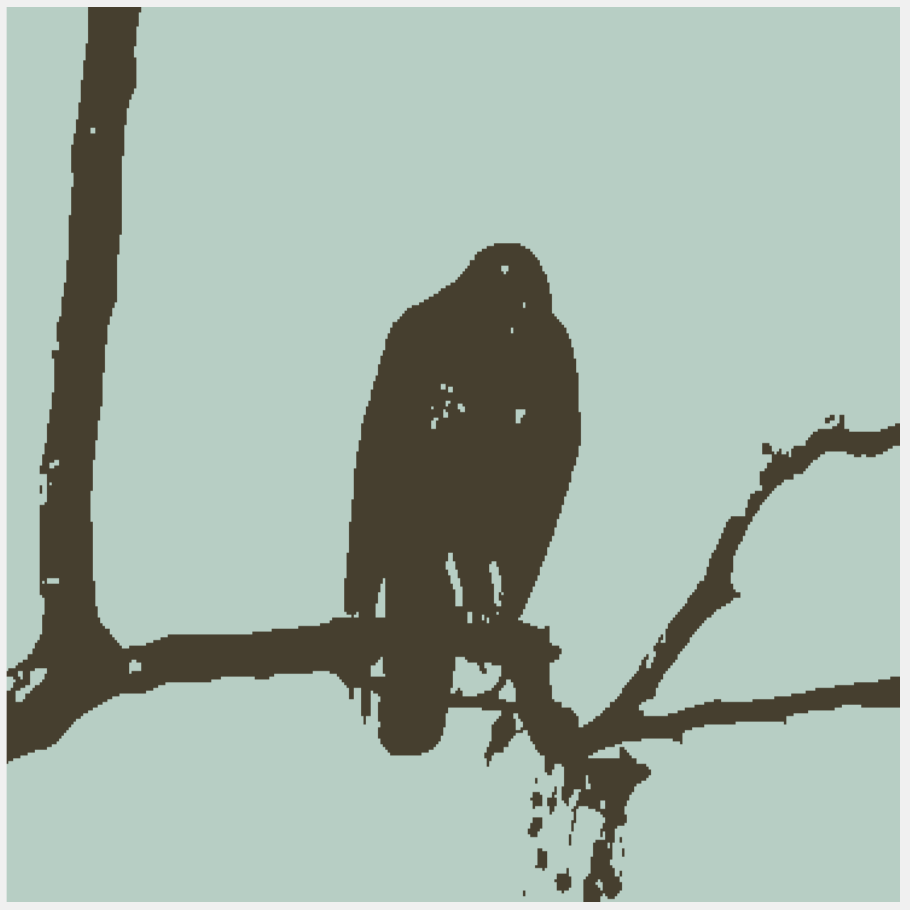}}  & \subcaptionbox{$L_1$ CV}{\includegraphics[width = 1.25in]{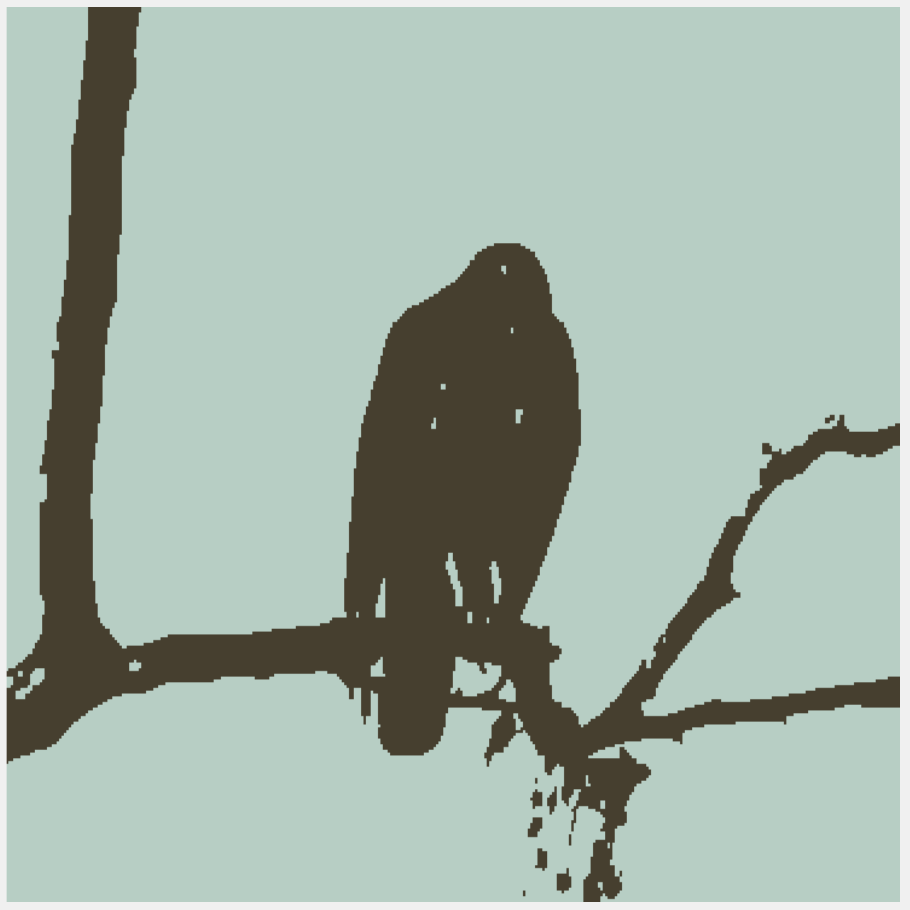}} & \subcaptionbox{$L_1 - 0.75L_2$ FR}{\includegraphics[width = 1.25in]{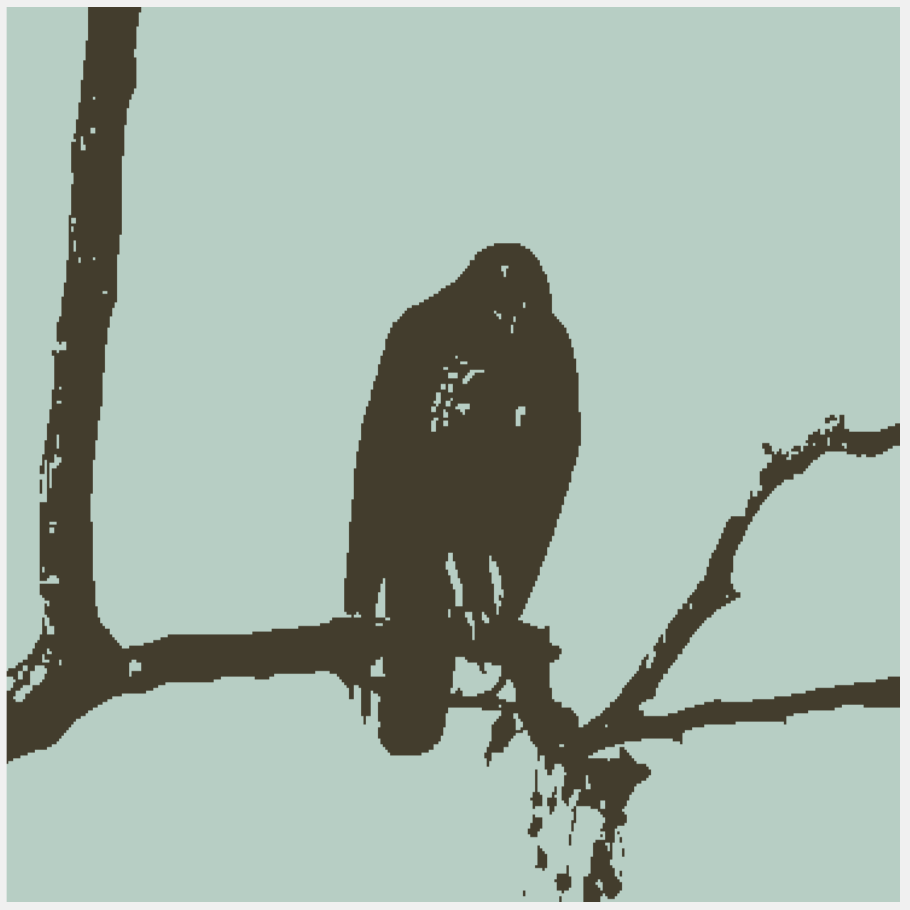}} & \subcaptionbox{$L_1$ FR}{\includegraphics[width = 1.25in]{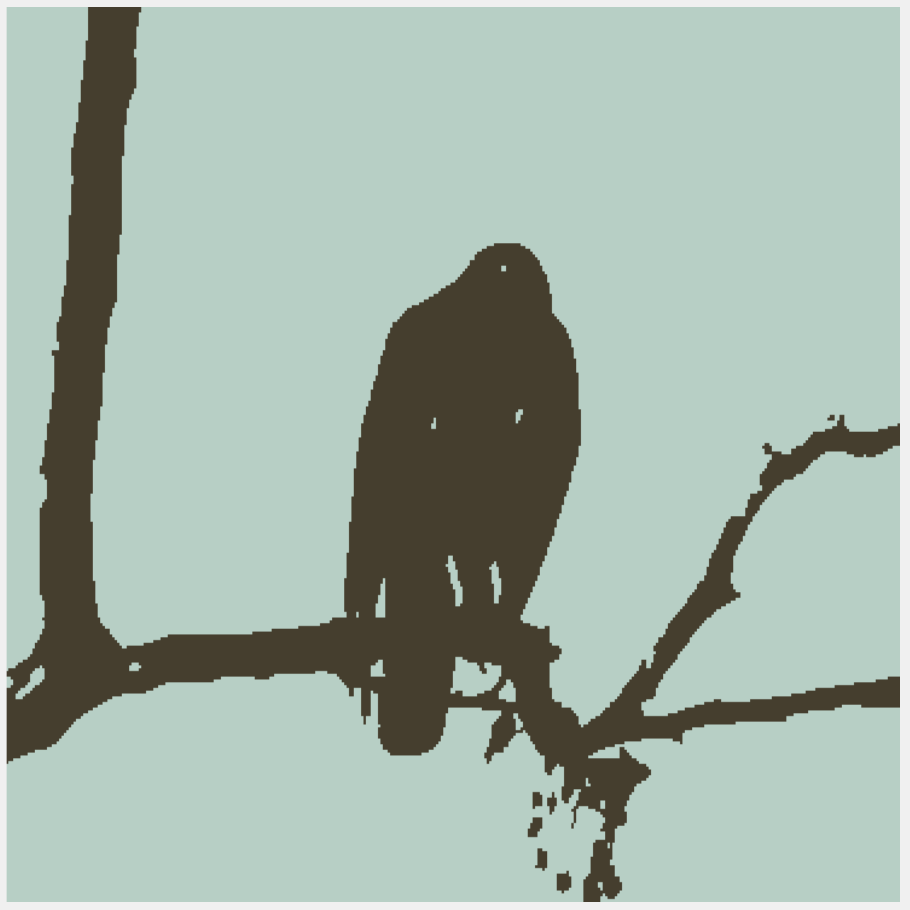}}
		\end{tabular}}
		\caption{Reconstruction results on Figure \ref{fig:hawk}.}
		\label{fig:hawk_result}
	\end{minipage}
	\begin{minipage}{\linewidth}
	\centering
	\resizebox{\textwidth}{!}{%
		\begin{tabular}{c@{}c@{}c@{}c@{}c}
			\subcaptionbox{Original}{\includegraphics[width = 1.00in]{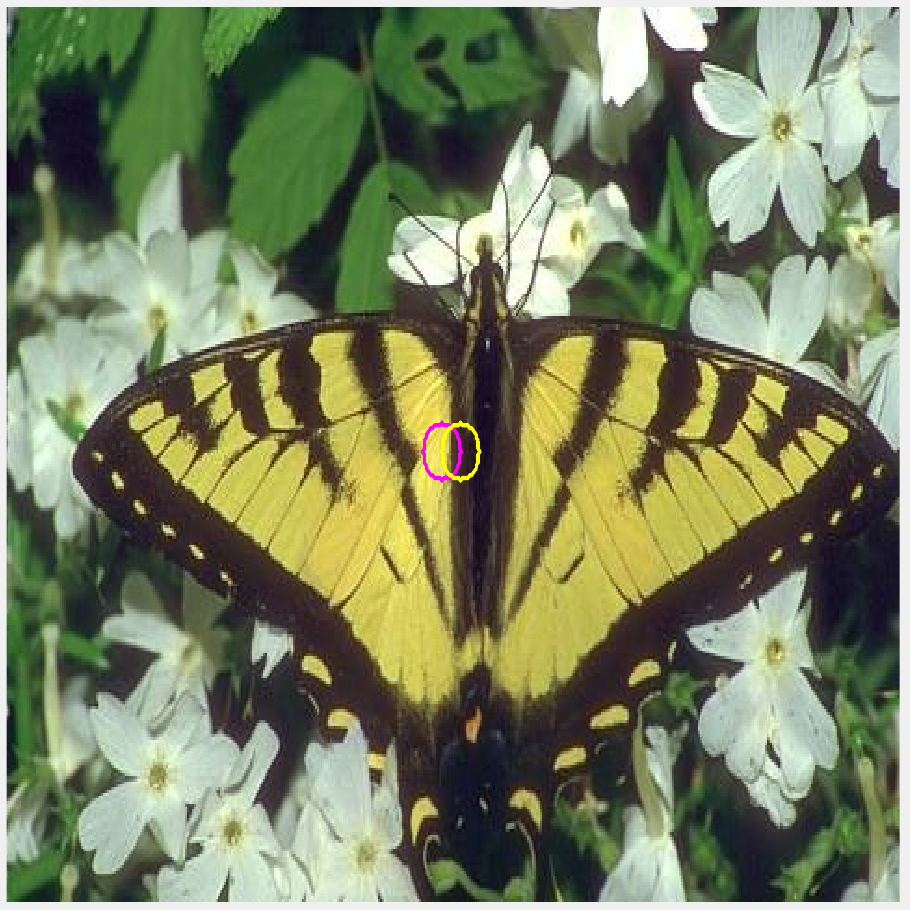}} &			\subcaptionbox{$L_1-0.75L_2$ CV}{\includegraphics[width = 1.00in]{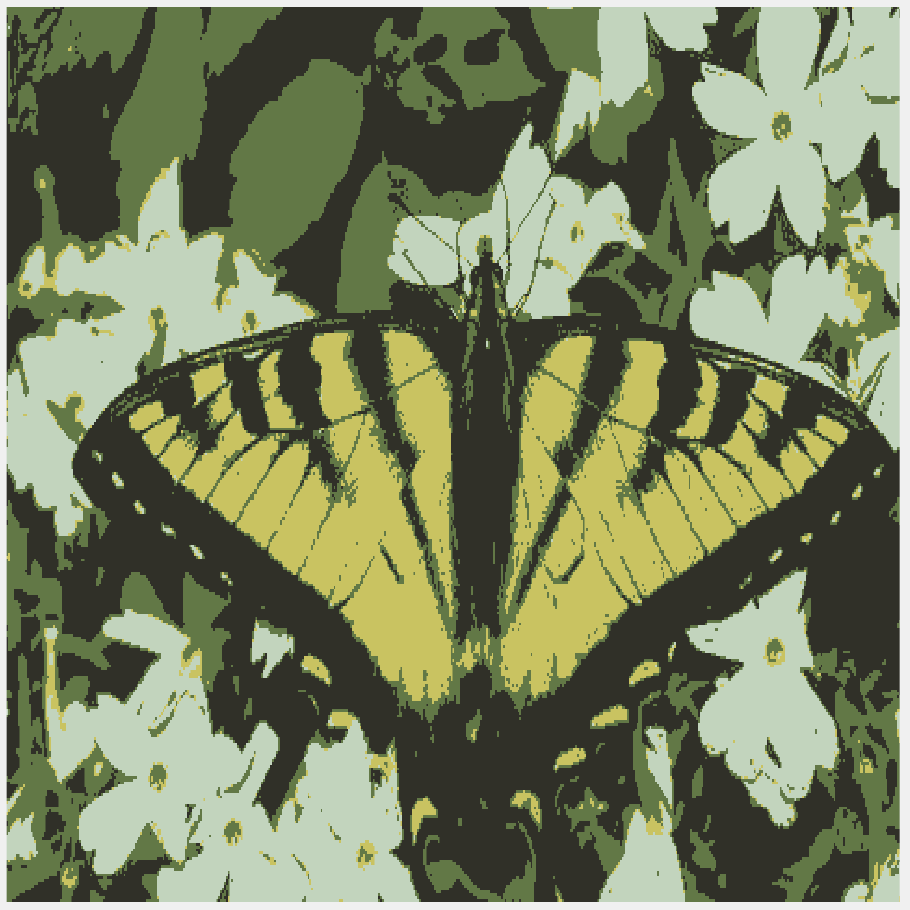}} &
			\subcaptionbox{$L_1$ CV}{\includegraphics[width = 1.00in]{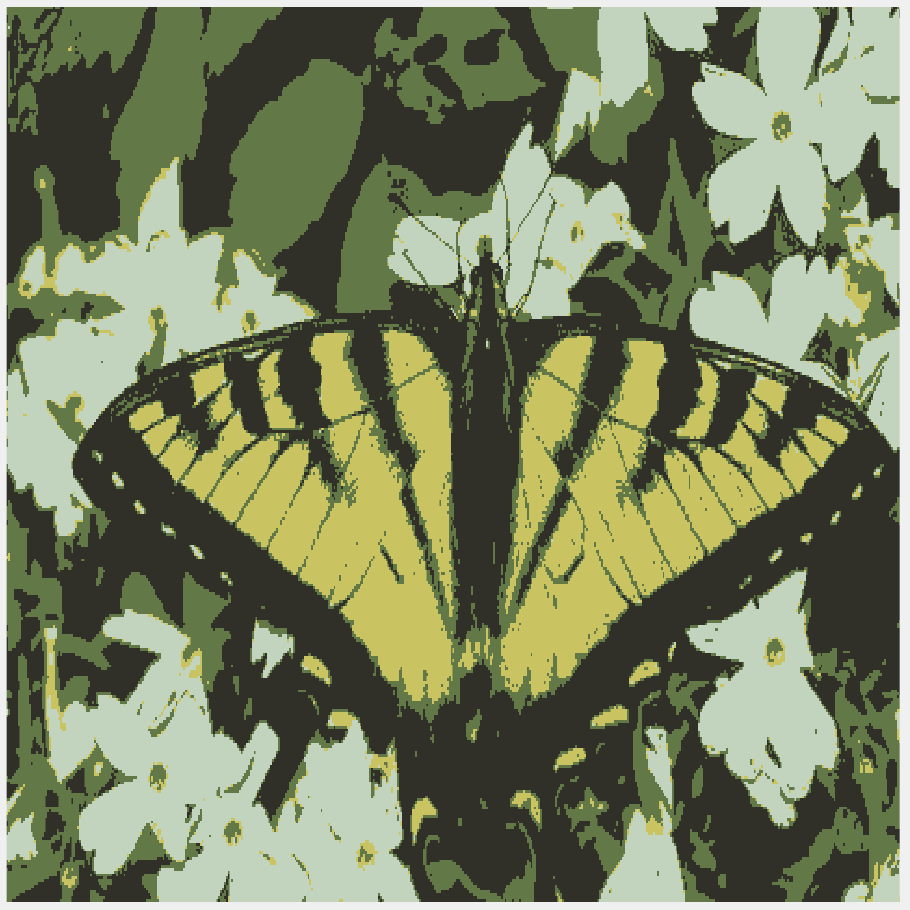}} & 			\subcaptionbox{$L_1 -0.75L_2$ FR}{\includegraphics[width = 1.00in]{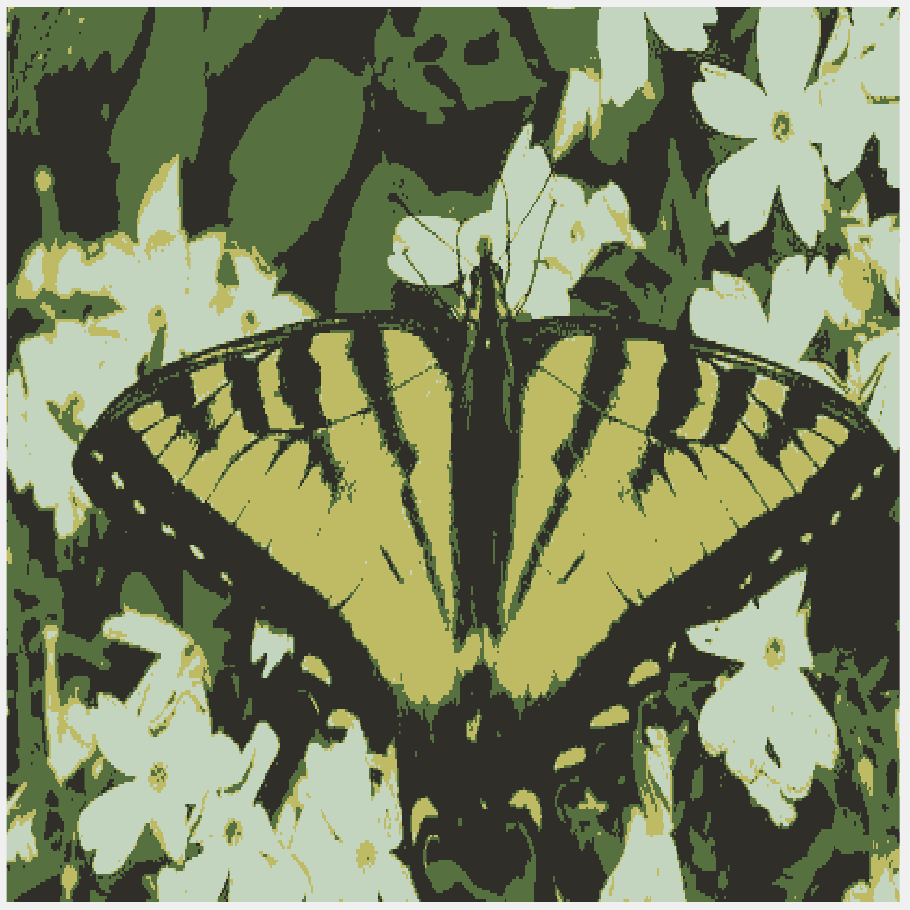}} & \subcaptionbox{$L_1$ FR}{\includegraphics[width = 1.00in]{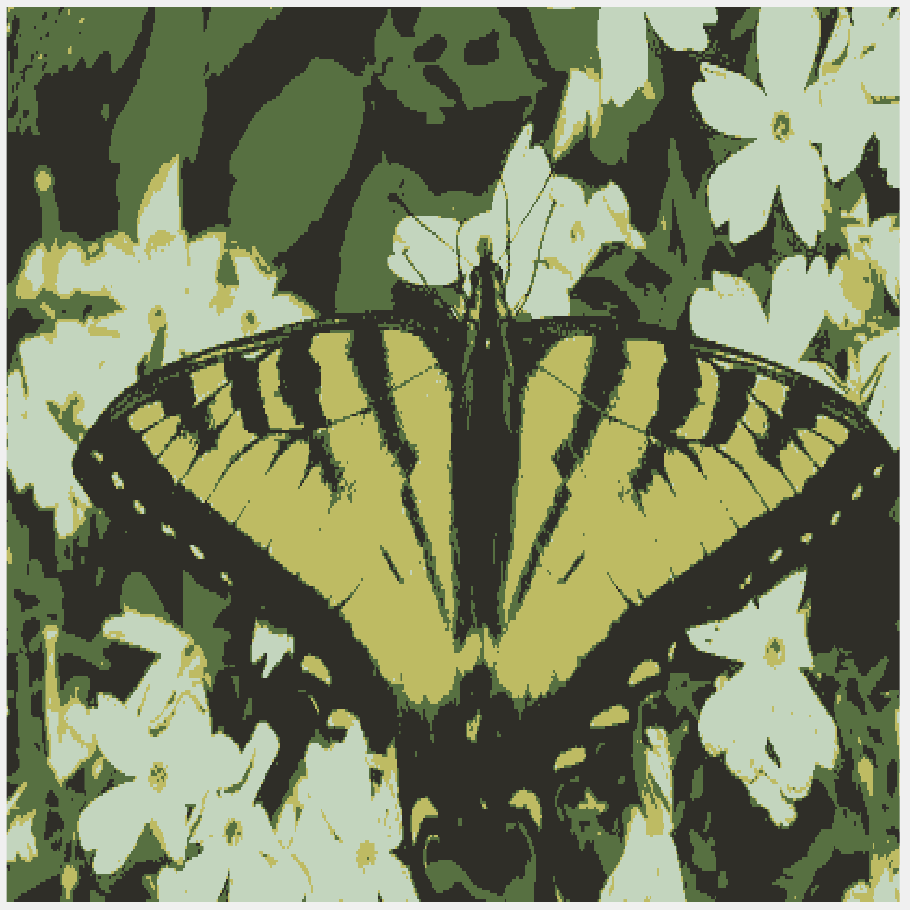}} 
		\end{tabular}}
		\caption{Reconstruction results on Figure \ref{fig:butterfly}. }
		\label{fig:butterfly_result}
	\end{minipage}
\end{figure}

\begin{figure}[t!]	
	\begin{minipage}{\linewidth}
		\centering
		\resizebox{\textwidth}{!}{%
			\begin{tabular}{c@{}c@{}c@{}c@{}c@{}}
				\subcaptionbox{Original}{\includegraphics[width = 1.00in]{./Figure/real_image/flower/flower-eps-converted-to.pdf}} & \subcaptionbox{$L_1-L_2$ FR }{\includegraphics[width = 1.00in]{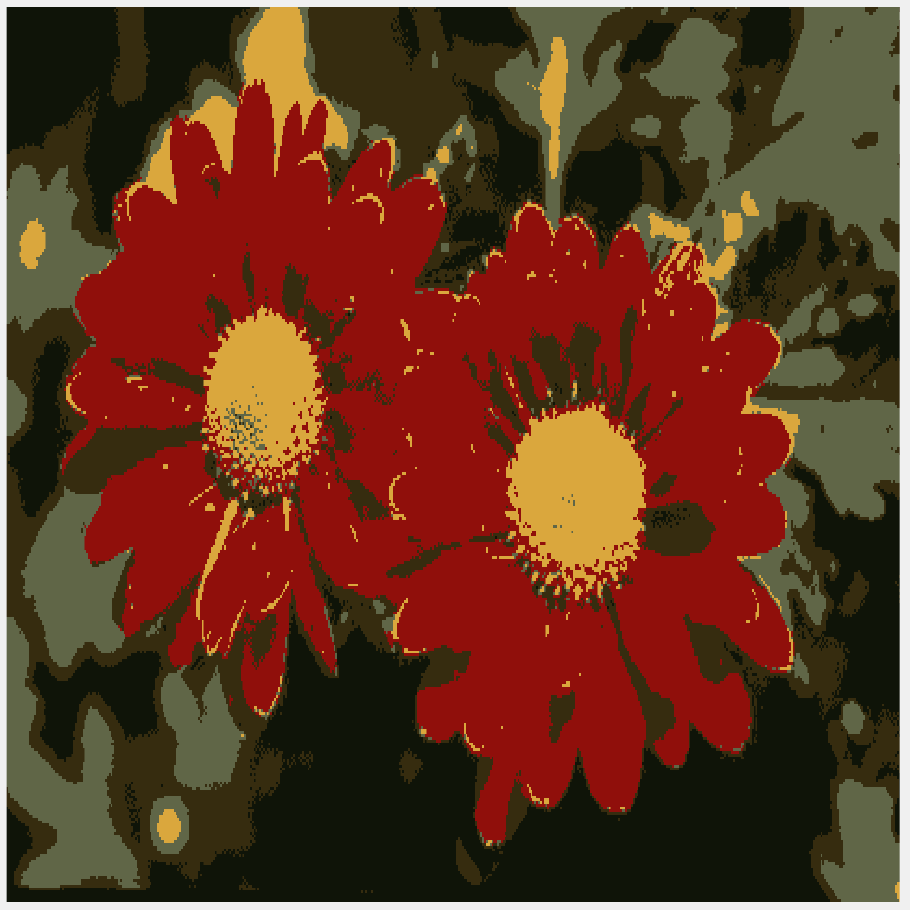}} & \subcaptionbox{$L_1-0.75L_2$ FR}{\includegraphics[width = 1.00in]{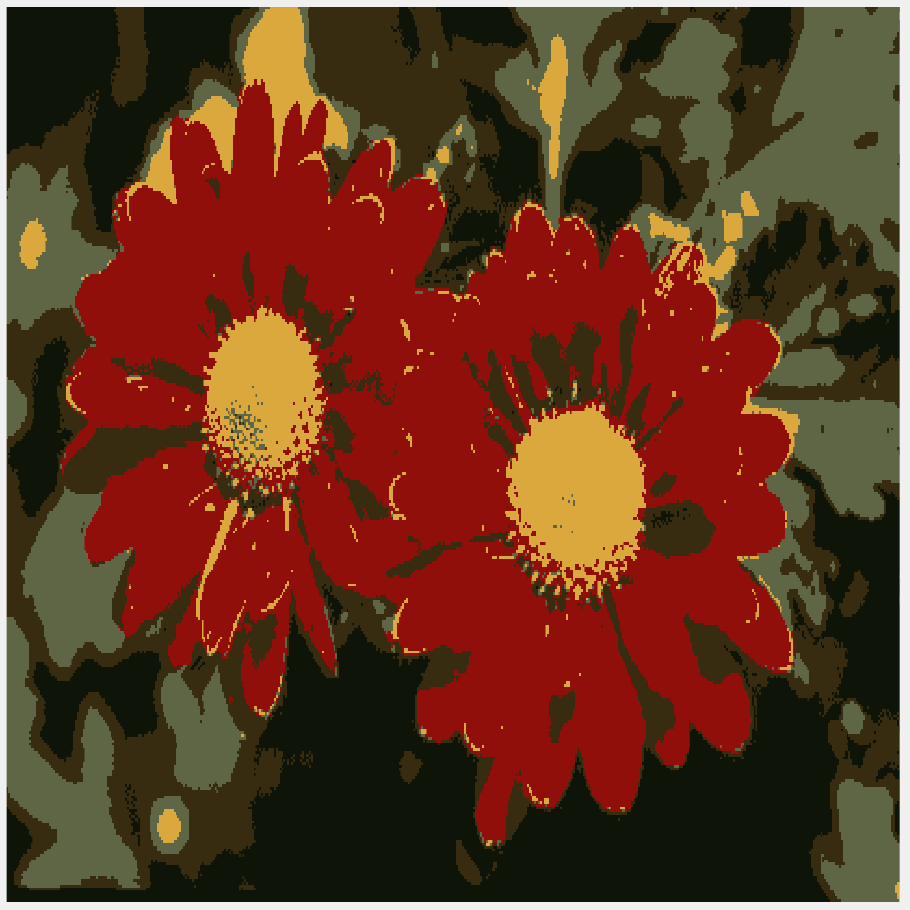}} & \subcaptionbox{$L_1-0.5L_2$ FR}{\includegraphics[width = 1.00in]{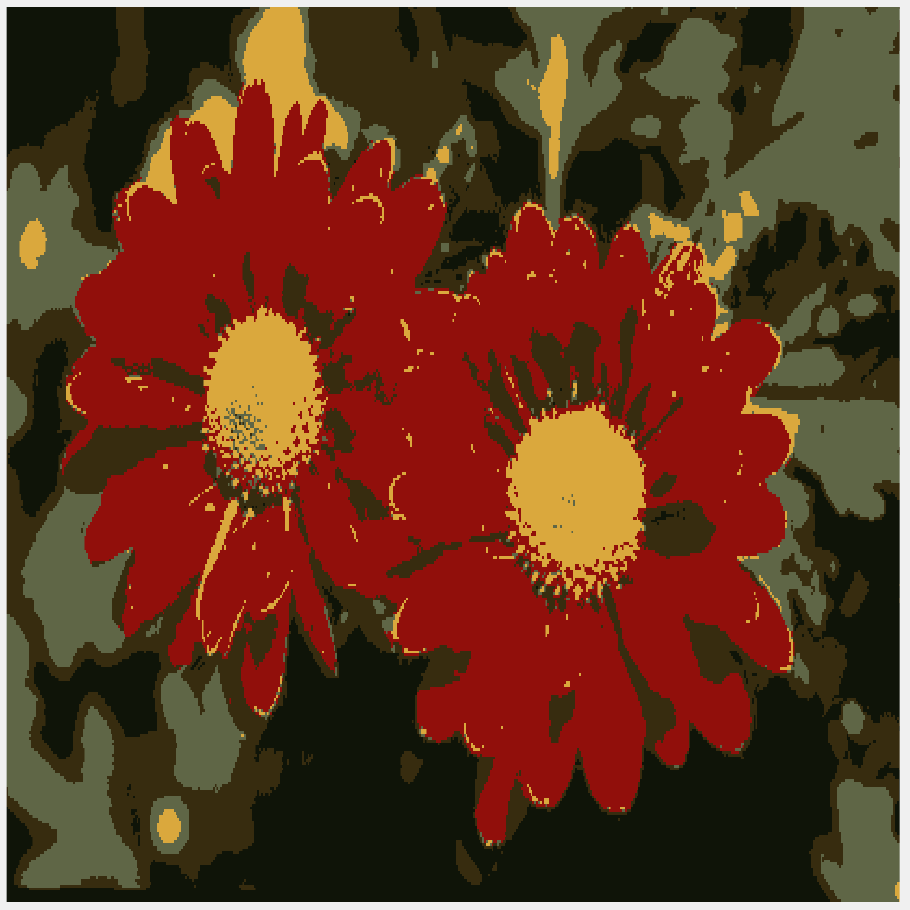}} &
				\subcaptionbox{$L_1-0.25L_2$ FR}{\includegraphics[width = 1.00in]{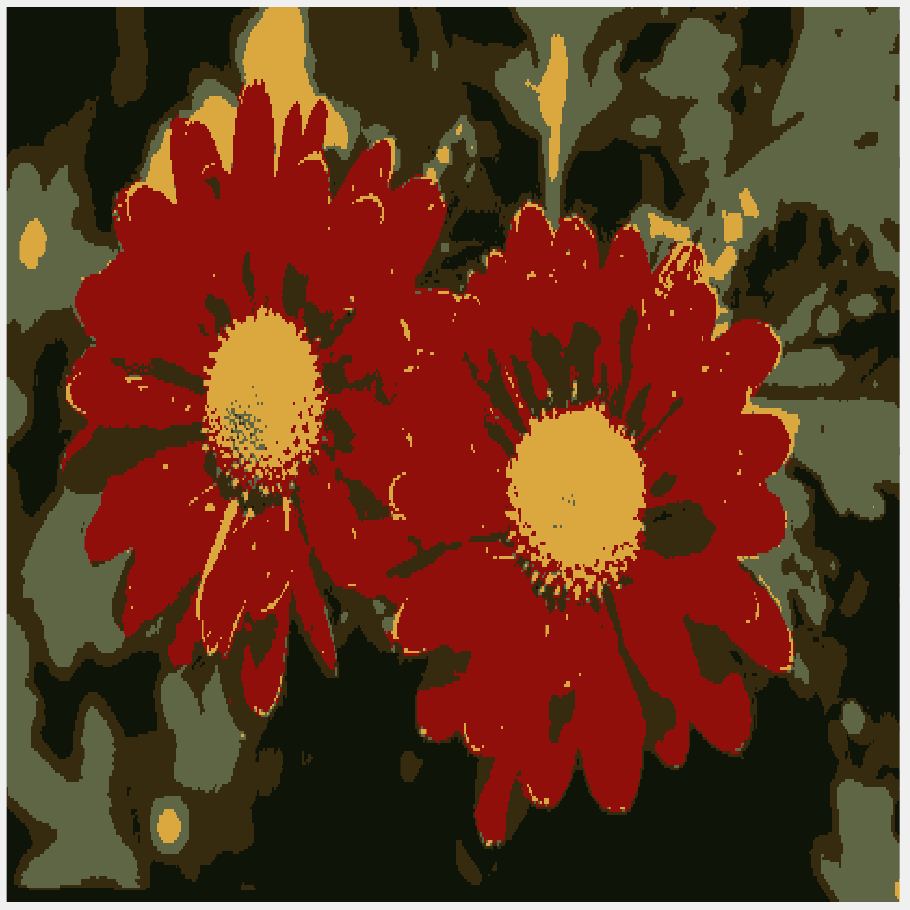}} \\ \subcaptionbox{$L_1$ FR}{\includegraphics[width = 1.00in]{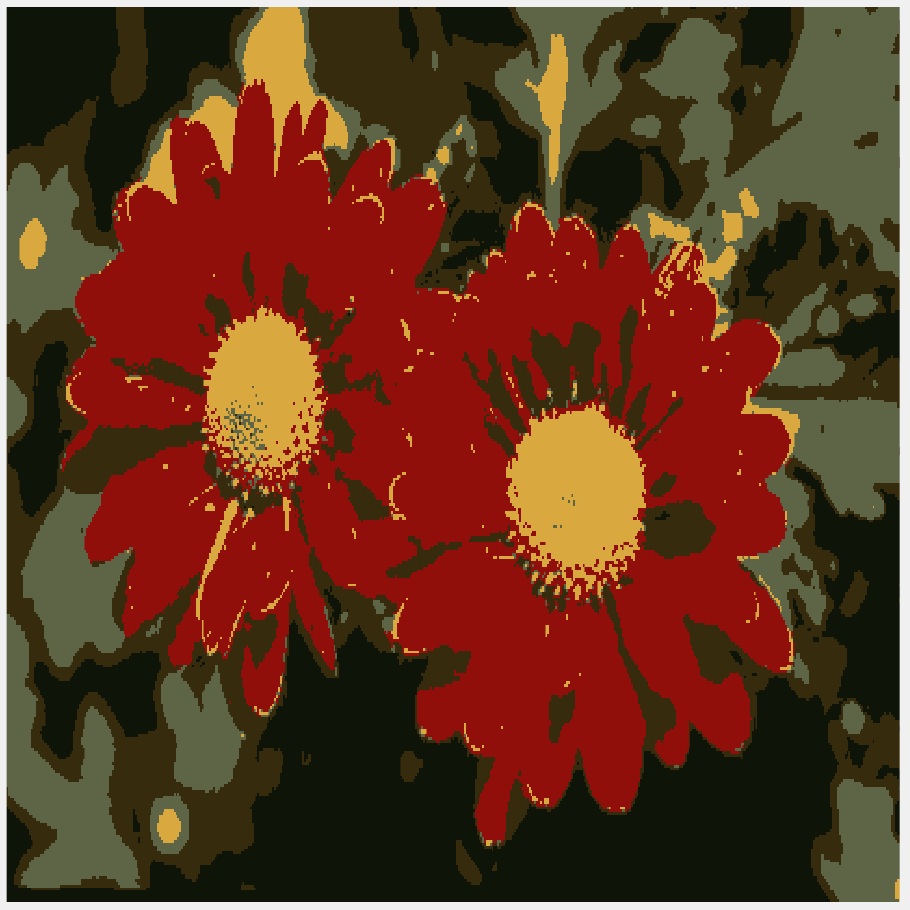}} &
				\subcaptionbox{$L_1+L_2^2$}{\includegraphics[width = 1.00in]{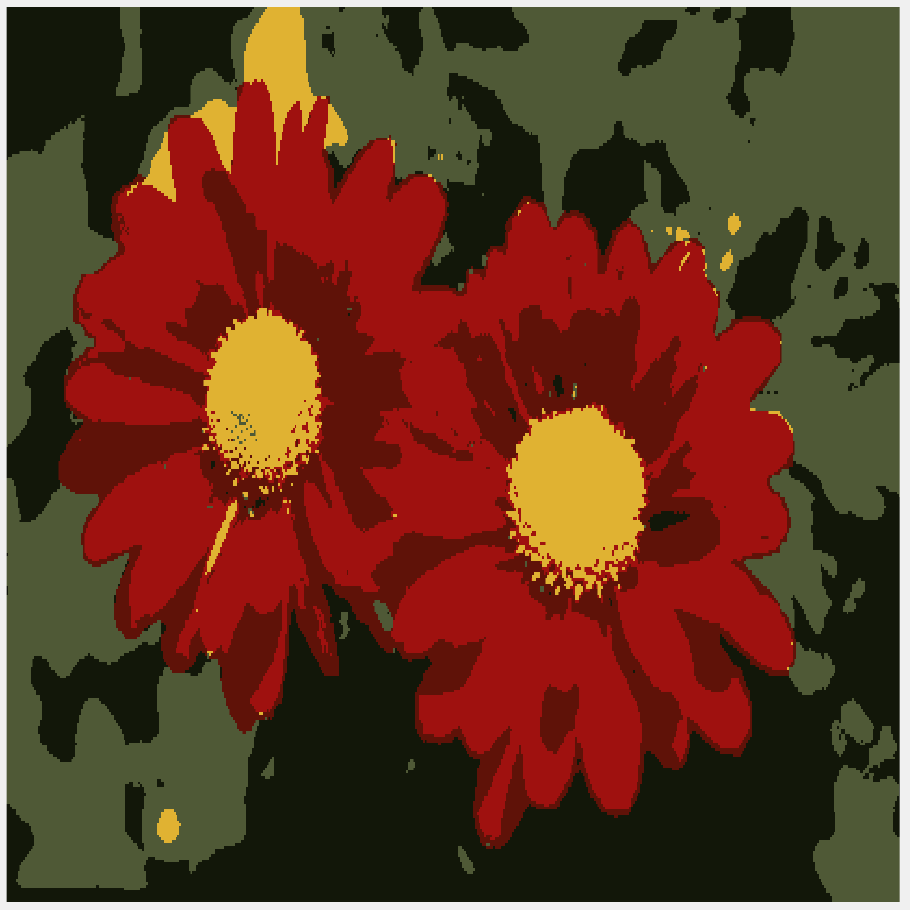}} & \subcaptionbox{$L_0$ \cite{xu2011image}}{\includegraphics[width = 1.00in]{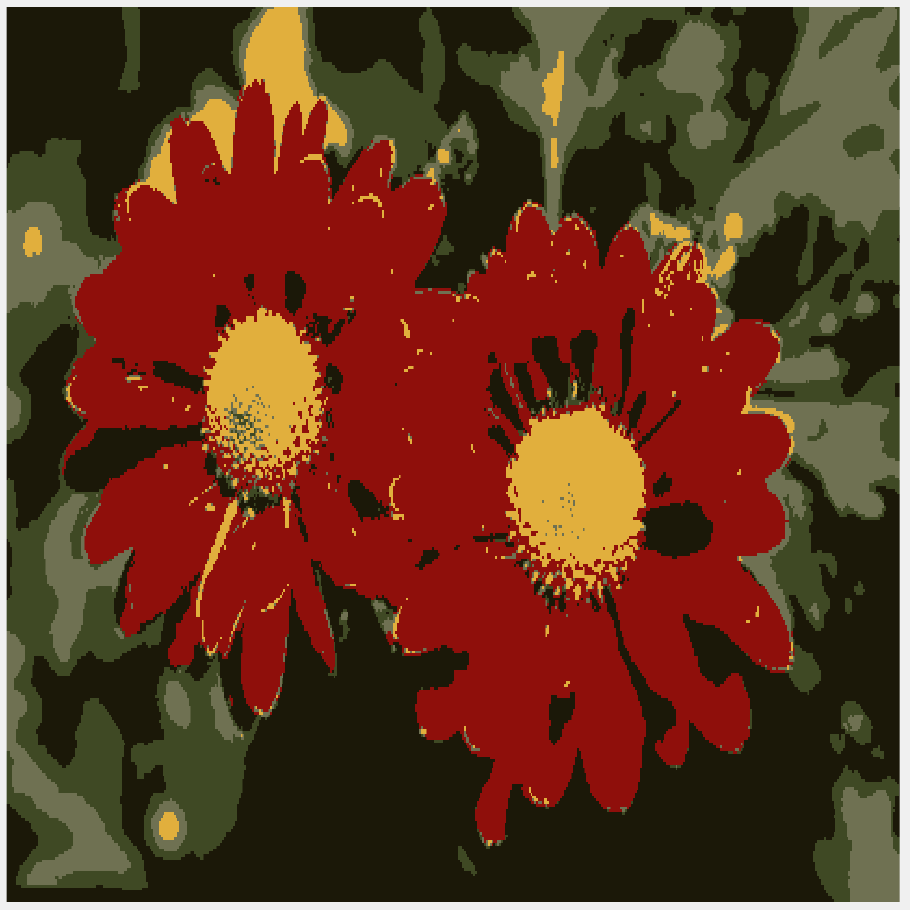}}& \subcaptionbox{$L_0$ \cite{storath2014fast}}{\includegraphics[width = 1.00in]{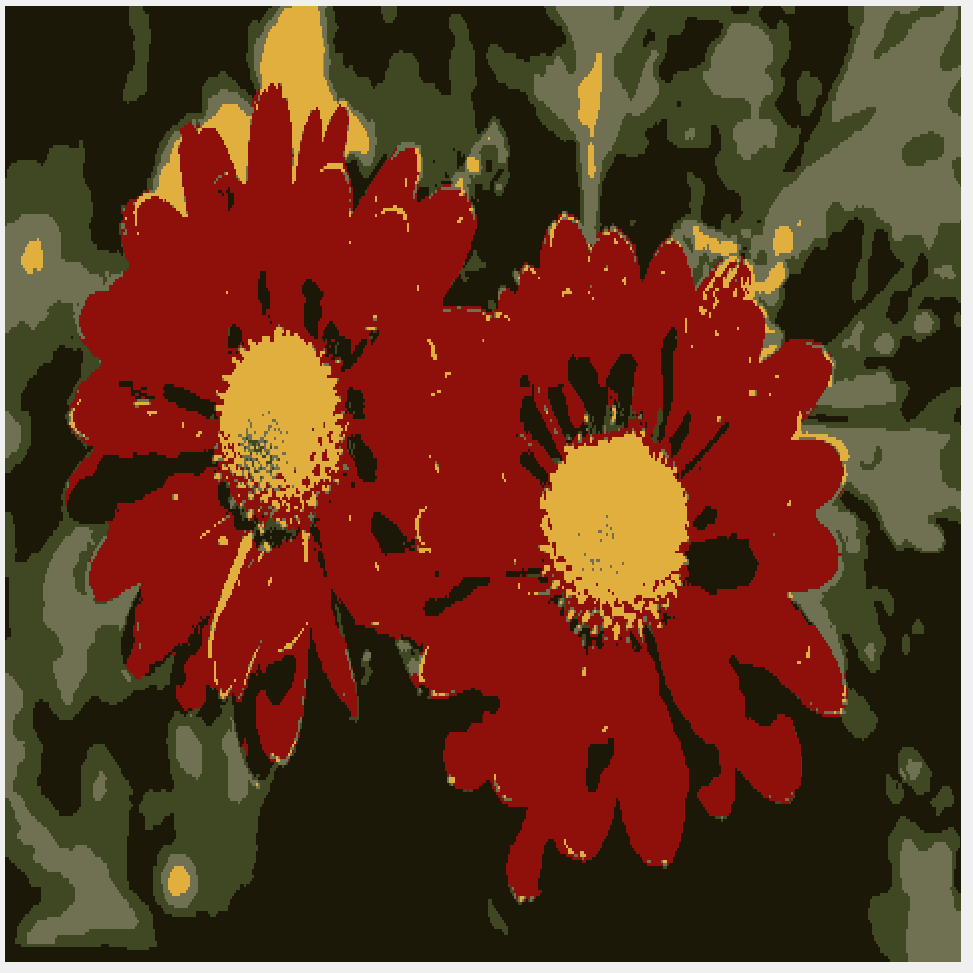}}
				& \subcaptionbox{$R_{MS}$}{\includegraphics[width = 1.00in]{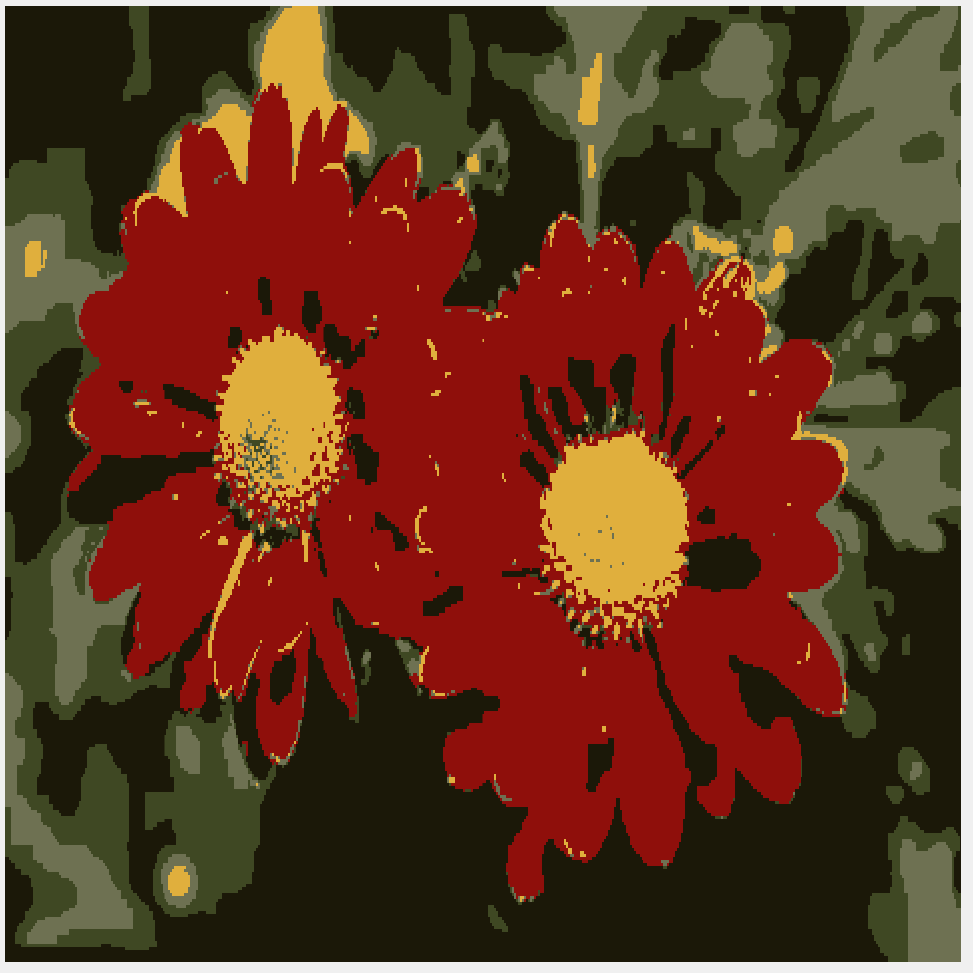}}
		\end{tabular}}
		\caption{Reconstruction results on Figure \ref{fig:flower}. }
		\label{fig:flower_result}
	\end{minipage}
	\begin{minipage}{\linewidth}
	\centering
	\resizebox{\textwidth}{!}{%
		\begin{tabular}{c@{}c@{}c@{}c@{}c@{}}
			\subcaptionbox{Original}{\includegraphics[width = 1.00in]{./Figure/real_image/pepper/pepper-eps-converted-to.pdf}} & \subcaptionbox{$L_1-L_2$ FR }{\includegraphics[width = 1.00in]{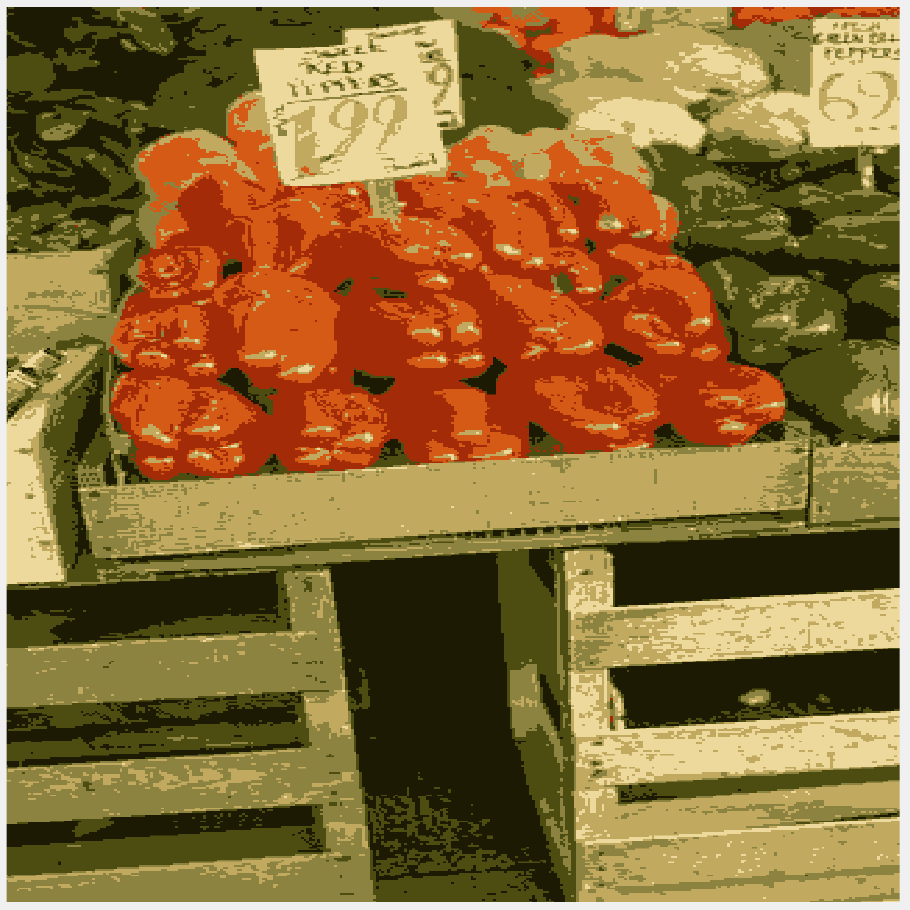}} & \subcaptionbox{$L_1-0.75L_2$ FR}{\includegraphics[width = 1.00in]{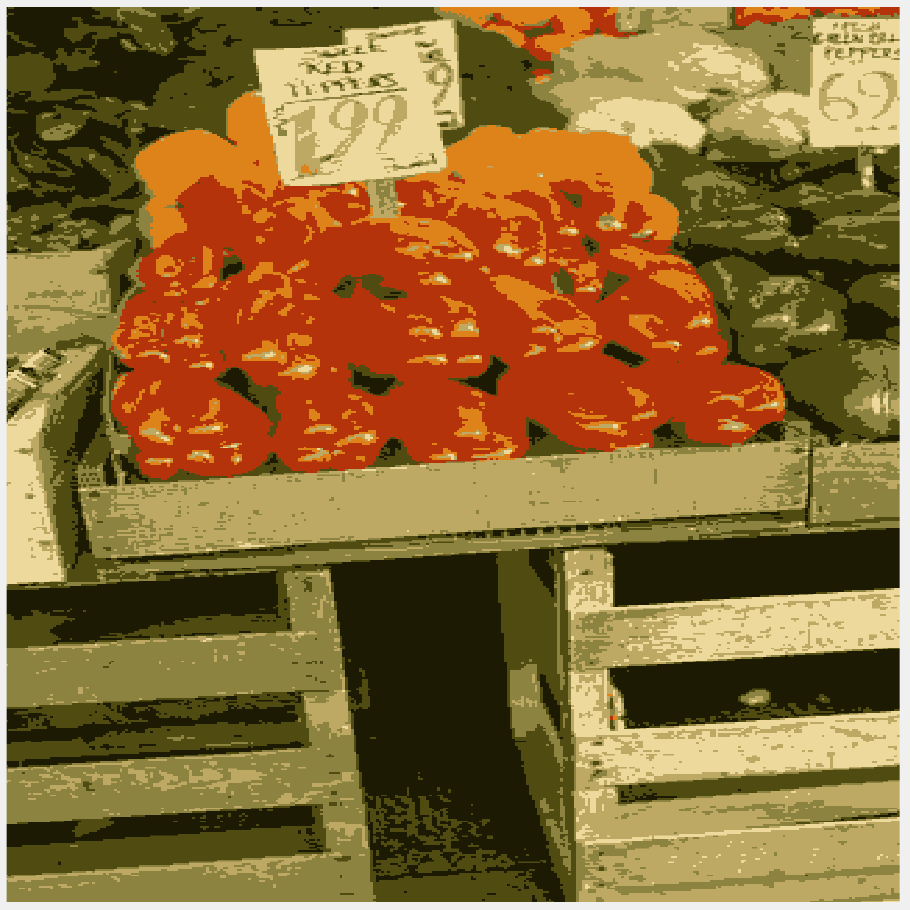}} & \subcaptionbox{$L_1-0.5L_2$ FR}{\includegraphics[width = 1.00in]{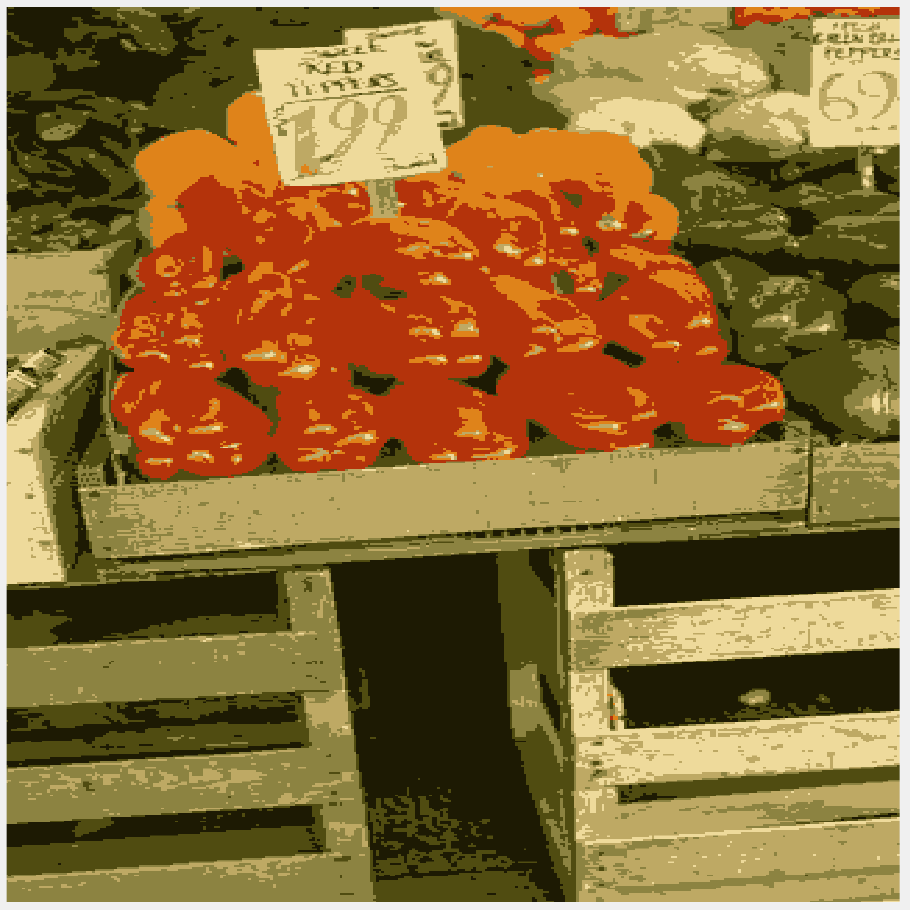}} &
			\subcaptionbox{$L_1-0.25L_2$ FR}{\includegraphics[width = 1.00in]{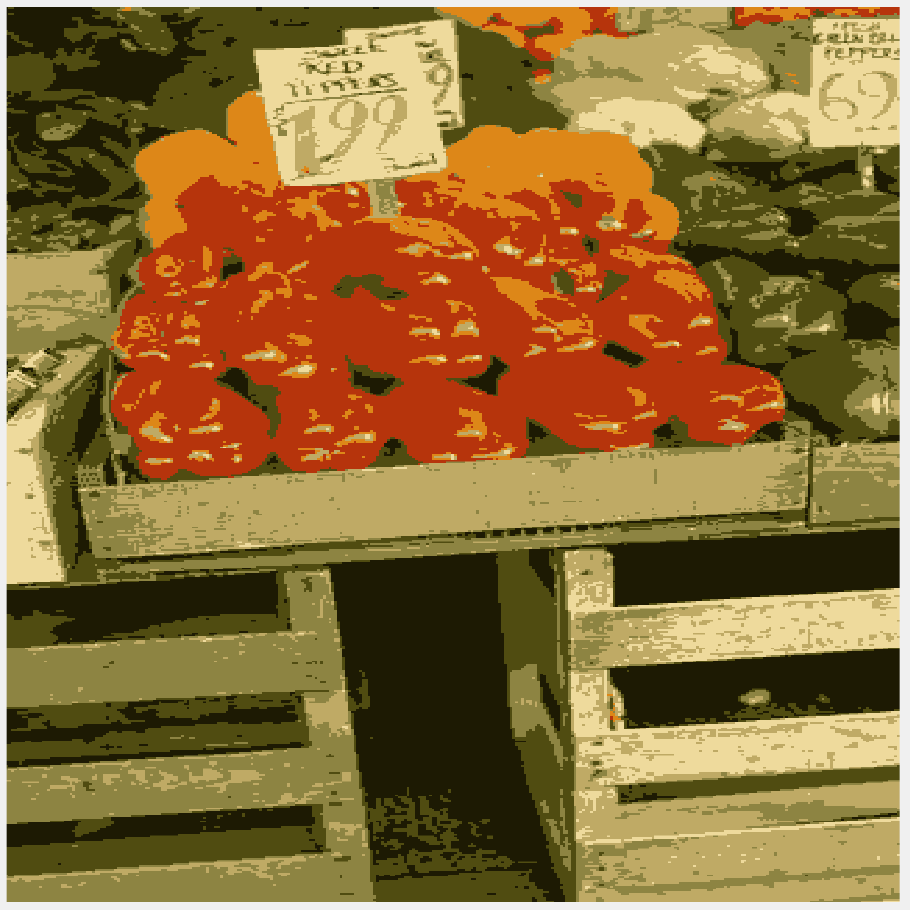}} \\ \subcaptionbox{$L_1$ FR}{\includegraphics[width = 1.00in]{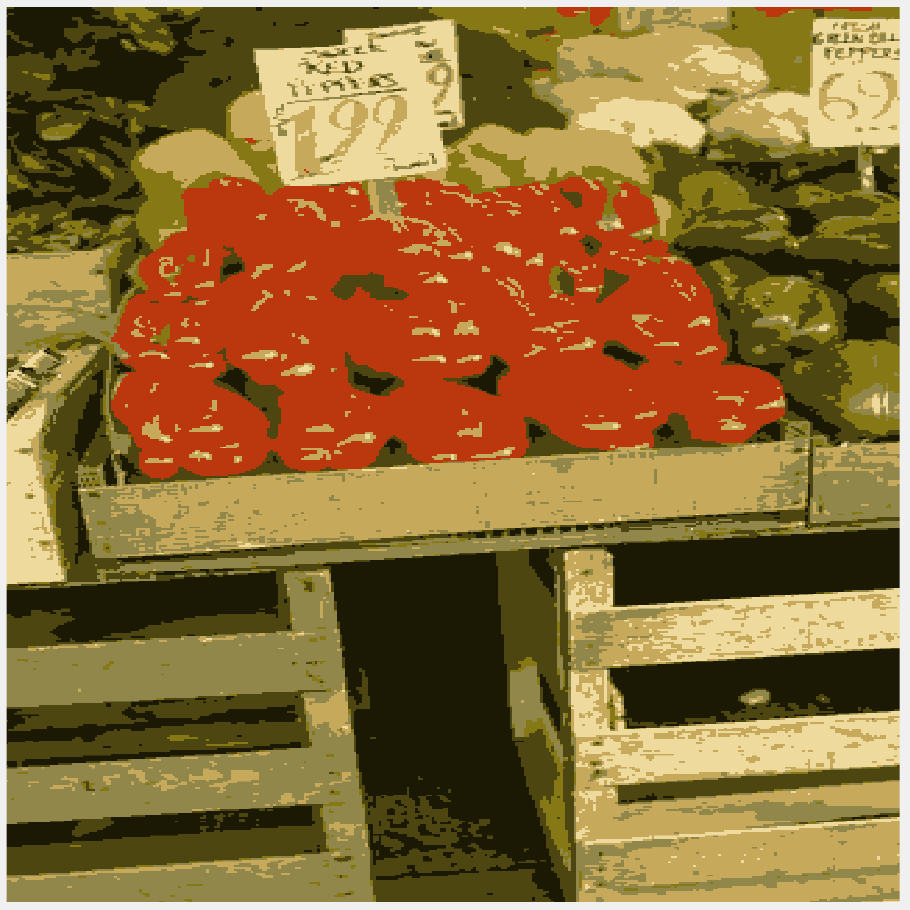}} &
			\subcaptionbox{$L_1+L_2^2$}{\includegraphics[width = 1.00in]{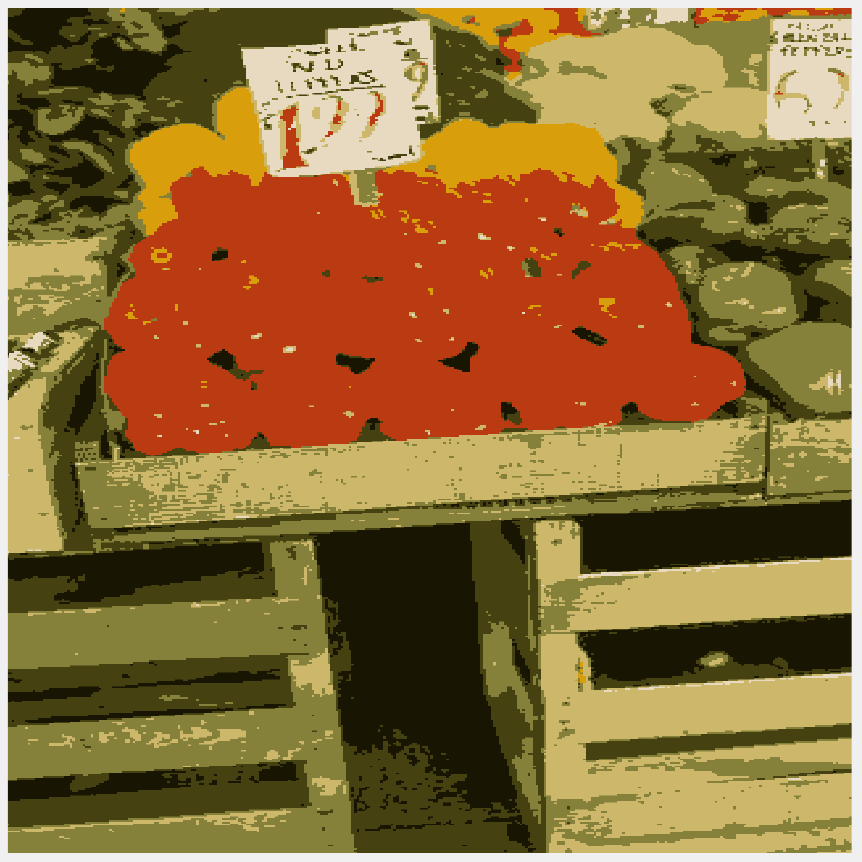}} & \subcaptionbox{$L_0$ \cite{xu2011image}}{\includegraphics[width = 1.00in]{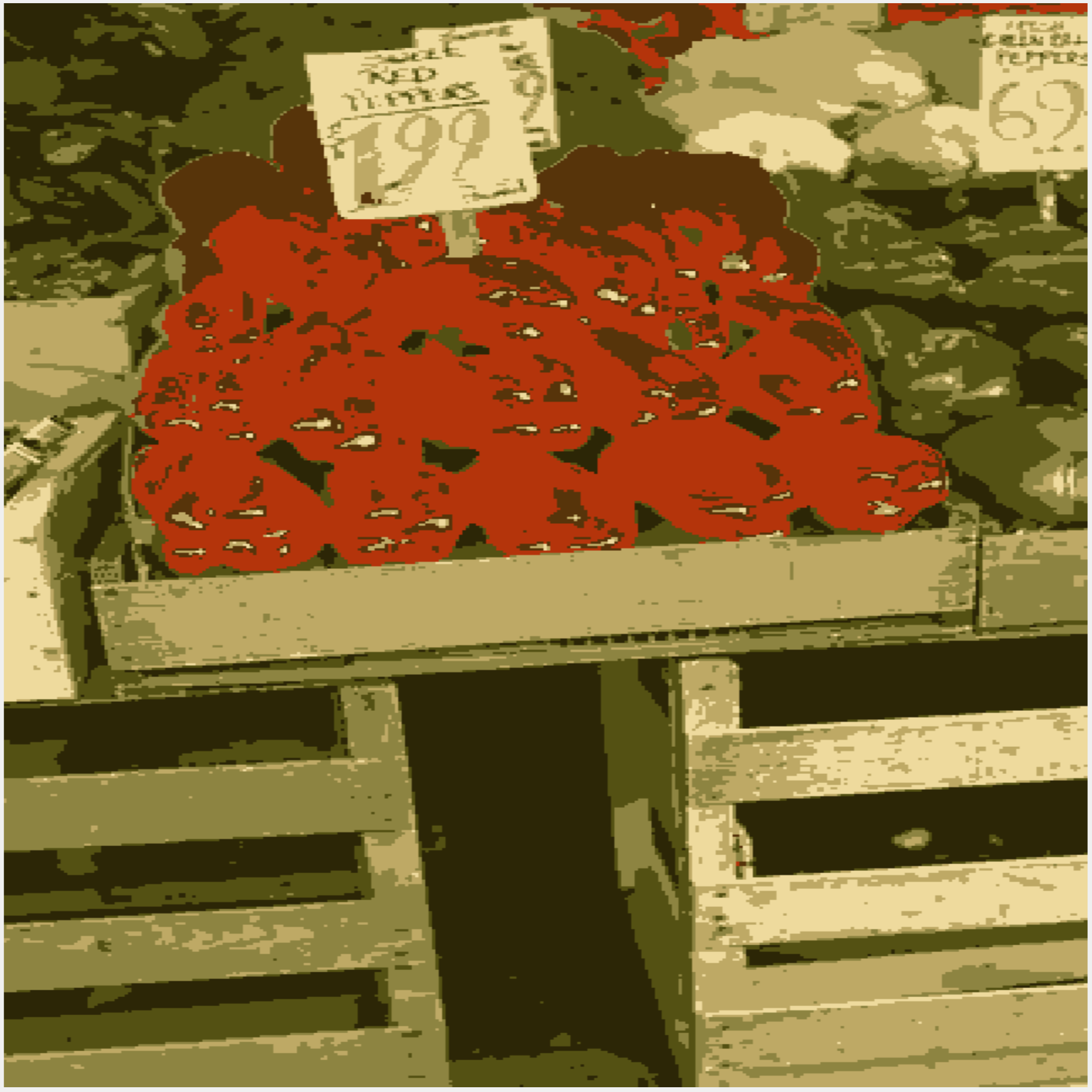}}& \subcaptionbox{$L_0$ \cite{storath2014fast}}{\includegraphics[width = 1.00in]{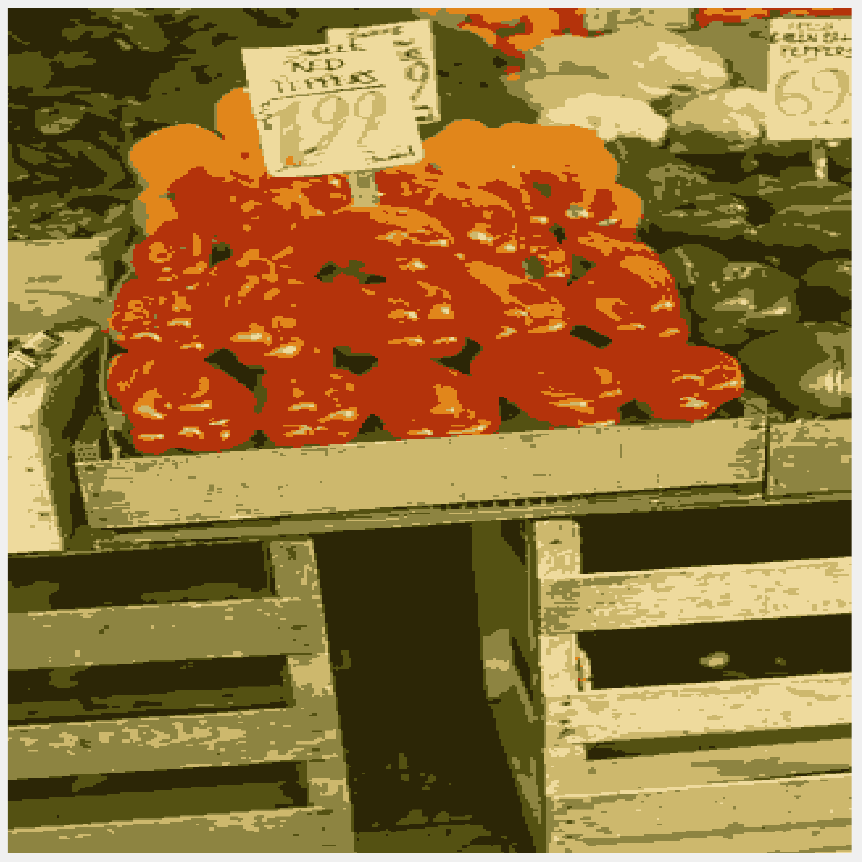}}& \subcaptionbox{$R_{MS}$}{\includegraphics[width = 1.00in]{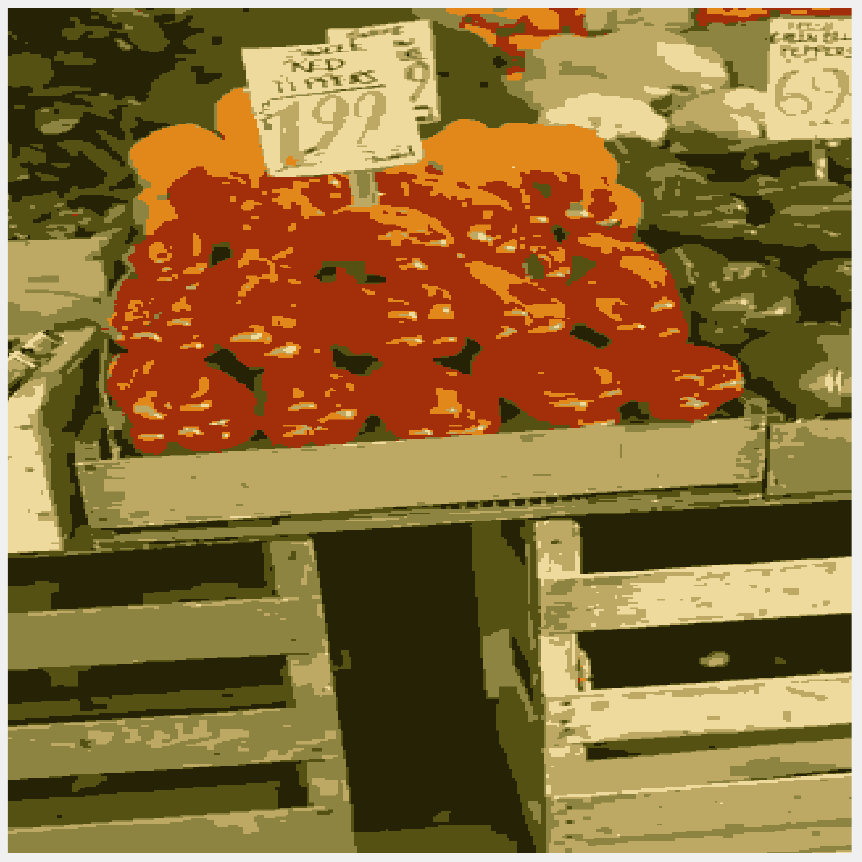}}
	\end{tabular}}
	\caption{Reconstruction results on Figure \ref{fig:pepper}. }
	\label{fig:pepper_result}
\end{minipage}
\end{figure}

\section{Conclusions and Future Works}
\label{sec:conclusions}
In this paper, we proposed AICV and AIFR models for piecewise-constant segmentation that can deal with both grayscale and color images. We developed alternating minimization algorithms utilizing DCA and PDHGLS to efficiently solve the models. Convergence analyses were provided to demonstrate that the objective functions were monotonically decreasing and to validate the efficacy of the algorithms. Numerical results illustrated that the AICV/AIFR models outperform their anisotropic counterparts on various  images in a robust manner. The segmentation results are comparable and sometimes better than those of the two-stage segmentation methods.

    In the future, we will consider the application of the weighted anisotropic-isotropic penalty to other types of segmentation approaches, such as piecewise-smooth formulations \cite{jung2017piecewise,le2007additive}, the Potts models \cite{pock2008convex,storath2015joint,wei2017new}, the fuzzy region model \cite{li2016multiphase}, and deep learning techniques \cite{jia2021regularized,jia2020nonlocal,kim2019mumford}. Since the two-stage methods are generally faster to run than our methods, we will leave the acceleration as a future work. Another future direction involves segmenting blurry images by combining our proposed models with some deblurring techniques. The numerical experiments demonstrated that there is no optimal, universal $\alpha$ for all images, which motivates us to develop an automatic method 
to select $\alpha$ for any given image in the future. As AICV/AIFR models indicate the success of using nonconvex penalty terms in image processing, we aim at other nonconvex penalties, such as transformed $L_1$ \cite{nikolova2000local, zhang2018minimization} and $L_1/L_2$ \cite{rahimi2019scale,wang2019accelerated}, for  image segmentation and other imaging problems including denoising and deblurring.
\section*{Acknowledgments}
We would like to thank the anonymous referees for their useful suggestions and feedback, which significantly improved the presentation of the paper.
\bibliographystyle{plain}
\bibliography{references}
\end{document}

%% file: ex_shared.tex
\usepackage{lipsum}
\usepackage{amsfonts}
\usepackage{amsmath}
\usepackage{graphicx}
\usepackage{epstopdf}
\usepackage{algorithmic}
\ifpdf
  \DeclareGraphicsExtensions{.eps,.pdf,.png,.jpg}
\else
  \DeclareGraphicsExtensions{.eps}
\fi

\newsiamremark{remark}{Remark}
\newsiamremark{hypothesis}{Hypothesis}
\crefname{hypothesis}{Hypothesis}{Hypotheses}
\newsiamthm{claim}{Claim}

\headers{AITV Segmentation}{K. Bui, F. Park, Y. Lou, and J. Xin}

\title{A Weighted Difference OF Anisotropic AND Isotropic Total Variation for Relaxed Mumford-Shah Multiphase and Color Image Segmentation \thanks{Submitted to the editors DATE.
\funding{The work was partially supported by NSF grants IIS-1632935, DMS-1846690, and DMS-1854434.}}}

\author{Kevin Bui\thanks{Department of Mathematics, University of California at Irvine, Irvine, CA 92697.
  (\email{kevinb3@uci.edu}, \email{jxin@math.uci.edu})}
\and Fredrick Park\thanks{Department of Mathematics \& Computer Science, Whittier College, Whittier, CA 90602
  (\email{fpark@whittier.edu})}
\and Yifei Lou\thanks{Department of Mathematical Science, University of Texas at Dallas, Richardson, TX 75080 (\email{yifei.lou@utdallas.edu})}
\and Jack Xin\footnotemark[2]}
\usepackage{amsopn}

\makeatletter
\newcommand*{\addFileDependency}[1]{
  \typeout{(#1)}
  \@addtofilelist{#1}
  \IfFileExists{#1}{}{\typeout{No file #1.}}
}
\makeatother
